\documentclass[final,11pt]{amsart}
\usepackage[T1]{fontenc}
\usepackage{charter}
\usepackage[bitstream-charter]{mathdesign}
\usepackage{charter}
\usepackage{paralist}

\usepackage[tight,raggedright]{subfigure}
\usepackage{graphicx}
\usepackage{framed}
\usepackage{algorithmic}
\usepackage{multirow}
\usepackage{pdflscape}
\usepackage[norelsize]{algorithm2e}
\usepackage{pbox}
\usepackage[disable]{todonotes}
\usepackage{footnote}
\makesavenoteenv{tabular}

\newcommand{\plotpath}{colorplots}

\def\e{\ensuremath{\mathrm{e}}}
\def\R{\ensuremath{\mathbb{R}}}
\def\E{\ensuremath{\mathbb{E}}}

\newcommand{\mat}[1]{\ensuremath{\mathbf{#1}}}
\def\matA{\mat{A}}
\def\matB{\mat{B}}
\def\matD{\mat{D}}
\def\matH{\mat{H}}
\def\matU{\mat{U}}
\def\matY{\mat{Y}}
\def\matC{\mat{C}}
\def\matW{\mat{W}}
\def\matP{\mat{P}}
\def\matI{\mat{I}}
\def\matR{\mat{R}}
\def\matS{\mat{S}}
\def\matT{\mat{T}}
\def\matV{\mat{V}}
\def\matX{\mat{X}}
\def\matZ{\mat{Z}}
\def\matF{\mat{F}}
\def\matM{\mat{M}}
\def\matQ{\mat{Q}}
\def\matSig{\boldsymbol{\Sigma}}

\def\matOmega{\boldsymbol{\Omega}}

\def\pinv{\dagger}
\def\transp{T}

\renewcommand{\vec}[1]{\ensuremath{\mathbf{#1}}}
\newcommand{\tr}[1]{\ensuremath{\operatorname{Tr}\left(#1\right)}}

\newcommand{\xiNorm}[1]{\ensuremath{\left\|#1\right\|_\xi}}
\newcommand{\TNorm}[1]{\ensuremath{\left\|#1\right\|_2}}
\newcommand{\TNormS}[1]{\ensuremath{\left\|#1\right\|_2^2}}
\newcommand{\FNorm}[1]{\ensuremath{\left\|#1\right\|_\mathrm{F}}}
\newcommand{\FNormS}[1]{\ensuremath{\left\|#1\right\|_\mathrm{F}^2}}
\newcommand{\VTNorm}[1]{\ensuremath{\left\|#1\right\|_2}}

\newcommand{\tracenorm}[1]{\ensuremath{\left\|#1\right\|_\star}}

\newcommand{\Prob}[1]{\ensuremath{\mathbb{P}\left\{#1\right\}}}
\newcommand{\const}[1]{\ensuremath{\mathrm{#1}}}

\newtheorem{theorem}{Theorem}

\newtheorem{lemma}{Lemma}
\theoremstyle{remark}

\title{
Revisiting the Nystr\"{o}m method for improved large-scale machine learning
}
\author[Alex Gittens]{Alex Gittens$^1$}
\thanks{$^1$Department of Applied and Computational Mathematics, %
California Institute of Technology, %
Pasadena, CA 91125. Email: \emph{gittens@caltech.edu}}
\author[Michael W. Mahoney]{Michael W. Mahoney$^2$}
\thanks{$^2$Department of Mathematics, %
Stanford University, %
Stanford, CA 9430. Email: \emph{mmahoney@cs.stanford.edu}}

\begin{document}

\begin{abstract}
We reconsider randomized algorithms for the low-rank approximation of 
symmetric positive semi-definite (SPSD) 
matrices such as Laplacian and kernel matrices that arise in data analysis 
and machine learning applications.
Our main results consist of an empirical evaluation of the performance 
quality and running time of sampling and projection methods on a diverse 
suite of SPSD matrices.
Our results highlight complementary aspects of sampling versus projection 
methods; they characterize the effects of common data preprocessing steps 
on the performance of these algorithms; and they point to important 
differences between uniform sampling and nonuniform sampling methods based 
on leverage scores.
In addition, our empirical results illustrate that existing theory is so 
weak that it does not provide even a qualitative guide to practice.
Thus, we complement our empirical results with a suite of worst-case 
theoretical bounds for both random sampling and random projection methods.
These bounds are qualitatively superior to existing bounds---\emph{e.g.}, 
improved additive-error bounds for spectral and Frobenius norm error and 
relative-error bounds for trace norm error---and they point to future 
directions to make these algorithms useful in even larger-scale machine
learning applications.
\end{abstract}

\maketitle




\makeatletter
\providecommand\@dotsep{5}
\def\listtodoname{List of Todos}
\def\listoftodos{\@starttoc{tdo}\listtodoname}
\makeatother

\section{Introduction}
\label{sxn:intro}

\todo[inline]{Use a remark environment for numbering and better formatting of the remarks}
\todo[inline]{Talwalkar: References [36] and [57] are often mixed up during 
citations, e.g., on page 2 (last paragraph), page 5 (above equation 5 and below 
equation 6).  In particular, [57] provides coherence-based bounds for the 
Nystrom method in the low-rank setting.}

We reconsider randomized algorithms for the low-rank 
approximation of symmetric positive semi-definite (SPSD) 
matrices such as Laplacian and kernel matrices that 
arise in data analysis and machine learning applications.
Our goal is to obtain an improved understanding, both empirically and 
theoretically, of the complementary strengths of sampling versus projection 
methods on realistic data.
Our main results consist of an empirical evaluation of the performance quality and 
running time of sampling and projection methods on a diverse suite of 
dense and sparse SPSD matrices drawn both from machine learning as well as 
more general data analysis applications.  
These results are not intended to be comprehensive but instead to be 
illustrative of how randomized algorithms for the low-rank 
approximation of SPSD matrices behave in a broad range of realistic machine 
learning and data analysis applications.

In addition to being of interest in their own right, our empirical results 
point to several directions that are not explained well by existing theory. 
(For example, that the results are much better than existing worst-case 
theory would suggest, and that sampling with respect to the statistical 
leverage scores leads to results that are complementary to those achieved 
by projection-based methods.)
Thus, we complement our empirical results with a suite of worst-case 
theoretical bounds for both random sampling and random projection methods.
These bounds are qualitatively superior to existing bounds---\emph{e.g.}, 
improved additive-error bounds for spectral and Frobenius norm error and 
relative-error bounds for trace norm error.
Importantly, by considering random sampling and random projection algorithms 
on an equal footing, we identify within our analysis deterministic 
structural properties of the input data and sampling/projection methods that 
are responsible for high-quality low-rank approximation.

In more detail, our main contributions are fourfold.
\def\pltopsep{.5em}
\begin{compactitem}
\item
First, we provide an empirical illustration of the complementary strengths 
and weaknesses of data-independent random projection methods and 
data-dependent random sampling methods when applied to SPSD matrices. 
We do so for a diverse class of SPSD matrices drawn from machine learning 
and more general data analysis applications, and we consider reconstruction
error with respect to the spectral, Frobenius, as well as trace norms.  
Depending on the parameter settings, the matrix norm of interest, the data 
set under consideration, etc., one or the other method might be preferable.
In addition, we illustrate how these empirical properties can often be 
understood in terms of the structural nonuniformities of the input data that
are of independent interest.
\item
Second, we consider the running time of high-quality sampling and 
projection algorithms.
For random sampling algorithms, the computational bottleneck is typically the
exact or approximate computation of the importance sampling distribution
with respect to which one samples; and for random projection methods, the
computational bottleneck is often the implementation of the random 
projection.
By exploiting and extending recent work on ``fast'' random projections and 
related recent work on ``fast'' approximation of the statistical leverage 
scores, we illustrate that high-quality leverage-based random sampling and 
high-quality random projection algorithms have comparable running times. 
Although both are slower than simple (and in general much lower-quality) 
uniform sampling, both can be implemented more quickly than a na\"{i}ve 
computation of an orthogonal basis for the top part of the spectrum.
\item
Third, our main technical contribution is a set of deterministic structural 
results that hold for any ``sketching matrix'' applied to an SPSD matrix.
(A precise statement of these results is given in 
Theorems~\ref{thm:spectral-deterministic-error},
\ref{thm:frobenius-deterministic-error}, 
and~\ref{thm:trace-deterministic-error}
in Section~\ref{sxn:theory-det}.)
We call these ``deterministic structural results'' since there is no
randomness involved in their statement or analysis and since they depend on
structural properties of the input data matrix and the way the sketching 
matrix interacts with the input data.
In particular, they highlight the importance of the statistical
leverage scores (and other related structural nonuniformities having to do 
with the subspace structure of the input matrix), which have proven 
important in other applications of random sampling and random projection 
algorithms.
\item
Fourth, our main algorithmic contribution is to show that when the low-rank
sketching matrix represents certain random projection or random sampling 
operations, then we obtain worst-case quality-of-approximation bounds that 
hold with high probability.
(A precise statement of these results is given in 
Lemmas~\ref{lem:sample-lev}, \ref{lem:proj-fourier}, 
\ref{lem:proj-gaussian}, and~\ref{lem:sample-unif} in
Section~\ref{sxn:theory-rand}.)
These bounds are qualitatively better than existing bounds (when nontrivial 
prior bounds even exist); they hold for reconstruction error of the input 
data with respect to the spectral norm and trace norm as well as the 
Frobenius norm; and they illustrate how high-quality random sampling 
algorithms and high-quality random projection algorithms can be treated from 
a unified perspective.
\end{compactitem}

A novel aspect of our work is that we adopt a unified 
approach to these low-rank approximation questions---unified in the sense 
that we consider both sampling and projection algorithms on an equal 
footing, and that we illustrate how the structural nonuniformities 
responsible for high-quality low-rank approximation in worst-case 
analysis also have important empirical consequences in a diverse class of 
SPSD matrices.
By identifying deterministic structural conditions responsible for 
high-quality low-rank approximation of SPSD matrices, we highlight 
complementary aspects of sampling and projection methods; and
by illustrating the empirical consequences of structural nonuniformities, we 
provide theory that is a much closer guide to practice than has been 
provided by prior work.
More generally, we should note that, although it is beyond the scope of 
this paper, our deterministic structural results could be used to check, in 
an \emph{a posteriori} manner, the quality of a sketching method for which 
one cannot establish an \emph{a priori} bound.

Our analysis is timely for several reasons.
First, in spite of the empirical successes of Nystr\"{o}m-based and other 
randomized low-rank methods, existing theory for the Nystr\"{o}m method is 
quite modest.
For example, existing worst-case bounds such as those 
of~\cite{dm_kernel_JRNL} are very weak, especially compared with existing 
bounds for least-squares regression and general low-rank matrix approximation 
problems~\cite{DMM08_CURtheory_JRNL,DMMS07_FastL2_NM10,Mah-mat-rev_BOOK}.%
\footnote{This statement may at first surprise the reader, since an SPSD 
matrix is an example of a general matrix, and one might suppose that the 
existing theory for general matrices could be applied to SPSD matrices.  
While this is true, these existing methods for general matrices do not in 
general respect the symmetry or positive semi-definiteness of the input.}
Moreover, many other worst-case bounds make very strong assumptions about the 
coherence properties of the input data~\cite{KMT12,Gittens12_TR}. 
Second, there have been conflicting views in the literature about the 
usefulness of uniform sampling versus nonuniform sampling based on the 
empirical statistical leverage scores of the data in realistic data analysis 
and machine learning applications.
For example, some work has concluded that the statistical leverage scores of 
realistic data matrices are fairly uniform, meaning that the coherence is 
small and thus uniform sampling is appropriate~\cite{WS01,KMT12}; while 
other work has demonstrated that leverage scores are often very nonuniform 
in ways that render uniform sampling inappropriate and that can be essential 
to highlight properties of downstream interest~\cite{Paschou07b,CUR_PNAS}. 
\todo[inline]{Talwalkar: In [38] we don't strictly claim that uniform sampling is 
the best.  A more accurate statement is the quote at the end of Section 4: "The 
empirical results suggest a trade-off between time and space requirements, as noted 
by Scholkopf and Smola (2002)[Chapter 10.2]. Adaptive techniques spend more time to 
find a concise subset of informative columns, but as in the case of the K-means 
algorithm, can provide improved approximation accuracy." }
Third, in recent years several high-quality numerical implementations of 
randomized matrix algorithms for least-squares and low-rank approximation 
problems have been developed~\cite{AMT10,MSM11_TR,WLRT08,RST09,MRT11}.
These have been developed from a ``scientific computing'' perspective, 
where condition numbers, spectral norms, etc. are of greater 
interest~\cite{algstat10_CHAPTER}, and where relatively strong homogeneity 
assumptions can be made about the input data.
In many ``data analytics'' applications, the questions one asks are very 
different, and the input data are much less well-structured.  Thus, we 
expect that some of our results will help guide the development of 
algorithms and implementations that are more appropriate for large-scale
analytics applications.

In the next section, Section~\ref{sxn:prelim}, we start by presenting some
notation, preliminaries, and related prior work.
Then, in Section~\ref{sxn:emp} we present our main empirical results; 
and in Section~\ref{sxn:theory} we present our main theoretical results.
We conclude in Section~\ref{sxn:conc} with a brief discussion of our 
results in a broader context.

\section{Notation, Preliminaries, and Related Prior Work}
\label{sxn:prelim}

In this section, we introduce the notation used throughout the paper, and we
address several preliminary considerations, including reviewing related 
prior work.

\subsection{Notation}
\label{sxn:prelim-notation}

Let $\matA \in \R^{n \times n}$ be an arbitrary SPSD matrix with eigenvalue 
decomposition $\matA = \matU \matSig \matU^\transp$, where we partition 
$\matU$ and $\matSig$ as
\begin{equation}
 \matU = \begin{pmatrix} \matU_1 & \matU_2 \end{pmatrix} \text{ and } 
 \matSig = \begin{pmatrix} \matSig_1 & \, \\ \, & \matSig_2 \end{pmatrix}.
 \label{eqn:eigendecompositionpartition}
\end{equation}
Here, $\matU_1$ has $k$ columns and spans the top $k$-dimensional eigenspace 
of $\matA$, and $\matSig_1 \in \R^{k \times k}$ is full-rank.%
\footnote{Variants of our results hold trivially if the rank of $\matA$ is 
$k$ or less, and so we focus on this more general case here.}
We denote the eigenvalues of $\matA$ with $\lambda_1(\mat{A}) \geq 
\ldots \geq \lambda_n(\matA).$

Given $\matA$ and a rank parameter $k$, the \emph{statistical leverage 
scores of $\matA$ relative to the best rank-$k$ approximation to $\matA$} 
equal the squared Euclidean norms of the rows of the $n \times k$ matrix 
$\matU_1$:
\begin{equation}
\ell_j = \|(\mat{U}_1)_j\|^2 .
\label{eqn:levscores}
\end{equation}
The leverage scores provide a more refined notion of the structural 
nonuniformities of $\matA$ than does the notion of \emph{coherence}, 
$\mu = \frac{n}{k} \max_{i\in\{1,\ldots,n\}} \ell_i $, which equals (up to scale) the largest 
leverage score; and they have been used historically in regression 
diagnostics to identify particularly influential or outlying data points.
Less obviously, the statistical leverage scores play a crucial role in 
recent work on randomized matrix algorithms: they define the key structural 
nonuniformity that must be dealt with in order to obtain high-quality 
low-rank and least-squares approximation of general matrices via random 
sampling and random projection methods~\cite{Mah-mat-rev_BOOK}.
Although Equation~(\ref{eqn:levscores}) defines them with respect to a 
particular basis, the statistical leverage scores equal the diagonal 
elements of the projection matrix onto the span of that basis, and thus they 
can be computed from any basis spanning the same space.
Moreover, they can be approximated more quickly than the time required to 
compute that basis with a truncated SVD or a QR 
decomposition~\cite{DMMW12_JMLR}.

We denote by $\matS$ an arbitrary $n \times \ell$ ``sketching'' matrix 
that, when post-multiplying a matrix $\matA$, maps points from $\R^{n}$ to 
$\R^{\ell}$.
We are most interested in the case where $\matS$ is a random matrix that represents 
a random sampling process or a random projection process, but we do not impose 
this as a restriction unless explicitly stated.
In order to provide high-quality low-rank matrix approximations, we  
control the error of our approximation in terms of the interaction of the 
sketching matrix $\matS$ with the eigenspaces of $\matA$, and thus we~let 
\begin{equation}
 \matOmega_1 = \matU_1^\transp \matS \quad \text{ and } \matOmega_2 = \matU_2^\transp \matS
 \label{eqn:interactiontermsdefinition}
\end{equation}
denote the projection of $\matS$ onto the top and bottom eigenspaces of 
$\matA$, respectively.

Recall that, by keeping just the top $k$ singular vectors, the matrix
$\matA_k := \matU_1 \matSig_1 \matU_1^T$ is the best rank-$k$ 
approximation to $\matA$, when measured with respect to any 
unitarily-invariant matrix norm, \emph{e.g.}, the spectral, Frobenius, or 
trace norm.
For a vector $\vec{x}\in\mathbb{R}^{n}$, let $\|\vec{x}\|_\xi,$ for $\xi=1,2,\infty,$ denote the 
$1$-norm, the Euclidean norm, and the $\infty$-norm, respectively, and let 
$\mbox{Diag}(\matA)$ denote the vector consisting of the diagonal entries of the matrix 
$\matA.$ Then, 
$\TNorm{\matA} = \|\mbox{Diag}(\matSig)\|_{\infty}$ denotes the 
\emph{spectral norm} of $\matA$;
$\FNorm{\matA} = \|\mbox{Diag}(\matSig)\|_2$ denotes the \emph{Frobenius 
norm} of $\matA$; and
$\tracenorm{\matA} = \|\mbox{Diag}(\matSig)\|_1$ denotes the \emph{trace 
norm} (or nuclear norm) of $\matA$.
Clearly, 
$$
\TNorm{\matA} \le \FNorm{\matA} \le \tracenorm{\matA} \le \sqrt{n} \FNorm{\matA} \le n \TNorm{\matA} .
$$
We quantify the quality of our algorithms by the ``additional error'' 
(above and beyond that incurred by the best rank-$k$ approximation to 
$\matA$).
In the theory of algorithms, bounds of the form provided by
(\ref{eqn:additive1}) and (\ref{eqn:additive2}) below are known as 
\emph{additive-error bounds}, the reason being that the additional error is 
an additive factor of the form $\epsilon$ times a size scale that is larger 
than the ``base error'' incurred by the best rank-$k$ approximation.
In this case, the goal is to minimize the ``size scale'' of the additional 
error.
Bounds of this form are very different and in general weaker than when the
additional error enters as a multiplicative factor, such as when the error
bounds are of the form
$\|\matA-\tilde{\matA}\| \le f(n,k,\eta)\|\matA-\matA_k\|$, 
where $f(\cdot)$ is some function and $\eta$ represents other parameters of 
the problem.
These latter bounds are of greatest interest when $f = 1+\epsilon$, for an 
error parameter $\epsilon$, as in (\ref{eqn:relative1}) and 
(\ref{eqn:relative2}) below.
These \emph{relative-error bounds}, in which the size scale of the additional
error equals that of the base error, provide a \emph{much} stronger notion of 
approximation than additive-error bounds.

\todo[inline]{Address Ilse's comments on the numerical analysts' usual use of the 
terminology relative-error, and her concern about not knowing how our TCS idea
of relative-error is calibrated (i.e. state that close to 1 is good, and why, and that
we want to achieve close to 1 with as few col samples as possible). So our
desidera are low relative-error achieved for few col samples}

\subsection{Preliminaries}
\label{sxn:prelim-prelim}


In many machine learning and data analysis applications, one is interested in symmetric 
positive semi-definite (SPSD) matrices, \emph{e.g.}, kernel matrices and 
Laplacian matrices. One common column-sampling-based approach to low-rank 
approximation of SPSD matrices is the so-called Nystr\"{o}m 
method~\cite{WS01,dm_kernel_JRNL,KMT12}. The Nystr\"{o}m method---
both randomized and deterministic variants---has 
proven useful in applications where the kernel matrices are reasonably 
well-approximated by low-rank matrices; and it has been applied to Gaussian 
process regression, spectral clustering and image segmentation, manifold 
learning, and a range of other common machine learning 
tasks~\cite{WS01,WRST02_TR,FBCM04,TKR08,ZK10,KMT12}.
The simplest Nystr\"om-based procedure selects columns from the 
original data set uniformly at random and then uses those columns to 
construct a low-rank SPSD approximation.
Although this procedure can be effective in practice for certain input 
matrices, two extensions (both of which are more expensive) can 
substantially improve the performance, \emph{e.g.}, lead to lower 
reconstruction error for a fixed number of column samples, both in theory 
and in practice. 
The first extension is to sample columns with a judiciously-chosen 
nonuniform importance sampling distribution; and the second extension is to 
randomly mix (or combine linearly) columns before sampling them.
For the random sampling algorithms, an important question is what importance 
sampling distribution should be used to construct the sample; while for the 
random projection algorithms, an important question is how to implement the 
random projections.
In either case, appropriate consideration should be paid to questions such as
whether the data are sparse or dense, how the eigenvalue spectrum decays, 
the nonuniformity properties of eigenvectors, \emph{e.g.}, as quantified by 
the statistical leverage scores, whether one is interested in reconstructing 
the matrix or performing a downstream machine learning task, and so on.

The following sketching model subsumes both of these classes of methods.
\begin{itemize}
\item
\emph{SPSD Sketching Model.} 
Let $\matA$ be an $n \times n$ positive semi-definite matrix, and 
let $\matS$ be a matrix of size $n \times \ell$, where $\ell \ll n$. 
Take
\[
 \matC = \matA \matS \quad \text{ and } \matW = \matS^\transp \matA \matS.
\]
Then $\matC \matW^\pinv \matC^\transp$ is a low-rank approximation to $\matA$ with rank at most $\ell.$
\end{itemize}
\todo[inline]{Should define/introduce/explain CUR methods if we want to refer to them}
We should note that the SPSD Sketching Model, formulated in this way, is 
\emph{not} guaranteed to be numerically stable: if $\matW$ is 
ill-conditioned, then instabilities may arise in forming the product 
$\matC \matW^\dagger \matC^\transp$. 
Thus, we are also interested in $\matC \matW_k^\pinv \matC^\transp$, where 
$\matW_k$ is the best rank-$k$ approximation to $\matW$, and where $k$ is a 
rank parameter.
For example, one might specify $k$ and then ``oversample'' by choosing 
$\ell>k$ but still be interested in an approximation that has rank no greater
than $k$.
Often, ``filtering'' a low-rank approximation in this way through a (lower) 
rank-$k$ space has a regularization effect:
for example, relative-error CUR matrix decompositions are implicitly 
regularized by letting the ``middle matrix'' have rank no greater than 
$k$~\cite{DMM08_CURtheory_JRNL,CUR_PNAS}; and
\cite{CD11_TR} considers a regularization of the uniform column sampling 
Nystr\"om extension where, before forming the extension, all singular values 
of $\matW$ smaller than a threshold are truncated to zero.
For our empirical evaluation, we consider both cases, which we
refer to as ``non-rank-restricted'' and ``rank-restricted,'' respectively.
For our theoretical results, for simplicity of notation, we do \emph{not} 
describe the generalization of our results to this rank-restricted model; 
but we note that our analysis could be extended to include this, \emph{e.g.}, 
by letting the sketching matrix $\matS$ be a combination of a sampling 
operation and an operation that projects to the best rank-$k$ approximation.

The choice of distribution for the sketching matrix $\matS$ leads to 
different classes of low-rank approximations.
For example, if $\matS$ represents the process of column sampling, either uniformly 
or according to a nonuniform importance sampling distribution, then we refer 
to the resulting approximation as a Nystr\"{o}m extension; if $\matS$ consists of 
random linear combinations of most or all of the columns of $\matA$, then we 
refer to the resulting approximation as a projection-based SPSD approximation.
In this paper, we focus on Nystr\"om extensions and projection-based SPSD 
approximations that fit the above SPSD Sketching Model.
In particular, we do not consider adaptive schemes, which iteratively select 
columns to progressively decrease the approximation error. 
While these methods often perform well in 
practice~\cite{BW09_PNAS,BW09_JRNL,FGK11,KMT12}, rigorous analyses of them 
are hard to come by---interested readers are referred to the discussion 
in~\cite{FGK11,KMT12}.

\todo[inline]{Talwalkar: You should perhaps also also include column-projection 
approximations (defined in [36]) in your discussion / experiments.  I would 
argue that a truly unified approach would include Nystrom, Column-Projection 
and Random Projection (and this indeed is what we do in our DFC work [43])}

\todo[inline]{Talwalkar: Considering parallel run times would be quite interesting.
For instance, random projection can be trivially parallelized, and in a large-scale
setting, reporting parallel runtime is appropriate.}

\subsection{The Power Method}
\label{sxn:prelim-power}

One can obtain the optimal rank-$k$ approximation to $\matA$ by forming an 
SPSD sketch where the sketching matrix $\matS$ is an orthonormal basis for the range of $\matA_k,$ 
because with such a~choice,
\begin{align*}
 \matC \matW^\pinv \matC^\transp & = 
   \matA \matS (\matS^\transp \matA \matS)^\pinv \matS^\transp \matA 
   = \matA (\matS \matS^\transp \matA \matS \matS^\transp)^\pinv \matA
   = \matA (\matP_{\matA_k} \matA \matP_{\matA_k} )^\pinv \matA 
   = \matA \matA_k^\pinv \matA = \matA_k.
\end{align*}
Of course, one cannot quickly obtain such a basis; this motivates 
considering sketching matrices $\matS_q$ obtained using the power method: that is,
taking $\matS_q = \matA^q \matS_0$ where $q$ is a positive integer and $\matS_0 \in \R^{n \times \ell}$
with $l \geq k.$ As $q \rightarrow \infty,$ assuming $\matU_1^\transp \matS_0$
has full row-rank, the matrices $\matS_q$ increasingly capture the dominant 
$k$-dimensional eigenspaces of $\matA$~\cite[Chapter 8]{GL96}, so one can 
reasonably expect that the sketching matrix $\matS_q$ produces SPSD sketches 
of $\matA$ with lower additional error.

SPSD sketches produced using $q$ iterations
of the power method have lower error than sketches produced without using the 
power method, but are roughly $q$ times more costly to produce. Thus, the power
method is most applicable when $\matA$ is such that one can compute the product 
$\matA^q \matS_0$ fast. 
We consider the empirical performance of sketches
produced using the power method in Section~\ref{sxn:emp}, and
we consider the theoretical performance in Section~\ref{sxn:theory}.

\subsection{Related Prior Work}
\label{sxn:prelim-prior}
Motivated by large-scale data analysis and machine learning applications, 
recent theoretical and empirical work has focused on ``sketching'' methods 
such as random sampling and random projection algorithms. 
A large part of the recent body of this work on randomized matrix algorithms 
has been summarized 
in the recent monograph of Mahoney~\cite{Mah-mat-rev_BOOK} and the recent 
review article of Halko, Martinsson, and Tropp~\cite{HMT09_SIREV}.
Here, we note that, on the empirical side, both random projection 
methods (\emph{e.g.}, \cite{BM01,FM03,VW11} and \cite{BDT11_TR})
and random sampling methods (\emph{e.g.}, \cite{Paschou07b,CUR_PNAS}) have 
been used in applications for clustering and classification of general data 
matrices; 
and that some of this work has highlighted the importance of the statistical 
leverage scores that we use in this 
paper~\cite{Paschou07b,CUR_PNAS,Mah-mat-rev_BOOK,chingwa-astro-draft}. 
In parallel, so-called Nystr\"om-based methods have also been used in 
machine learning applications.
Originally used by Williams and Seeger to solve regression and 
classification problems involving Gaussian processes when the SPSD matrix 
$\matA$ is well-approximated by a low-rank matrix~\cite{WS01,WRST02_TR}, the 
Nystr\"om extension has been used in a large body of subsequent work.
For example, applications of the Nystr\"{o}m method 
to large-scale machine learning problems
include~\cite{TKR08,KMT09,KMT09c,MTJ11_TR} and~\cite{ZTK08,LKL10,ZK10}, and
applications in statistics and signal processing
include~\cite{PWT05,BW07_WKSHP,BW07_TR,SW08,BW08,BW09_PNAS,BW09_JRNL}.

Much of this work has focused on new proposals for selecting columns 
(\emph{e.g.}, \cite{ZTK08,ZK09,LZS10,AW10,LKL10})
and/or coupling the method with downstream applications
(\emph{e.g.}, \cite{BJ05,CMT10_AISTATS,JYNLZ11_TR,HM11,MPARG11,Bac12_TR}).
The most detailed results are provided by~\cite{KMT12} (as well as the 
conference papers on which it is based~\cite{KMT09,KMT09b,KMT09c}).
Interestingly, they observe that uniform sampling performs quite well, 
suggesting that in the data they considered the leverage scores are
quite uniform, which also motivated the related 
work~\cite{TalRos10,MT11_AISTATS}.
This is in contrast with applications in genetics~\cite{Paschou07b},
term-document analysis~\cite{CUR_PNAS}, and astronomy~\cite{chingwa-astro-draft}, 
where the statistical leverage scores were seen to be very nonuniform in 
ways of interest to the downstream scientist; we return to this issue 
in Section~\ref{sxn:emp}.

On the theoretical side, much of the work has followed that of Drineas and 
Mahoney~\cite{dm_kernel_JRNL}, who provided the first rigorous bounds for the 
Nystr\"om extension of a general SPSD matrix. 
They show that when $\Omega(k \epsilon^{-4} \ln \delta^{-1})$ columns 
are sampled with an importance sampling distribution that is 
proportional to the square of the diagonal entries of~$\matA$, then 
\begin{equation}
   \| \matA - \matC \matW^\pinv \matC^\transp \|_{\xi} \leq \| \mat{A} - \mat{A}_k \|_{\xi} + \epsilon \sum\nolimits_{k=1}^n (\matA)_{ii}^2 \\
\label{eqn:nystrom-dm}
\end{equation}
holds with probability $1-\delta$, where $\xi=2,F$ represents the Frobenius 
or spectral norm.
(Actually, they prove a stronger result of the form given in 
Equation~(\ref{eqn:nystrom-dm}), except with $\matW^\pinv$ replaced with 
$\matW_k^\pinv$, where $\matW_k$ represents the best rank-$k$ approximation 
to $\matW$~\cite{dm_kernel_JRNL}.)
Subsequently, Kumar, Mohri, and Talwalkar show that if $\mu k \ln(k/\delta))$ columns
are sampled uniformly at 
random with replacement from an $\matA$ that has \emph{exactly} rank $k$, 
then one achieves exact recovery, \emph{i.e.}, 
$\matA=\matC\matW^\pinv\matC^\transp$, with high probability~\cite{KMT09}. 
Gittens extends this to the case where $\matA$ is only approximately 
low-rank~\cite{Gittens12_TR}. 
In particular, he shows that if $\ell = \Omega( \mu k \ln k)$ columns are 
sampled uniformly at random (either with or without replacement),~then 
\begin{equation}
 \TNorm{\matA - \matC \matW^\pinv \matC^\transp} \leq \TNorm{\matA - \matA_k} \left( 1 + \frac{2n}{\ell} \right)
\label{eqn:nystrom-git1}
\end{equation}
with probability exceeding $1 - \delta$ and
\begin{equation}
 \TNorm{\matA - \matC \matW^\pinv \matC^\transp} \leq \TNorm{\matA - \matA_k} + \frac{2}{\delta} \cdot \tracenorm{\matA - \matA_k}
\label{eqn:nystrom-git2}
\end{equation}
with probability exceeding $1 - 2\delta.$

We have described these prior theoretical bounds in detail to emphasize how
strong, relative to the prior work, our new bounds are.
For example, 
Equation~(\ref{eqn:nystrom-dm}) provides an additive-error approximation 
with a very large scale; 
the bounds of Kumar, Mohri, and Talwalkar require a sampling complexity 
that depends on the coherence of the input matrix~\cite{KMT09}, which means 
that unless the coherence is very low one needs to sample essentially all 
the rows and columns in order to reconstruct the matrix; 
Equation~(\ref{eqn:nystrom-git1}) provides a bound where the additive scale 
depends on $n$; and
Equation~(\ref{eqn:nystrom-git2}) provides a spectral norm bound where the 
scale of the additional error is the (much larger) trace norm.
Table~\ref{table:bounds-comparison} compares the bounds on the approximation errors of
SPSD sketches derived in this work to those available in the literature. We note further
that Wang and Zhang recently established lower-bounds on the worst-case relative spectral
and trace norm errors of uniform Nystr\"om extensions~\cite{WZ13}. Our Lemma~\ref{lem:sample-unif}
provides matching upper bounds, showing the optimality of these estimates. 

A related stream of research concerns projection-based low-rank approximations of
general (\emph{i.e.}, non-SPSD) matrices~\cite{HMT09_SIREV,Mah-mat-rev_BOOK}. Such approximations are formed by first
constructing an approximate basis for the top left invariant subspace of $\matA,$
and then restricting $\matA$ to this space. Algorithmically, one constructs $\matY = \matA \matS,$
where $\matS$ is a sketching matrix, then takes $\matQ$ to be a basis obtained from the
QR decomposition of $\matY,$ and then forms the low-rank approximation $\matQ\matQ^\transp \matA.$  
The survey paper~\cite{HMT09_SIREV} proposes two schemes for the approximation of SPSD matrices that 
fit within this paradigm: $\matQ (\matQ^\transp \matA \matQ) \matQ^\transp$
and $(\matA \matQ) (\matQ^\transp \matA \matQ)^\pinv (\matQ^\transp \matA).$ The first scheme---for which
\cite{HMT09_SIREV} provides quite sharp error bounds when $\matS$ is a matrix of i.i.d. 
standard Gaussian random variables---has the salutary property
of being numerically stable. 
On the other hand, although~\cite{HMT09_SIREV} does not provide any 
theoretical guarantees for the second scheme, it points out that this latter 
scheme produces noticeably more accurate approximations in practice.
In Section~\ref{sxn:emp}, we provide empirical evidence of the superior 
performance of the second scheme, and we show that it is actually an 
instantiation of the power method (as described in 
Section~\ref{sxn:prelim-power}) with $q=2.$ 
Accordingly, the deterministic and stochastic error bounds provided in 
Section~\ref{sxn:theory} are applicable to this SPSD sketch.

It is worth noting that in~\cite{WZ13}, the authors propose a modified Nystr\"om method wherein the matrix
$\matW$ is replaced by $\matC^\pinv \matA (\matC^\pinv)^\transp,$ so that the low rank approximation to
$\matA$ is given by $\matC \matC^\pinv \matA (\matC^\pinv)^\transp \matC^\transp.$ Note that $\matC \matC^\pinv$ is
another expression for the orthoprojector $\matQ \matQ^\transp$ onto the range of $\matY = \matA \matS,$ so 
this Nystr\"om method is an instantiation of the projection-based low-rank approximations
analyzed in~\cite{HMT09_SIREV}. However,~\cite{WZ13}, unlike~\cite{HMT09_SIREV}, considers the case where $\matC$
is constructed by sampling from the columns of $\matA$ adaptively. The low-rank approximation produced by 
the algorithm proposed in~\cite{WZ13} satisfies
\[
 \E\FNorm{\matA - \matC \matC^\pinv \matA (\matC^\pinv)^\transp \matC^\transp} \leq (1 + \epsilon) \FNorm{\matA - \matA_k}
\]
when $\const{O}(k/\epsilon^2)$ columns are sampled.

\begin{table}[ht]
\small
\begin{center}
\begin{tabular}{|p{1in}|c|c|c|c|}
\hline
Source
 & $\ell$
 & $\|\matA - \matC\matW^\pinv \matC^\transp\|_2$ 
 & $\|\matA - \matC\matW^\pinv \matC^\transp\|_F$ 
 & $\|\matA - \matC\matW^\pinv \matC^\transp\|_\star$ \\
 \hline
 \multicolumn{5}{|c|}{Prior works} \\
 \hline
\cite{dm_kernel_JRNL}
 & $\Omega(\epsilon^{-4}k)$
 & $\mbox{opt}_2 + \epsilon \sum_{i=}^n A_{ii}^2$ 
 & $\mbox{opt}_F + \epsilon \sum_{i=1}^n A_{ii}^2$  
 & -- \\
\hline
 \cite{BW09_PNAS}
 & $\Omega(1)$
 & --  
 & --  
 & $\mbox{O}\left(\frac{n-\ell}{n}\right)\tracenorm{\mat{A}}$ \\
\hline
 \cite{TalRos10}
 & $\Omega(\mu_r r \ln r)$
 & 0
 & 0
 & 0 \\
\hline
\cite{KMT12}
 & $ \Omega(1)$
 & $\mbox{opt}_2 + \frac{n}{\sqrt{\ell}} \TNorm{\matA}$ 
 & $\mbox{opt}_F + n(\frac{k}{\ell})^{1/4} \TNorm{\matA}$ 
 & -- \\
\hline
\multicolumn{5}{|c|}{This work}\\
\hline
 Lemma~\ref{lem:sample-unif}, uniform column \linebreak sampling
 & $\Omega\left(\frac{ \mu_k k \ln k}{(1-\epsilon)^{2}}\right)$
 & $\mbox{opt}_2(1 + \frac{n}{\epsilon\ell})$ 
 & $\mbox{opt}_F + \epsilon^{-1} \mbox{opt}_\star$  
 & $\mbox{opt}_\star(1 + \epsilon^{-1})$ \\
\hline
Lemma~\ref{lem:sample-lev} leverage-based column sampling
 & $\Omega\left(\frac{k \ln(k/\beta)}{\beta \epsilon^2 }\right)$
 & $\mbox{opt}_2 + \epsilon^2 \mbox{opt}_\star$
 & $\mbox{opt}_F + \epsilon \mbox{opt}_\star$ 
 & $(1 + \epsilon^2)\mbox{opt}_\star$\\
\hline 
 Lemma~\ref{lem:proj-fourier}, Fourier-based projection
 & $\Omega(\epsilon^{-1} k \ln n)$
 &$ \big(1 + \frac{1}{1 - \sqrt{\epsilon}}\big) \mbox{opt}_2 + \frac{\epsilon \mbox{opt}_\star}{(1 - \sqrt{\epsilon})k}
  $
 & $\mbox{opt}_F + \sqrt{\epsilon} \mbox{opt}_\star$
 & $(1 + \epsilon)\mbox{opt}_\star$\\
\hline
 Lemma~\ref{lem:proj-gaussian}, Gaussian-based projection
 & $\Omega(k\epsilon^{-1})$
 & $(1 + \epsilon^2) \mbox{opt}_2 + \frac{\epsilon}{k}\mbox{opt}_\star$
 & $\mbox{opt}_F + \epsilon \mbox{opt}_\star $
 & $(1+\epsilon^2) \mbox{opt}_\star$\\
 \hline
\end{tabular}
\end{center}
\caption{Comparison of our bounds on the approximation errors of several types
of SPSD sketches with those provided in prior works. Only the asymptotically 
largest terms (as $\epsilon \rightarrow 0$) are displayed and constants are 
omitted, for simplicity. Here, $\epsilon \in (0,1),$ $\mbox{opt}_\xi$ is the smallest $\xi$-norm error possible when
approximating $\matA$ with a rank-$k$ matrix ($k \geq \ln n)$, $r = \mbox{rank}(\matA),$ $\ell$ is 
the number of column samples sufficient for the stated bounds to hold, $k$ is a target rank,
and $\mu_s$ is the coherence of $\matA$ relative to the best rank-$s$ approximation to $\matA.$
The parameter $\beta \in (0,1]$ allows for the possibility of sampling using $\beta$-approximate
leverage scores (see Section~\ref{sxn:theory-rand-levscore}) rather than the exact leverage
scores.
With the exception of \cite{dm_kernel_JRNL}, which 
samples columns with probability proportional to their Euclidean norms, 
and our novel leverage-based Nystr\"om bound,
these bounds are for sampling columns or linear combinations of columns uniformly
at random. All bounds hold with constant probability.
}
\label{table:bounds-comparison}
\end{table}

\subsection{An overview of our bounds}
Our bounds in Table~\ref{table:bounds-comparison} (established as Lemmas~\ref{lem:sample-lev}--\ref{lem:sample-unif} 
in Section~\ref{sxn:theory-rand}) exhibit a common structure: for the spectral and Frobenius
norms, we see that the additional error is on a larger scale than the optimal error, and the trace norm
bounds all guarantee relative error approximations. 
This follows from the fact, as detailed in Section~\ref{sxn:theory-det}, 
that low-rank approximations that conform to the SPSD sketching model can be 
understood as forming column-sample/projection-based approximations to the 
\emph{square root} of $\mat{A}$, and thus squaring this approximation yields 
the resulting
approximation to $\mat{A}.$ The squaring process unavoidably results in potentially large additional errors
in the case of the spectral and Frobenius norms---
whether or not the additional errors are large in practice depends upon the properties of the matrix
and the form of stochasticity used in the sampling process. For instance, from our bounds it is clear that 
Gaussian-based SPSD sketches are expected to have lower additional error in the spectral norm
than any of the other sketches considered.

From Table~\ref{table:bounds-comparison}, we also 
see, in the case of uniform Nystr\"om extensions, a necessary dependence on the
coherence of the input matrix since columns are sampled uniformly at random. 
However, we also see that the scales of the additional error of the Frobenius and
trace norm bounds are substantially improved over those in prior results. 
The large additional error in the spectral norm
error bound is necessary in the worse case~\cite{Gittens12_TR}.
Lemmas~\ref{lem:sample-lev}, \ref{lem:proj-fourier} and
\ref{lem:proj-gaussian} in Section~\ref{sxn:theory-rand}---which respectively
address leverage-based, Fourier-based, and Gaussian-based SPSD sketches---show that
spectral norm additive-error bounds with additional error on a substantially smaller scale
can be obtained if one first mixes the columns before
sampling from $\matA$ or one samples from a judicious nonuniform distribution over
the~columns.

Table~\ref{table:sketchstatistics} compares the minimum, mean, and maximum approximation errors of 
several SPSD sketches of four matrices (described in Section~\ref{sxn:emp-datasets}) to the optimal
rank-$k$ approximation errors. We consider
three regimes for $\ell,$ the number of column samples used to construct the sketch: $\ell = \const{O}(k),$
$\ell = \const{O}(k \ln k),$ and $\ell = \const{O}(k \ln n).$
These
matrices exhibit a diverse range of properties: {\emph e.g.}, Enron is sparse and has a slowly decaying
spectrum, while Protein is dense and has a rapidly decaying spectrum. Yet we notice that the sketches
perform quite well on each of these matrices. In particular, when $\ell = \const{O}(k \ln n),$ the average
errors of the sketches are within $1 + \epsilon$ of the optimal rank-$k$ approximation errors, where $\epsilon \in [0,1].$
Also note that the leverage-based sketches consistently have lower average errors (in all of the three norms considered)
than all other sketches. Likewise, the uniform Nystr\"om extensions usually have larger average errors than the other sketches.
These two sketches represent opposite extremes: uniform Nystr\"om extensions (constructed using uniform column
sampling) are constructed using no knowledge about the matrix, while leverage-based sketches use an
importance sampling distribution derived from the SVD of the matrix to determine which columns to use in the 
construction of the sketch.

Table~\ref{table:theory-practice-gap} illustrates the gap between the theoretical results 
currently available in the literature and what is observed in practice: it depicts the ratio
between the error bounds in Table~\ref{table:bounds-comparison} and the average errors observed
over 30 runs of the SPSD approximation algorithms (the error bound from~\cite{TalRos10} is not considered in the table,
as it does not apply at the number of samples $\ell$ used in the experiments). 

\begin{landscape}
\begin{center}
\begin{table}
\tiny
\begin{center}
 \begin{tabular}{cc}
 \begin{tabular}{|c|}
  \hline
  Enron, $k = 60$ \\
 \hline
$\|\matA - \matC\matW^\pinv\matC^\transp\|_2/\|\matA - \matA_k\|_2$ \\
\hline
\begin{tabular}{lccc}
& $\ell = k+8$ & $\ell = k\ln k$ & $\ell = k \ln n$ \\
Nystr\"om & 1.386/1.386/1.386 & 1.386/1.386/1.386 & 1.386/1.386/1.386 \\
SRFT sketch & 1.378/1.379/1.381 & 1.357/1.360/1.364 & 1.310/1.317/1.323 \\
Gaussian sketch & 1.378/1.380/1.381 & 1.357/1.360/1.364 & 1.314/1.318/1.323 \\
Leverage sketch & 1.321/1.381/1.386 & 1.039/1.188/1.386 & 1.039/1.042/1.113
\end{tabular} \\
\hline
$\|\matA - \matC\matW^\pinv\matC^\transp\|_{\mathrm{F}}/\|\matA - \matA_k\|_{\mathrm{F}}$ \\
\hline
\begin{tabular}{lccc}
& $\ell = k+8$ & $\ell = k\ln k$ & $\ell = k \ln n$ \\
Nystr\"om & 1.004/1.004/1.004 & 0.993/0.994/0.994 & 0.972/0.972/0.973 \\
SRFT sketch & 1.004/1.004/1.004 & 0.994/0.994/0.994 & 0.972/0.972/0.972 \\
Gaussian sketch & 1.004/1.004/1.004 & 0.994/0.994/0.994 & 0.972/0.972/0.972 \\
Leverage sketch & 1.002/1.002/1.003 & 0.994/0.995/0.996 & 0.988/0.989/0.989
\end{tabular} \\ 
\hline
$\|\matA - \matC\matW^\pinv\matC^\transp\|_{\star}/\|\matA - \matA_k\|_{\star}$ \\
\hline
\begin{tabular}{lccc}
& $\ell = k+8$ & $\ell = k\ln k$ & $\ell = k \ln n$ \\
Nystr\"om & 1.002/1.002/1.003 & 0.984/0.984/0.984 & 0.943/0.944/0.944 \\
SRFT sketch & 1.002/1.002/1.002 & 0.984/0.984/0.984 & 0.944/0.944/0.944 \\
Gaussian sketch & 1.002/1.002/1.002 & 0.984/0.984/0.984 & 0.944/0.944/0.944 \\
Leverage sketch & 1.002/1.002/1.003 & 0.990/0.991/0.992 & 0.977/0.978/0.980
\end{tabular} \\ 
\hline
 \end{tabular}
 &
  \begin{tabular}{|c|}
  \hline
  Protein, $k=10$ \\
\hline
$\|\matA - \matC\matW^\pinv\matC^\transp\|_2/\|\matA - \matA_k\|_2$ \\
\hline
\begin{tabular}{lccc}
& $\ell = k+8$ & $\ell = k\ln k$ & $\ell = k \ln n$ \\
Nystr\"om & 1.570/2.104/2.197 & 1.496/2.100/2.196 & 1.023/1.350/2.050 \\
SRFT sketch & 1.835/1.950/2.039 & 1.686/1.874/2.009 & 1.187/1.287/1.405 \\
Gaussian sketch & 1.812/1.956/2.058 & 1.653/1.894/2.007 & 1.187/1.293/1.438 \\
Leverage sketch & 1.345/1.644/2.166 & 1.198/1.498/2.160 & 0.942/0.994/1.073
\end{tabular} \\
\hline
$\|\matA - \matC\matW^\pinv\matC^\transp\|_{\mathrm{F}}/\|\matA - \matA_k\|_{\mathrm{F}}$ \\
\hline
\begin{tabular}{lccc}
& $\ell = k+8$ & $\ell = k\ln k$ & $\ell = k \ln n$ \\
Nystr\"om & 1.041/1.054/1.065 & 1.023/1.042/1.054 & 0.867/0.877/0.894 \\
SRFT sketch & 1.049/1.054/1.058 & 1.032/1.037/1.043 & 0.873/0.877/0.880 \\
Gaussian sketch & 1.049/1.054/1.060 & 1.032/1.039/1.043 & 0.874/0.878/0.883 \\
Leverage sketch & 1.027/1.036/1.054 & 1.011/1.018/1.034 & 0.862/0.868/0.875
\end{tabular} \\ 
\hline
$\|\matA - \matC\matW^\pinv\matC^\transp\|_{\star}/\|\matA - \matA_k\|_{\star}$ \\
\hline
\begin{tabular}{lccc}
& $\ell = k+8$ & $\ell = k\ln k$ & $\ell = k \ln n$ \\
Nystr\"om & 1.011/1.014/1.018 & 0.988/0.994/0.998 & 0.760/0.764/0.770 \\
SRFT sketch & 1.013/1.015/1.016 & 0.990/0.993/0.995 & 0.762/0.764/0.766 \\
Gaussian sketch & 1.013/1.015/1.017 & 0.991/0.993/0.994 & 0.762/0.765/0.767 \\
Leverage sketch & 1.004/1.008/1.014 & 0.982/0.985/0.991 & 0.758/0.765/0.771
\end{tabular} \\ 
\hline
 \end{tabular} \\
 \\
 
 \begin{tabular}{|c|}
  \hline
  AbaloneD, $\sigma = .15$, $k = 20$ \\
\hline
$\|\matA - \matC\matW^\pinv\matC^\transp\|_2/\|\matA - \matA_k\|_2$ \\
\hline
\begin{tabular}{lccc}
& $\ell = k+8$ & $\ell = k\ln k$ & $\ell = k \ln n$ \\
Nystr\"om & 2.168/2.455/2.569 & 2.022/2.381/2.569 & 1.823/2.204/2.567 \\
SRFT sketch & 2.329/2.416/2.489 & 2.146/2.249/2.338 & 1.741/1.840/1.918 \\
Gaussian sketch & 2.347/2.409/2.484 & 2.161/2.254/2.361 & 1.723/1.822/1.951 \\
Leverage sketch & 1.508/1.859/2.377 & 1.152/1.417/2.036 & 0.774/0.908/1.091
\end{tabular} \\
\hline
$\|\matA - \matC\matW^\pinv\matC^\transp\|_{\mathrm{F}}/\|\matA - \matA_k\|_{\mathrm{F}}$ \\
\hline
\begin{tabular}{lccc}
& $\ell = k+8$ & $\ell = k\ln k$ & $\ell = k \ln n$ \\
Nystr\"om & 1.078/1.090/1.098 & 1.061/1.078/1.091 & 1.026/1.040/1.054 \\
SRFT sketch & 1.088/1.089/1.090 & 1.074/1.075/1.077 & 1.034/1.035/1.037 \\
Gaussian sketch & 1.087/1.089/1.091 & 1.073/1.075/1.077 & 1.033/1.035/1.036 \\
Leverage sketch & 1.028/1.040/1.059 & 0.998/1.006/1.020 & 0.959/0.963/0.968
\end{tabular} \\ 
\hline
$\|\matA - \matC\matW^\pinv\matC^\transp\|_{\star}/\|\matA - \matA_k\|_{\star}$ \\
\hline
\begin{tabular}{lccc}
& $\ell = k+8$ & $\ell = k\ln k$ & $\ell = k \ln n$ \\
Nystr\"om & 1.022/1.024/1.026 & 1.010/1.014/1.016 & 0.977/0.980/0.983 \\
SRFT sketch & 1.024/1.024/1.024 & 1.014/1.014/1.014 & 0.980/0.980/0.981 \\
Gaussian sketch & 1.024/1.024/1.024 & 1.014/1.014/1.014 & 0.980/0.980/0.981 \\
Leverage sketch & 1.009/1.012/1.016 & 0.994/0.997/1.000 & 0.965/0.968/0.971
\end{tabular} \\ 
\hline
\end{tabular} 
 &
  \begin{tabular}{|c|}
  \hline
  WineS, $\sigma = 1,$ $k= 20$ \\
\hline
$\|\matA - \matC\matW^\pinv\matC^\transp\|_2/\|\matA - \matA_k\|_2$ \\
\hline
\begin{tabular}{lccc}
& $\ell = k+8$ & $\ell = k\ln k$ & $\ell = k \ln n$ \\
Nystr\"om & 1.989/2.001/2.002 & 1.987/1.998/2.002 & 1.739/1.978/2.002 \\
SRFT sketch & 1.910/1.938/1.966 & 1.840/1.873/1.905 & 1.624/1.669/1.709 \\
Gaussian sketch & 1.903/1.942/1.966 & 1.839/1.873/1.910 & 1.619/1.670/1.707 \\
Leverage sketch & 1.242/1.762/1.995 & 1.000/1.317/1.987 & 1.000/1.000/1.005
\end{tabular} \\
\hline
$\|\matA - \matC\matW^\pinv\matC^\transp\|_{\mathrm{F}}/\|\matA - \matA_k\|_{\mathrm{F}}$ \\
\hline
\begin{tabular}{lccc}
& $\ell = k+8$ & $\ell = k\ln k$ & $\ell = k \ln n$ \\
Nystr\"om & 1.036/1.040/1.043 & 1.028/1.034/1.038 & 0.998/1.009/1.018 \\
SRFT sketch & 1.038/1.039/1.039 & 1.029/1.030/1.030 & 1.000/1.000/1.001 \\
Gaussian sketch & 1.038/1.039/1.039 & 1.029/1.030/1.030 & 1.000/1.000/1.001 \\
Leverage sketch & 1.004/1.011/1.018 & 0.996/1.000/1.005 & 0.994/0.995/0.997
\end{tabular} \\
\hline
$\|\matA - \matC\matW^\pinv\matC^\transp\|_{\star}/\|\matA - \matA_k\|_{\star}$ \\
\hline
\begin{tabular}{lccc}
& $\ell = k+8$ & $\ell = k\ln k$ & $\ell = k \ln n$ \\
Nystr\"om & 1.013/1.015/1.016 & 1.002/1.005/1.007 & 0.965/0.970/0.976 \\
SRFT sketch & 1.014/1.014/1.015 & 1.004/1.004/1.004 & 0.970/0.970/0.970 \\
Gaussian sketch & 1.014/1.014/1.015 & 1.004/1.004/1.004 & 0.970/0.970/0.970 \\
Leverage sketch & 1.002/1.005/1.009 & 0.997/0.999/1.002 & 0.995/0.996/0.997
\end{tabular} \\ 
\hline
 \end{tabular}
 \end{tabular}
\end{center}
 \caption{The min/mean/max ratios of the errors of several non-rank-restricted SPSD sketches to the optimal rank-$k$ approximation error for several of the
 matrices considered in Table~\ref{table:datasets}. Here $k$ is the target rank and $\ell$ is the number of column samples used to form
 the SPSD sketches. The min/mean/max ratios were computed using 30 trials for combination of $\ell$ and sketching~method.}
 \label{table:sketchstatistics}
\end{table}
\end{center}
\end{landscape}

\todo[inline]{Make sure the footnote associated with this table is on the 
right page after the document is finalized}
\begin{savenotes}
\begin{table}[!p]
\fontsize{8}{10}
\selectfont
\begin{center}
\begin{tabular}{|p{1.5in}|c|c|c|}
\hline
source, sketch
 & pred./obs. spectral error
 & pred./obs. Frobenius error
 & pred./obs. trace error \\
  \hline
  \multicolumn{4}{|c|}{Enron, $k = 60$ } \\
  \hline
\cite{dm_kernel_JRNL}, column sampling with probabilities proportional to squared diagonal entries
 & 3041.0
 & 66.2
 & -- \\
 \cline{1-1}
\cite{BW09_PNAS}, uniform column sampling with replacement
 & --  
 & --  
 & 2.0\\
 \cline{1-1}
\cite{KMT12}, uniform column sampling without replacement
 & 331.2
 & 77.7
 & -- \\
 \cline{1-1}
 Lemma~\ref{lem:sample-lev},  leverage-based column sampling
 & 1287.0
 & 20.5
 & 1.2 \\
 \cline{1-1}
Lemma~\ref{lem:proj-fourier}, Fourier-based
 & 102.1
 & 42.0
 & 1.6 \\
 \cline{1-1}
Lemma~\ref{lem:proj-gaussian}, Gaussian-based
 & 20.1
 & 7.6
 & 1.4 \\
 \cline{1-1}
Lemma~\ref{lem:sample-unif}, uniform column sampling with replacement
 & 9.4
 & 285.1
 & 9.5 \\
 \hline
 \multicolumn{4}{|c|}{Protein, $k = 10$ }\\
 \hline
\cite{dm_kernel_JRNL}, column sampling with probabilities proportional to squared diagonal entries
 & 125.2
 & 18.6
 & -- \\
 \cline{1-1}
\cite{BW09_PNAS}, uniform column sampling with replacement
 & --  
 & --  
 & 3.6\\
 \cline{1-1}
\cite{KMT12}, uniform column sampling without replacement
 & 35.1
 & 20.5
 & -- \\
 \cline{1-1}
Lemma~\ref{lem:sample-lev},  leverage-based
 & 42.4
 & 6.2
 & 2.0 \\
 \cline{1-1}
Lemma~\ref{lem:proj-fourier}, Fourier-based
 & 155.0
 & 20.4
 & 3.1\\
 \cline{1-1}
Lemma~\ref{lem:proj-gaussian}, Gaussian-based
 & 5.7
 & 5.6
 & 2.2 \\
 \cline{1-1}
Lemma~\ref{lem:sample-unif}, uniform column sampling with replacement
 & 90.0
 & 63.4
 & 14.3 \\
 \hline
 \multicolumn{4}{|c|}{AbaloneD, $\sigma = .15, k = 20$ } \\
 \hline
\cite{dm_kernel_JRNL}, column sampling with probabilities proportional to squared diagonal entries
 & 360.8
 & 42.5
 & -- \\
 \cline{1-1}
\cite{BW09_PNAS}, uniform column sampling with replacement
 & --  
 & --  
 & 2.0 \\
 \cline{1-1}
\cite{KMT12}, uniform column sampling without replacement
 & 62.0
 & 45.7
 & -- \\
 \cline{1-1}
Lemma~\ref{lem:sample-lev},  leverage-based
 & 235.4
 & 14.1
 & 1.3 \\
 \cline{1-1}
Lemma~\ref{lem:proj-fourier}, Fourier-based
 & 70.1 
 & 36.0
 & 1.7 \\
 \cline{1-1}
Lemma~\ref{lem:proj-gaussian}, Gaussian-based
 & 8.7
 & 8.3
 & 1.3\\
 \cline{1-1}
Lemma~\ref{lem:sample-unif}, uniform column sampling with replacement
 & 13.2
 & 166.2
 & 9.0 \\
 \hline
 \multicolumn{4}{|c|}{WineS, $\sigma = 1, k = 20$ }\\
 \hline
 \cite{dm_kernel_JRNL}, column sampling with probabilities proportional to squared diagonal entries
 & 408.4
 & 41.1
 & -- \\
 \cline{1-1}
\cite{BW09_PNAS}, uniform column sampling with replacement
 & --  
 & --  
 & 2.1\\
 \cline{1-1}
\cite{KMT12}, uniform column sampling without replacement
 & 70.3
 & 44.3
 & -- \\
 \cline{1-1}
Lemma~\ref{lem:sample-lev},  leverage-based
 & 244.6
 & 12.9
 & 1.2\\
 \cline{1-1}
Lemma~\ref{lem:proj-fourier}, Fourier-based
 & 94.8
 & 36.0
 & 1.7\\
 \cline{1-1}
Lemma~\ref{lem:proj-gaussian}, Gaussian-based
 & 11.4
 & 8.1
 & 1.4\\
 \cline{1-1}
Lemma~\ref{lem:sample-unif}, uniform column sampling with replacement
 & 13.2
 & 162.2
 & 9.1\\
 \hline
\end{tabular}
\end{center}
\caption{Comparison of the empirically observed approximation errors to the guarantees provided in this and other works, for several
datasets. Each approximation was formed using $\ell = 6 k\ln k$ samples. To evaluate the error guarantees,
$\delta = 1/2$ was taken{\protect\footnote{Taking $\delta$ too much smaller
results in estimates for the number of samples required that exceed the dimensions of the matrices considered.}} and all constants present in the statements of the bounds were replaced with ones. The observed errors were
taken to be the average errors over 30 runs of the approximation algorithms.
The datasets, described in Section~\ref{sxn:emp-datasets}, are representative of several classes of matrices prevalent
in machine learning applications.}
\label{table:theory-practice-gap}
\end{table}
\end{savenotes}

\noindent
Several trends can be identified; among them, we
note that the bounds provided in this paper for Gaussian-based sketches come quite close to capturing 
the errors seen in practice, and the Frobenius and trace norm error guarantees of the 
leverage-based and Fourier-based sketches tend to more closely reflect the empirical behavior than
the error guarantees provided in prior work for Nystr\"om sketches. Overall, the trace norm error
bounds are quite accurate. On the other hand, prior bounds are sometimes more informative in the case 
of the spectral norm (with the notable exception of the Gaussian sketches). Several important points
can be gleaned from these observations. First, the accuracy of the Gaussian error bounds suggests that the main theoretical
contribution of this work, the deterministic structural results given as Theorems~\ref{thm:spectral-deterministic-error} 
through~\ref{thm:trace-deterministic-error}, captures the underlying behavior of the SPSD sketching process. This supports our belief that this
work provides a foundation for truly informative error bounds. Given that this is the case, it is clear that the 
analysis of the stochastic elements of the SPSD sketching process is much sharper in the Gaussian case than in the 
leverage-score, Fourier, and uniform Nystr\"om cases. We expect that, at least in the case of leverage and Fourier-based sketches, the 
stochastic analysis can and will be sharpened to produce error guarantees almost as informative as the ones we have provided for
Gaussian-based sketches.

\section{Empirical Aspects of SPSD Low-rank Approximation}
\label{sxn:emp}

\todo[inline]{Talwalkar: The empirical results are very detailed, which is great, 
but having a short summary of take-away messages (e.g., which methods perform best 
for which types of data) might be helpful.}

In this section, we present our main empirical results, which consist of 
evaluating sampling and projection algorithms applied to a diverse set of 
SPSD matrices.
In addition to understanding the relative merits, in terms of both running 
time and solution quality, of different sampling/projection schemes, we 
would like to understand the effects of various data preprocessing decisions. 
The bulk of our empirical evaluation considers two random projection
procedures and two random sampling procedures for the sketching matrix 
$\matS$:
for random projections, we consider using SRFTs (Subsampled Randomized Fourier Transforms) as well as uniformly 
sampling from Gaussian mixtures of the columns; and for random sampling, we 
consider sampling columns uniformly at random as well as sampling 
columns according to a nonuniform importance sampling distribution that 
depends on the empirical statistical leverage scores.
In the latter case of leverage score-based sampling, we also consider 
the use of both the (na\"{i}ve and expensive) exact algorithm as well as 
a (recently-developed fast) approximation algorithm.
Section~\ref{sxn:emp-datasets} starts with a brief description of 
the data sets we consider; 
Section~\ref{sxn:algorithms} describes the details of our SPSD sketching 
algorithms;
and then Section~\ref{sxn:emp-decisions} 
briefly describes the effect of various data preprocessing decisions.
In Section~\ref{sxn:emp-reconstruction}, we present our main 
results on reconstruction quality for the random sampling and random 
projection methods; and, in Section~\ref{sxn:approx-levmethods}, we 
discuss running time issues, and we present our main results for running time 
and reconstruction quality for both exact and approximate versions of 
leverage-based sampling.

We emphasize that we don't intend these results to be ``comprehensive'' but 
instead to be ``illustrative'' case-studies---that are representative of a 
much wider range of applications than have been considered previously.
In particular, we would like to illustrate the tradeoffs between these 
methods in different realistic applications in order, \emph{e.g.}, to provide 
directions for future work.
For instance, \emph{prima facie}, algorithms based on leverage-based column 
sampling might be expected to be more expensive than those based on uniform 
column sampling or random projections, but (based on previous work for general 
matrices~\cite{DMM08_CURtheory_JRNL,DMMS07_FastL2_NM10,Mah-mat-rev_BOOK})
they might also be expected to deliver lower approximation errors.
Similarly, using approximate leverage scores to construct the importance 
sampling distribution might be expected to perform worse than using exact 
leverage scores, but this might be acceptable given its computational 
advantages. 
In addition to clarifying some of these issues, our empirical evaluation 
also illustrates ways in which existing theory is insufficient to 
explain the success of sampling and projection methods.
This motivates our improvements to existing theory that we describe in 
Section~\ref{sxn:theory}.

With respect to our computational environment,
all of our computations were conducted using 64-bit MATLAB R2012a under 
Ubuntu on a 2.6--GHz quad-core Intel i7 machine with 6Gb of RAM. 
To allow for accurate timing comparisons, all computations were carried out 
in a single thread. 
When applied to an $n \times n$ SPSD matrix $\mat{A}$, 
our implementation of the SRFT requires $\mathrm{O}(n^2 \ln n)$ operations, 
as it applies MATLAB's \texttt{fft} to the entire matrix $\mat{A}$ and
\emph{then} it samples $\ell$ columns from the resulting matrix. 
We note that the SRFT computation can be made more competitive: a more rigorous 
implementation of the SRFT algorithm could reduce this running time to 
$\mathrm{O}(n^2 \ln \ell)$; but due to the complexities involved in 
optimizing pruned FFT codes, we did not pursue this~avenue.

\subsection{Data Sets}
\label{sxn:emp-datasets}

Table~\ref{table:datasets} provides summary statistics for the data sets used 
in our empirical evaluation. 
In order to illustrate the complementary strengths and weaknesses of 
different sampling versus projection methods in a wide range of realistic 
applications, we consider four classes of matrices which are commonly 
encountered in machine learning and data analysis applications: 
normalized Laplacians of very sparse graphs drawn from ``informatics graph'' 
applications;
dense matrices corresponding to Linear Kernels from machine learning 
applications;
dense matrices constructed from a Gaussian Radial Basis Function Kernel 
(RBFK); and
sparse RBFK matrices constructed using Gaussian radial basis functions, 
truncated to be nonzero only for nearest neighbors.
Although not exhaustive, this collection of data sets represents a wide range of 
data sets with very different (sparsity, spectral, leverage score, etc.) 
properties that have been of interest recently not only in machine learning 
but in data analysis more~generally.

\begin{table}[t]
\begin{tabular}{|l|c|l|l|l|}
\hline
Name & Description & n & d  & \%nnz \\
\hline
\hline
\multicolumn{5}{|c|}{Laplacian Kernels} \\
\hline
HEP      & arXiv High Energy Physics collaboration graph & 9877  & NA  & 0.06 \\
GR       & arXiv General Relativity collaboration graph  & 5242  & NA  & 0.12 \\
Enron    & subgraph of the Enron email graph             & 10000 & NA  & 0.22 \\
Gnutella & Gnutella peer to peer network on Aug. 6, 2002 & 8717  & NA  & 0.09 \\
\hline
\hline
\multicolumn{5}{|c|}{Linear Kernels} \\
\hline
Dexter  & bag of words                             & 2000 & 20000  & 83.8 \\
Protein & derived feature matrix for S. cerevisiae & 6621 & 357    & 99.7 \\
SNPs    & DNA microarray data from cancer patients & 5520 & 43     & 100 \\
Gisette & images of handwritten digits             & 6000 & 5000   & 100 \\
\hline
\hline
\multicolumn{5}{|c|}{Dense RBF Kernels} \\
\hline
AbaloneD & physical measurements of abalones & 4177 & 8   & 100 \\
WineD    & chemical measurements of wine     & 4898 & 12  & 100 \\
\hline
\hline
\multicolumn{5}{|c|}{Sparse RBF Kernels} \\
\hline
AbaloneS & physical measurements of abalones & 4177 & 8  & 82.9/48.1 \\
WineS    & chemical measurements of wine     & 4898 & 12 & 11.1/88.0 \\
\hline
\end{tabular}
\caption{The data sets used in our empirical evaluation 
(\cite{LKF07}, \cite{KY04}, \cite{GGBD05}, \cite{GSPDK06}, 
\cite{Netal02}, \cite{Corke96}, \cite{UCIMachineLearningRepository}). 
Here, $n$ is the number of data points, $d$ is the number of features in 
the input space before kernelization, and \%nnz is the percentage of nonzero
entries in the matrix.
For Laplacian ``kernels,'' $n$ is the number of nodes in the graph (and thus 
there is no $d$ since the graph is ``given'' rather than ``constructed'').
The \%nnz for the Sparse RBF Kernels depends on the $\sigma$ parameter; 
see Table~\ref{table:datasets_stats}.
}
\label{table:datasets}
\end{table}

\todo[inline]{ Talwalkar: it would be nice to remind the user why low-rank approximation is 
useful for each of the 4 categories of data that you look at.  In particular, such 
motivation is important in cases where the matrices have a slowly decaying 
spectrum (e.g., Laplacian matrices).}

To understand better the Laplacian data, recall that, given an undirected graph with 
weighted adjacency matrix $\mat{W}$, its normalized graph Laplacian is 
\[
  \mat{A} = \mat{I} - \mat{D}^{-1/2} \mat{W} \mat{D}^{-1/2},
\]
where $\mat{D}$ is the diagonal matrix of weighted degrees of the nodes of the
graph, \emph{i.e.}, $D_{ii} = \sum_{j \neq i} W_{ij}$.
This Laplacian is an SPSD matrix, but note that \emph{not} all SPSD matrices 
can be written as the Laplacian of a~graph.

The remaining datasets are positive-semidefinite kernel matrices associated with datasets 
drawn from a variety of application areas. Recall that, given given 
points $\vec{x}_1, \ldots, \vec{x}_n \in \R^d$ and a function
$\kappa : \R^d \times \R^d \rightarrow \R,$ the $n \times n$ matrix with elements
\[
 A_{ij} = \kappa(\vec{x}_i, \vec{x}_j)
\]
is called the kernel matrix of $\kappa$ with respect to $\vec{x}_1, \ldots, \vec{x}_n.$
Appropriate choices of $\kappa$ ensure that $\matA$ is positive semidefinite. When this 
is the case, the entries $\matA_{ij}$ can be interpreted as measuring, in a sense
determined by the choice of $\kappa$, the similarity of points $i$ and $j$. Specifically, if 
$\matA$ is SPSD, then $\kappa$ determines a so-called \emph{feature map}
$\Phi_\kappa: \R^d \rightarrow \R^n$ such that
\[
 A_{ij} = \langle \Phi_\kappa(\vec{x}_i), \Phi_\kappa(\vec{x}_j) \rangle
\]
measures the similarity (correlation) of $\vec{x}_i$ and $\vec{x}_j$ in feature space~\cite{SS01-book}.

When $\kappa$ is the usual Euclidean inner-product, so that 
\[
 A_{ij} = \langle \vec{x}_i, \vec{x}_k \rangle,
\]
$\matA$ is called a Linear Kernel matrix. Gaussian RBFK matrices, defined by 
\[
 A_{ij}^\sigma = \exp\bigg(\frac{-\TNormS{\vec{x}_i -
\vec{x}_j}}{\sigma^2}\bigg),
\]
correspond to the similarity measure $\kappa(\vec{x}, \vec{y}) = \exp(-\|\vec{x} - \vec{y}\|^2_2/\sigma^2).$
Here $\sigma$, a nonnegative number, defines the scale of the kernel.
Informally, $\sigma$ defines the ``size scale'' over which pairs of 
points $\vec{x}_i$ and $\vec{x}_j$ ``see'' each other.
Typically $\sigma$ is determined by a global cross-validation criterion, as 
$\mat{A}^\sigma$ is generated for some specific machine learning task; and, 
thus, one may have no \emph{a priori} knowledge of the behavior of the 
spectrum or leverage scores of $\mat{A}^\sigma$ as $\sigma$ is varied. 
Accordingly, we consider Gaussian RBFK matrices with different values of $\sigma$. 

Finally, given the same data points, $\vec{x}_1, \ldots, \vec{x}_n$, one
can construct sparse Gaussian RBFK matrices 
\[
 A_{ij}^{(\sigma,\nu,C)} = \left[ \left( 1 - \frac{\TNorm{\vec{x}_i - \vec{x}_j}}{C}\right)^\nu \right]^+ \cdot \exp\bigg(\frac{-\TNormS{\vec{x}_i -
\vec{x}_j}}{\sigma^2}\bigg),
\]
where $[x]^+ = \max\{0, x\}.$
When $\nu$ is larger than $(d+1)/2,$ this kernel matrix is positive semidefinite~\cite{Genton01}.
 Increasing $\nu$ shrinks the magnitudes of the off-diagonal
entries of the matrix toward zero.
As the cutoff point $C$ decreases the matrix becomes more sparse; in particular,
$C \rightarrow 0$ ensures that $\matA^{(\sigma,\nu,C)} \rightarrow \mat{I}.$ On the other
hand, $C \rightarrow \infty$ ensures that $\matA^{(\sigma,\nu,C)}$ approaches the
(dense) Gaussian RBFK matrix $\mat{A}^\sigma.$ For simplicity, in our empirical evaluations, we fix 
$\nu = \lceil (d+1)/2 \rceil$ and $C = 3 \sigma$, and we vary $\sigma$.
As with the effect of varying $\sigma$, the effect of varying the 
sparsity parameter $C$ is not obvious \emph{a priori}--- $C$ is typically chosen
according to a global criterion to ensure good performance at a specific machine learning 
task, without consideration for its effect on the spectrum or leverage 
scores of~$A_{ij}^{(\sigma,\nu,C)}$.

\begin{table}[t]
\begin{tabular}{|l|l|l|l|l|l|l|l|}
\hline
Name & \%nnz 
     & $\Big\lceil\tfrac{\FNormS{\mat{A}}}{\TNormS{\mat{A}}} \Big\rceil$
     & $k$ 
     & $\tfrac{\lambda_{k+1}}{\lambda_k}$ 
     & $100 \tfrac{\FNorm{\mat{A} - \mat{A}_k}}{\FNorm{\mat{A}}}$ 
     & $100 \tfrac{\tracenorm{\mat{A} - \mat{A}_k}}{\tracenorm{\mat{A}}}$ 
     & \pbox{3cm}{$k$th-largest \\ leverage score \\
     scaled by $n/k$ } \\
\hline
\hline
HEP & 0.06 & 3078 & 20 & 0.998 & 7.8 & 0.4 & 128.8 \\ 
HEP & 0.06 & 3078 & 60 & 0.998 & 13.2 & 1.1 & 41.9 \\
GR & 0.12 & 1679 & 20 & 0.999 & 10.5 & 0.74 & 71.6 \\
GR & 0.12 & 1679 & 60 & 1 & 17.9 & 2.16 & 25.3 \\
Enron & 0.22 & 2588 & 20 & 0.997 & 7.77 & 0.352 & 245.8 \\
Enron & 0.22 & 2588 & 60 & 0.999 & 12.0 & 0.94 & 49.6 \\
Gnutella & 0.09 & 2757 & 20 & 1 & 8.1 & 0.41 & 166.2 \\
Gnutella & 0.09 & 2757 & 60 & 0.999 & 13.7 & 1.20 & 49.4 \\
\hline
\hline
Dexter  & 83.8 & 176 & 8  & 0.963 & 14.5 & .934 & 16.6 \\
Protein & 99.7 & 24  & 10 & 0.987 & 42.6 & 7.66 & 5.45 \\
SNPs    & 100  & 3   & 5  & 0.928 & 85.5 & 37.6 & 2.64 \\
Gisette & 100  & 4   & 12 & 0.90  & 90.1 & 14.6 & 2.46 \\
\hline
\hline
AbaloneD (dense, $\sigma = .15$) & 100 & 41 & 20 & 0.992 & 42.1 & 3.21 & 18.11 \\
AbaloneD (dense, $\sigma = 1$)   & 100 & 4  & 20 & 0.935 & 97.8 & 59   & 2.44 \\
WineD (dense, $\sigma = 1$)      & 100 & 31 & 20 & 0.99  & 43.1 & 3.89 & 26.2 \\
WineD (dense, $\sigma = 2.1$)    & 100 & 3  & 20 & 0.936 & 94.8  & 31.2 & 2.29 \\
\hline
\hline
AbaloneS (sparse, $\sigma = .15$) & 82.9 & 400 & 20 & 0.989 & 15.4 & 1.06 & 48.4 \\
AbaloneS (sparse, $\sigma = 1$)   & 48.1 & 5   & 20 & 0.982 & 90.6 & 21.8 & 3.57 \\
WineS (sparse, $\sigma = 1$)      & 11.1 & 116 & 20 & 0.995 & 29.5 & 2.29 & 49.0 \\
WineS (sparse, $\sigma = 2.1$)    & 88.0   & 39  & 20 & 0.992 & 41.6 & 3.53 & 24.1 \\
\hline
\end{tabular}
\caption{Summary statistics for the data sets from Table~\ref{table:datasets}
that we used in our empirical evaluation.}
\label{table:datasets_stats}
\end{table}

To illustrate the diverse range of properties exhibited by these four classes
of data sets, consider Table~\ref{table:datasets_stats}.
Several observations are particularly relevant to our
discussion below.
\begin{itemize}
\item
All of the Laplacian Kernels drawn from informatics graph applications are 
extremely sparse in terms of number of nonzeros, and they all tend to have 
very slow spectral decay, as illustrated both by the quantity 
$\big\lceil\FNormS{\mat{A}}/\TNormS{\mat{A}}\big\rceil$ (this is the 
\emph{stable rank}, which  is a numerically stable (under)estimate of the 
rank of $\mat{A}$) as well as by the relatively small fraction of the 
Frobenius norm that is captured by the best rank-$k$ approximation to 
$\mat{A}$.
For the Laplacian Kernels we considered two values of the rank parameter $k$ 
that were chosen (somewhat) arbitrarily; many of the results we report 
continue to hold qualitatively if $k$ is chosen to be (say) an order of 
magnitude larger. 
\item
Both the Linear Kernels and the Dense RBF Kernels are much denser and are 
much more well-approximated by moderately to very low-rank matrices.
In addition, both the Linear Kernels and the Dense RBF Kernels have 
statistical leverage scores that are much more uniform---there are several 
ways to illustrate this, none of them perfect, and here, we illustrate this 
by considering the $k^{th}$ largest leverage score, scaled by the factor $n/k$
(if $\matA$ were exactly rank $k$, this would be the coherence of $\matA$).
For the Linear Kernels and the Dense RBF Kernels, this quantity is typically one to
two orders of magnitude smaller than for the Laplacian Kernels.
\item
For the Dense RBF Kernels, we consider two values of the $\sigma$ parameter, 
again chosen (somewhat) arbitrarily.
For both AbaloneD and WineD, we see that decreasing $\sigma$ from $1$ to 
$0.15$, \emph{i.e.}, letting data points ``see'' fewer nearby points, has two 
important effects:
first, it results in matrices that are much \emph{less} well-approximated by 
low-rank matrices; and
second, it results in matrices that have \emph{much} more heterogeneous leverage
scores.
For example, for AbaloneD, the fraction of the Frobenius norm that is captured 
decreases from $97.8$ to $42.1$ and the scaled $k^{th}$ largest leverage score 
increases from $2.44$ to~$18.11$.
\item
For the Sparse RBF Kernels, there are a range of sparsities, ranging from 
above the sparsity of the sparsest Linear Kernel, but all are denser
than the Laplacian Kernels.
Changing the $\sigma$ parameter has the same effect (although it is even 
more pronounced) for Sparse RBF Kernels as it has for Dense RBF Kernels.
In addition, ``sparsifying'' a Dense RBF Kernel also has the effect of 
making the matrix less well approximated by a low-rank matrix and of making
the leverage scores more nonuniform.
For example, for AbaloneD with $\sigma=1$ (respectively, $\sigma=0.15$), 
the fraction of the Frobenius norm that is captured decreases from $97.8$ 
(respectively, $42.1$) to $90.6$ (respectively, $15.4$), 
and the scaled $k^{th}$ largest leverage score increases from $2.44$ 
(respectively, $18.11$) to $3.57$ (respectively,~$48.4$).
\end{itemize}
As we see below, when we consider the RBF Kernels as the width parameter 
and sparsity are varied, we observe a range of intermediate cases between 
the extremes of the (``nice'') Linear Kernels and the (very ``non-nice'') Laplacian Kernels.

\subsection{SPSD Sketching Algorithms}
\label{sxn:algorithms}
The sketching matrix $\matS$ may be selected in a variety of ways. We will provide
empirical results for two sampling-based SPSD sketches and two projection-based SPSD
sketches. In the former case, the sketching matrix $\matS$ contains
exactly one nonzero in each column, corresponding to a single sample from the columns
of $\matA.$ In the latter case, $\matS$ is dense, and mixes the columns of 
 $\matA$ before sampling from the resulting matrix. 

In more detail, we consider two types of sampling-based SPSD sketches (\emph{i.e.} Nystr\"om extensions):
those constructed by sampling columns uniformly at random with replacement, and those
constructed by sampling columns from a distribution based upon the leverage scores
of the matrix filtered through the optimal rank-$k$ approximation of the matrix.
In the case of column sampling, the sketching matrix $\matS$ is simply the first $\ell$
columns of a matrix that was chosen uniformly at random from the set of all permutation matrices.

In the case of leverage-based sampling, $\matS$ has a more complicated distribution.  Recall
that the leverage scores relative to the best rank-$k$ approximation to $\matA$ are the squared
Euclidean norms of the rows of the $n \times k$ matrix $\matU_1:$ 
\[
 \ell_j = \|(\matU_1)_j\|^2.
\]
It follows from the orthonormality of $\matU_1$ that $\sum\nolimits_j (\ell_j/k) = 1,$ and the leverage scores
can thus be interpreted as a probability distribution over the columns of $\matA.$ To construct a
sketching matrix corresponding to sampling from this distribution, we first select the columns to be 
used by sampling with replacement from this distribution. Then,
$\matS$ is constructed as $\matS = \matR \matD$ where $\matR \in \R^{n \times \ell}$ 
is a column selection matrix that samples columns of $\matA$ from the given 
distribution---\emph{i.e.}, $\matR_{ij} = 1$ iff the $i$th column of $\matA$ 
is the $j$th column selected---and $\matD$ is a diagonal rescaling matrix 
satisfying $\matD_{jj} = \frac{1}{\sqrt{\ell p_i}}$ iff $\matR_{ij} = 1$. It is often expensive to
compute the leverage scores exactly; and so in Section~\ref{sxn:approx-levmethods}, we consider the empirical
performance of sketches based on several different approximation algorithms for the leverage scores. The sketching matrices
for these approximations take the same form; the only difference is the distribution used to select the
column samples.

The two projection-based sketches we consider are based upon Gaussians and the real Fourier transform.
In the former case, $\matS$ is a matrix of i.i.d. $\mathcal{N}(0,1)$ random variables. In the latter case,
$\matS$ is a \emph{subsampled randomized Fourier transform} (SRFT) matrix; that is, $\matS  = \sqrt{\frac{n}{\ell}} \matD\matF\matR$, where 
$\matD$ is a diagonal matrix of Rademacher random variables, $\matF$ is the real Fourier transform matrix,
and $\matR$ restricts to $\ell$ columns

In the figures, we refer to sketches constructed by selecting columns uniformly at random 
with the label `unif', leverage score-based sketches
with `lev', Gaussian sketches with `gaussian', and Fourier sketches with `srft'.

\subsection{Effects of Data Analysis Preprocessing Decisions}
\label{sxn:emp-decisions}

\todo[inline]{Talwalkar:
You may want to include k-means as a competitor in Section 3.3?  
It lacks a solid theoretical grounding, but in our work [38], 
k-means was clearly the best method.}

Before proceeding with our main empirical results, we pause to describe 
the effects of various machine learning and data analysis ``design 
decisions'' on the behavior of SPSD sketching algorithms in general 
as well as on the behavior of the statistical leverage scores in particular.
We should emphasize that, for ``worst case'' matrices, very little can be 
said in this regard.
Thus, these observations are based on our experiences with a diverse range of
data sets, including those from Section~\ref{sxn:emp-datasets}.
While not completely general, these observations are likely to hold in modified
form for many other realistic data, and they can potentially be useful as 
heuristic guides to practice.
For example, if preprocessing does not significantly change the leverage 
score distribution, then one could compute the leverage scores on the raw 
data and use these to sample columns from the processed data or to certify 
that the data have low coherence. 
Likewise, the behavior of the leverage scores as the rank parameter $k$ is
varied or as the $\sigma$ scale parameter of RBF kernels varies is of 
interest, as it is expensive to compute the leverage scores anew for each 
value of $k$ or $\sigma$ as part of a cross-validation computation. 
\todo[inline]{Incorporate Ilse's point that the leverage scores could potentially
be cheaply updated--- it's a reasonable, but not clear how best to do so, so leave as research direction}

One common preprocessing step is to ``whiten'' the data before applying a 
machine learning algorithm.
If the data are given in the form of $\mat{X} \in \R^{n \times d}$ where the 
$i$th row of $\mat{X}$ is an observation of $d$ covariates, then these 
covariates may have different means and characteristic size scales 
(\emph{i.e.}, variances).
In this case, it is often appropriate to transform the covariates so they 
all have zero mean and are on the same size scale. 
The whitening transform generates a new matrix $\hat{\mat{X}},$ 
corresponding to these transformed covariates, by removing the mean of each
column and rescaling the columns so they all have unit norm. 
In our experience, whitening modifies the statistical leverage scores, often 
by making them somewhat more homogeneous, but for a fixed rank parameter $k$ 
it does not change them too substantially, \emph{e.g.}, to within no more 
than a multiplicative factor of $2$.
Given the sensitivity of matrix reconstruction algorithms to various 
structural properties of the input data that we describe below, however, 
the more important observation is that whitening tends to decrease the 
effective rank of the input data set, and at the same time it often tends to 
shrink the spectral gaps.
As shown below, this has observable consequences on the reconstruction 
errors of all the sketching methods considered, but in particular those 
involving approximate leverage score computations. 

Another preprocessing decision has to do with the choice of rank $k$ with 
which to describe the data.
This is typically determined according to an exogeneously-specified ``model 
selection'' criterion that does not explicitly take into account the 
spectrum or leverage score structure of the input matrix.
It enters our discussion since we consider sampling columns with 
probabilities proportional to their statistical leverage scores 
\emph{relative to a rank-$k$ space}, and thus the leverage scores 
depend on $k$.
In our experience, increasing $k$ tends to uniformize or homogeneize the 
leverage scores, often gradually, but sometimes quite substantially.
(We should note, however, that there are exceptions to this, where one 
observes very strong localization on low-order eigenvectors of data 
matrices~\cite{CM11_TR}.)


Yet another preprocessing decision has to do with the choice of the 
$\sigma$ scale parameter in Gaussian RBFK matrices.
As with the rank parameter, the scale parameter $\sigma$ in practice is
determined according to an exogeneously-specified model selection criterion 
that does not explicitly take into account the spectrum or leverage score 
structure of the input matrix.
In our experience, as $\sigma$ increases, the leverage scores become more 
and more uniform; and they become more heterogeneous as $\sigma$ decreases. 
Informally, as a data point ``sees'' more data points, any outlying effect
is mitigated.
Varying $\sigma$ also has an effect on the spectrum.
As a general rule, letting $\sigma \rightarrow 0$ tends to make the spectrum 
of $\mat{A}^\sigma$ flatter, \emph{i.e.}, decay more slowly, and letting 
$\sigma \rightarrow \infty$ makes $\mat{A}^\sigma$ lower-rank.
Recall that the diagonal entries of $\mat{A}^\sigma$ are identically one, 
and as $\sigma \rightarrow \infty,$ $\mat{A}^\sigma$ tends to the matrix of 
all ones. 
That is, increasing $\sigma$ corresponds to considering all the observations 
$\vec{x}_i$ as being equally dissimilar, 
so all columns are equally 
noninformative. 
On the other hand, as $\sigma \rightarrow 0,$ $\mat{A}^\sigma$ approaches 
the identity, 
and very dissimilar observations (in the sense that 
$\TNorm{\vec{x}_i - \vec{x}_j}$ is large) are penalized more heavily than 
similar observations, and thus there is some nonuniformity in the columns 
of $\mat{A}^\sigma.$ 
In some cases, we observed that, as the scale $\sigma$ decreases, the 
leverage scores stabilize, identifying the same columns as being important 
or influential over a range of scales.

\subsection{Reconstruction Accuracy of Sampling and Projection Algorithms}
\label{sxn:emp-reconstruction}

Here, we describe the performances of the SPSD
sketches described in Section~\ref{sxn:algorithms}---column sampling uniformly at random without 
replacement, column sampling according to the nonuniform leverage score 
probabilities, and sampling using Gaussian and SRFT mixtures of the 
columns---in terms of reconstruction accuracy for the data sets described in 
Section~\ref{sxn:emp-datasets}.
We describe general observations we have made about each class of 
matrices in turn, and then we summarize our observations.
We consider only the use of exact leverage scores here, and we
postpone until Section~\ref{sxn:approx-levmethods} a discussion of running 
time issues and similar reconstruction results when approximate leverage 
scores are used for the importance sampling distribution.
In each case, we present results for both the ``non-rank-restricted'' case 
as well as the ``rank-restricted'' case. 
Recall that by non-rank-restricted, we mean that the error
\begin{equation}
\xiNorm{\mat{A} - \mat{C} \mat{W}^\dagger \mat{C}^\transp}/\xiNorm{\mat{A} - \mat{A}_k}
\label{eqn:relerr1}
\end{equation}
is plotted; while by rank-restricted, we mean that the error
\begin{equation}
\xiNorm{\mat{A} - \mat{C} \mat{W}_k^\dagger \mat{C}^\transp}/\xiNorm{\mat{A} - \mat{A}_k} 
\label{eqn:relerr2}
\end{equation}
is plotted (\emph{viz.}, the matrix $\mat{W}$ in Eqn.~(\ref{eqn:relerr1}) has been
replaced with the low-rank approximation $\mat{W}_k$).
Note that previous work has shown that relative-error guarantees can be 
obtained, \emph{e.g.}, with CUR matrix decompositions, not only when one 
projects onto the span of judiciously-chosen columns, analogously to
Eqn.~(\ref{eqn:relerr1}) and as our worst-case guarantees in this paper are 
formulated, but also when one restricts the rank of the low-rank 
approximation to be no greater than $k$ by projecting onto the best rank-$k$ 
approximation to the original matrix~\cite{DMM08_CURtheory_JRNL}. 
We evaluate the ``rank-restricted'' case of the form of 
Eqn.~(\ref{eqn:relerr2}), that depends on projecting onto the best rank-$k$ 
approximation of the subsample (and not the original matrix) since it is more 
algorithmically tractable;
but we note that similar but ``smoother'' results (\emph{e.g.}, the error is 
much more monotonic as a function of the number of samples, when compared 
with the ``rank-restricted'' results we present below) are obtained 
empirically with this more expensive rank-restriction procedure.
The data points plotted in each figure of this section represent the average errors observed
over 30 trials.

Finally, we note that previous work has shown that the statistical leverage 
scores reflect an important nonuniformity structure in the columns of 
general data matrices~\cite{CUR_PNAS,Mah-mat-rev_BOOK}; that randomly 
sampling columns according to this distribution results in lower worst-case 
error (for problems such as least-squares approximation and low-rank 
approximation of general matrices) than sampling columns uniformly at 
random~\cite{DMM08_CURtheory_JRNL,DMMS07_FastL2_NM10,Mah-mat-rev_BOOK}; and
that leverage scores have proven useful in a wide range of practical 
applications~\cite{Paschou07b,CUR_PNAS,Mah-mat-rev_BOOK,chingwa-astro-draft}.
In spite of this, ours is the first work to implement and evaluate leverage 
score sampling for low-rank approximation of SPSD matrices.

\subsubsection{Graph Laplacians}

\begin{figure}[!p]
 \centering
 \subfigure[GR, $k = 20$]{\includegraphics[width=1.6in, keepaspectratio=true]{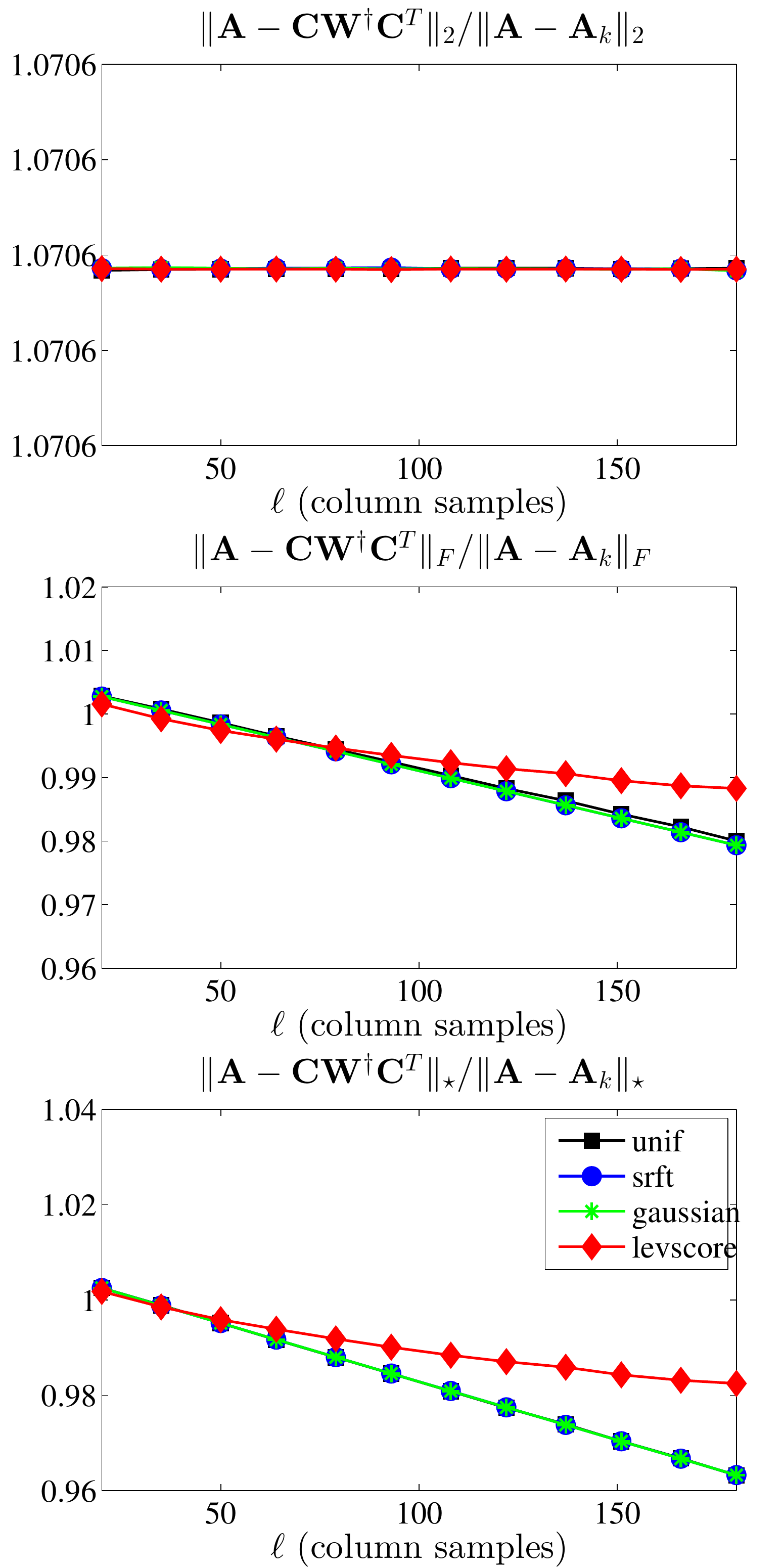}}%
 \subfigure[GR, $k = 60$]{\includegraphics[width=1.6in, keepaspectratio=true]{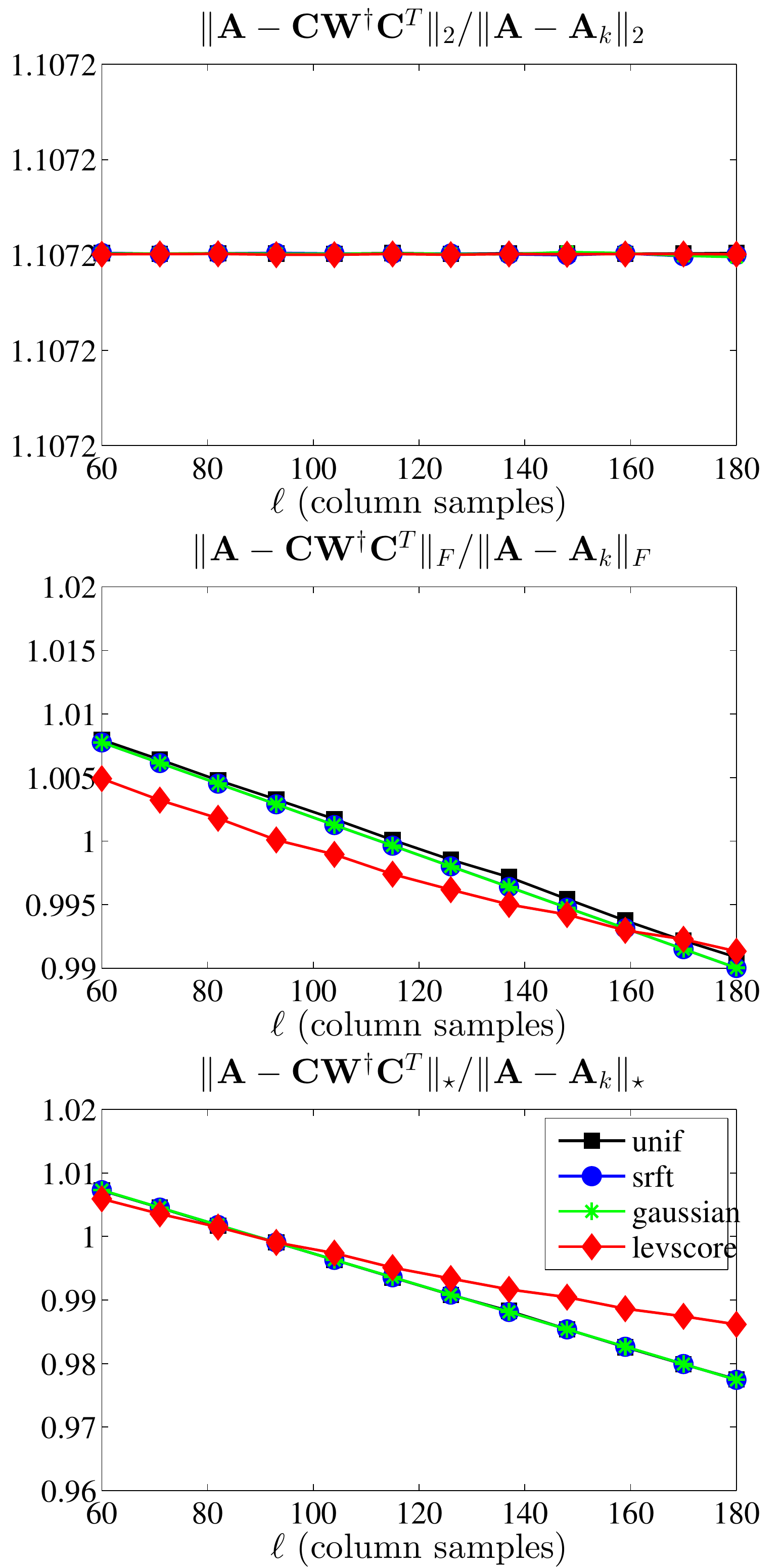}}%
 \subfigure[HEP, $k = 20$]{\includegraphics[width=1.6in, keepaspectratio=true]{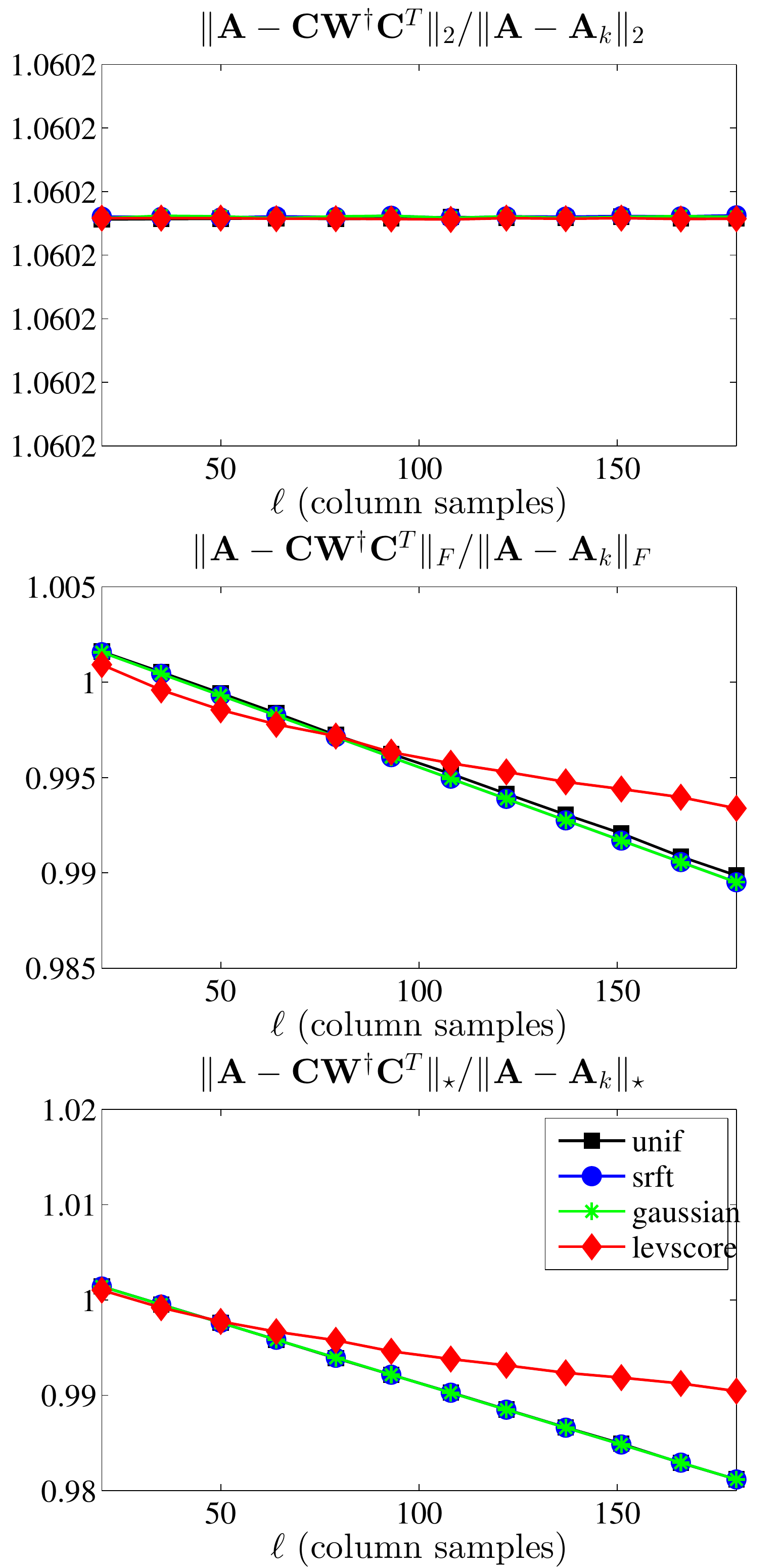}}%
 \subfigure[HEP, $k = 60$]{\includegraphics[width=1.6in, keepaspectratio=true]{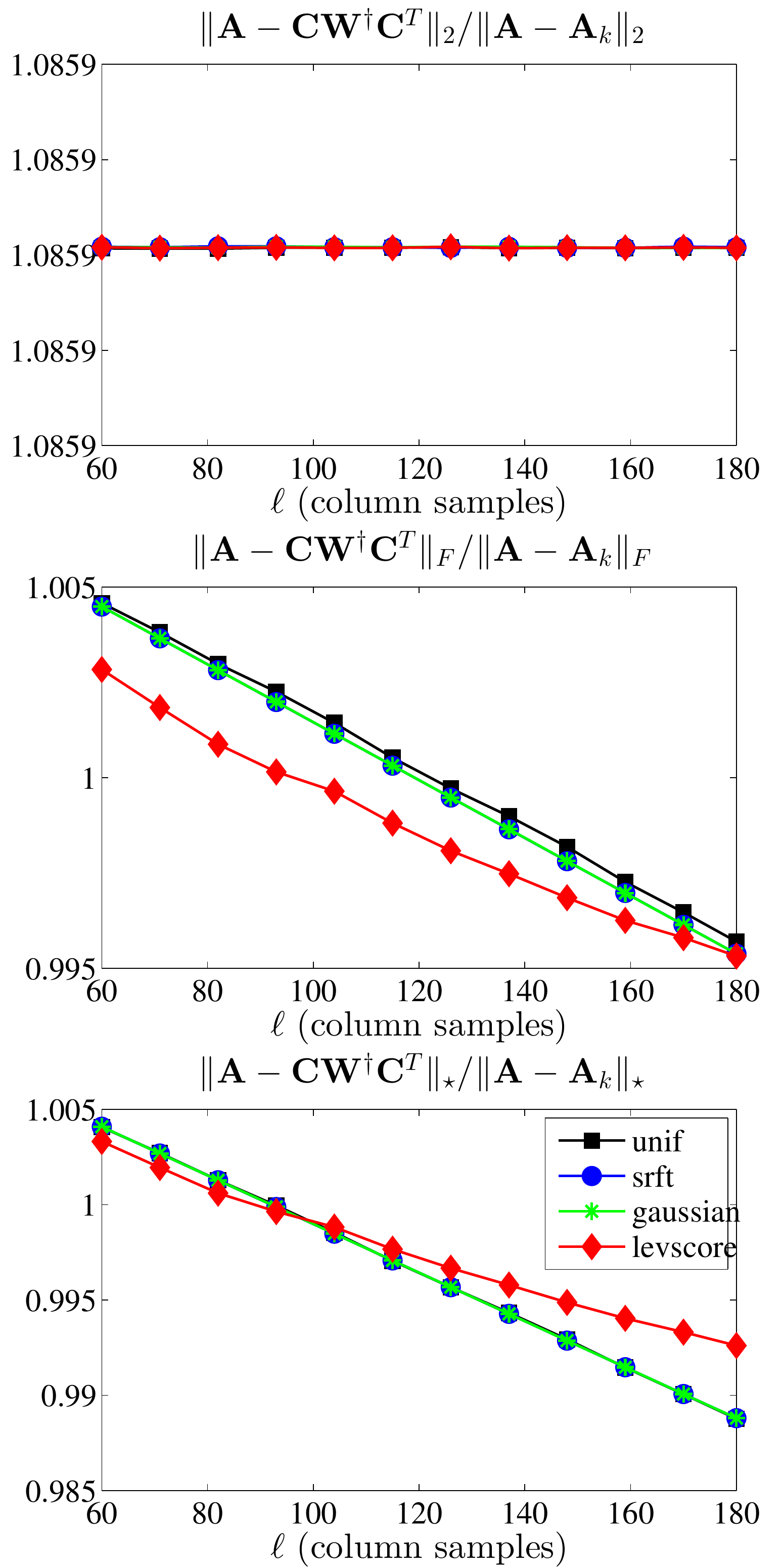}}%
 \\%
 \subfigure[GR, $k = 20$]{\includegraphics[width=1.6in, keepaspectratio=true]{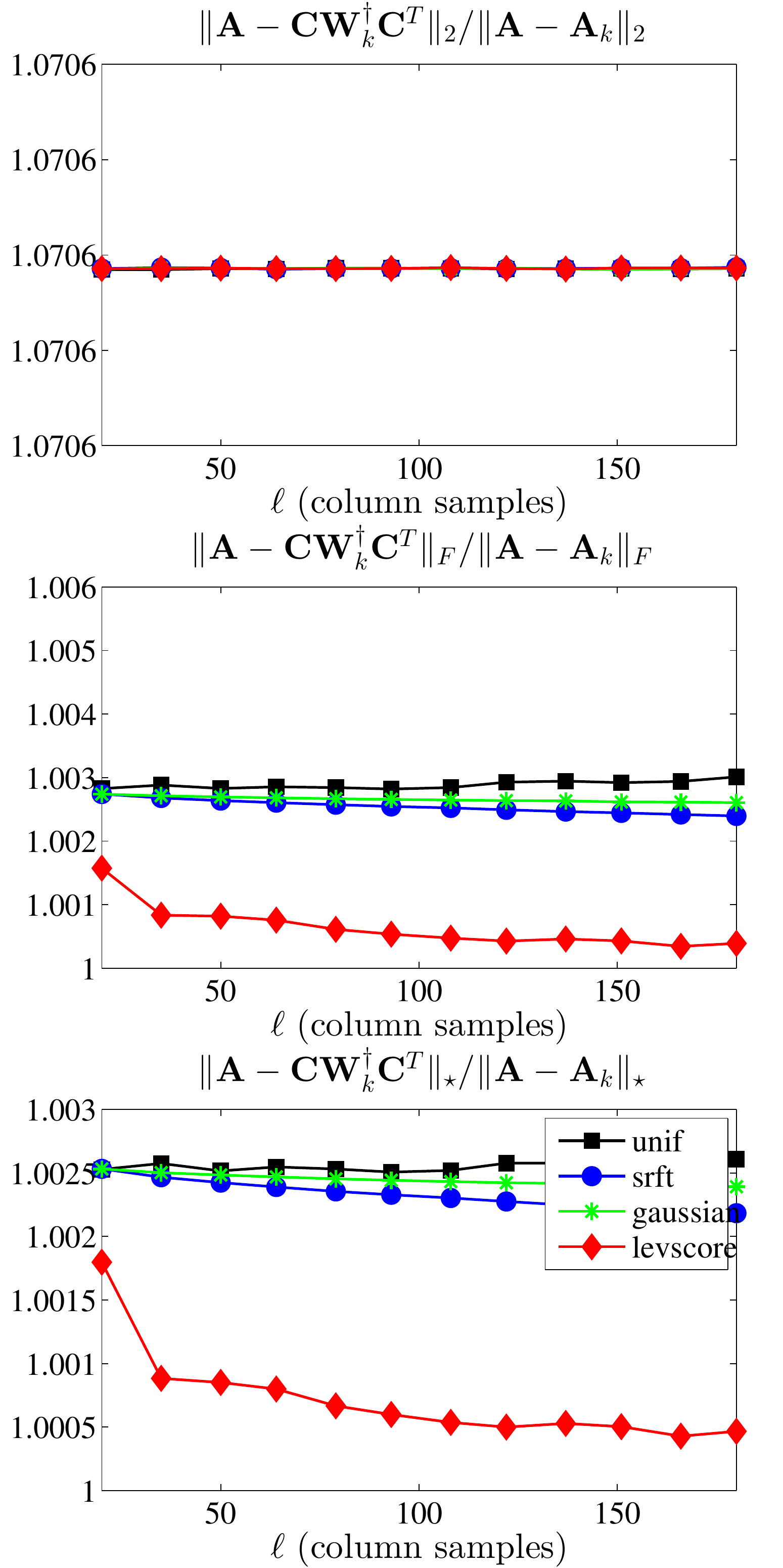}}%
 \subfigure[GR, $k = 60$]{\includegraphics[width=1.6in, keepaspectratio=true]{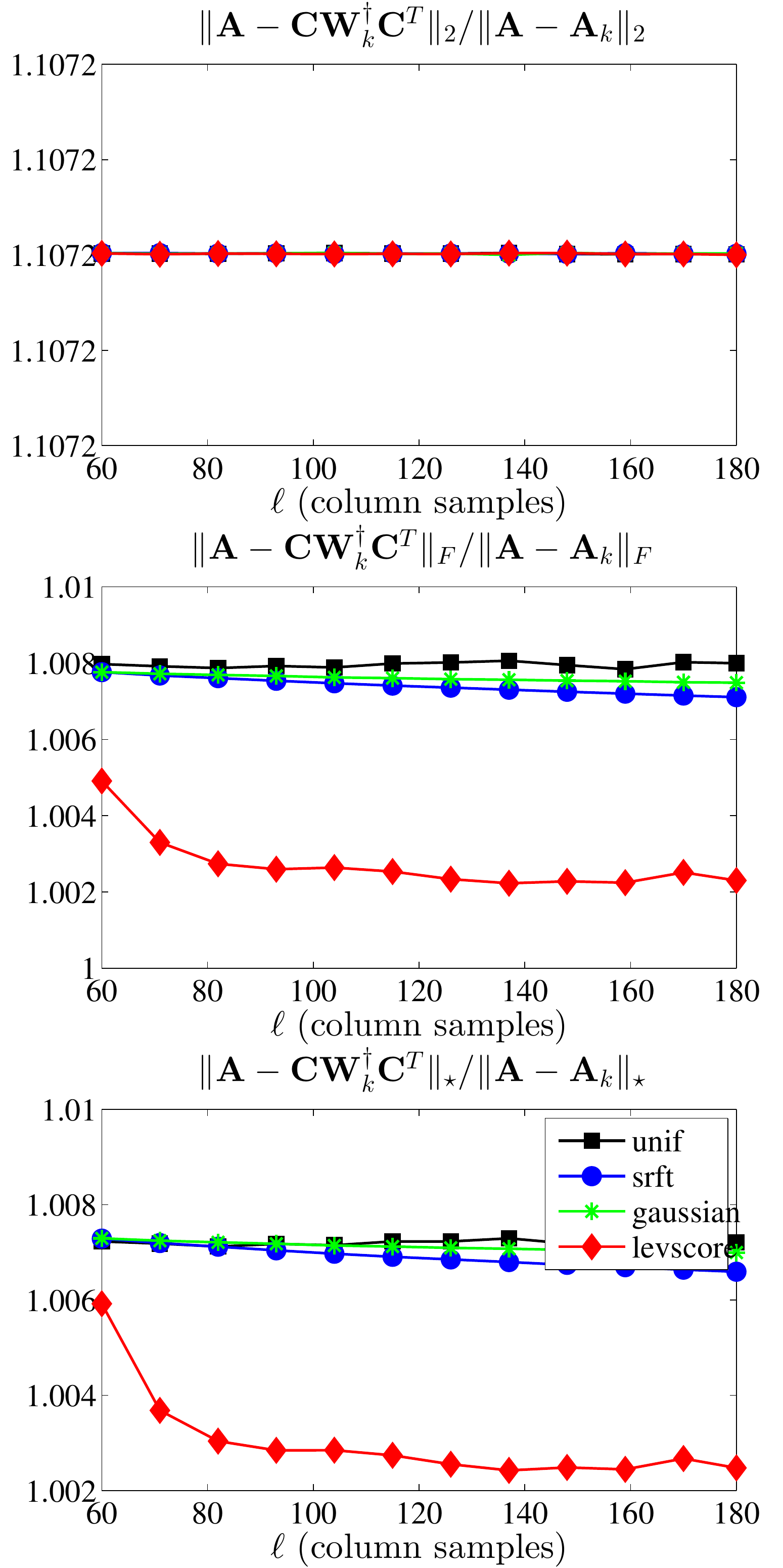}}%
 \subfigure[HEP, $k = 20$]{\includegraphics[width=1.6in, keepaspectratio=true]{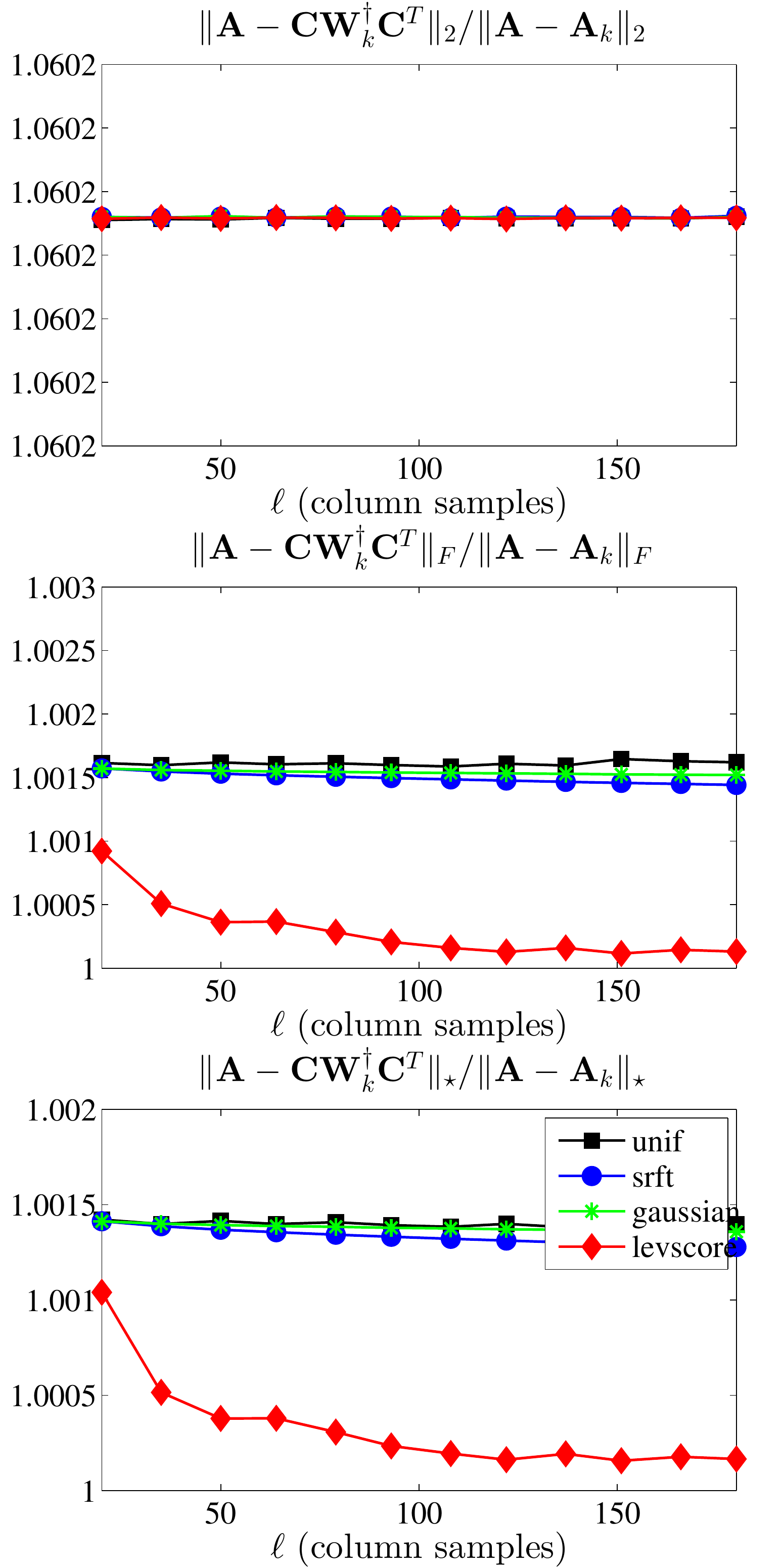}}%
 \subfigure[HEP, $k = 60$]{\includegraphics[width=1.6in, keepaspectratio=true]{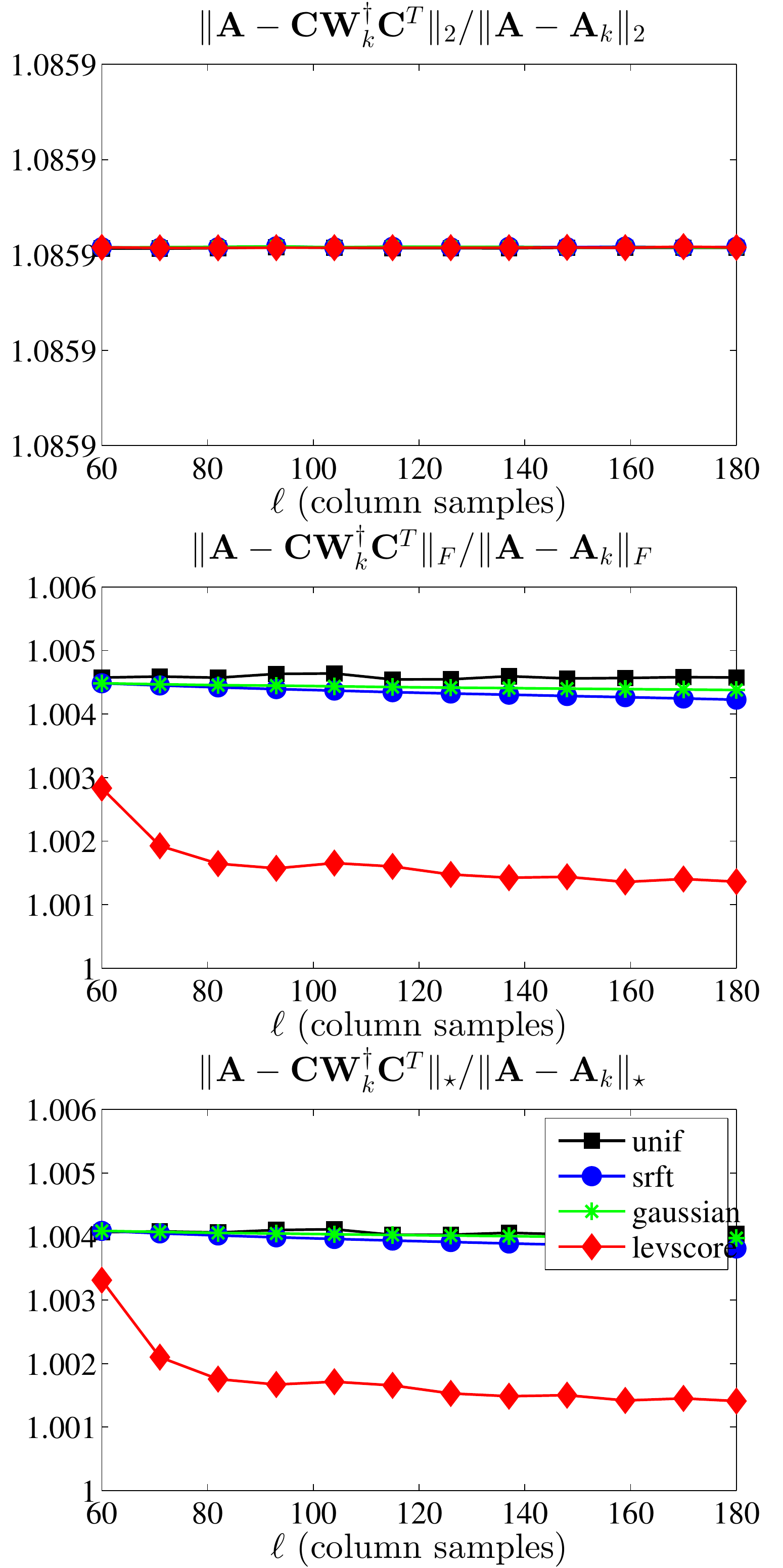}}
 \caption{The spectral, Frobenius, and trace norm errors 
 (top to bottom, respectively, in each subfigure) of several
 (non-rank-restricted in top panels and rank-restricted in bottom panels)
 SPSD sketches, as a function of the number of columns samples $\ell$, 
 for the GR and HEP Laplacian data sets, with two choices of the rank parameter $k$.}%
 \label{fig:laplacian-exact-errors-1}
\end{figure}

\begin{figure}[p]
 \centering
 \subfigure[Enron, $k = 20$]{\includegraphics[width=1.6in, keepaspectratio=true]{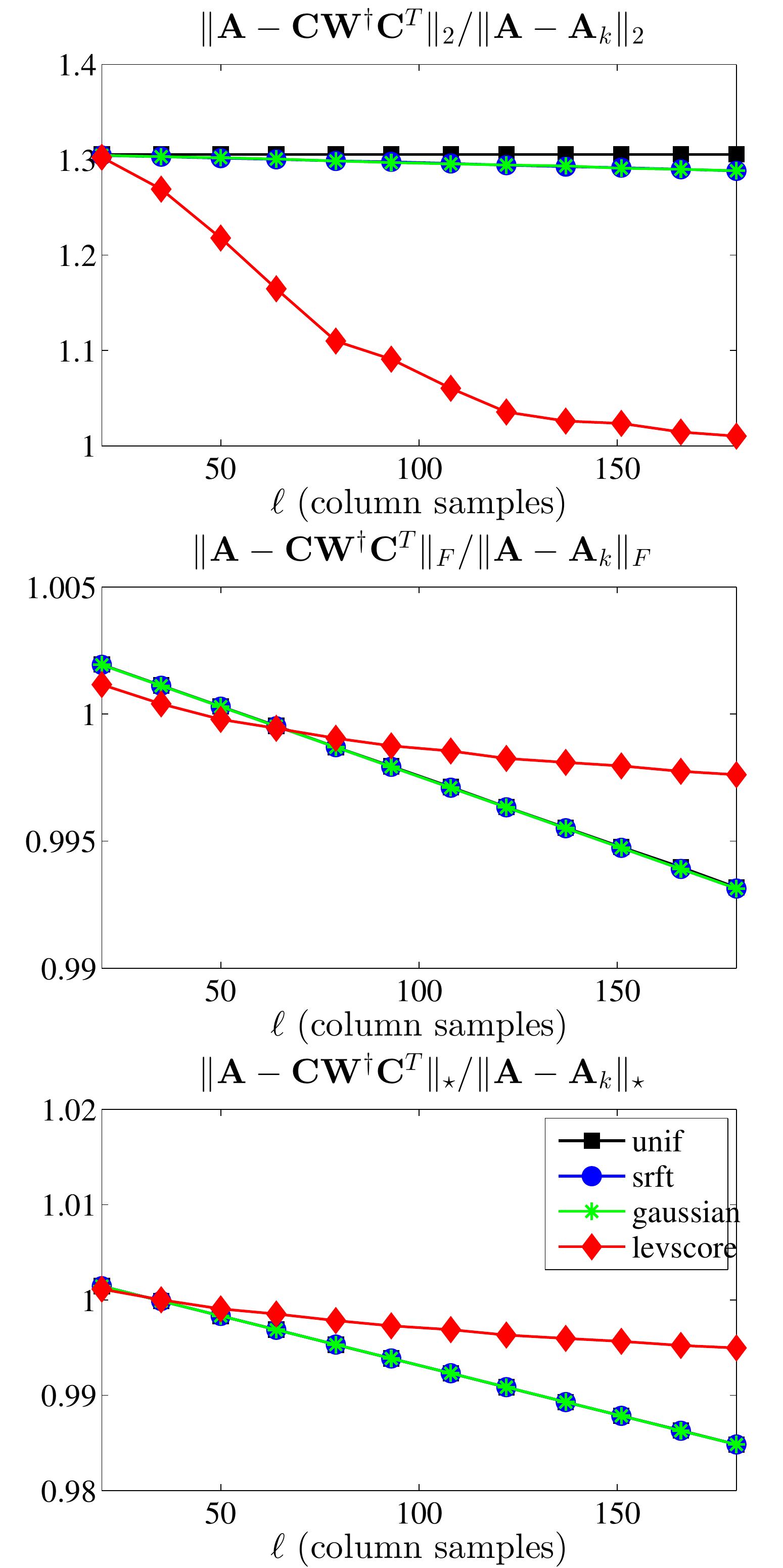}}%
 \subfigure[Enron, $k = 60$]{\includegraphics[width=1.6in, keepaspectratio=true]{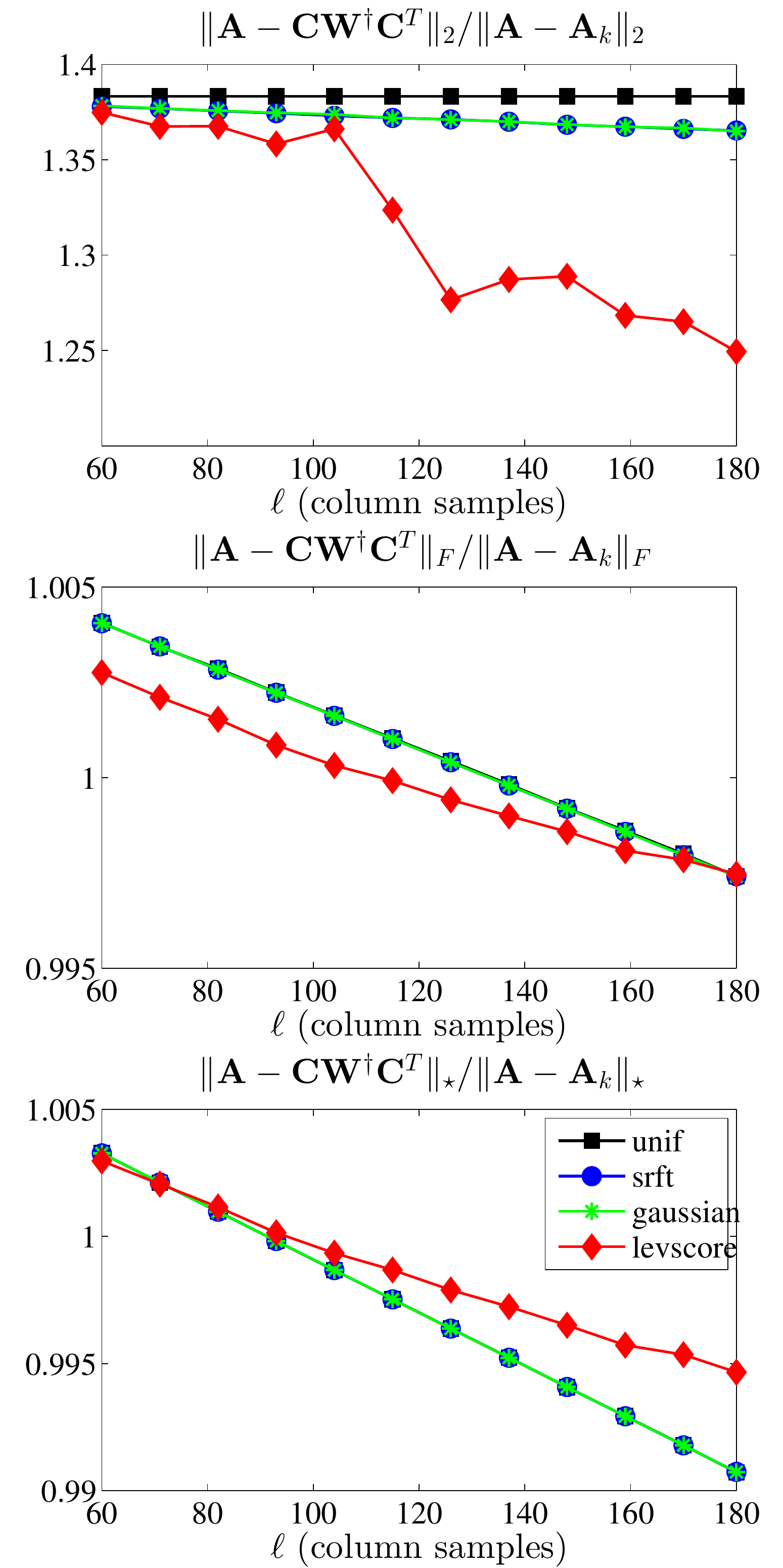}}%
 \subfigure[Gnutella, $k = 20$]{\includegraphics[width=1.6in, keepaspectratio=true]{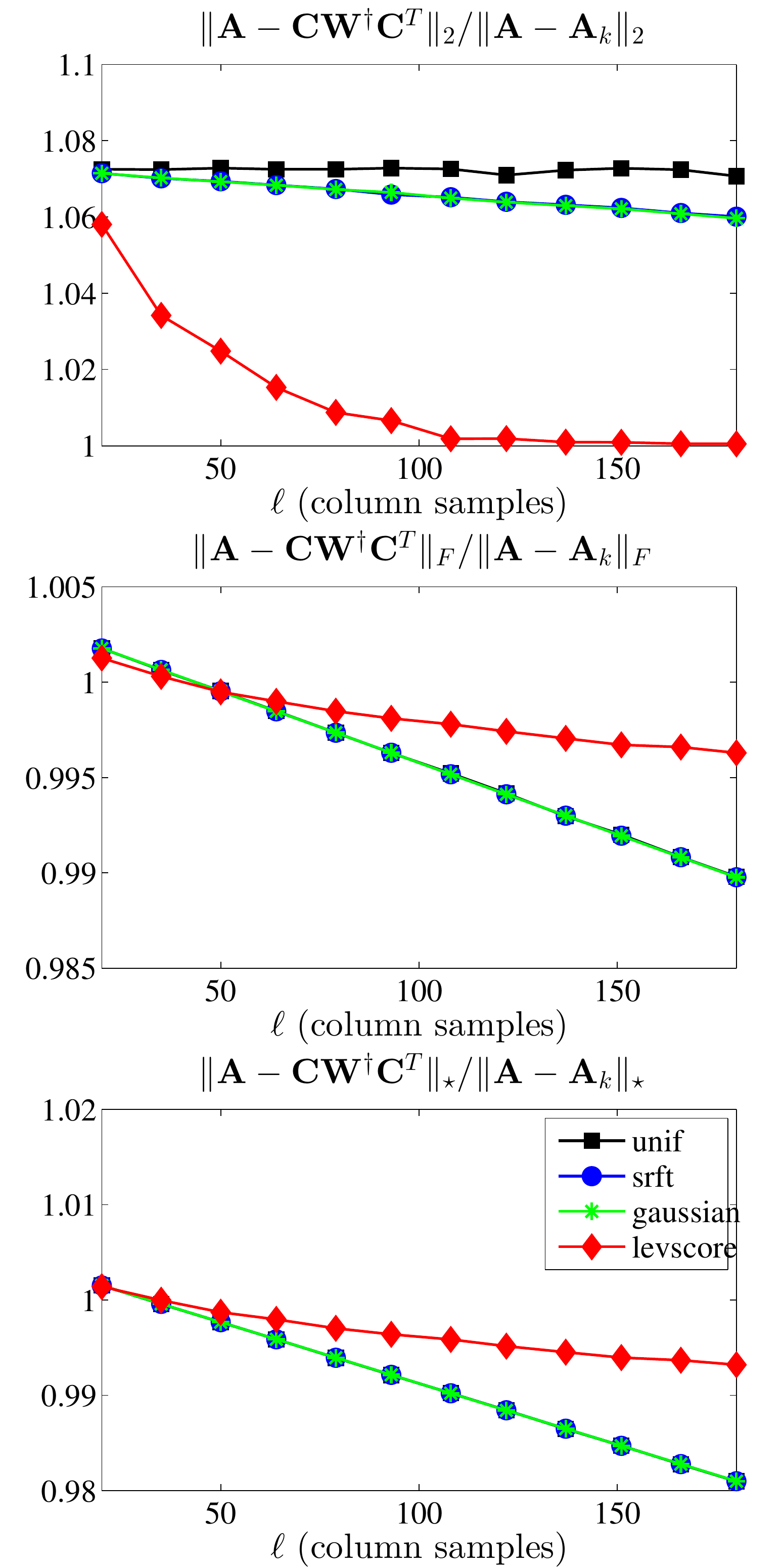}}%
 \subfigure[Gnutella, $k = 60$]{\includegraphics[width=1.6in, keepaspectratio=true]{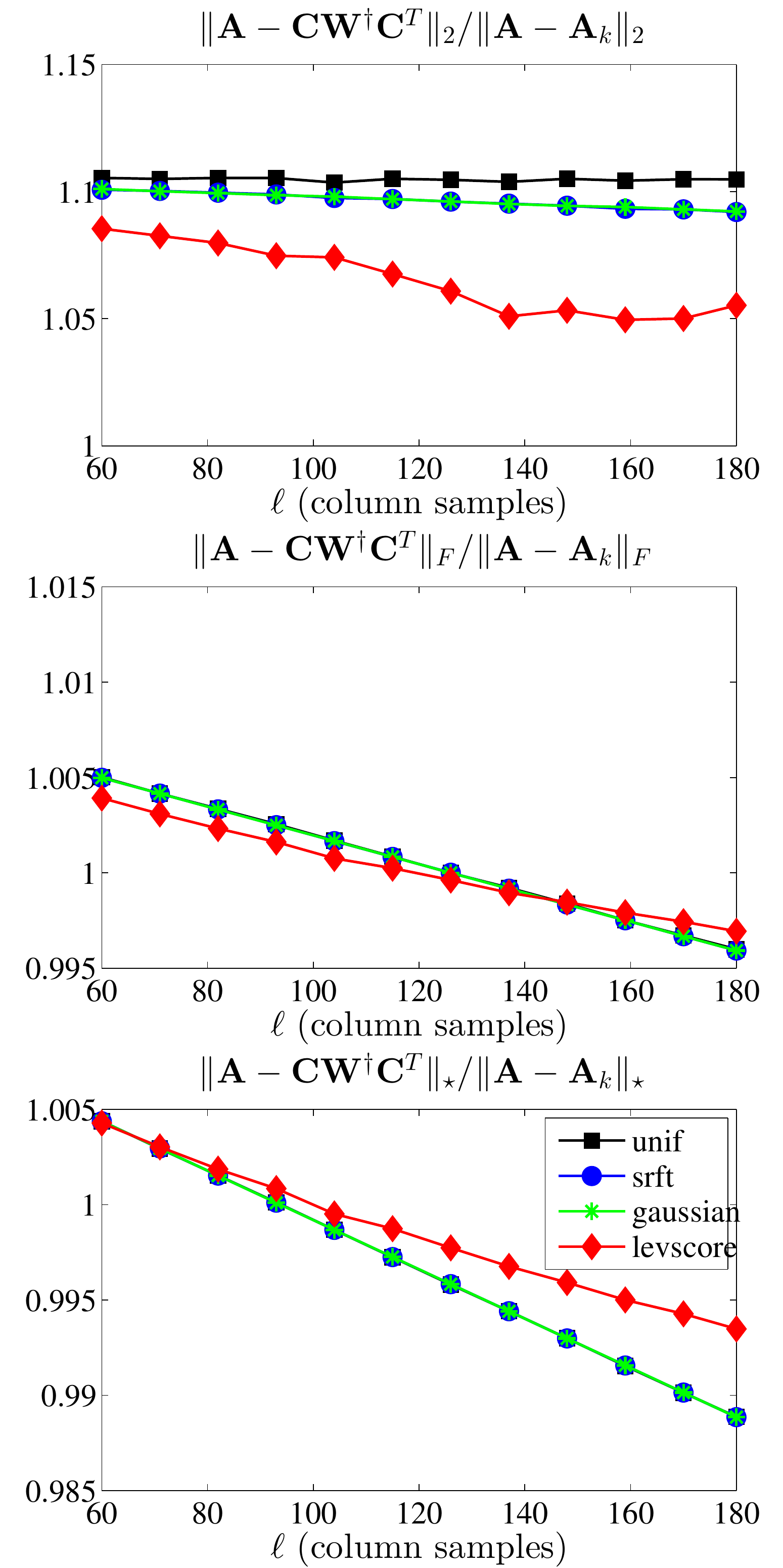}}%
 \\%
 \subfigure[Enron, $k = 20$]{\includegraphics[width=1.6in, keepaspectratio=true]{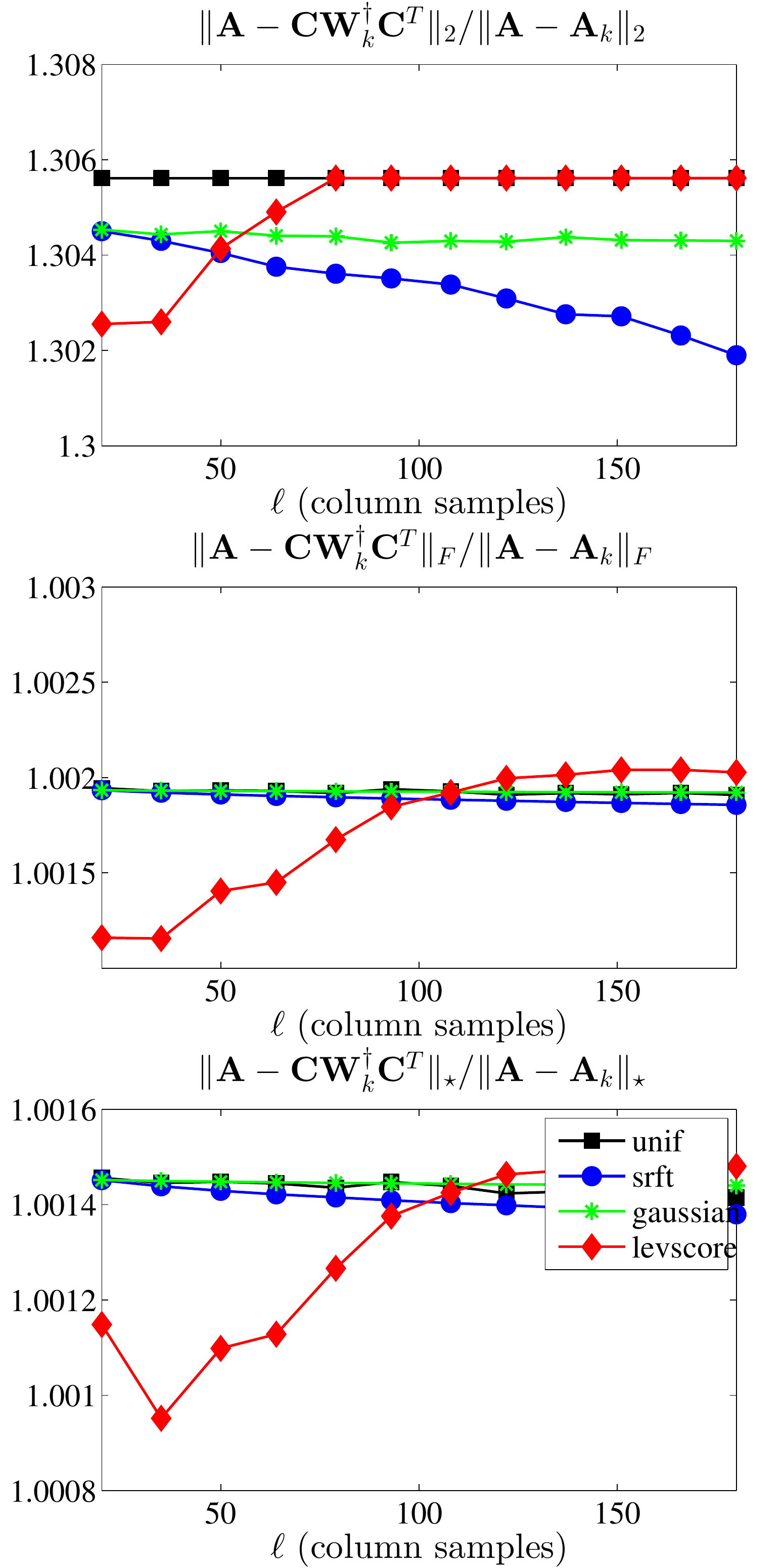}}%
 \subfigure[Enron, $k = 60$]{\includegraphics[width=1.6in, keepaspectratio=true]{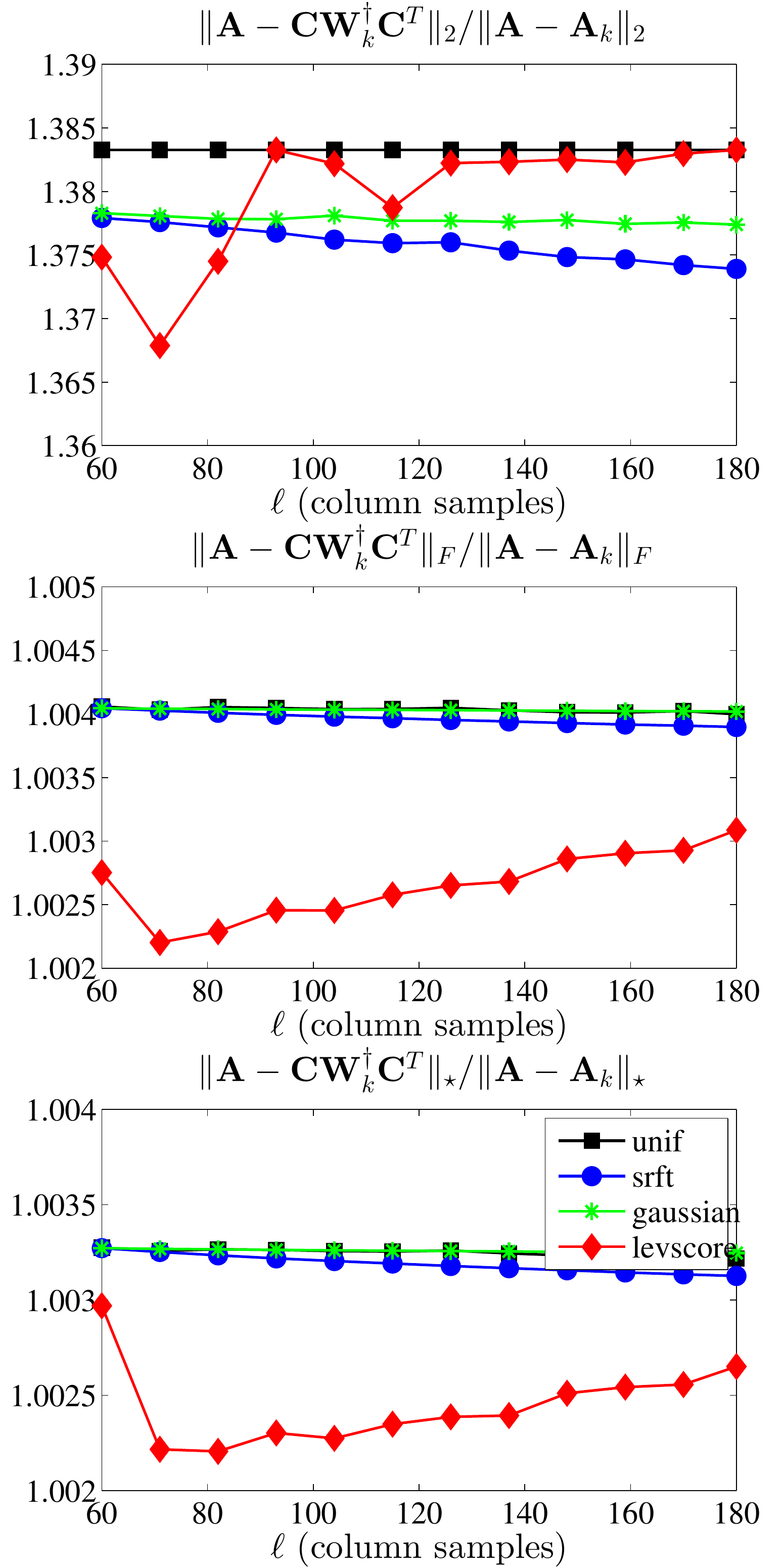}}%
 \subfigure[Gnutella, $k = 20$]{\includegraphics[width=1.6in, keepaspectratio=true]{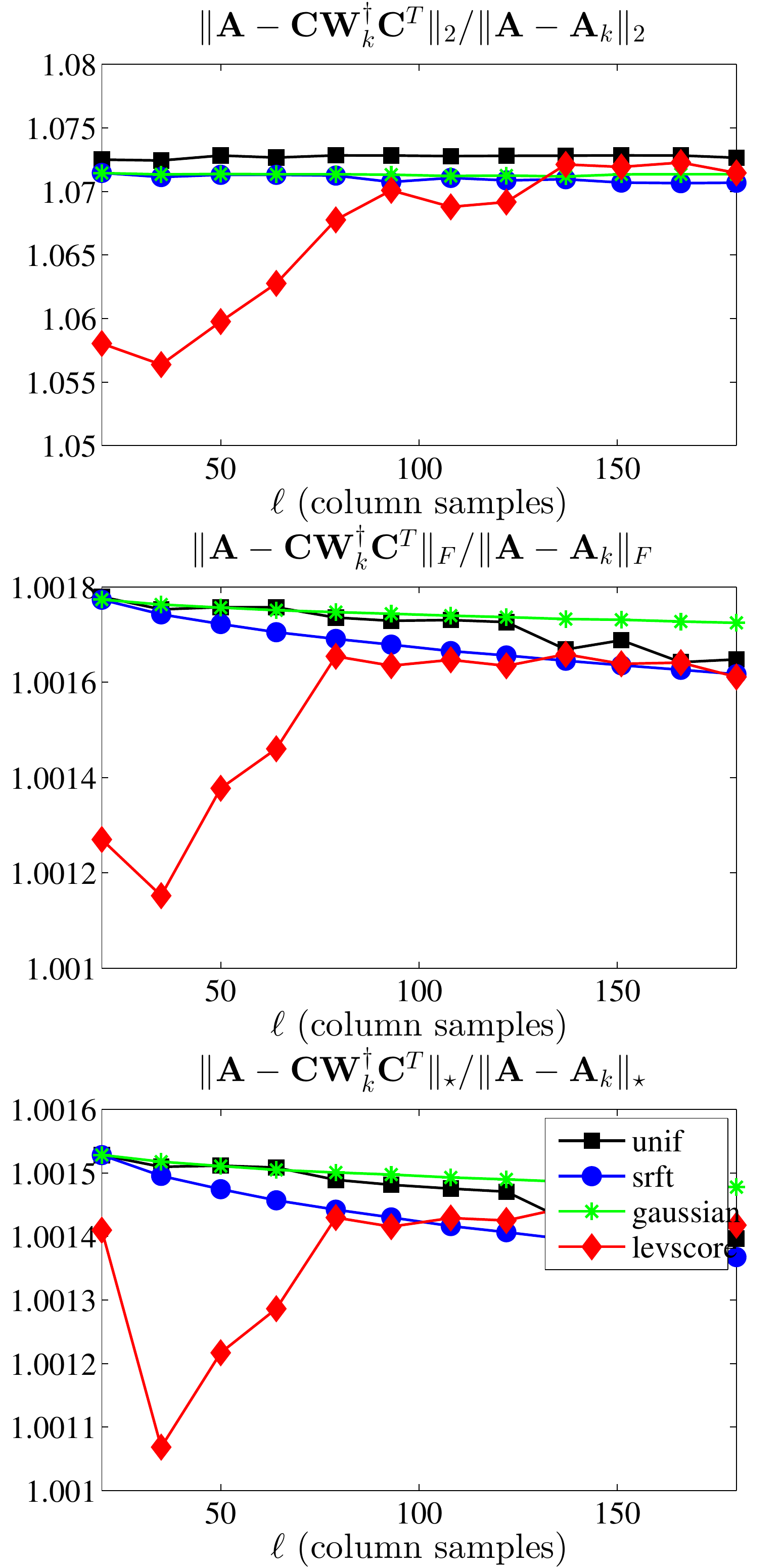}}%
 \subfigure[Gnutella, $k = 60$]{\includegraphics[width=1.6in, keepaspectratio=true]{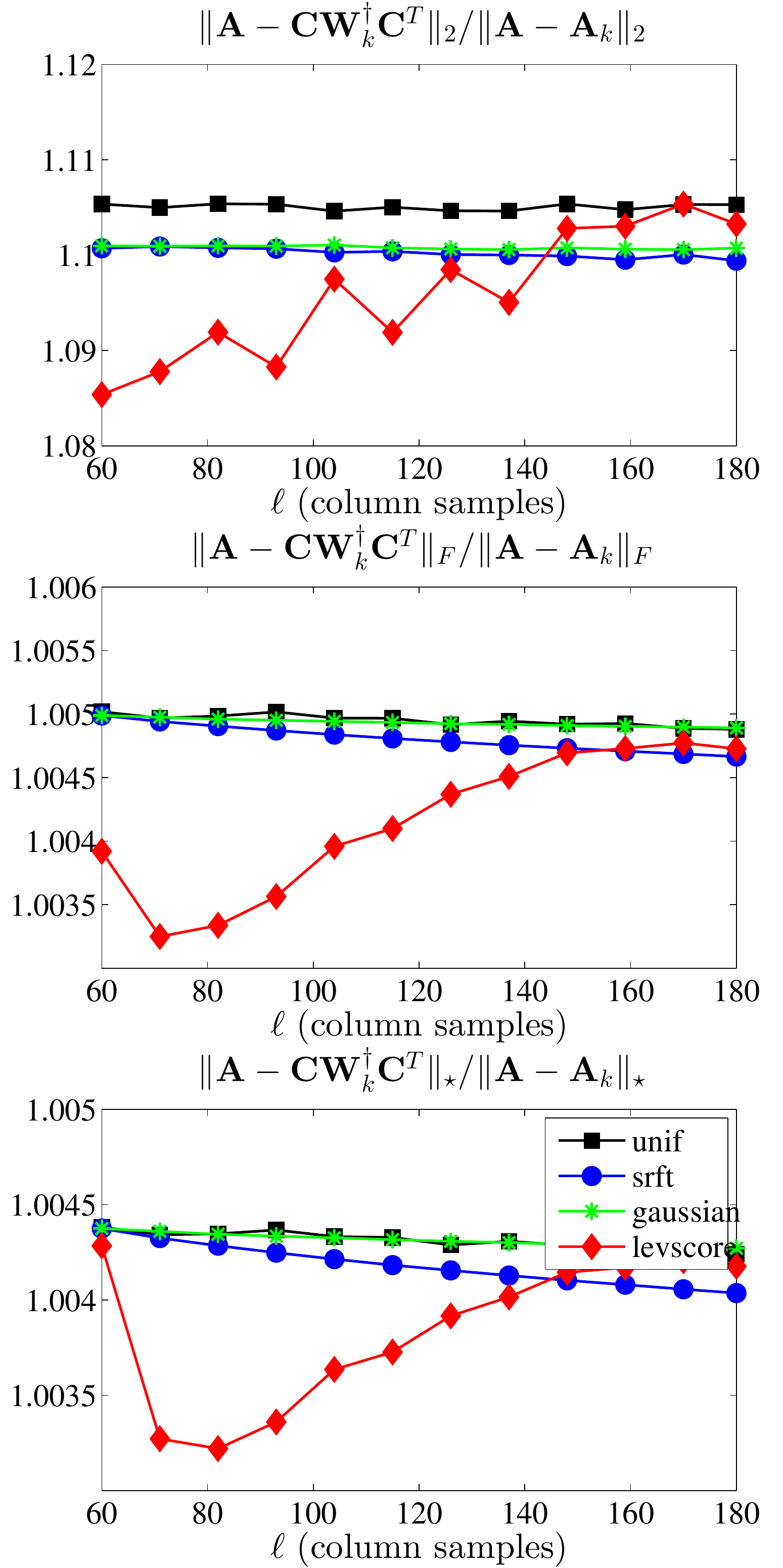}}%
 \caption{The spectral, Frobenius, and trace norm errors 
 (top to bottom, respectively, in each subfigure) of several
 (non-rank-restricted in top panels and rank-restricted in bottom panels)
 SPSD sketches, as a function of the number of columns samples $\ell$, 
 for the Enron and Gnutella Laplacian data sets, with two choices of the rank parameter $k$.}%
 \label{fig:laplacian-exact-errors-2}
\end{figure}

Figure~\ref{fig:laplacian-exact-errors-1} and
Figure~\ref{fig:laplacian-exact-errors-2} 
show the reconstruction error results for sampling and projection methods 
applied to several normalized graph Laplacians.
The former shows GR and HEP, each for two values of the rank parameter, 
and the latter shows Enron and Gnutella, again each for two values of the 
rank parameter.
Both figures show the spectral, Frobenius, and trace norm approximation 
errors, as a function of the number of column samples~$\ell$, relative to 
the error of the optimal rank-$k$ approximation of $\mat{A}$. 
In both figures, the first four (\emph{i.e.}, top) subfigures show the 
results for the non-rank-restricted case, and the last four (\emph{i.e.}, 
bottom) subfigures show the results for the rank-restricted case. 
In particular, in the rank-restricted case, the low-rank approximation is 
``filtered'' through a rank-$k$ space, and thus the approximation ratio is 
always greater than unity.

These and subsequent figures contain a lot of information, some of which is
peculiar to the given data sets and some of which is more general.
In light of subsequent discussion, several observations are worth making 
about the results presented in these two figures.
\begin{itemize}
\item
All of the SPSD sketches provide quite accurate 
approximations---relative to the best possible approximation factor for that 
norm, and relative to bounds provided by existing theory, as reviewed in 
Section~\ref{sxn:prelim-prior}---even with only $k$ column samples (or 
in the case of the Gaussian and SRFT mixtures, with only $k$ linear 
combinations of vectors). 
Upon examination, this is partly due to the extreme sparsity and extremely 
slow spectral decay of these data sets which means, as shown in 
Table~\ref{table:datasets}, that only a small fraction of the (spectral or 
Frobenius or trace) mass is captured by the optimal rank $20$ or $60$ 
approximation. 
Thus, although an SPSD sketch constructed from $20$ or $60$ vectors also only 
captures a small portion of the mass of the matrix, the relative error is 
small, since the scale of the residual error is large.  
\item
The scale of the Y axes is different between different figures and 
subfigures.
This is to highlight properties within a given plot, but it can hide 
several things.
In particular, note that the scale for the spectral norm is generally
larger than for the Frobenius norm, which is generally larger than for the
trace norm, consistent with the size of those norms; and that the scale is 
larger for higher-rank approximations, \emph{e.g.} compare GR $k=20$ with 
GR $k=60$, also consistent with the larger amount of mass captured by 
higher-rank approximations.
\item
Both the non-rank-restricted and rank-restricted results are the same for 
$\ell=k$.
For $\ell > k$, the non-rank-restricted errors tend to decrease (or 
at least not increase, as for GR and HEP the spectral norm error is flat as 
a function of $\ell$), which is intuitive.
While the rank-restricted errors also tend to decrease for $\ell > k$, 
the decrease is much less (since the rank-restricted plots are bounded below 
by unity) and the behavior is much more complicated as a function of 
increasing $\ell$.
\item
The X axes ranges from $k$ to $9k$ for the $k=20$ plots and from $k$ to $3k$ for the 
$k=60$ plots.
As a practical matter, choosing $\ell$ between $k$ and (say) $2k$ or $3k$ is 
probably of greatest interest.
In this regime, there is an interesting tradeoff for the non-rank-restricted
plots: for moderately large values of $\ell$ in this regime, the error for 
leverage-based sampling is moderately better than for uniform sampling or 
random projections, while if one chooses $\ell$ to be much larger then the 
improvements from leverage-based sampling saturate and the uniform sampling 
and random projection methods are better.
This is most obvious in the Frobenius norm plots, although it is also seen in
the trace norm plots, and it suggests that some combination of leverage-based
sampling and uniform sampling might be best.
\item
For the rank-restricted plots, in some cases, \emph{e.g.}, with GR and HEP, 
the errors for leverage-based sampling are much better than for the other 
methods and quickly improve with increasing $\ell$ until they saturate; 
while in other cases, \emph{e.g.}, with Enron and Gnutella, the errors for
leverage-based sampling improve quickly and then degrade with increasing 
$\ell$.
Upon examination, the former phenomenon is similar to what was observed in 
the non-rank-restricted case and is due to the strong ``bias'' provided by 
the leverage score importance sampling distribution to the top part of the 
spectrum, allowing the sampling process to focus very quickly on the 
low-rank part of the input matrix.
(In some cases, this is due to the fact that the heterogeneity of 
the leverage score importance sampling distribution means that one is likely 
to choose the same high leverage columns multiple times, rather than 
increasing the accuracy of the sketch by adding new columns whose 
leverage scores are lower.) 
The latter phenomenon of degrading error quality as $\ell$ is increased is 
more complex and seems to be due to some sort of ``overfitting'' caused by 
this strong bias and by choosing many more than $k$ columns.  
\item
The behavior of the approximations with respect to the spectral norm is 
quite different from the behavior in the Frobenius and trace norms. 
In the latter, as the number of samples $\ell$ increases, the errors tend to 
decrease, although in an erratic manner for some of the rank-restricted 
plots; while for the former, the errors tend to be much flatter as a function 
of increasing $\ell$ for at least the Gaussian, SRFT, and uniformly sampled 
sketches.
\end{itemize}
All in all, there seems to be quite complicated behavior for low-rank 
sketches for these Laplacian data sets.
Several of these observations can also be made for subsequent figures; but
in some other cases the (very sparse and not very low rank) structural 
properties of the data are primarily responsible.

\subsubsection{Linear Kernels}

\begin{figure}[p]
 \centering
 \subfigure[Dexter, $k = 8$]{\includegraphics[width=1.6in, keepaspectratio=true]{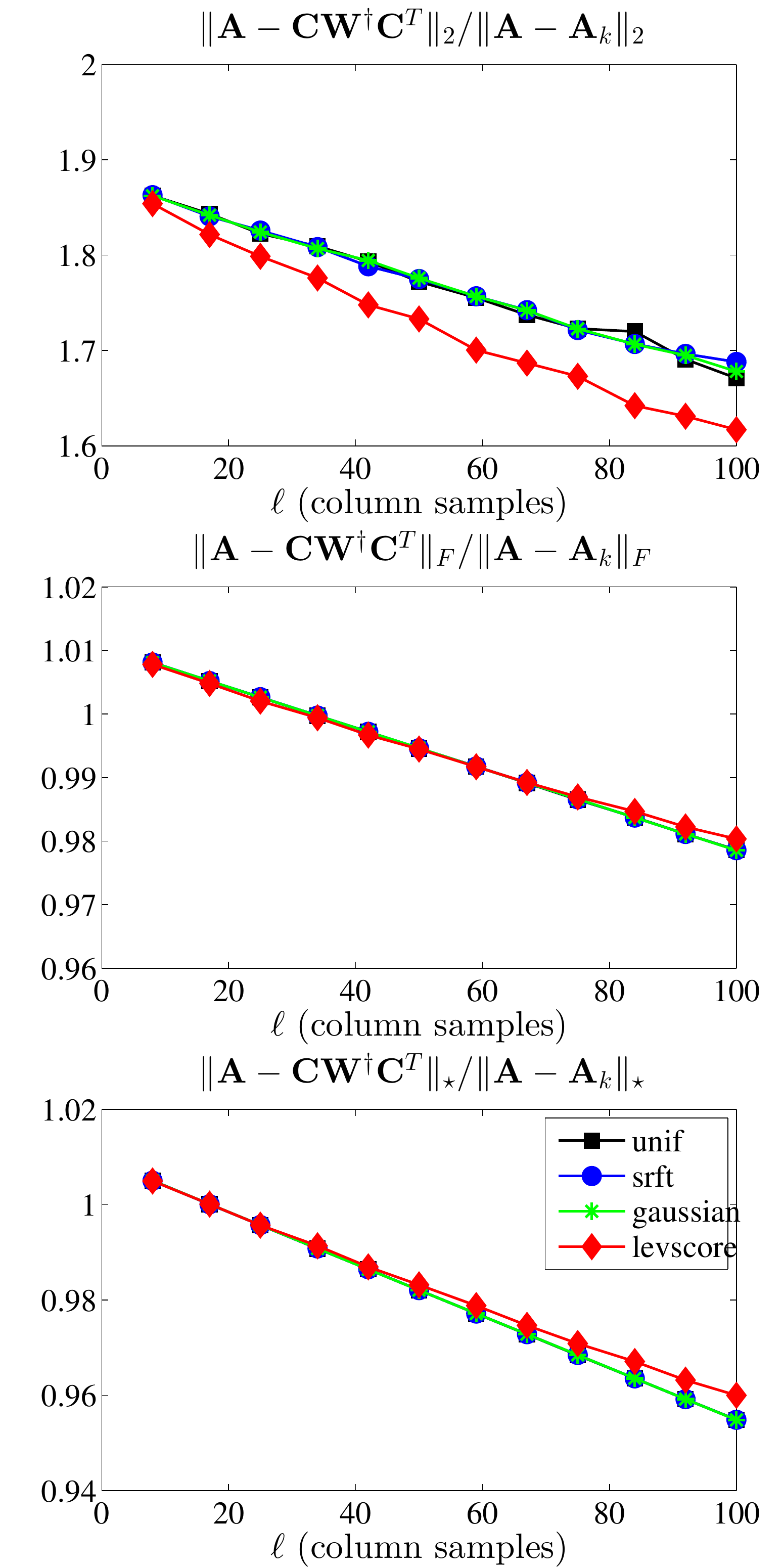}}%
 \subfigure[Protein, $k = 10$]{\includegraphics[width=1.6in, keepaspectratio=true]{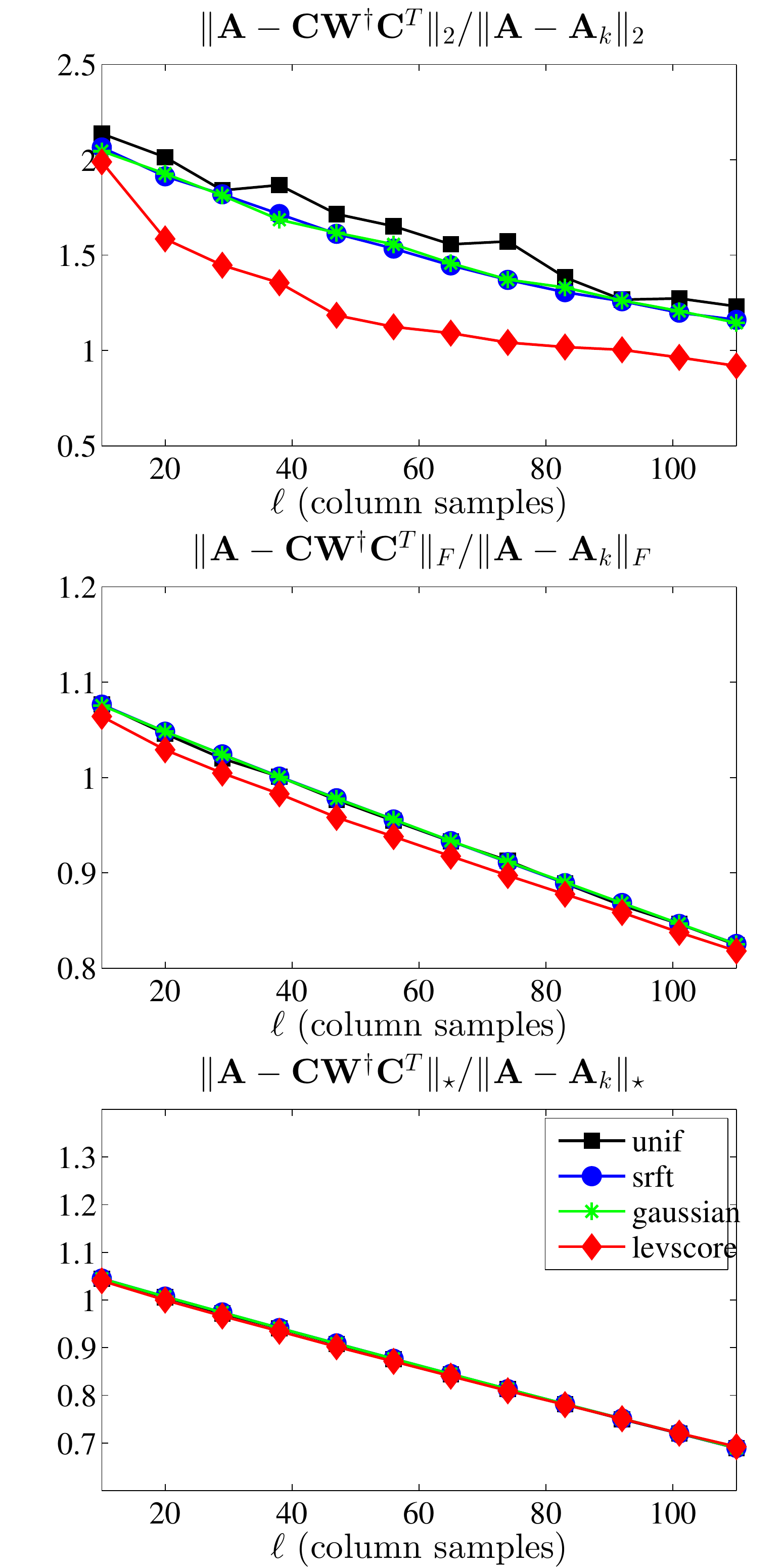}}%
 \subfigure[SNPs, $k = 5$]{\includegraphics[width=1.6in, keepaspectratio=true]{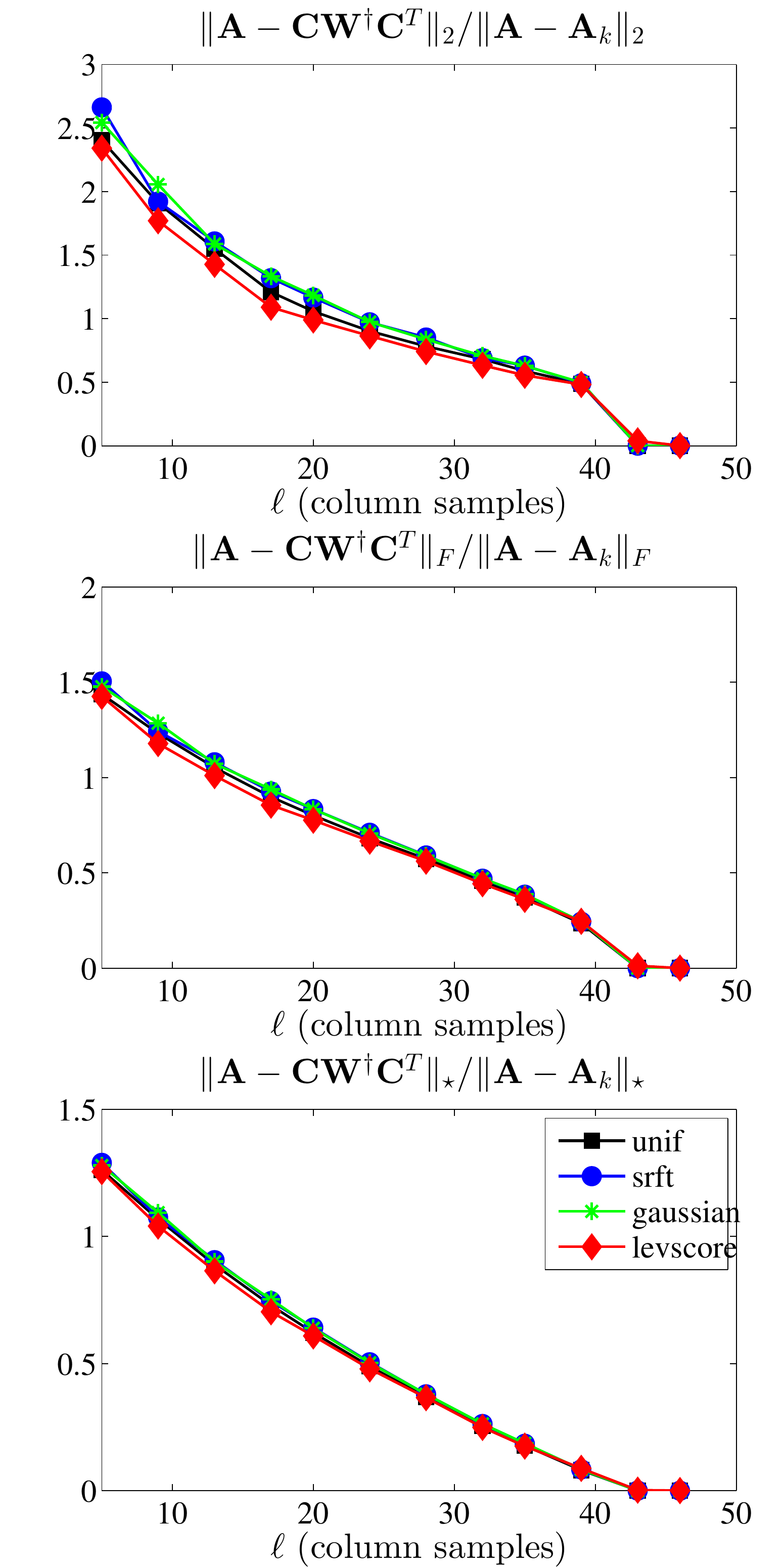}}%
 \subfigure[Gisette, $k = 12$]{\includegraphics[width=1.6in, keepaspectratio=true]{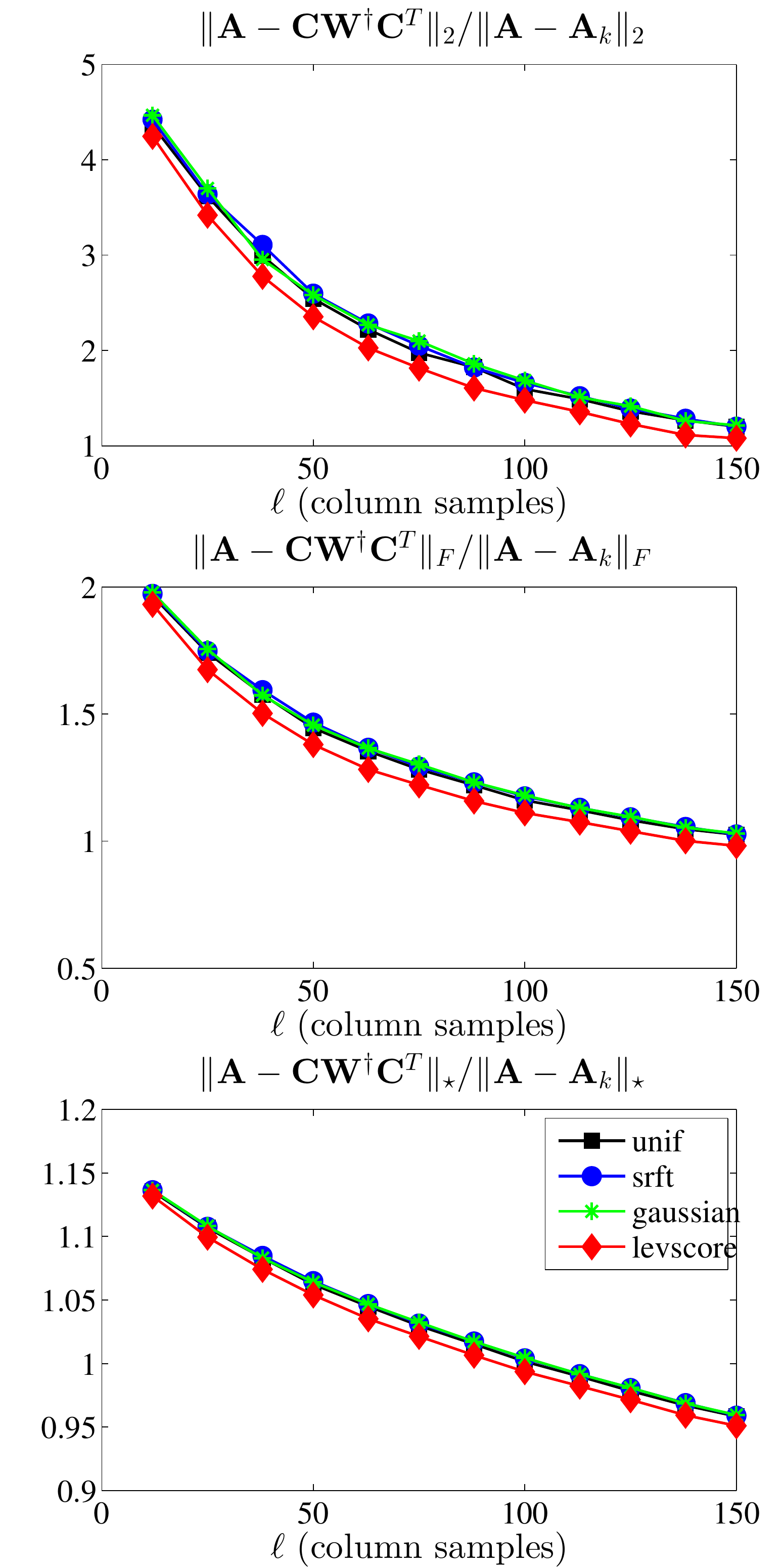}}%
 \\%
 \subfigure[Dexter, $k = 8$]{\includegraphics[width=1.6in, keepaspectratio=true]{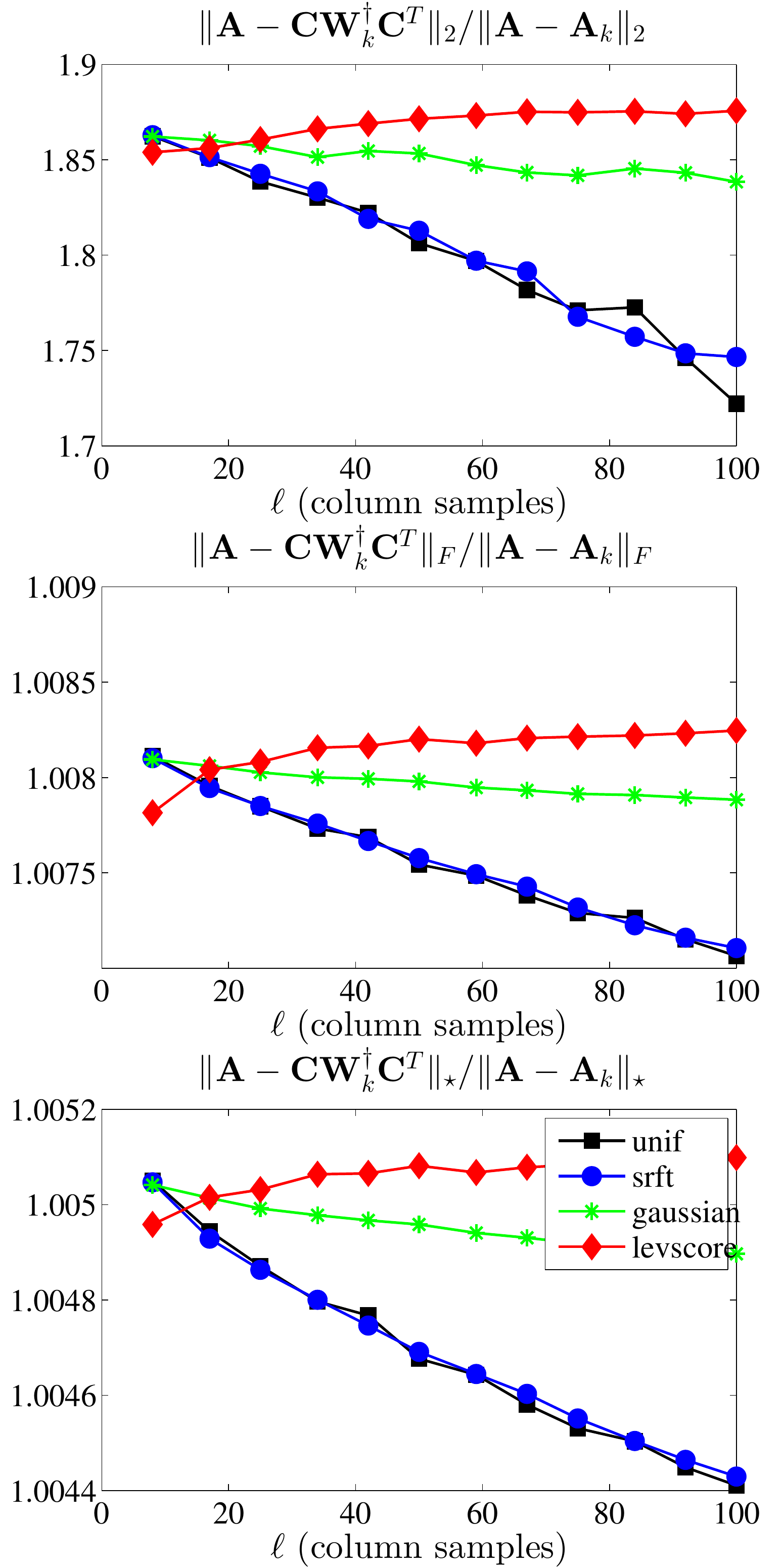}}%
 \subfigure[Protein, $k = 10$]{\includegraphics[width=1.6in, keepaspectratio=true]{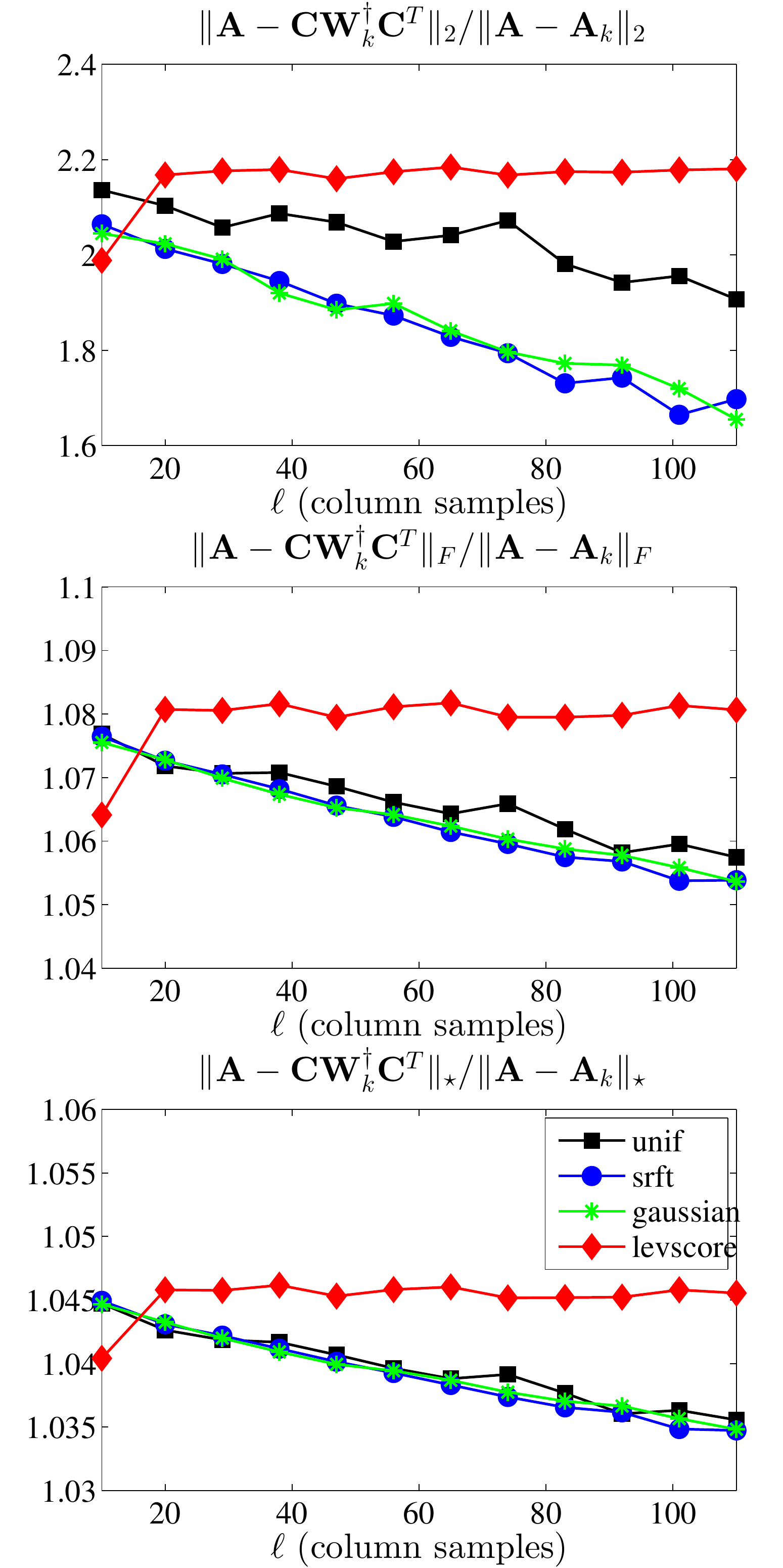}}%
 \subfigure[SNPs, $k = 5$]{\includegraphics[width=1.6in, keepaspectratio=true]{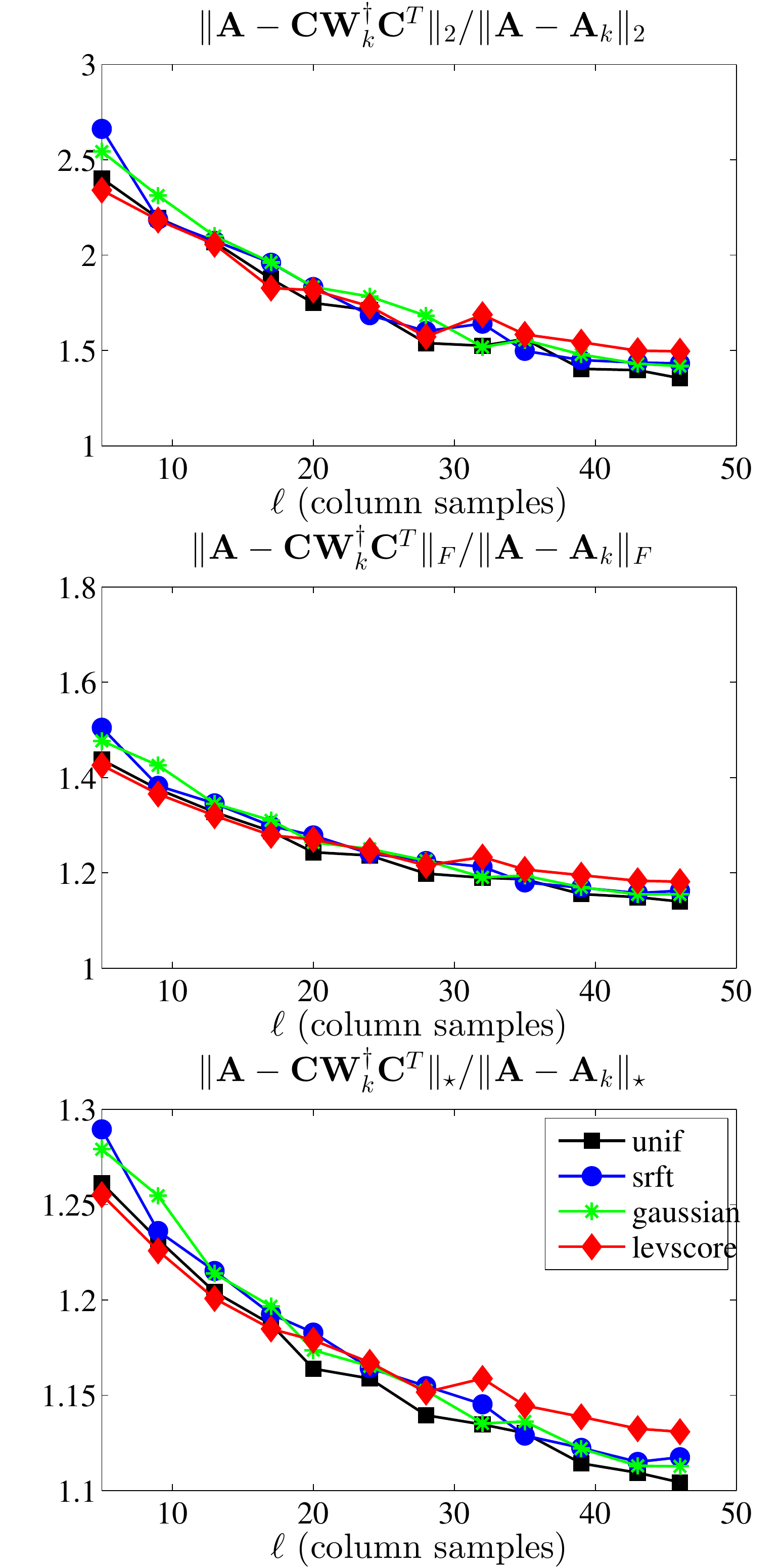}}%
 \subfigure[Gisette, $k = 12$]{\includegraphics[width=1.6in, keepaspectratio=true]{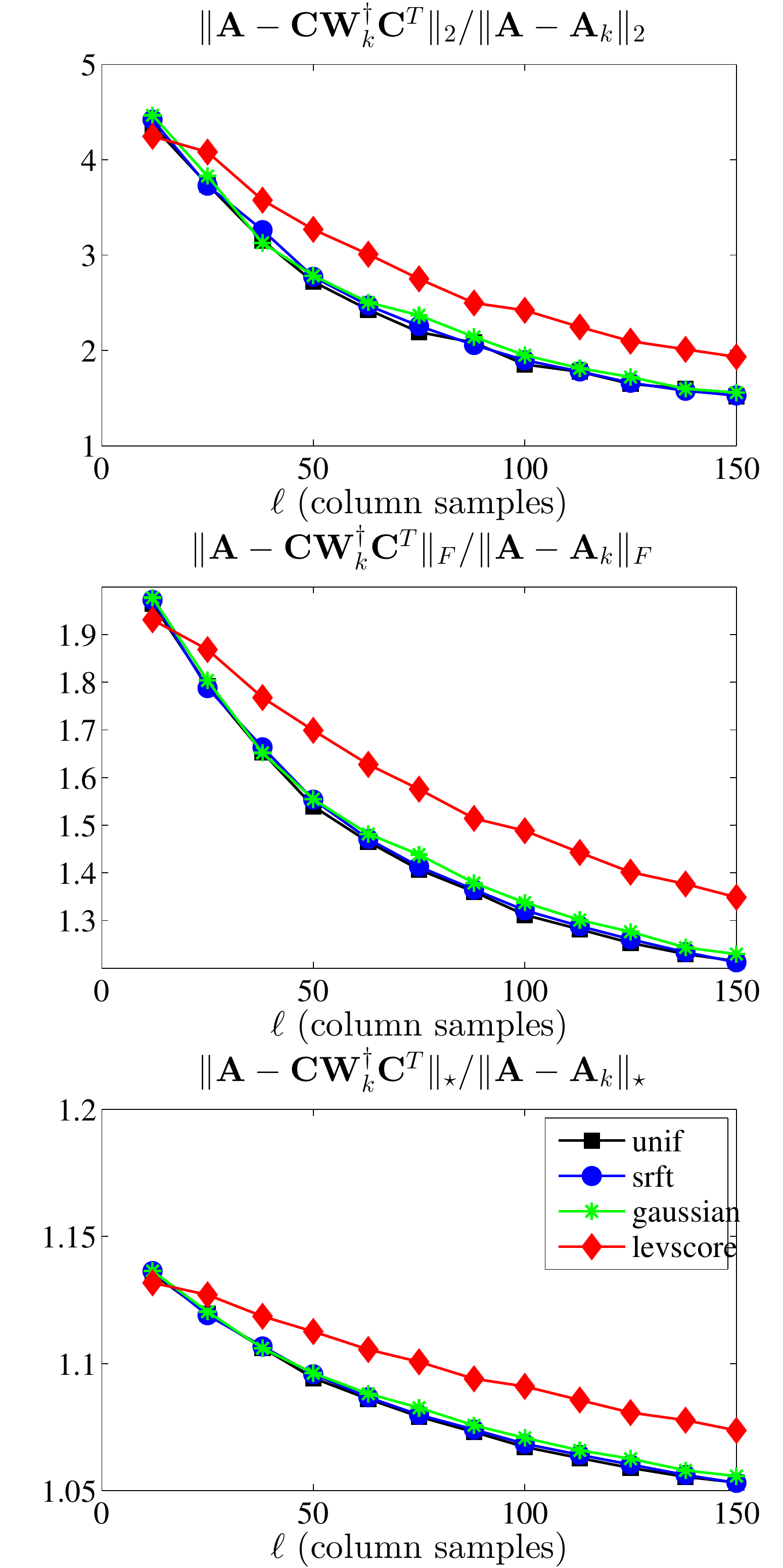}}%
 \caption{The spectral, Frobenius, and trace norm errors 
 (top to bottom, respectively, in each subfigure) of several
 (non-rank-restricted in top panels and rank-restricted in bottom panels)
 SPSD sketches, as a function of the number of columns samples $\ell$, 
 for the Linear Kernel data sets.
 }%
 \label{fig:linearkernel-exact-errors}
\end{figure}

Figure~\ref{fig:linearkernel-exact-errors}
shows the reconstruction error results for sampling and projection methods 
applied to several Linear Kernels.  
The data  sets (Dexter, Protein, SNPs, and Gisette) are all quite low-rank 
and have fairly uniform leverage scores.
Several observations are worth making about the results presented in this 
figure.
\begin{itemize}
\item
All of the methods perform quite similarly for the non-rank-restricted 
case: all have errors that decrease smoothly with increasing $\ell$, and in 
this case there is little advantage to using methods other than uniform 
sampling (since they perform similarly and are more expensive).
Also, since the ranks are so low and the leverage scores are so uniform, the 
leverage score sketch is no longer significantly distinguished by its 
tendency to saturate quickly.
\item
The scale of the Y axes is much larger than for the Laplacian data sets,
mostly since the matrices are much more well-approximated by low-rank 
matrices, although the scale decreases as one goes from spectral to 
Frobenius to trace reconstruction error, as before.
\item
For SNPs and Gisette, the rank-restricted reconstruction results are 
very similar for all four methods, with a smooth decrease in error as
$\ell$ is increased, although interestingly using leverage scores is
slightly worse for Gisette.
For Dexter and Protein, the situation is more complicated: using the SRFT 
always leads to smooth decrease as $\ell$ is increased, and uniform sampling
generally behaves the same way also; Gaussian projections behave this way 
for Protein, but 
for Dexter Gaussian projections are noticably worse than SRFT and uniform 
sampling; and, except for very small values of $\ell$, leverage-based 
sampling is worse still and gets noticably worse as $\ell$ is increased.
Even this poor behavior of leverage score sampling on the Linear Kernels is 
notably worse than for the rank-restricted Laplacians, where there was a 
range of moderately small $\ell$ where leverage score sampling was much 
superior to other methods.
\end{itemize}
These linear kernels (and also to some extent the dense RBF kernels below
that have larger $\sigma$ parameter) are examples of relatively ``nice'' 
machine learning data sets that are similar to matrices where uniform 
sampling has been shown to perform well 
previously~\cite{TKR08,KMT09,KMT09c,KMT12}; and for these matrices our 
empirical results agree with these prior works.

\subsubsection{Dense and Sparse RBF Kernels}

\begin{figure}[p]
 \centering
 \subfigure[AbaloneD, $\sigma = .15, k = 20$]{\includegraphics[width=1.6in, keepaspectratio=true]{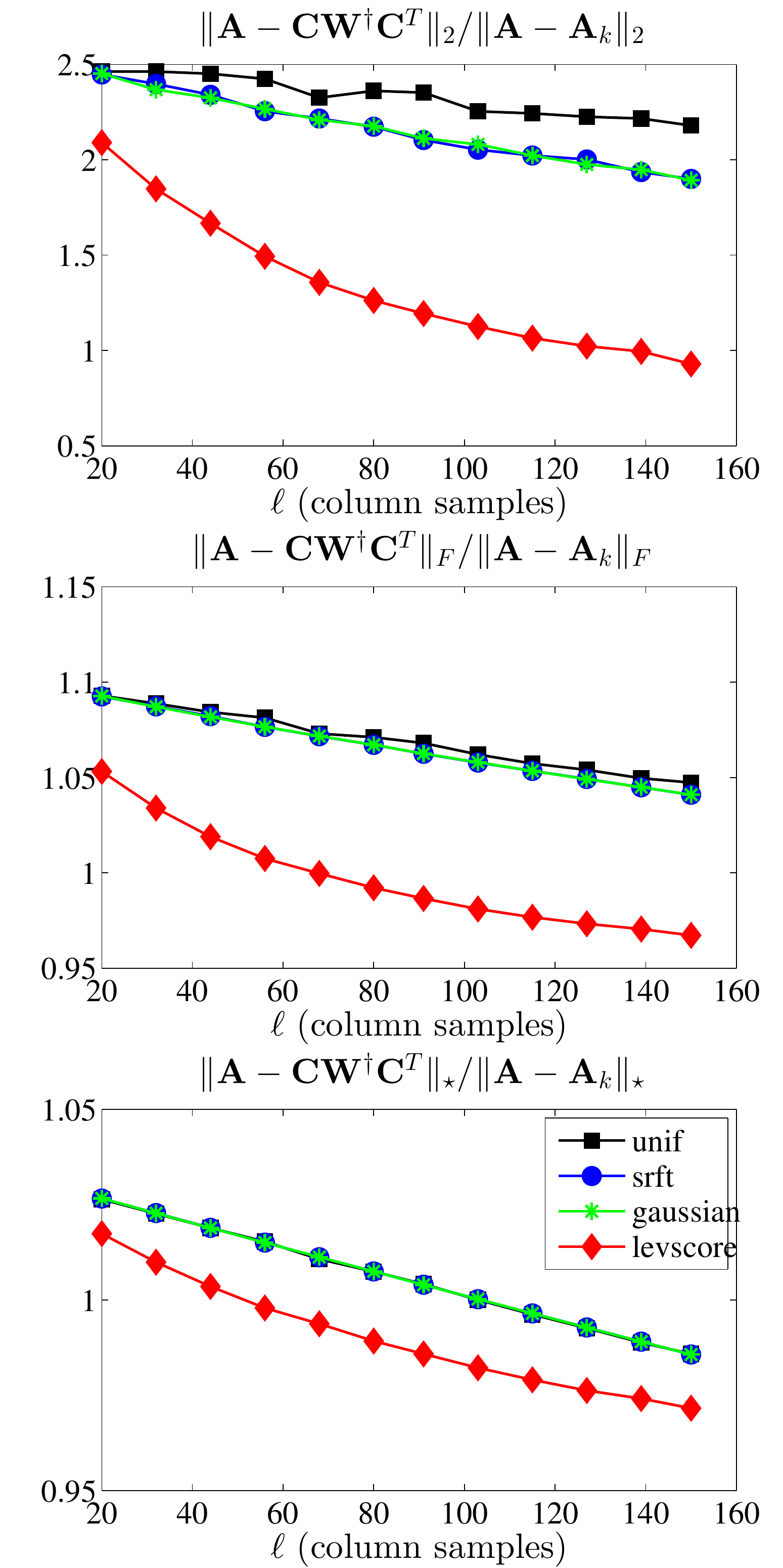}}%
 \subfigure[AbaloneD, $\sigma = 1, k = 20$]{\includegraphics[width=1.6in, keepaspectratio=true]{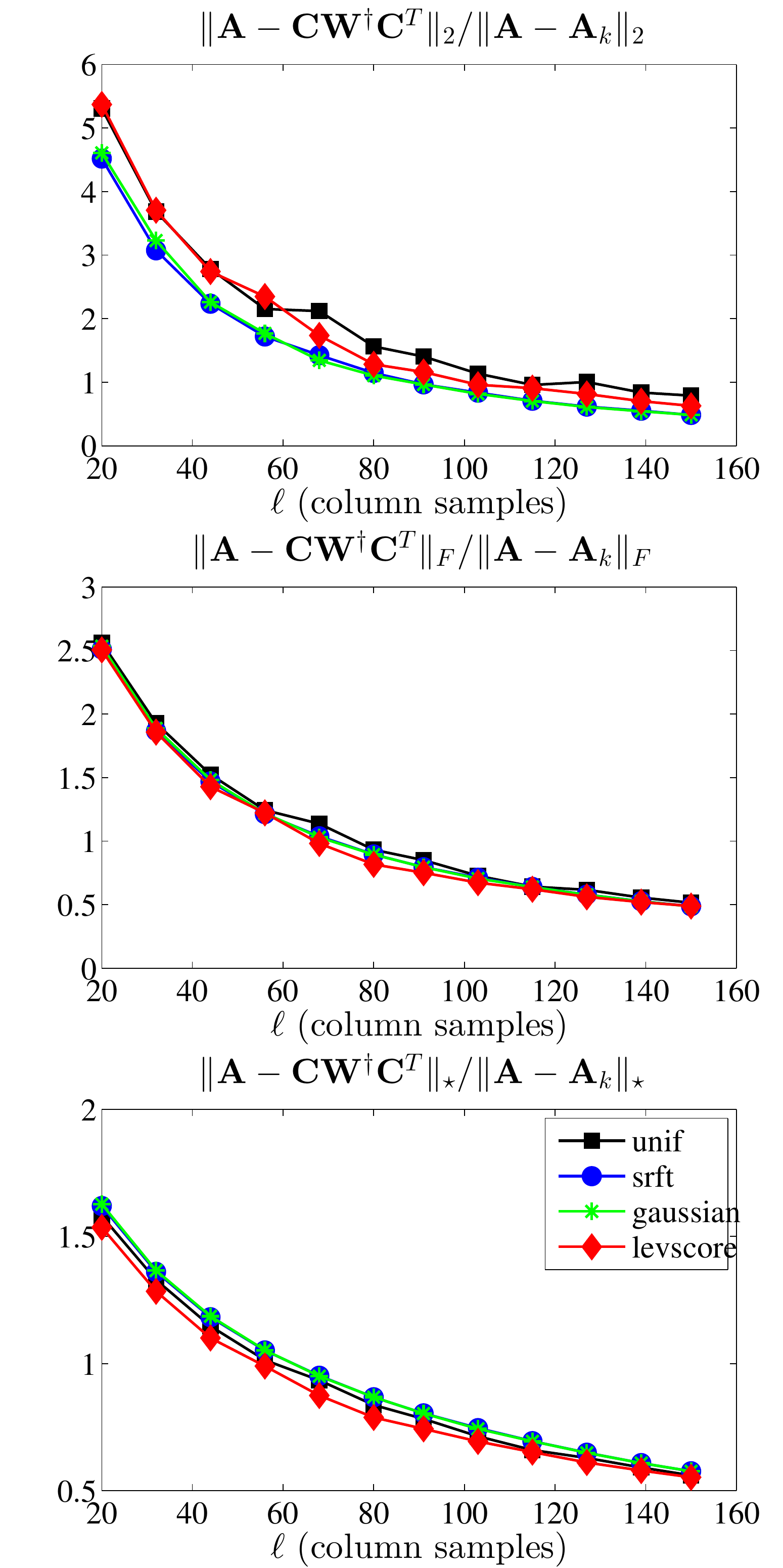}}%
 \subfigure[WineD, $\sigma = 1, k = 20$]{\includegraphics[width=1.6in, keepaspectratio=true]{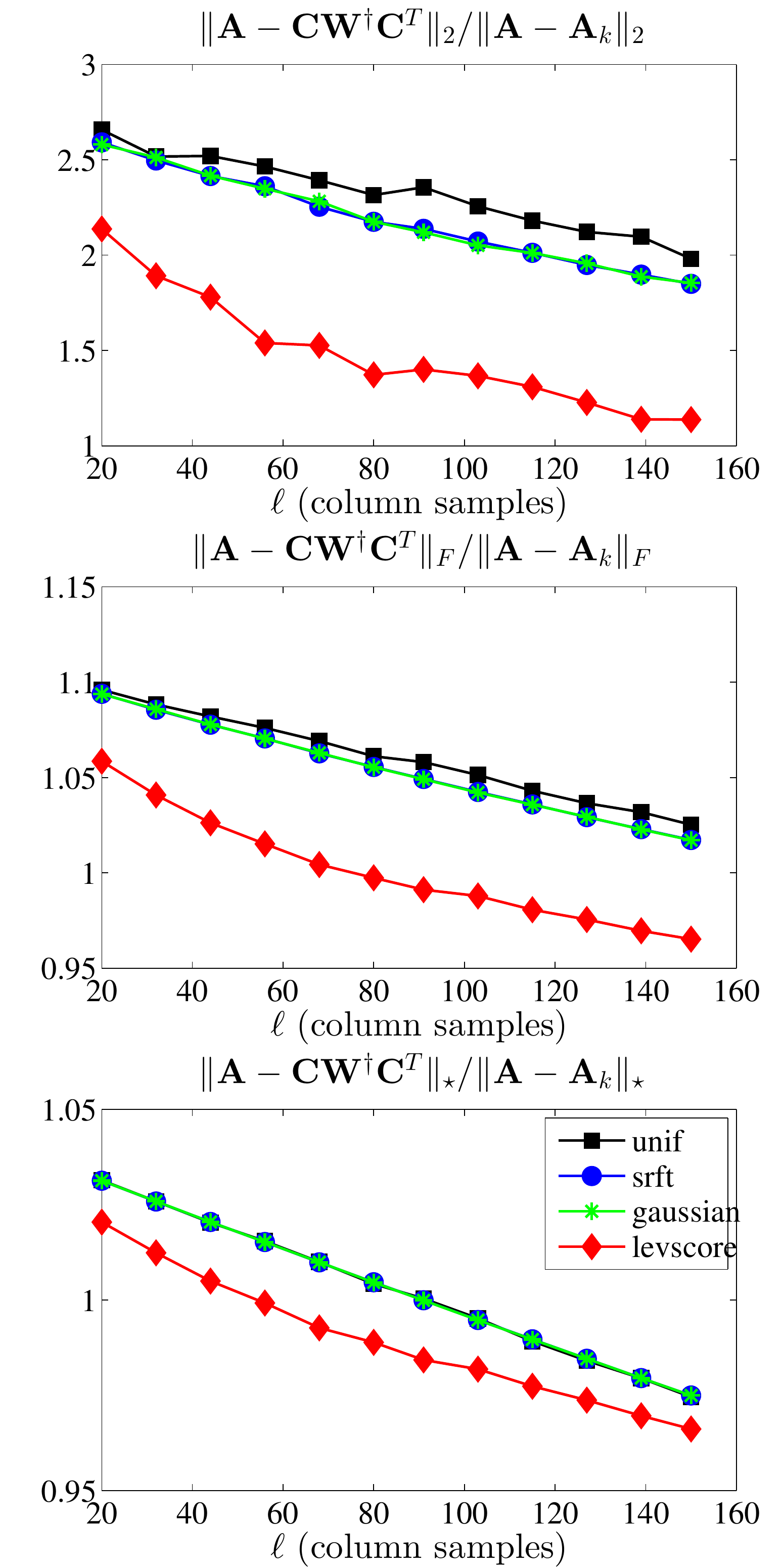}}%
 \subfigure[WineD, $\sigma = 2.1, k = 20$]{\includegraphics[width=1.6in, keepaspectratio=true]{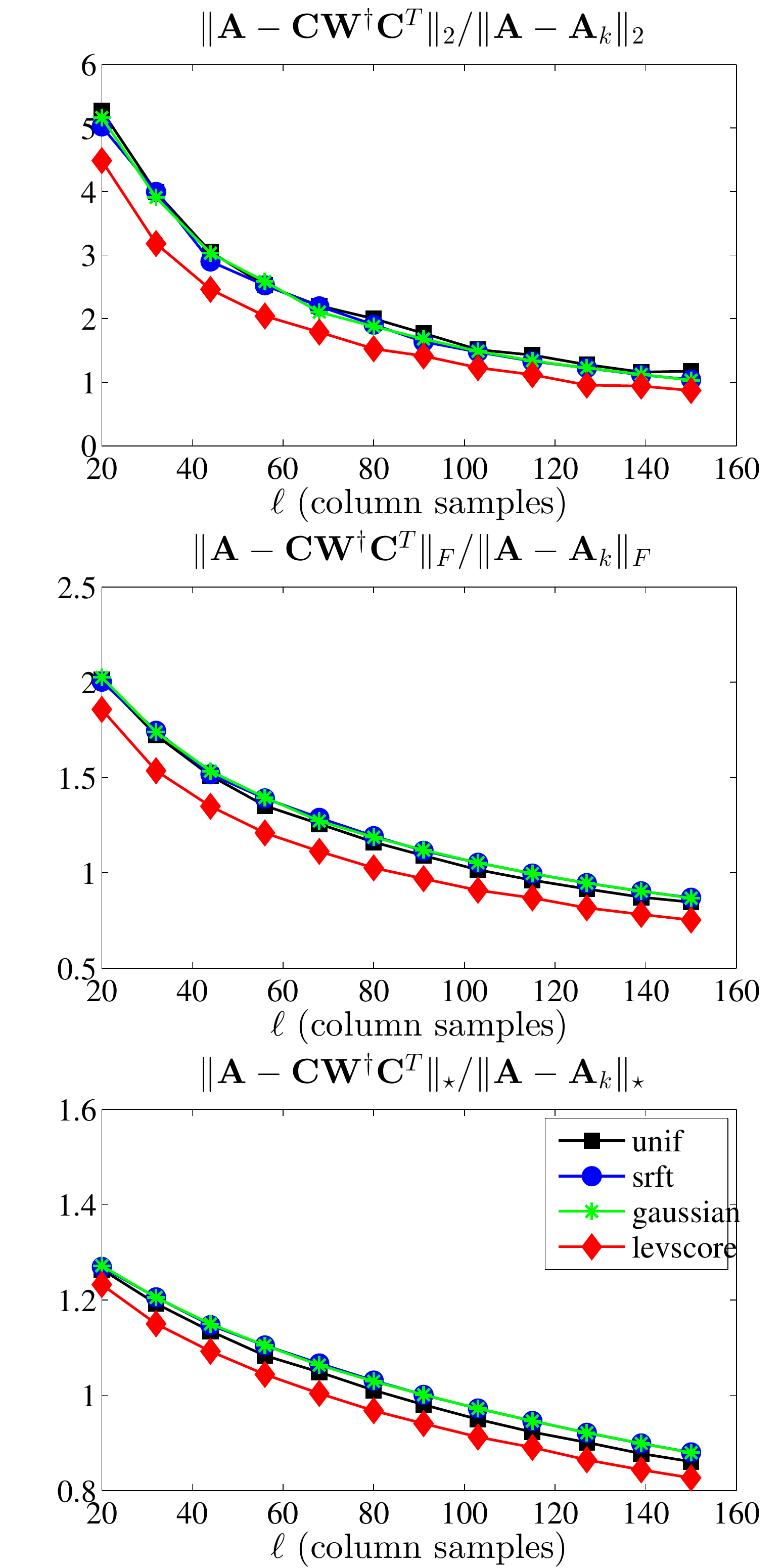}}%
 \\%
 \subfigure[AbaloneD, $\sigma = .15, k = 20$]{\includegraphics[width=1.6in, keepaspectratio=true]{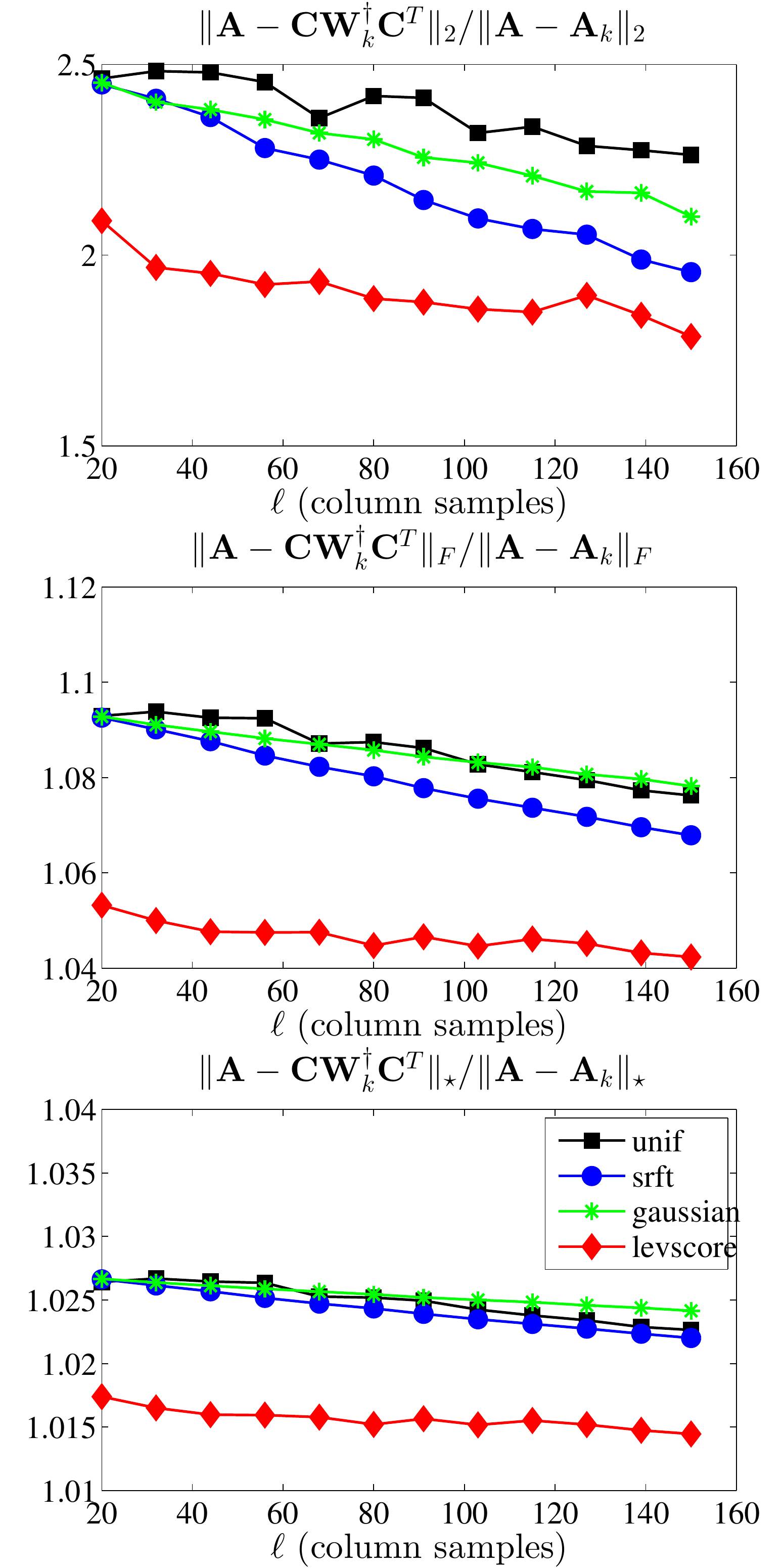}}%
 \subfigure[AbaloneD, $\sigma = 1, k = 20$]{\includegraphics[width=1.6in, keepaspectratio=true]{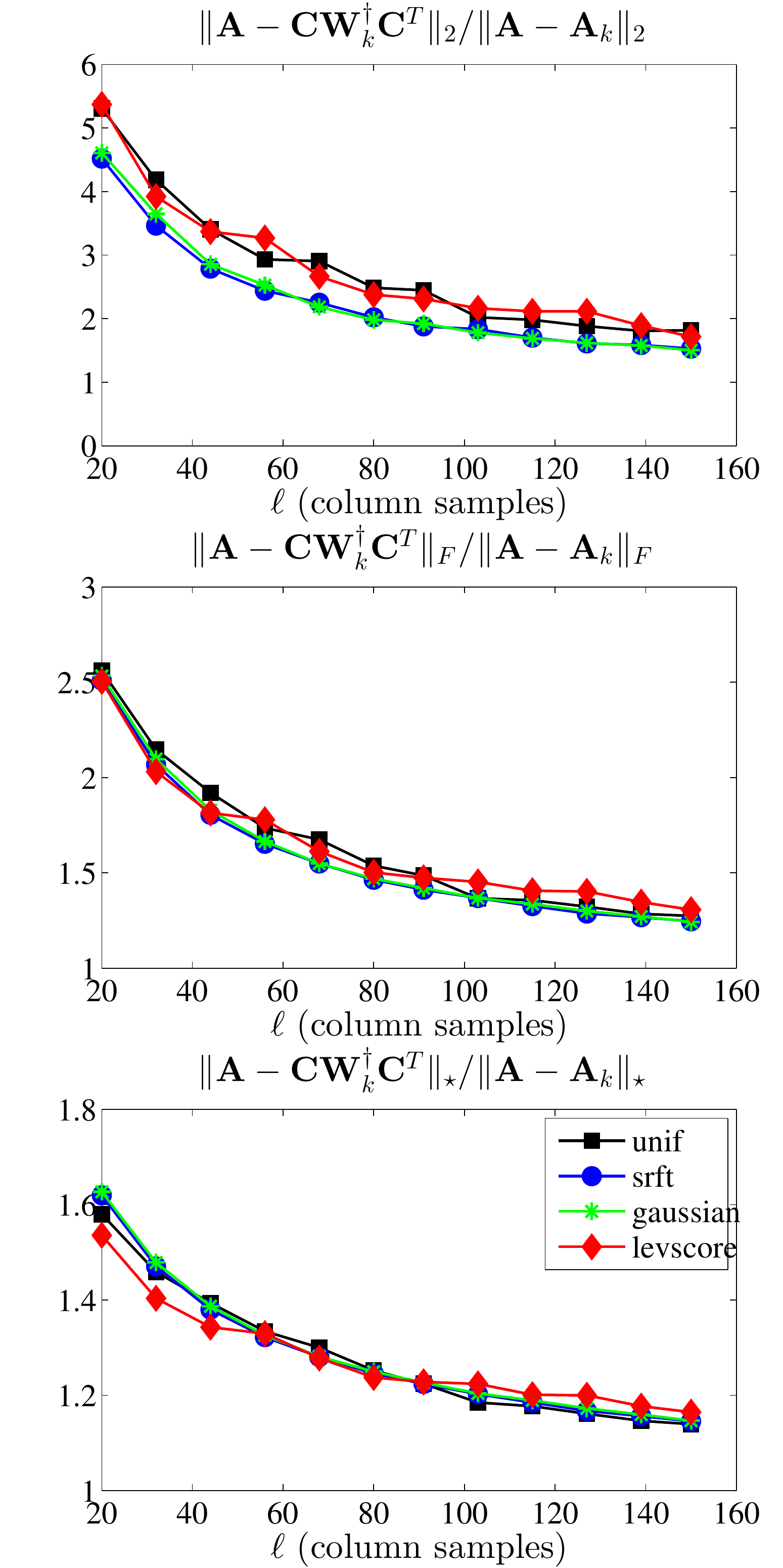}}%
 \subfigure[WineD, $\sigma = 1, k = 20$]{\includegraphics[width=1.6in, keepaspectratio=true]{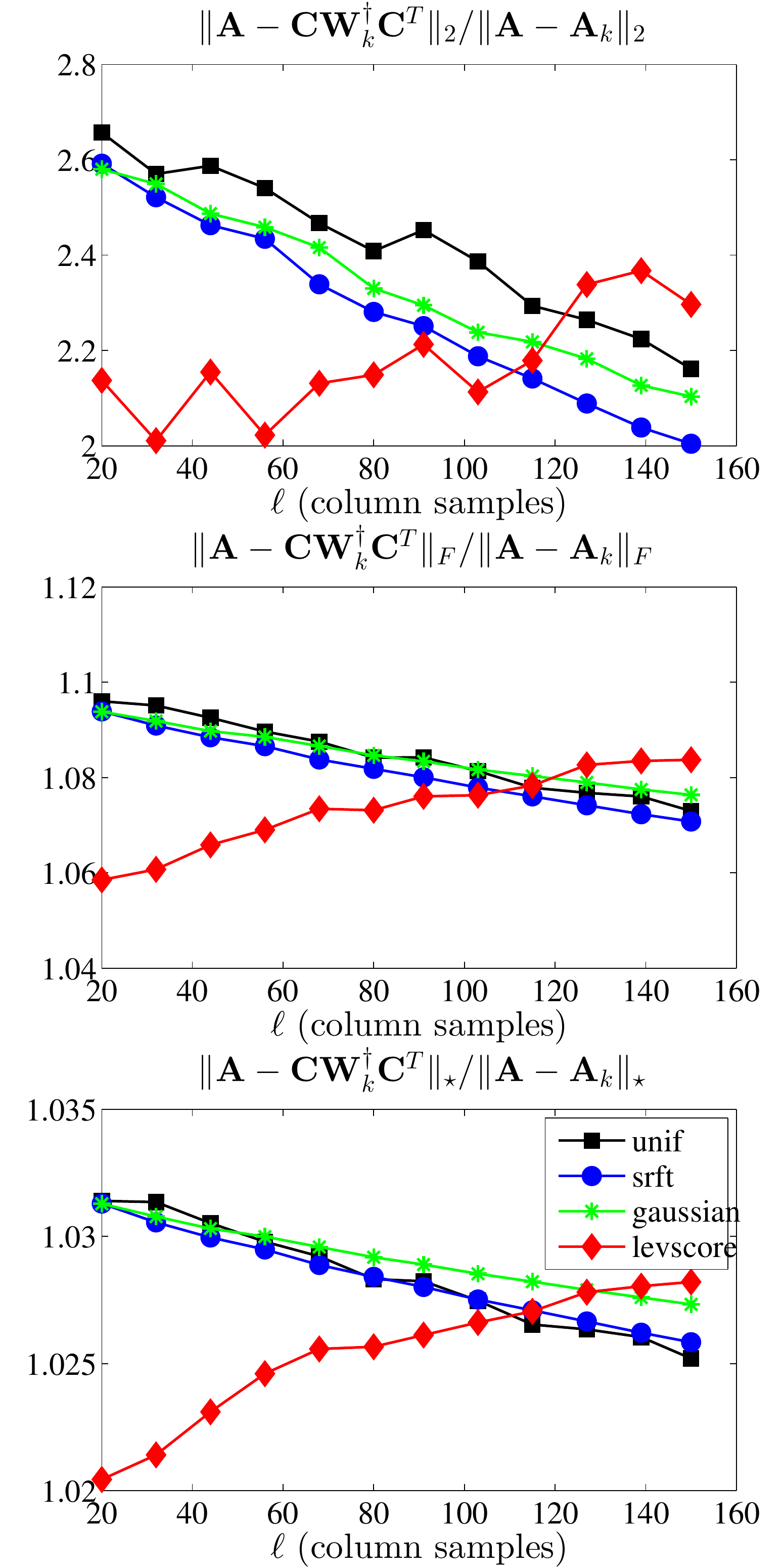}}%
 \subfigure[WineD, $\sigma = 2.1, k = 20$]{\includegraphics[width=1.6in, keepaspectratio=true]{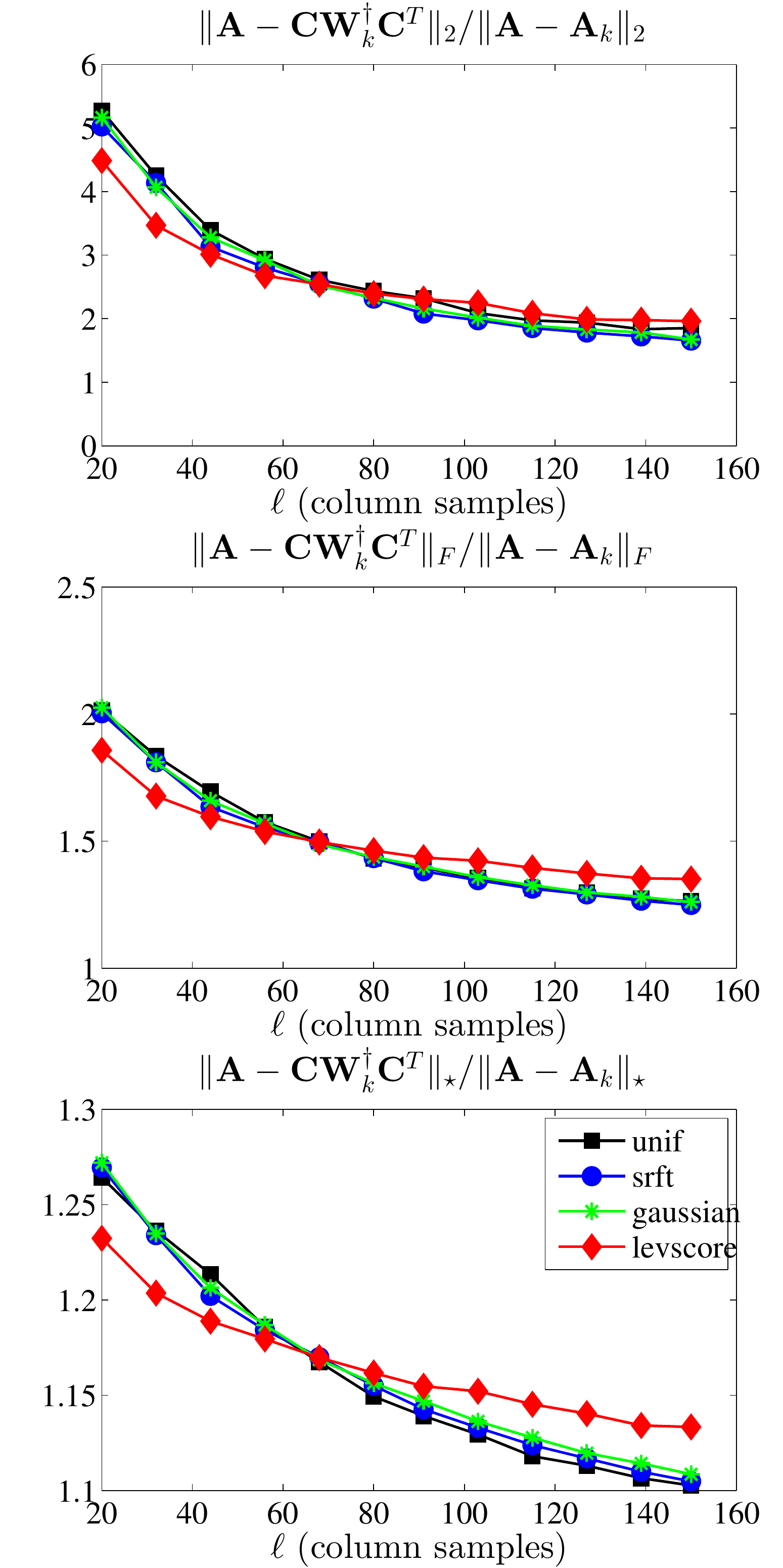}}
 \caption{The spectral, Frobenius, and trace norm errors 
 (top to bottom, respectively, in each subfigure) of several
 (non-rank-restricted in top panels and rank-restricted in bottom panels)
 SPSD sketches, as a function of the number of columns samples $\ell$, 
 for several dense RBF data sets.}%
 \label{fig:denserbf-exact-errors}
\end{figure}

\begin{figure}[p]
 \centering
 \subfigure[AbaloneS, $\sigma = .15, k = 20$]{\includegraphics[width=1.6in, keepaspectratio=true]{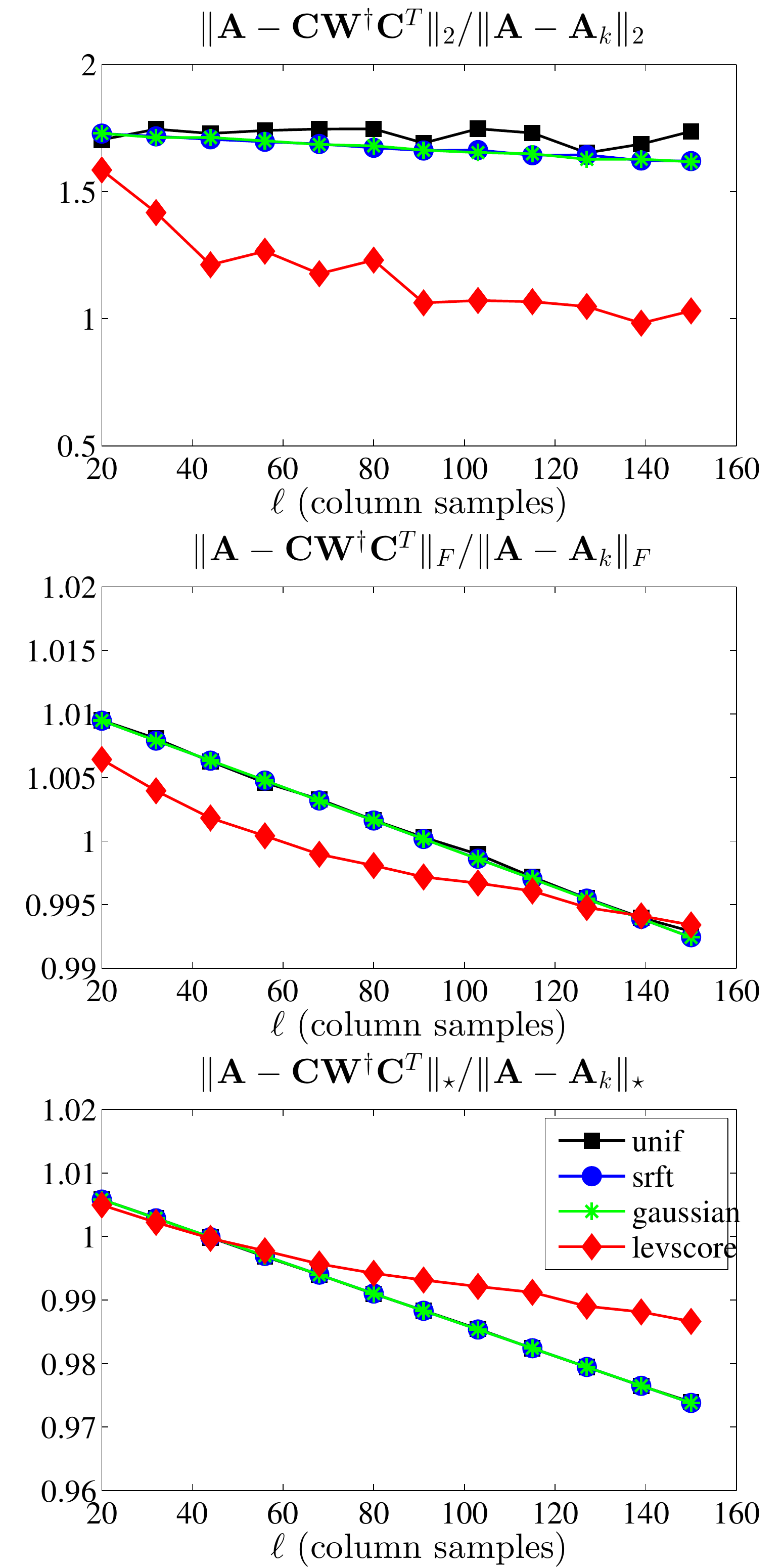}}%
 \subfigure[AbaloneS, $\sigma = 1, k = 20$]{\includegraphics[width=1.6in, keepaspectratio=true]{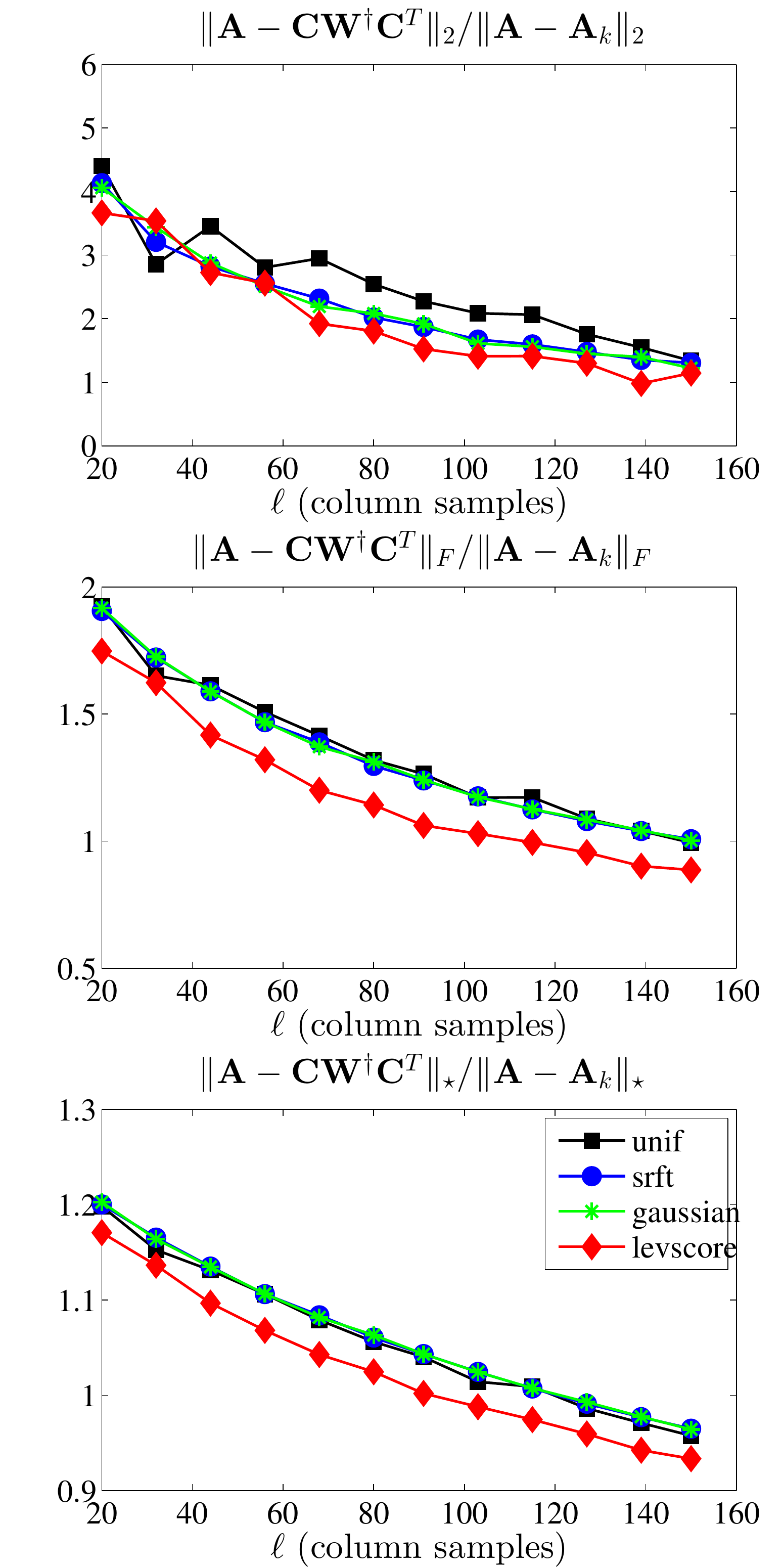}}%
 \subfigure[WineS, $\sigma = 1, k = 20$]{\includegraphics[width=1.6in, keepaspectratio=true]{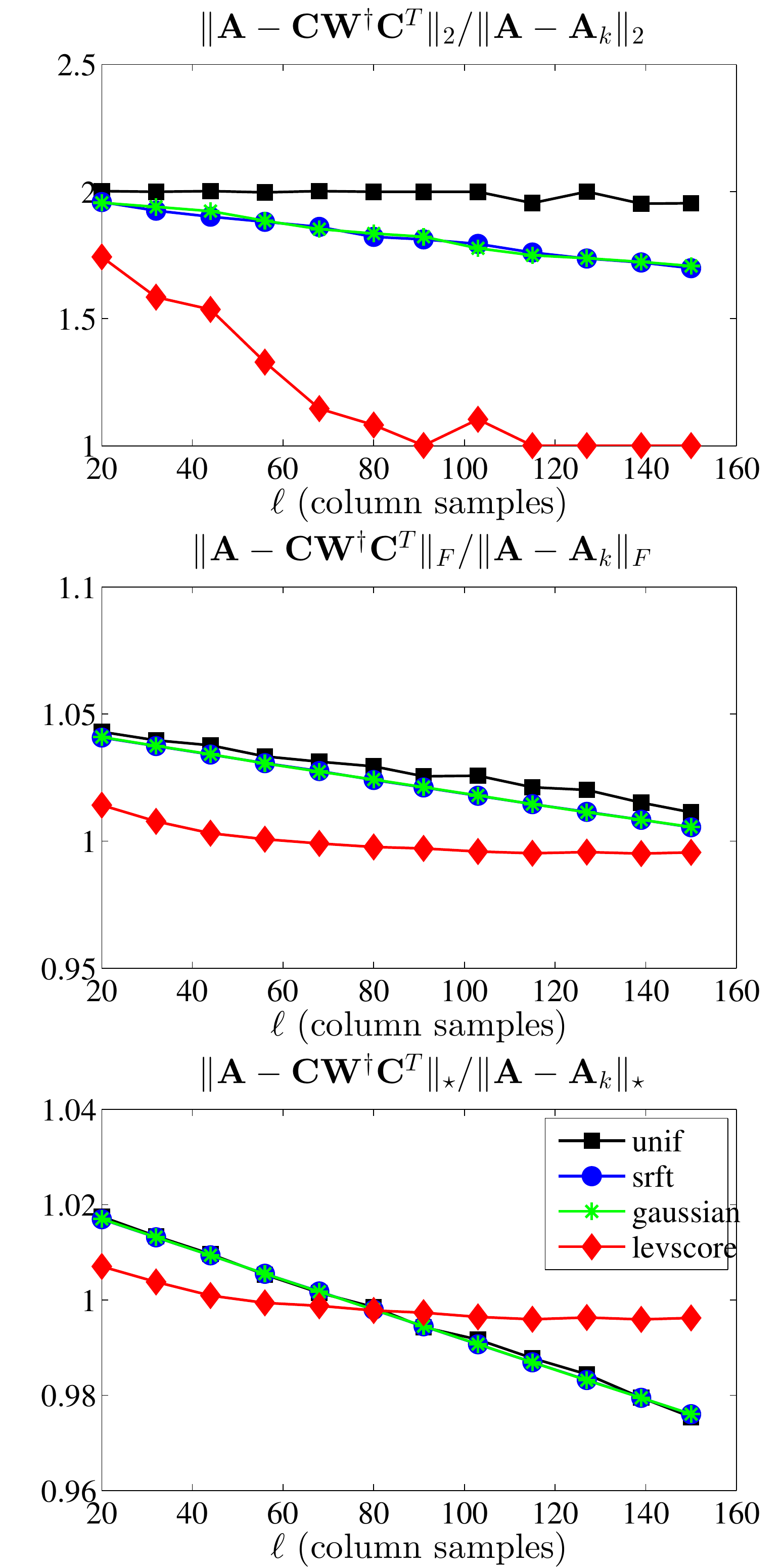}}%
 \subfigure[WineS, $\sigma = 2.1, k = 20$]{\includegraphics[width=1.6in, keepaspectratio=true]{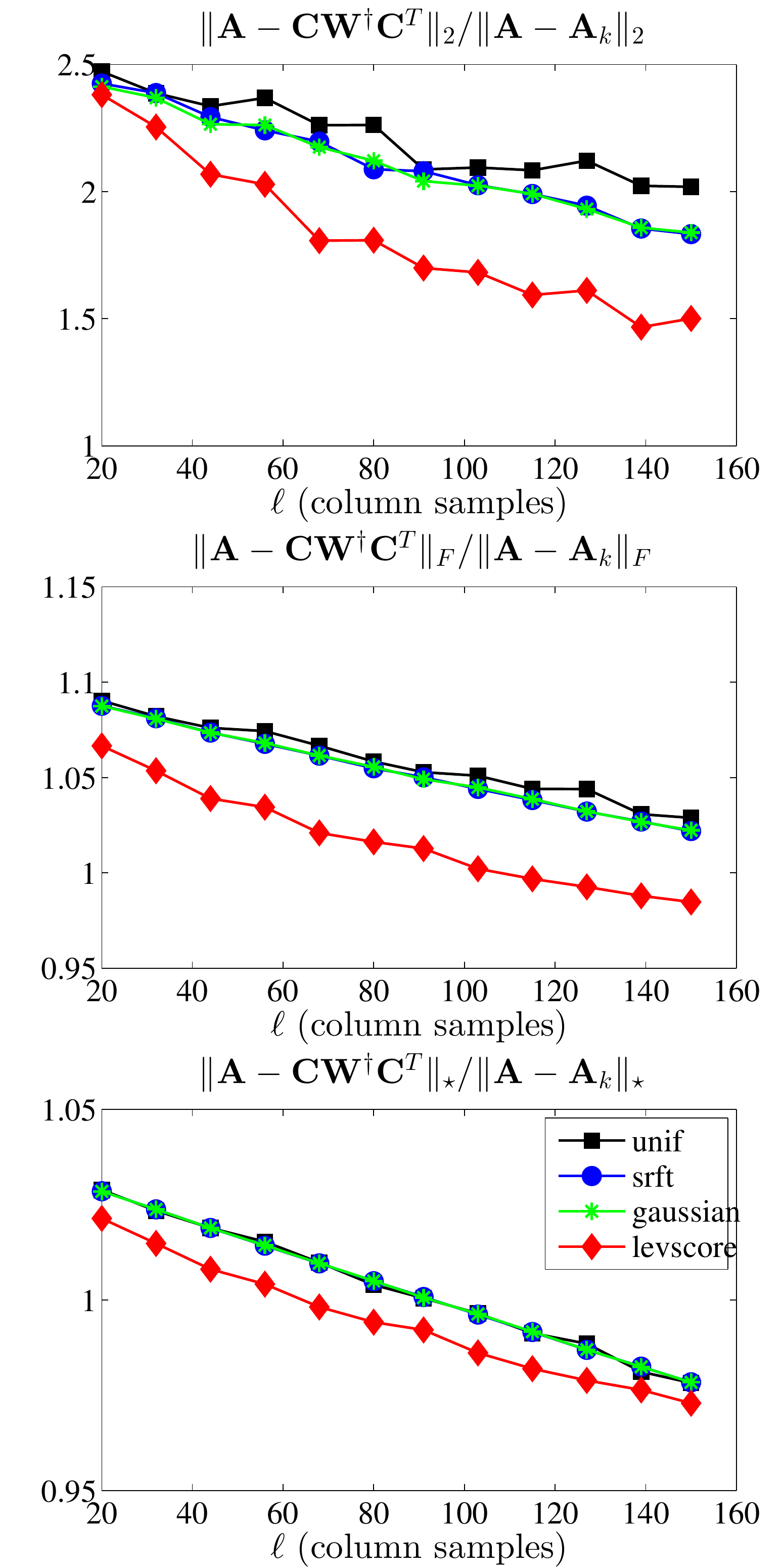}}%
 \\%
 \subfigure[AbaloneS, $\sigma = .15, k = 20$]{\includegraphics[width=1.6in, keepaspectratio=true]{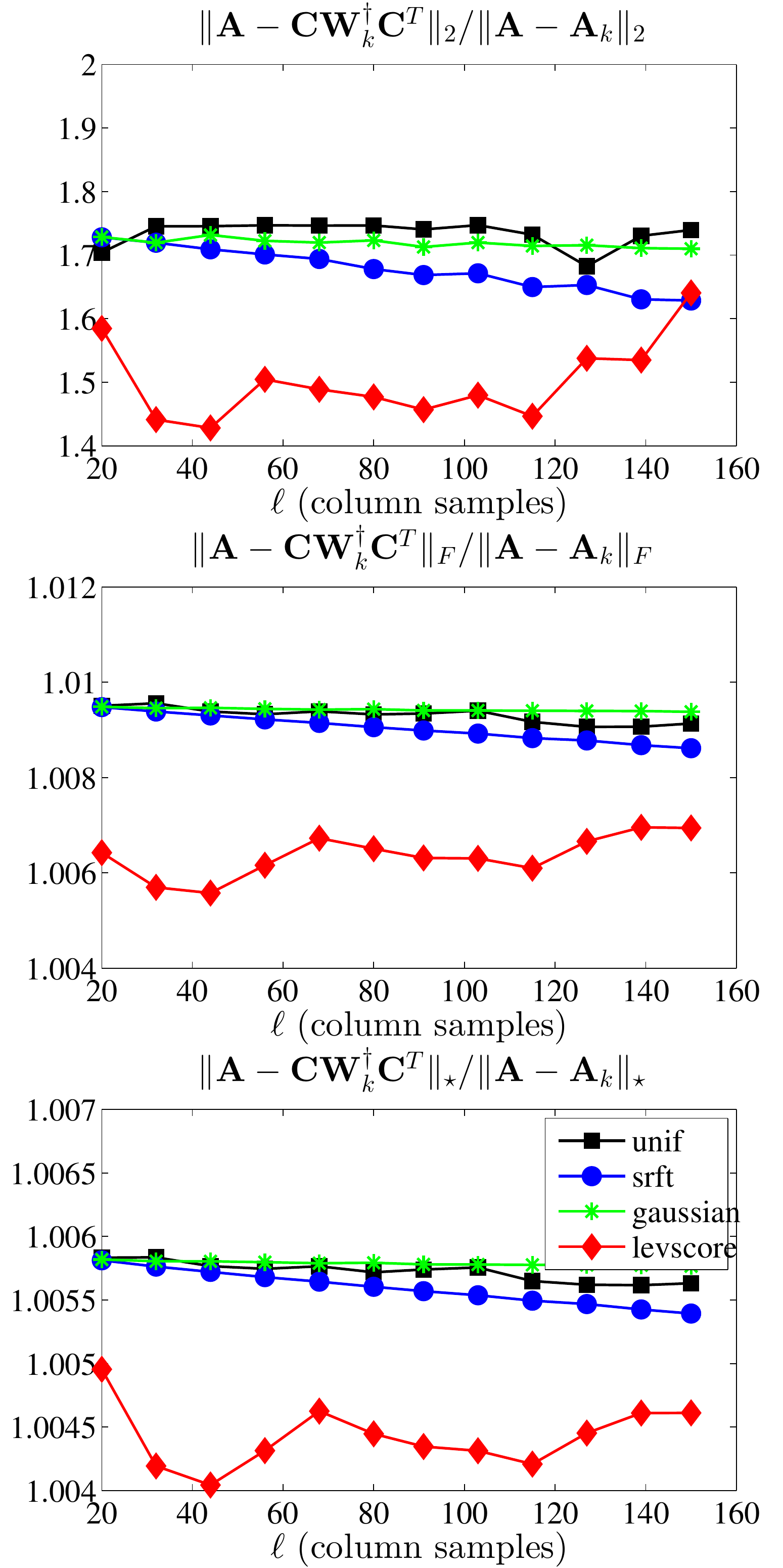}}%
 \subfigure[AbaloneS, $\sigma = 1, k = 20$]{\includegraphics[width=1.6in, keepaspectratio=true]{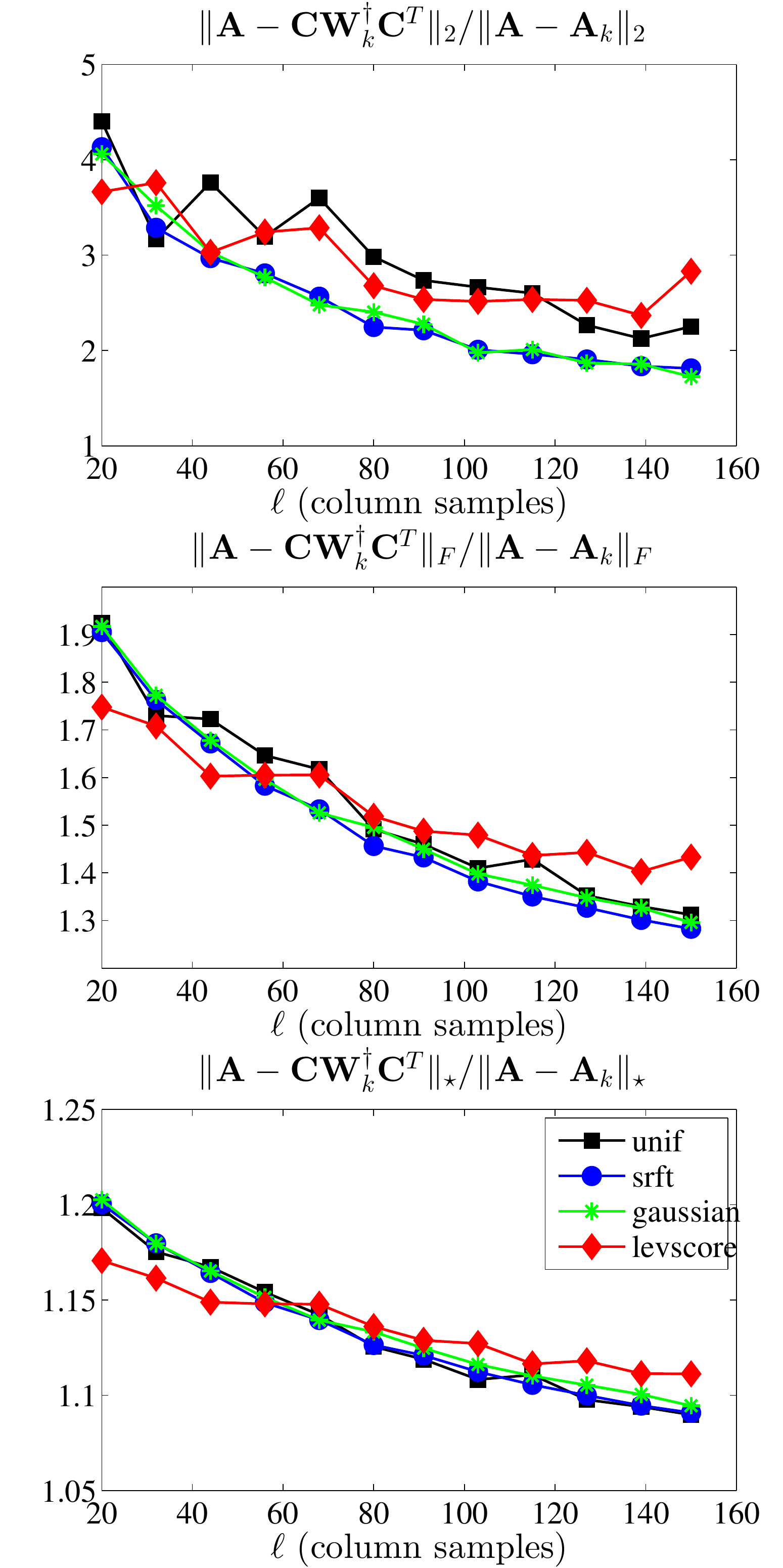}}%
 \subfigure[WineS, $\sigma = 1, k = 20$]{\includegraphics[width=1.6in, keepaspectratio=true]{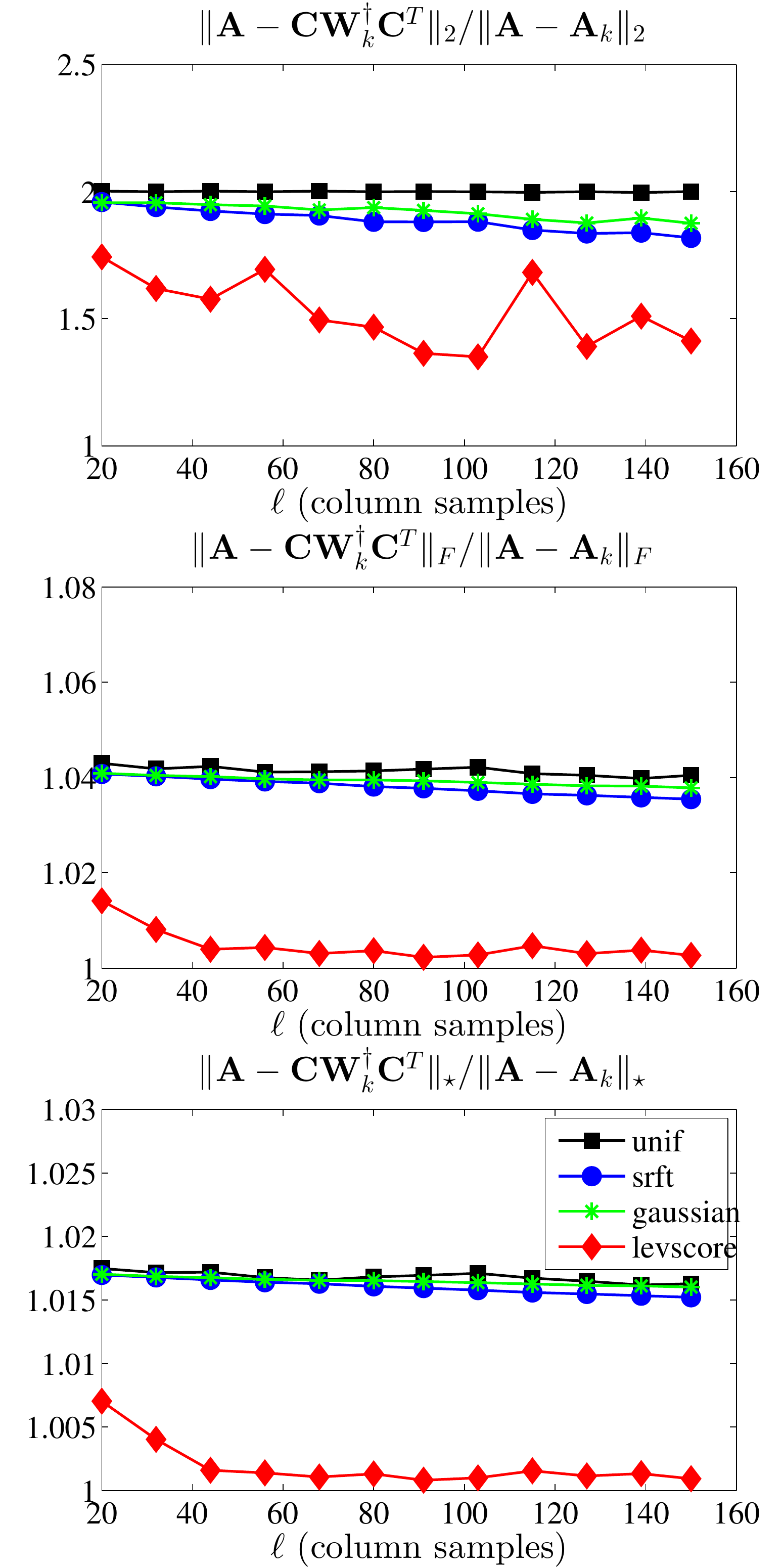}}%
 \subfigure[WineS, $\sigma = 2.1, k = 20$]{\includegraphics[width=1.6in, keepaspectratio=true]{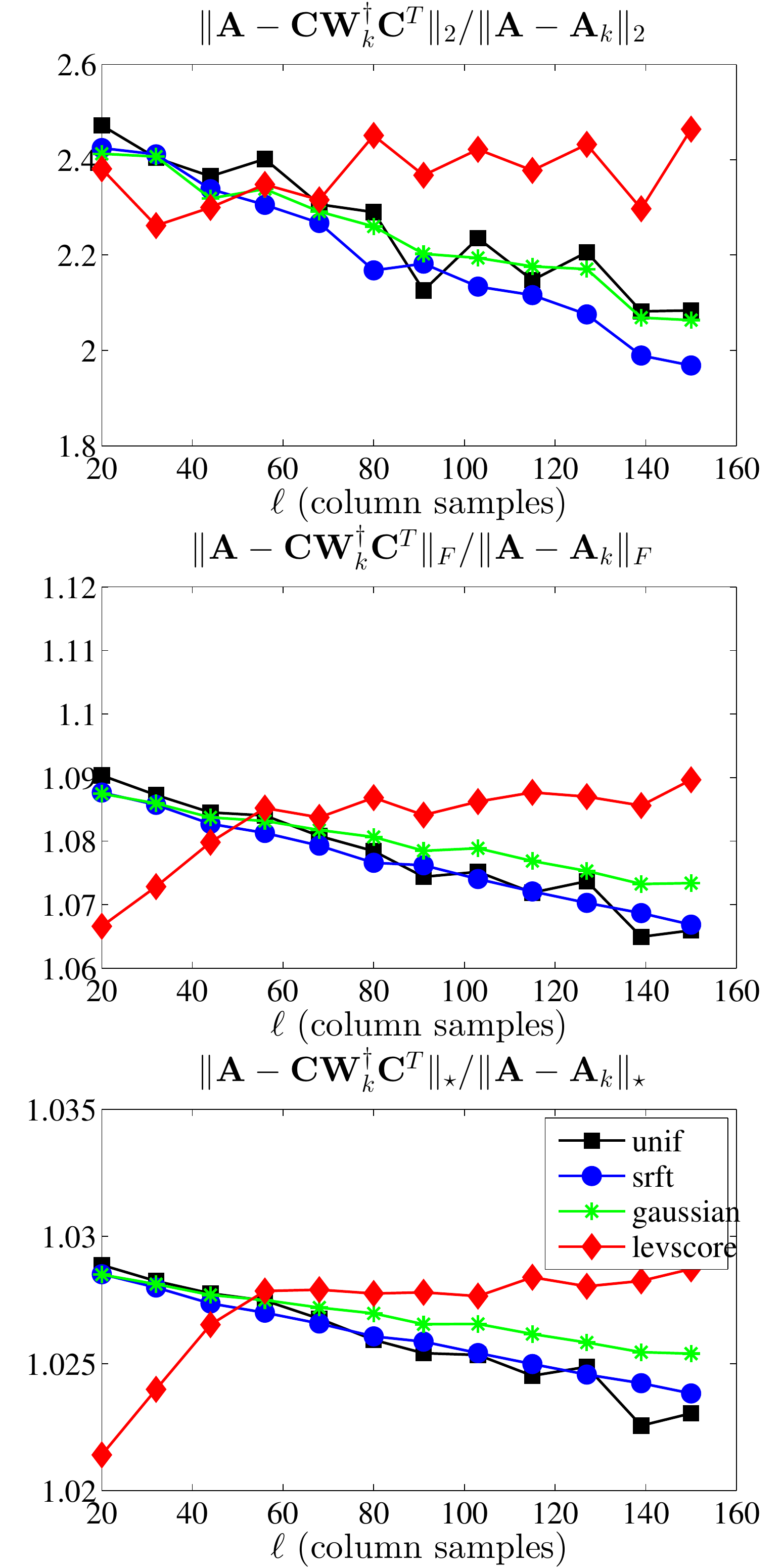}}%
 \caption{The spectral, Frobenius, and trace norm errors 
 (top to bottom, respectively, in each subfigure) of several
 (non-rank-restricted in top panels and rank-restricted in bottom panels)
 SPSD sketches, as a function of the number of columns samples $\ell$, 
 for several sparse RBF data sets.
 }%
 \label{fig:sparserbf-exact-errors}
\end{figure}

Figure~\ref{fig:denserbf-exact-errors} and
Figure~\ref{fig:sparserbf-exact-errors} present 
the reconstruction error results for sampling and projection methods 
applied to several dense RBF and sparse RBF kernels.
Several observations are worth making about the results presented in these 
figures.
\begin{itemize}
\item
For the non-rank-restricted results, all of the methods have errors that 
decrease with increasing $\ell$.
In particular, for larger values of $\sigma$ and for denser data, the 
decrease is somewhat more regular, and the four methods tend to perform 
similarly.
For larger values of $\sigma$ and sparser data, leverage
score sampling is somewhat better.
This parallels what we observed with the Linear Kernels, except that here the 
leverage score sampling is somewhat better for all values of $\ell$.
\item
For the non-rank-restricted results for the smaller values of $\sigma$, 
leverage score sampling tends to be much better than uniform sampling and 
projection-based methods.
For the sparse data, however, this effect saturates; and we again observe 
(especially when $\sigma$ is smaller in AbaloneS and WineS) the tradeoff we 
observed previously with the Laplacian data---leverage score sampling is 
better when $\ell$ is moderately larger than $k$, while uniform sampling and 
random projections are better when $\ell$ is much larger than $k$.
\item
For the rank-restricted results, we see that when $\sigma$ is large, all 
of the results tend to perform similarly.
(The exception to this is WineS, for which leverage score sampling starts out 
much better than other methods and then gets worse as $\ell$ is increased.)
On the other hand, when $\sigma$ is small, the results are more complex.
Leverage score sampling is typically much better than other methods, 
although the results are quite choppy as a function of $\ell$, and in some 
cases the effect diminished as $\ell$ is increased.
\end{itemize}
Recall from Table~\ref{table:datasets_stats} that for smaller values of 
$\sigma$ and for sparser kernels, the SPSD matrices are less 
well-approximated by low-rank matrices, and they have more heterogeneous 
leverage scores.
Thus, they are more similar to the Laplacian data than the Linear Kernel 
data; and this suggests (as we have observed) that leverage score sampling 
should perform relatively better than uniform column sampling and 
projection-based schemes when in these two cases.
In particular, nowhere do we see that leverage score sampling performs much 
worse than other methods, as we saw with the rank-restricted Linear 
Kernel results.

\subsubsection{Summary of Comparison of Sampling and Projection Algorithms}

Before proceeding, there are several summary observations that we can make 
about sampling versus projection methods for the data sets we have considered.
\begin{itemize}
\item
Linear Kernels and to a lesser extent Dense RBF Kernels with larger $\sigma$
parameter have relatively low-rank and relatively uniform leverage scores, 
and in these cases uniform sampling does quite well.
These data sets correspond most closely with those that have been studied 
previously in the machine learning literature, and for these data sets our 
results are in agreement with that prior work.
\item
Sparsifying RBF Kernels and/or choosing a smaller $\sigma$ parameter tends
to make these kernels less well-approximated by low-rank matrices and to 
have more heterogeneous leverage scores.
In general, these two properties need not be directly related---the spectrum is 
a property of eigenvalues, while the leverage scores are determined by the 
eigenvectors---but for the data we examined they are related, in that 
matrices with more slowly decaying spectra also often have more 
heterogeneous leverage scores.
\item
For Dense RBF Kernels with smaller $\sigma$ and Sparse RBF Kernels, leverage 
score sampling tends to do much better than other methods.
Interestingly, the Sparse RBF Kernels have many properties of very sparse 
Laplacian Kernels corresponding to relatively-unstructured informatics 
graphs, an observation which should be of interest for researchers who 
construct sparse graphs from data using, \emph{e.g.}, ``locally linear'' 
methods, to try to reconstruct hypothesized low-dimensional~manifolds.
\item
Reconstruction quality under leverage score sampling saturates, as a 
function of choosing more samples $\ell$; this is seen both for 
non-rank-restricted and rank-restricted situations.
As a consequence, there can be a tradeoff between leverage score sampling
or other methods being better, depending on the values of $\ell$ that are 
chosen.
\item
Although they are potentially ill-conditioned, non-rank-restricted approximations 
behave better in terms of reconstruction quality.
Rank-constrained approximations tend to have much more complicated behavior
as a function of increasing the numbe of samples $\ell$, including choppier
and non-monotonic behavior.
This is particularly severe for leverage score sampling, but it occurs with
other methods; and it suggests that other forms of regularization (other than
what is essentially a Tikhonov form of regularization for the 
rank-restricted cases) might be appropriate.
\end{itemize}
In general, \emph{all} of the sampling and projection methods we considered 
perform \emph{much} better on the SPSD matrices we considered than previous
worst-case bounds (\emph{e.g.},~\cite{dm_kernel_JRNL,KMT12,Gittens12_TR}) 
would suggest.
(That is, even the worst results correspond to single-digit approximation 
factors in relative scale.)
This observation is intriguing, because the motivation of leverage score
sampling (and, recall, that in this context random projections should be
viewed as performing uniform random sampling in a randomly-rotated basis 
where the leverage scores have been approximately 
uniformized~\cite{Mah-mat-rev_BOOK}) is very much tied to the Frobenius 
norm, and so there is no {\em a priori} reason to expect its good 
performance to extend to the spectral or trace norms.
Motivated by this, we revisit the question of proving improved 
worst-case theoretical bounds in Section~\ref{sxn:theory}.

Before describing these improved theoretical results, however, we address
in Section~\ref{sxn:approx-levmethods} running time questions.
After all, a na\"{i}ve implementation of sampling with exact leverage 
scores is slower than other methods (and much slower than uniform sampling).
As shown below, by using the recently-developed approximation algorithm 
of~\cite{DMMW12_JMLR}, not only does this approximation algorithm run in time 
comparable with random projections (for certain parameter settings), but
it leads to approximations that soften the strong bias that the exact 
leverage scores provide toward the best rank-$k$ approximation to the 
matrix, thereby leading to improved reconstruction results in many cases.

\subsection{Reconstruction Accuracy of Leverage Score Approximation Algorithms}
\label{sxn:approx-levmethods}

A na\"{i}ve view might assume that computing probabilities that permit 
leverage-based sampling requires an $O(n^3)$ computation of the full 
SVD, or at least the full computation of a partial SVD, and thus that it
would be much more expensive than recently-developed random projection 
methods.
Indeed, an ``exact'' computation of the leverage scores with a QR 
decomposition or truncated SVD takes roughly $\mathrm{O}(n^2k)$ time (and 
the running time results of Section~\ref{sxn:emp-reconstruction} actually 
used this na\"{i}ve procedure).
Recent work, however, has shown that relative-error approximations to all 
the statistical leverage scores can be computed more quickly than this 
exact algorithm~\cite{DMMW12_JMLR}. 
Here, we implement and evaluate a version of this algorithm, and we evaluate
it both in terms of running time and in terms of reconstruction quality on 
the diverse suite of real data matrices we considered above.
We note that ours is the first work to provide an empirical 
evaluation of an implementation of the leverage score approximation 
algorithms of~\cite{DMMW12_JMLR}, illustrating empirically the tradeoffs 
between cost and efficiency in a practical setting.

\subsubsection{Description of the Fast Approximation Algorithm of~\cite{DMMW12_JMLR}}

\begin{algorithm}[t]
\begin{framed}

\textbf{Input:} $\mat{A} \in \mathbb{R}^{n \times d}$ (with SVD
$\mat{A}=\mat{U}\boldsymbol{\Sigma}\mat{V}^\transp$), 
error parameter $\epsilon \in (0,1/2]$.

\vspace{0.1in}

\textbf{Output:} $\tilde{\ell_i}, i=1,\ldots,n$, approximations to the leverage
scores of $\mat{A}.$

\begin{enumerate}

\item Let $\boldsymbol{\Pi}_1 \in\R^{r_1\times n}$ be an SRFT with
\[
r_1 = \Omega(\epsilon^{-2} (\sqrt{d} + \sqrt{\ln n})^2 \ln d)
\]

\item Compute $\boldsymbol{\Pi}_1 \mat{A} \in\R^{r_1\times d}$ and its QR factorization
$\boldsymbol{\Pi}_1 \mat{A} = \mat{Q}\mat{R}$. 

\item  Let $\boldsymbol{\Pi}_2 \in \R^{d \times r_2}$ be a matrix of i.i.d. standard 
Gaussian random variables, where
\[
r_2 = \Omega\left(\epsilon^{-2}\ln n\right).
\]

\item Construct the product $\matOmega=\mat{A}\mat{R}^{-1}\boldsymbol{\Pi}_2.$
\item For $i =1, \ldots,n$ compute $\tilde{\ell}_i=\TNormS{\Omega_{(i)}}$.
\end{enumerate}

\end{framed}
\caption{Algorithm (originally Algorithm~1 in~\cite{DMMW12_JMLR}) for 
approximating the leverage scores~$\ell_i$ of an $n \times d$ matrix 
$\mat{A}$, where $n \gg d$, to within a multiplicative factor of 
$1 \pm \epsilon$. 
The running time of the algorithm is
$\mathrm{O}( nd\ln(\sqrt{d} + \sqrt{\ln n}) + nd \epsilon^{-2} \ln n + d^2 \epsilon^{-2} (\sqrt{d} + \sqrt{\ln n})^2 \ln d ).$
}
\label{alg:tall_levscore_approx}
\end{algorithm}

\begin{algorithm}[t]
\begin{framed}

\textbf{Input:} $\mat{A} \in \mathbb{R}^{n \times d},$ a rank parameter
$k,$ and an error parameter $\epsilon \in (0, 1/2].$ 

\vspace{0.1in}

\textbf{Output:} $\hat{\ell_i}, i =1, \ldots, n$, approximations to the 
leverage scores of $\mat{A}$ filtered through its dominant dimension-$k$ subspace.
\begin{enumerate}
\item Construct $\boldsymbol{\Pi}\in\mathbb{R}^{d\times 2k}$ with i.i.d. standard
Gaussian entries. 
\item Compute $\mat{B}=\left(\mat{A} \mat{A}^\transp\right)^q \mat{A} \boldsymbol{\Pi}
\in \mathbb{R}^{n \times 2k}$ with 
\[
 q \geq \left\lceil \frac{\ln\left(1+\sqrt{\frac{k}{k-1}}+\e\sqrt{\frac{2}{k}}
 \sqrt{\min\left\{n,d\right\}-k}\right)}{2\ln \left(1+\epsilon/10\right)-1/2}
 \right\rceil,
\]
\item Approximate the leverage scores of $\mat{B}$ by calling 
Algorithm~\ref{alg:tall_levscore_approx} with inputs $\mat{B}$ and $\epsilon$; 
let $\hat{\ell}_i$ for $i=1,\ldots,n$ be the outputs of 
Algorithm~\ref{alg:tall_levscore_approx}.
\end{enumerate}
\end{framed}
\caption{Algorithm (originally Algorithm~4 in~\cite{DMMW12_JMLR}) for 
approximating the leverage scores (relative to the best rank-$k$ 
approximation to $\mat{A}$) of a general $n \times d$ matrix $\mat{A}$ with 
those of a matrix that is close by in the spectral norm (or the Frobenius norm if $q=0$). 
This algorithm runs in time $\mathrm{O}(ndkq) + T_1$, where $T_1$ is the 
running time of Algorithm~\ref{alg:tall_levscore_approx}.
}
\label{alg:spectral_levscore_approx}
\end{algorithm}

\begin{algorithm}[t]
\begin{framed}

\textbf{Input:} $\mat{A} \in \mathbb{R}^{n \times d}$, a rank parameter $k$, and an iteration parameter $q$.

\vspace{0.1in}

\textbf{Output:} $\hat{\ell_i}, i \in=1,\ldots,n,$ approximations to the 
leverage scores of $\mat{A}$ filtered through its dominant dimension-$k$ subspace.
\begin{enumerate}
\item Construct an SRHT matrix $\boldsymbol{\Pi}\in\mathbb{R}^{d\times r},$ where 
\[
r \geq \left\lceil 36 \epsilon^{-2} [ \sqrt{k} + 
\sqrt{8 \ln(k d)} ]^2 \ln(k) \right\rceil. 
\]
%
\item Compute $\mat{B}=\left(\mat{A} \mat{A}^\transp\right)^q \mat{A} \boldsymbol{\Pi}
\in \mathbb{R}^{n \times r},$ where $q \geq 0$ is an integer.
\item Return the exact leverage scores of $\mat{B}.$
%
\end{enumerate}
\end{framed}
\caption{Algorithm for approximating the leverage scores (relative to the 
best rank-$k$ approximation to $\mat{A}$) of a general $n \times d$ matrix 
$\mat{A}$ with those of a matrix that is close by in the spectral norm. 
This is a modified version of Algorithm~\ref{alg:spectral_levscore_approx}, 
 in which the random projection is implemented with an SRFT rather than a 
Gaussian random matrix, and where the number of ``iterations'' $q$ is 
prespecified.
This algorithm runs in time $\mathrm{O}(n d \ln r + n d r q + nr^2)$ since $\mat{A} \boldsymbol{\Pi}$ can be computed in time $\mathrm{O}(nd \ln r).$
}
\label{alg:frob_levscore_approx}
\end{algorithm}

Algorithm~\ref{alg:tall_levscore_approx} (which originally appeared as 
Algorithm~1 in~\cite{DMMW12_JMLR}) takes as input an arbitrary $n \times d$ 
matrix $\mat{A}$, where $n \gg d$, and it returns as output a $1\pm\epsilon$ 
approximation to \emph{all} of the statistical leverage scores of the
input matrix. The original algorithm of~\cite{DMMW12_JMLR} uses a subsampled 
Hadamard transform and requires $r_1$ to be somewhat larger than what we 
state in Algorithm~\ref{alg:tall_levscore_approx}. That an SRFT
with a smaller value of $r_1$ can be used instead is a consequence of 
the fact that Lemma~3 in~\cite{DMMW12_JMLR} is also satisfied by an SRFT matrix 
with the given $r_1;$ this is established in~\cite{tropp2011improved,BG12_TR}.
 
The running time of this algorithm, given in the caption of the algorithm, is 
roughly $O(nd \ln d)$ when $d = \Omega(\ln n)$.
Thus Algorithm~\ref{alg:tall_levscore_approx} generates relative-error
approximations to the leverage scores of a tall and skinny matrix $\mat{A}$ 
in time $o(nd^2)$, rather than the $\mathrm{O}(nd^2)$ time that would be required to 
compute a QR decomposition or a thin SVD of the $n \times d$ matrix $\mat{A}$.
The basic idea behind Algorithm~\ref{alg:tall_levscore_approx} is
as follows.
If we had a QR decomposition of $\mat{A}$, then we could postmultiply 
$\mat{A}$ by the inverse of the ``$R$'' matrix to obtain an orthogonal 
matrix spanning the column space of $\mat{A}$; and from this $n \times d$ 
orthogonal matrix, we could read off the leverage scores from the Euclidean 
norms of the rows.
Of course, computing the QR decomposition would require $O(nd^2)$ time.
To get around this, Algorithm~\ref{alg:tall_levscore_approx} premultiplies
$\mat{A}$ by a structured random projection $\boldsymbol{\Pi}_1$, computes a 
QR decomposition of $\boldsymbol{\Pi}_1 \mat{A}$, and postmultiplies $\mat{A}$ by 
$\mat{R}^{-1}$, \emph{i.e.}, the inverse of the ``$R$'' matrix from the QR decomposition of $\boldsymbol{\Pi}_1 \mat{A}$.
Since $\boldsymbol{\Pi}_1$ is an SRFT, premultiplying by it takes roughly 
$O(nd \ln d)$ time.  In addition, note that $\boldsymbol{\Pi}_1 \mat{A}$ needs to be post multiplied 
by a second random projection in order to compute all of the leverage 
scores in the allotted time; see~\cite{DMMW12_JMLR} for details.
This algorithm is simpler than the algorithm in which we are primarily 
interested that is applicable to square SPSD matrices, but we start with 
it since it illustrates the basic ideas of how our main algorithm works and 
since our main algorithm calls it as a subroutine.
We note, however, that this algorithm is directly useful for approximating 
the leverage scores of Linear Kernel matrices $\mat{A}=\mat{X}\mat{X}^\transp$, 
when $\mat{X}$ is a tall and skinny matrix. 

Consider, next, Algorithm~\ref{alg:spectral_levscore_approx} (which 
originally appeared as Algorithm~4 in~\cite{DMMW12_JMLR}), which takes as 
input an \emph{arbitrary} $n \times d$ matrix $\mat{A}$ and a rank parameter 
$k$, and returns as output a $1\pm\epsilon$ approximation to
\emph{all} of the statistical leverage scores (relative to the best rank-$k$ 
approximation) of the input.
An important technical point is that the problem of computing the leverage
scores of a matrix relative to a low-dimensional space is ill-posed,
essentially because the spectral gap between the $k^{th}$ and the $(k+1)^{st}$ 
eigenvalues can be small, and thus
Algorithm~\ref{alg:spectral_levscore_approx} actually computes approximations
to the leverage scores of a matrix that is near to $\mat{A}$ in the spectral 
norm (or the Frobenius norm if $q=0$).
See~\cite{DMMW12_JMLR} for details.
Basically, this algorithm uses Gaussian sampling to find a matrix close 
to $\mat{A}$ in the Frobenius norm or spectral norm, and then it approximates
the leverage scores of this matrix by using 
Algorithm~\ref{alg:tall_levscore_approx} on the smaller, very rectangular
matrix $\mat{B}$.
When $\mat{A}$ is square, as in our applications,
Algorithm~\ref{alg:spectral_levscore_approx} is typically more costly than 
direct 
computation of the leverage scores, 
at least for dense matrices (but it does have the advantage that the 
number of iterations is bounded, independent of properties of the matrix, 
which is not true for typical iterative methods to compute low-rank 
approximations). 

Of greater practical interest is Algorithm~\ref{alg:frob_levscore_approx}, which is a 
modification of Algorithm~\ref{alg:spectral_levscore_approx} in which the 
Gaussian random projection is replaced with an SRFT.
That is, Algorithm~\ref{alg:frob_levscore_approx} uses an SRFT projection to 
find a matrix close by to $\mat{A}$ in the Frobenius norm or spectral norm 
(depending on the value of $q$), and then it exactly computes the 
leverage scores of this matrix.  
This improves the running time to 
$\mathrm{O}( n^2 \ln(\sqrt{k} + \sqrt{\ln n}) + n^2 (\sqrt{k} + \sqrt{\ln n})^2 \ln(k) q + n (\sqrt{k} + \sqrt{\ln n})^4 \ln^2(k) ),$
which is $o(n^2 k)$ when $q = 0$.
Thus an important point for Algorithm~\ref{alg:frob_levscore_approx}
(as well as for Algorithm~\ref{alg:spectral_levscore_approx}) is the 
parameter $q$ which describes the number of iterations.
For $q=0$ iterations, we get an inexpensive Frobenius norm approximation; while for 
higher $q$, we get better spectral norm approximations that are more 
expensive\footnote{Observe that since $\matA$ is rectangular in Algorithms~\ref{alg:spectral_levscore_approx}
and~\ref{alg:frob_levscore_approx}, we approximate the leverage scores of 
$\matA$ with those of $\matB = (\matA\matA^\transp)^q\matA \boldsymbol{\Pi};$ in particular
the case $q =0$ corresponds to taking $\matB = \matA \boldsymbol{\Pi}.$ By way of contrast, when we use the power method to construct 
sketches of an SPSD matrix, we take $\matC = \matA^q \matS,$ so the case $q =1$ corresponds to $\matC = \matA\matS.$ }
This flexibility is of interest, as one may want
to approximate the actual leverage scores accurately or one may simply want to find crude approximations
useful for obtaining SPSD sketches with low reconstruction error.

Finally, note that although choosing the number of iterations $q$ as we did 
in Algorithm~\ref{alg:spectral_levscore_approx} is convenient for worst-case 
analysis, as a practical implementational matter it is easier either to 
choose $q$ based on spectral gap information revealed during the running of
the algorithm or to prespecify $q$ to be a small integer, \emph{e.g.}, $2$ or 
$3$, before the algorithm runs.
Both of these have an interpretation of accelerating the rate of decay of 
the spectrum with a power iteration, but they behave somewhat differently 
due to the different stopping conditions.
Below, we consider both variants.

\subsubsection{Running Time Comparisons}
\label{sxn:emp-running}

\begin{figure}[p]
 \centering
 \subfigure[GR, $k = 20$]{\includegraphics[width=1.6in, keepaspectratio=true]{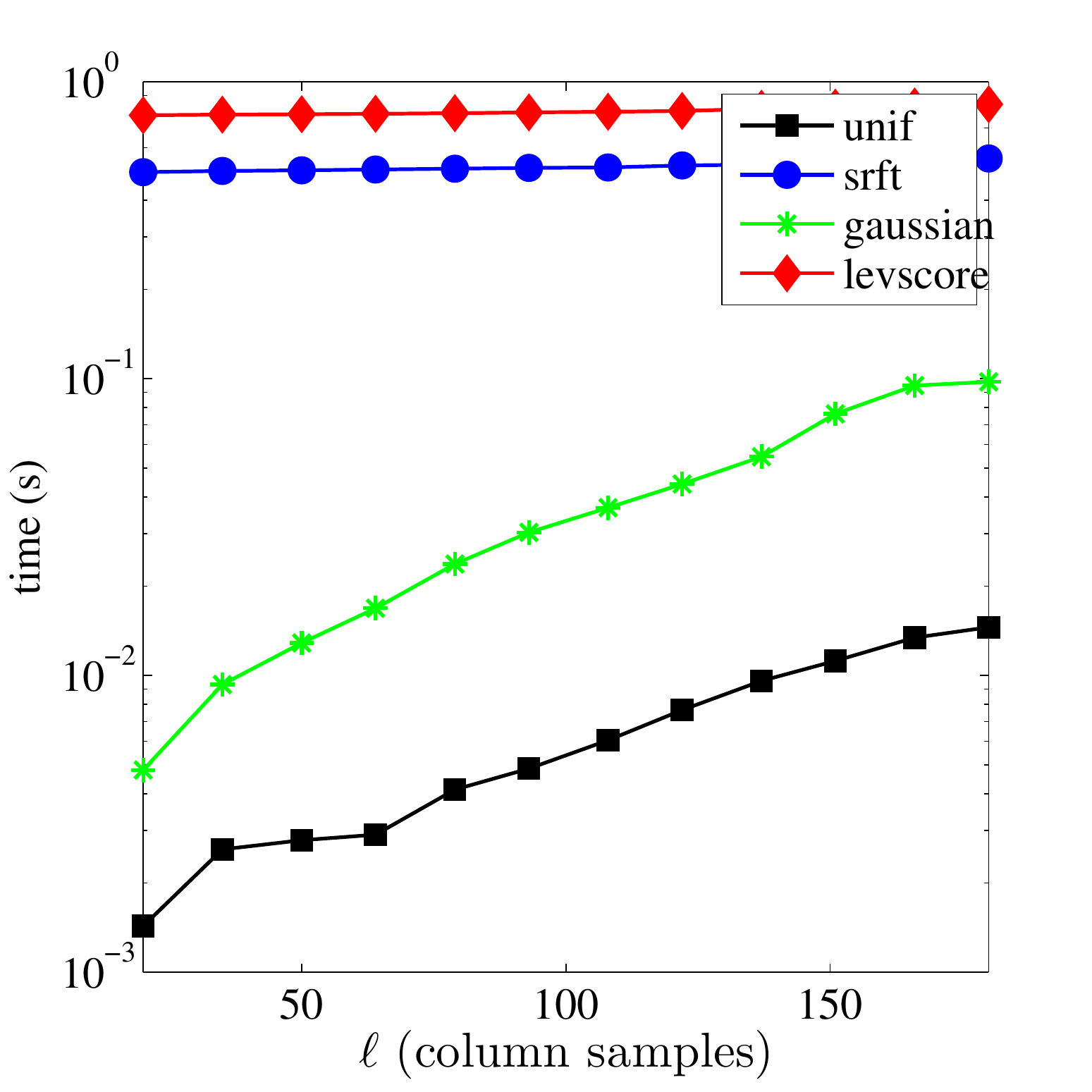}}%
 \subfigure[GR, $k = 60$]{\includegraphics[width=1.6in, keepaspectratio=true]{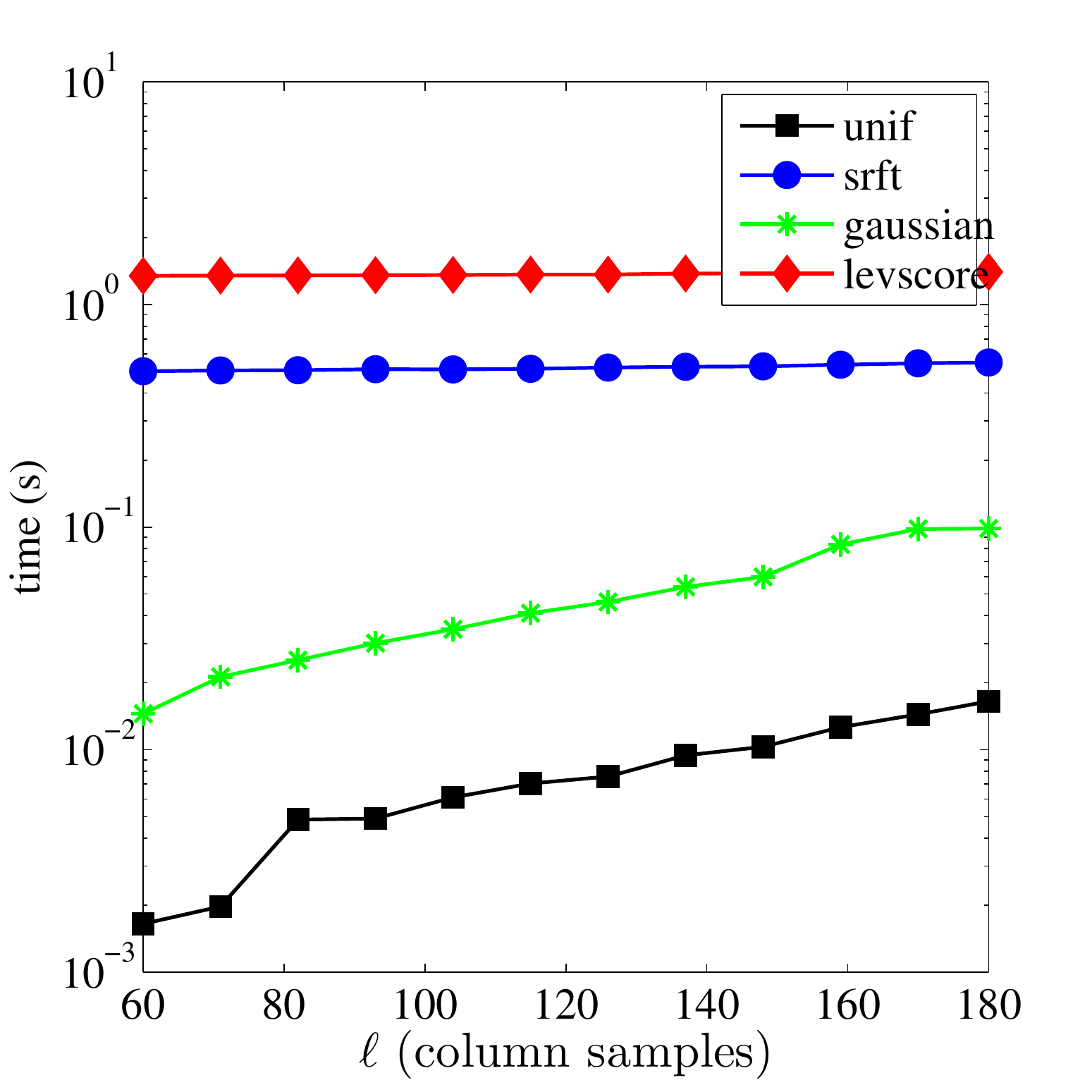}}%
 \subfigure[HEP, $k = 20$]{\includegraphics[width=1.6in, keepaspectratio=true]{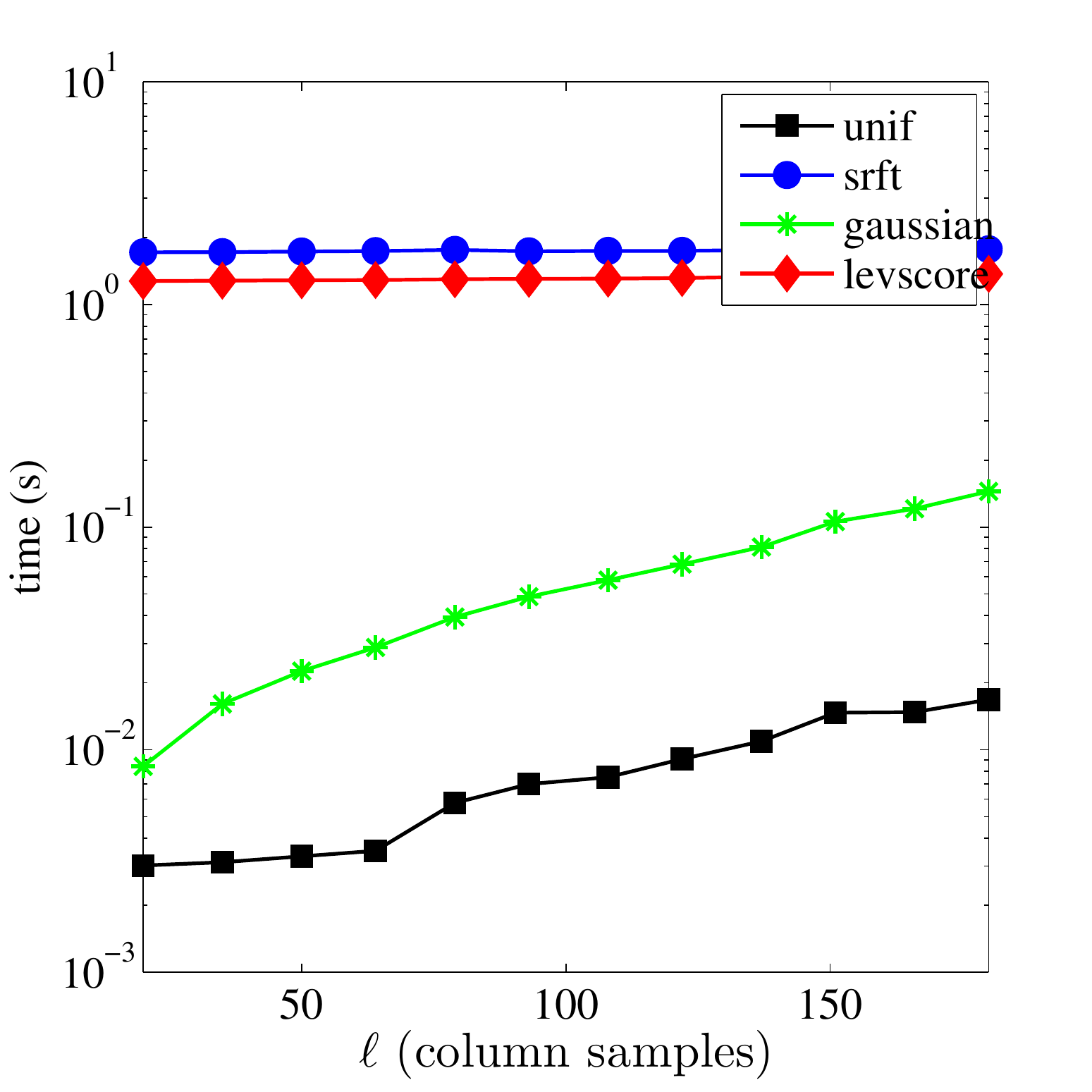}}%
 \subfigure[HEP, $k = 60$]{\includegraphics[width=1.6in, keepaspectratio=true]{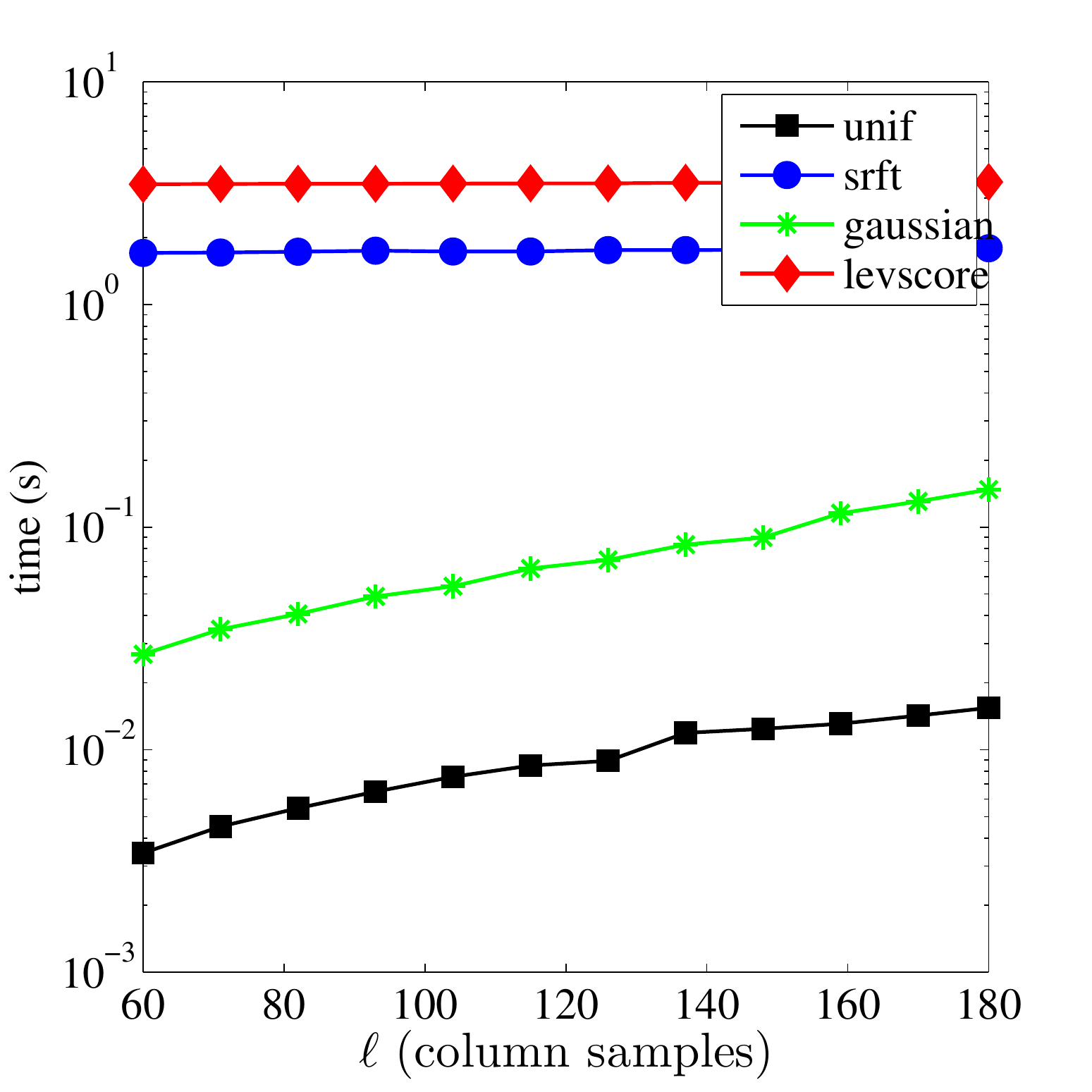}}
 \\%
 \subfigure[Dexter, $k = 8$]{\includegraphics[width=1.6in, keepaspectratio=true]{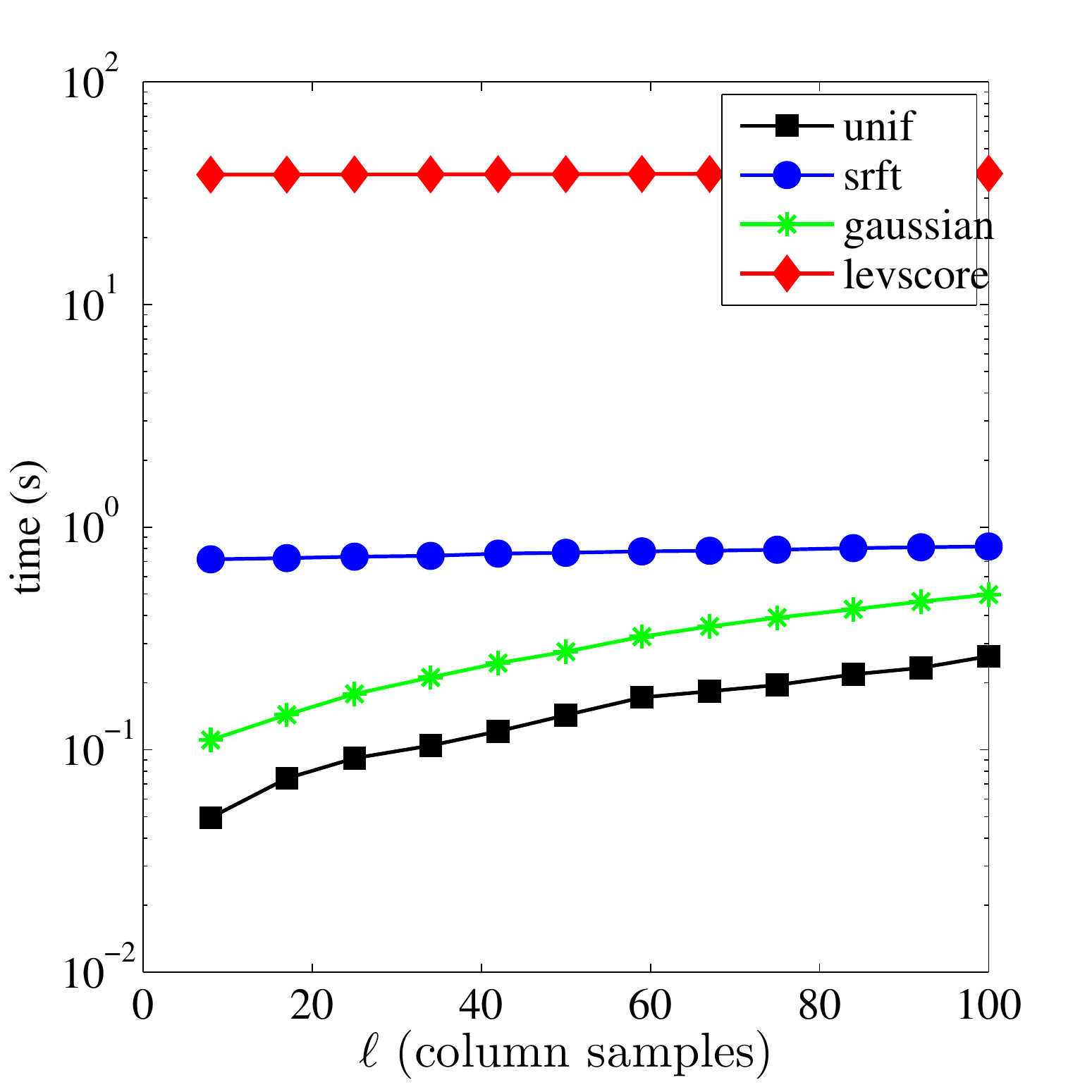}}%
 \subfigure[Protein, $k = 10$]{\includegraphics[width=1.6in, keepaspectratio=true]{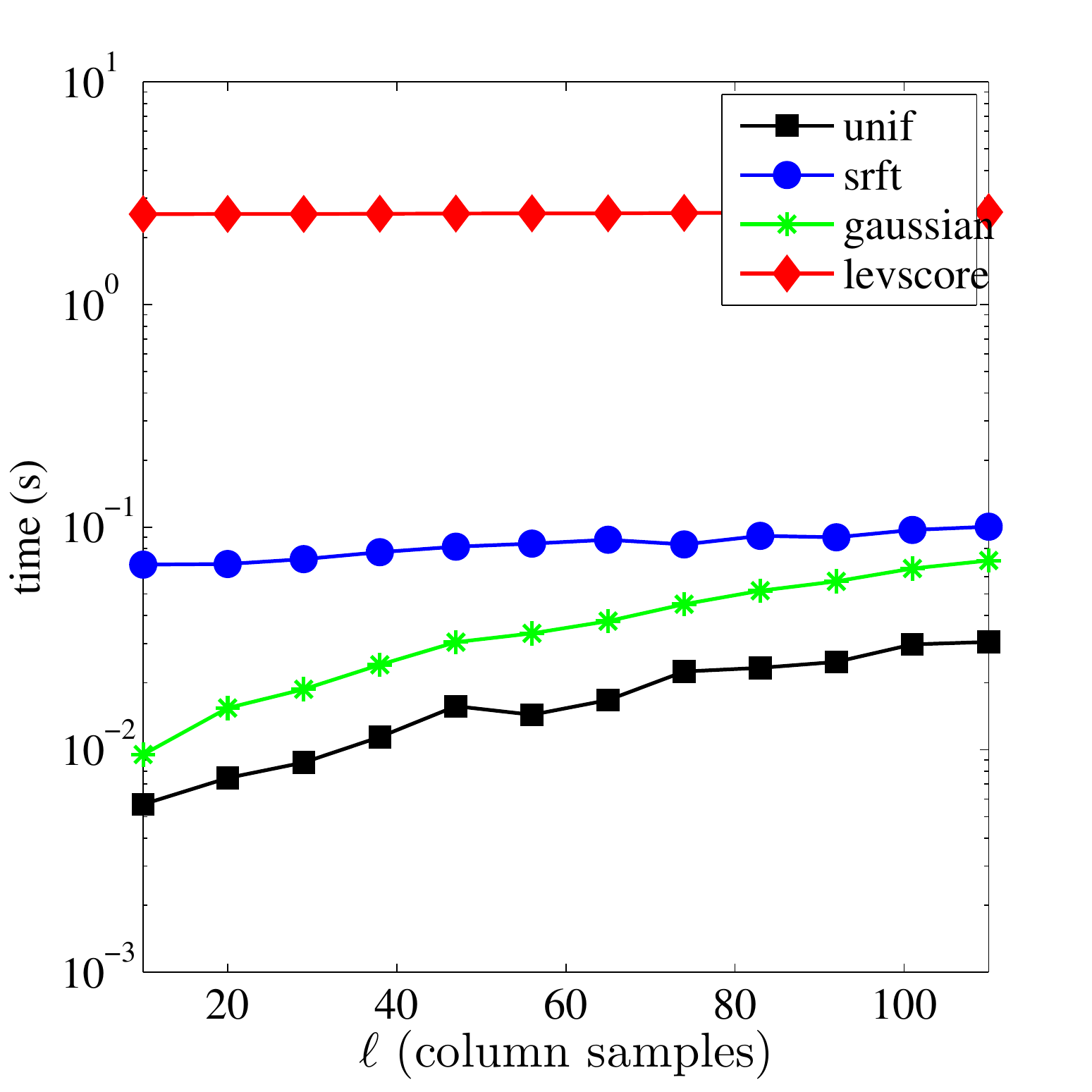}}%
 \subfigure[SNPs, $k = 5$]{\includegraphics[width=1.6in, keepaspectratio=true]{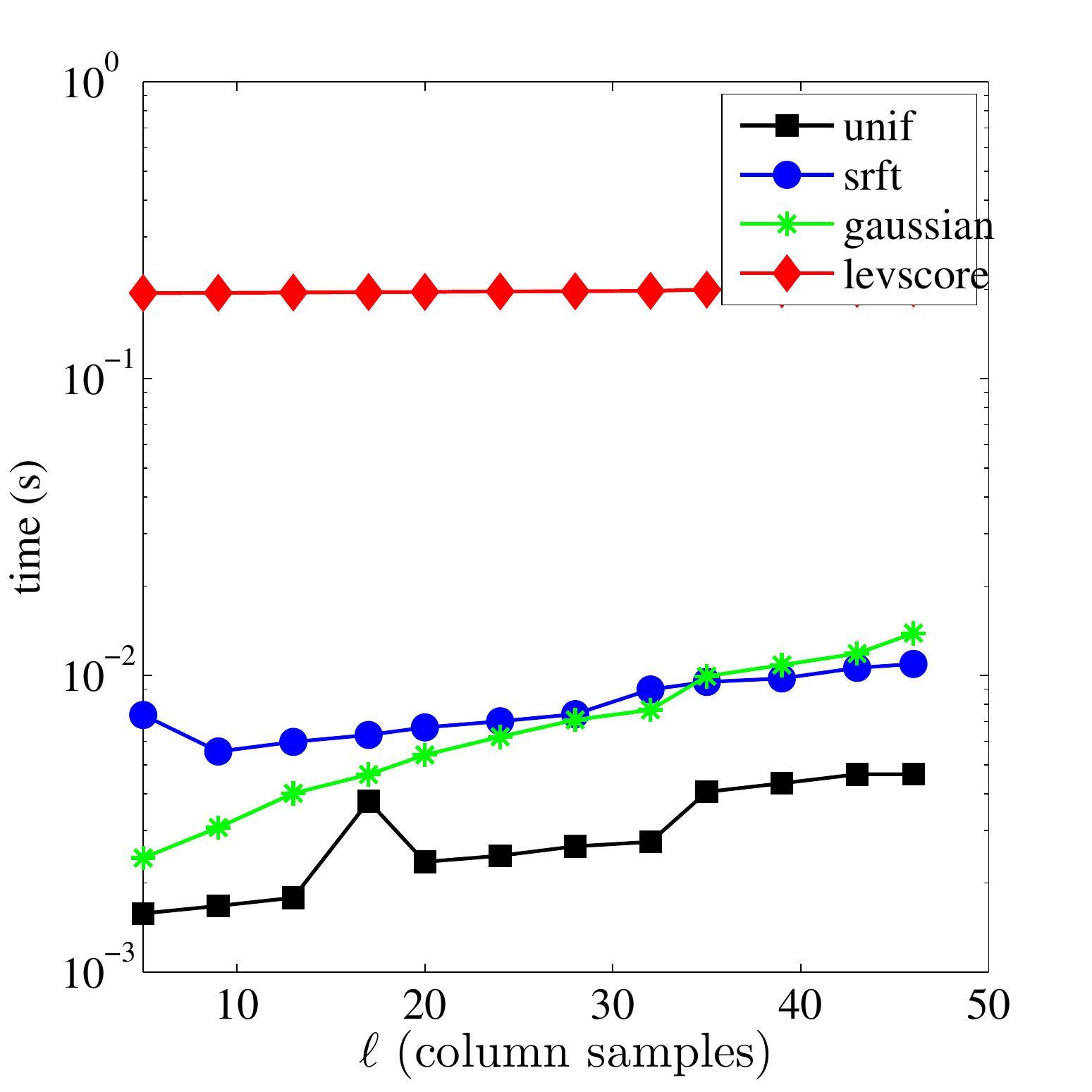}}%
 \subfigure[Gisette, $k = 12$]{\includegraphics[width=1.6in, keepaspectratio=true]{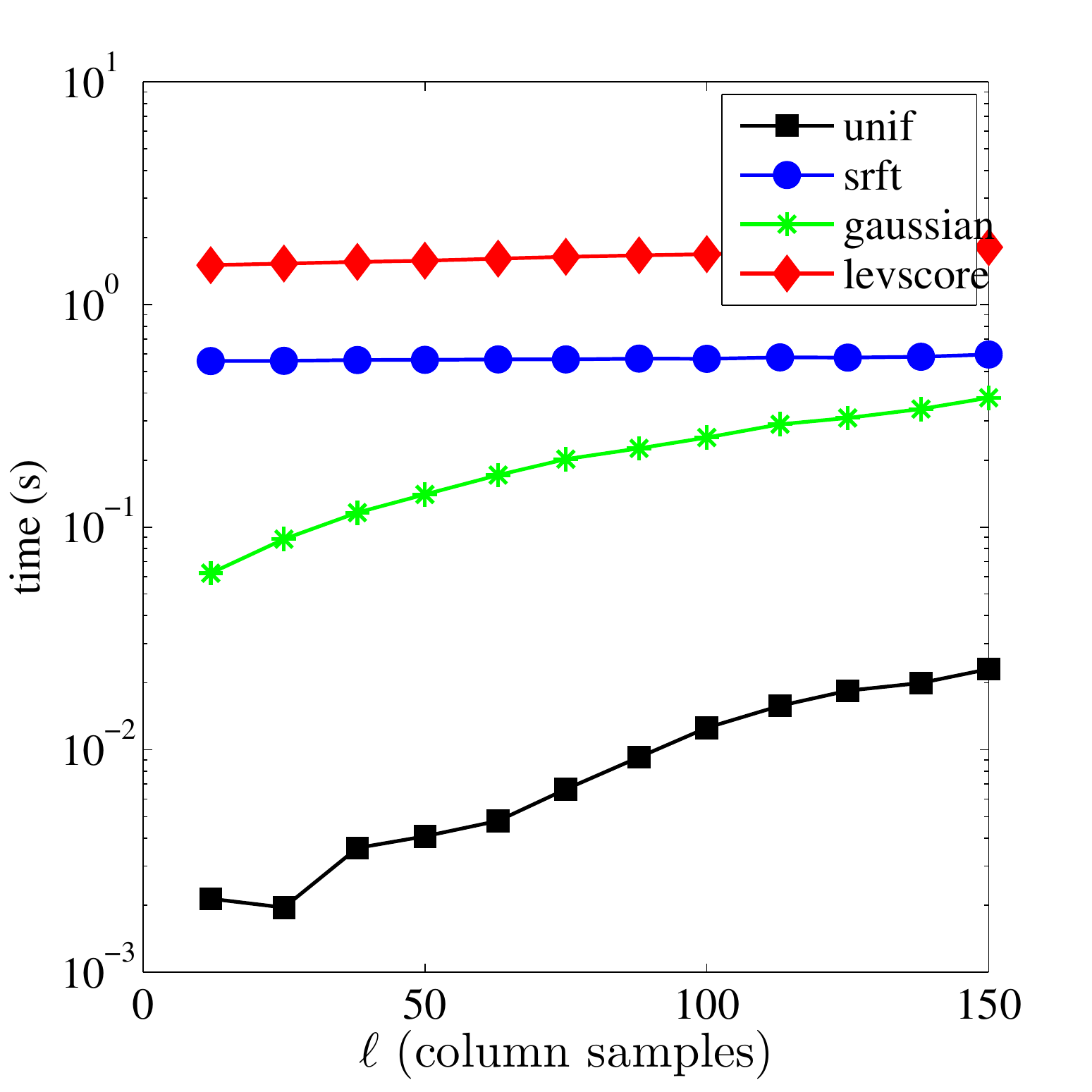}}%
 \\%
 \subfigure[AbaloneD, $\sigma = .15, k = 20$]{\includegraphics[width=1.6in, keepaspectratio=true]{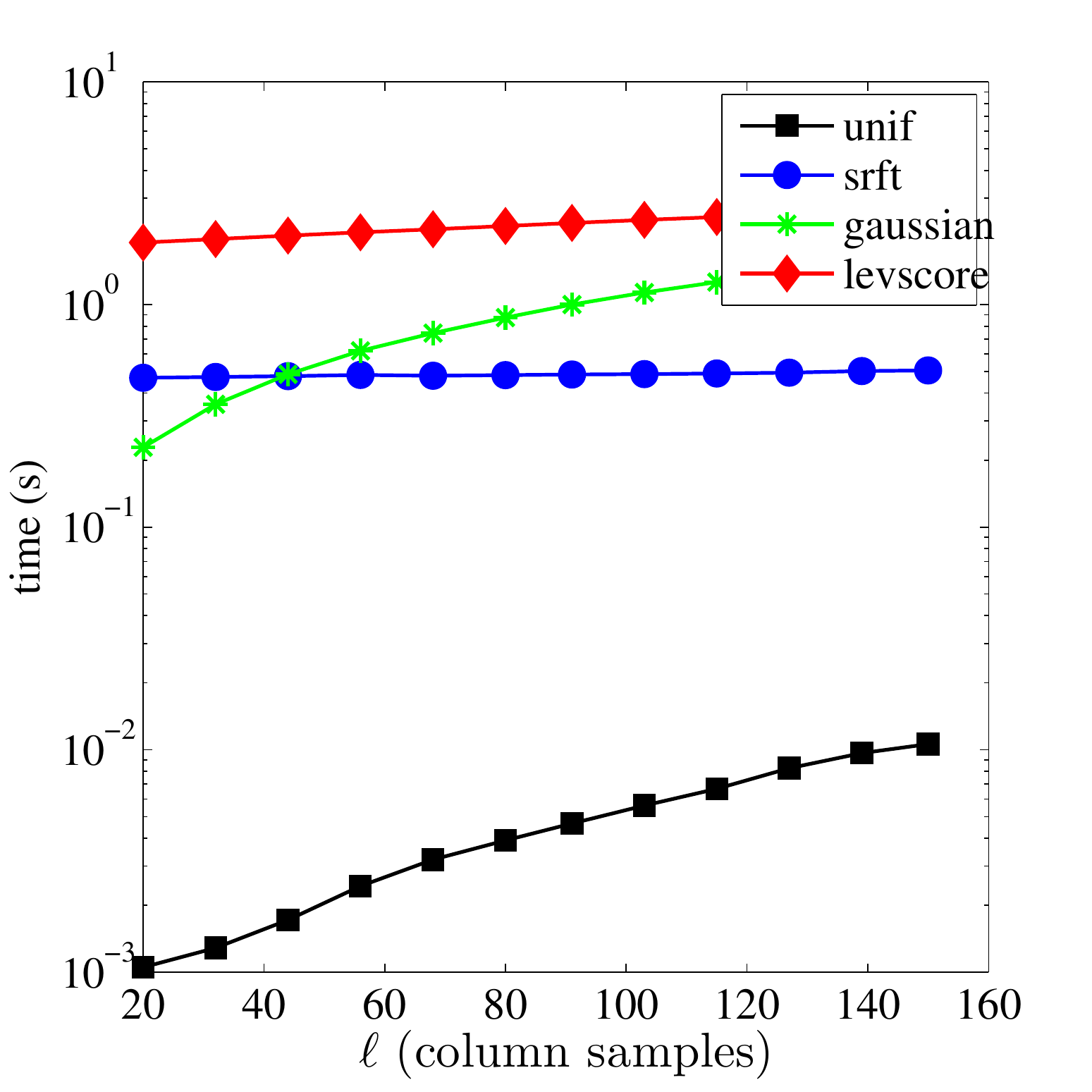}}%
 \subfigure[AbaloneD, $\sigma = 1, k = 20$]{\includegraphics[width=1.6in, keepaspectratio=true]{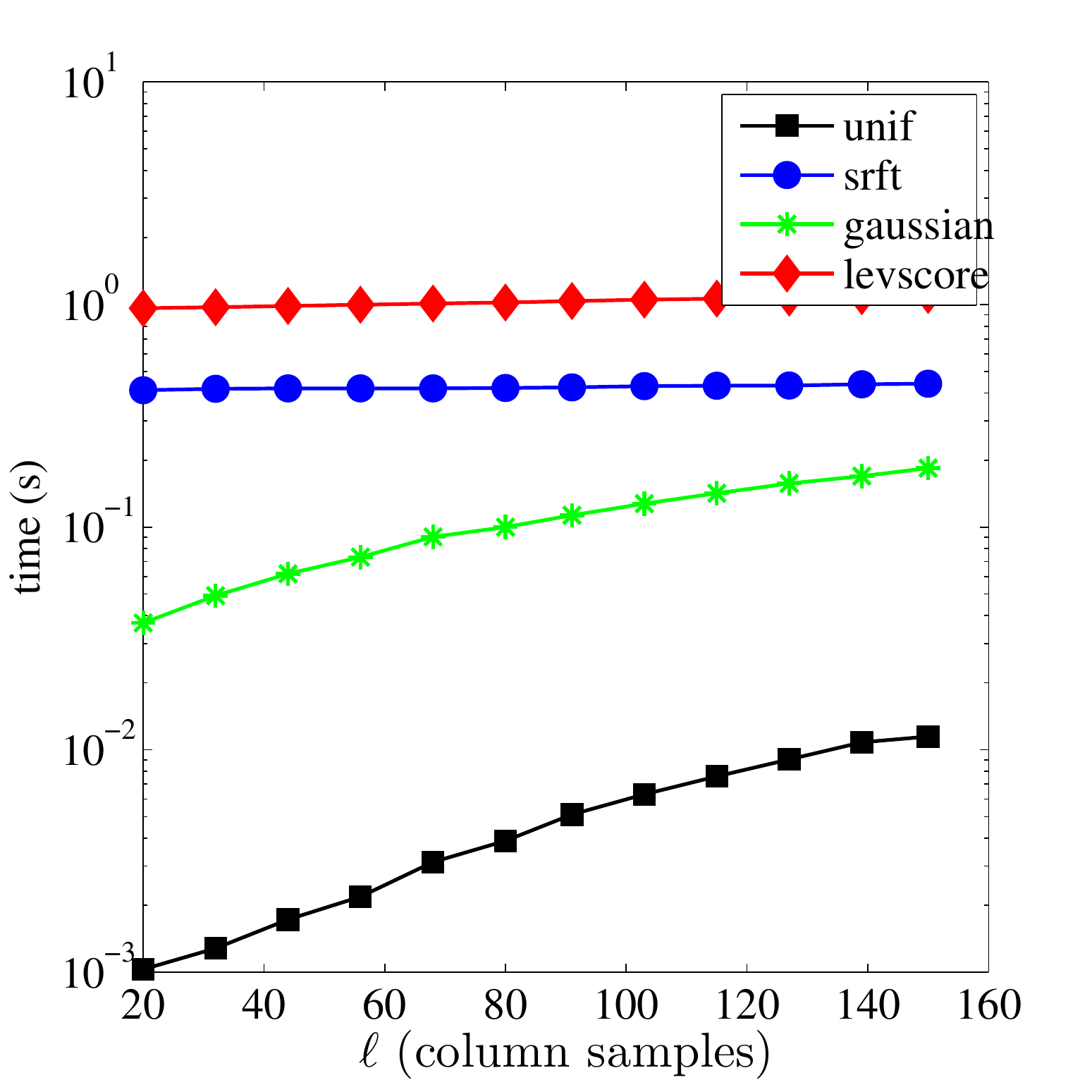}}%
 \subfigure[WineD, $\sigma = 1, k = 20$]{\includegraphics[width=1.6in, keepaspectratio=true]{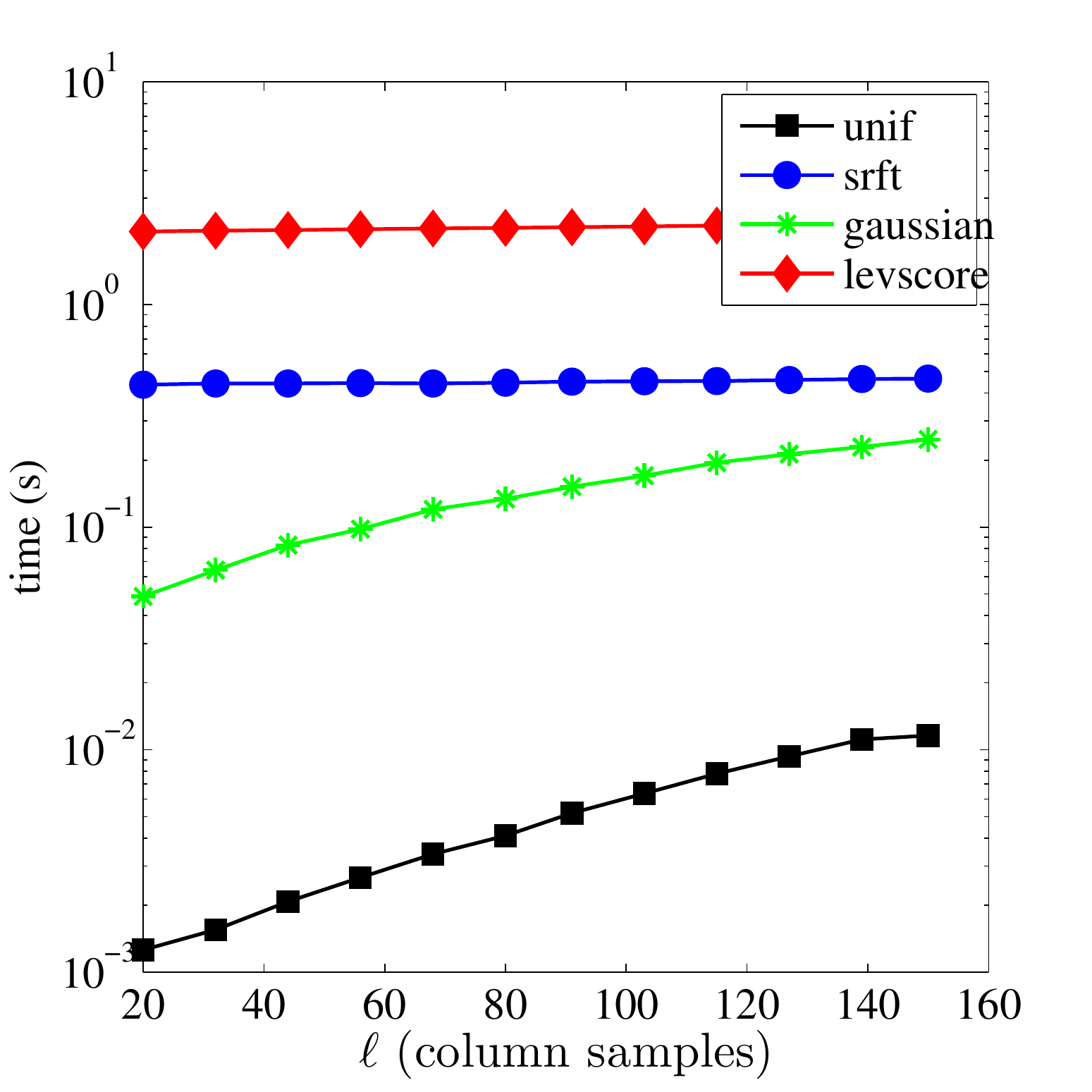}}%
 \subfigure[WineD, $\sigma = 2.1, k = 20$]{\includegraphics[width=1.6in, keepaspectratio=true]{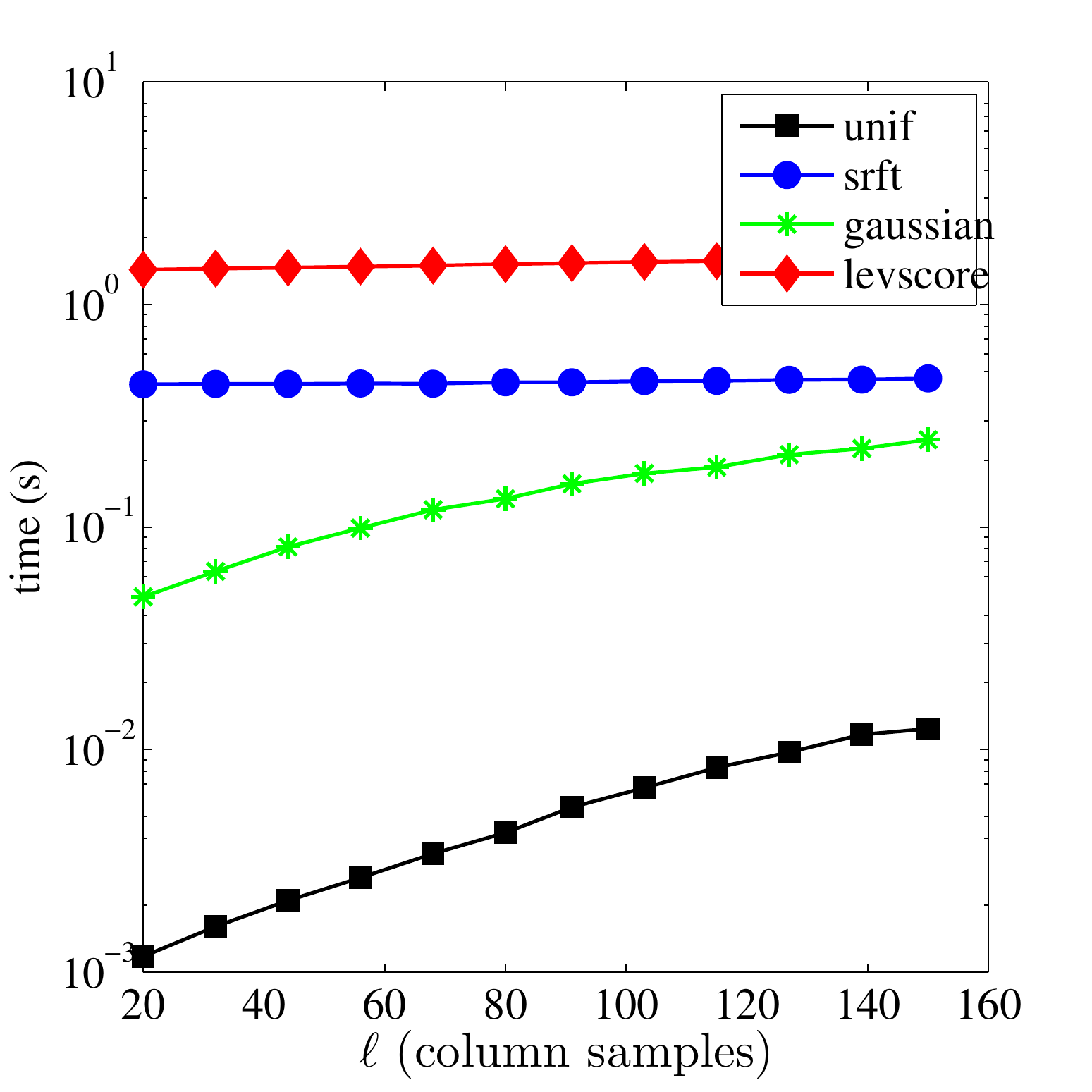}}%
 \\%
 \subfigure[AbaloneS, $\sigma = .15, k = 20$]{\includegraphics[width=1.6in, keepaspectratio=true]{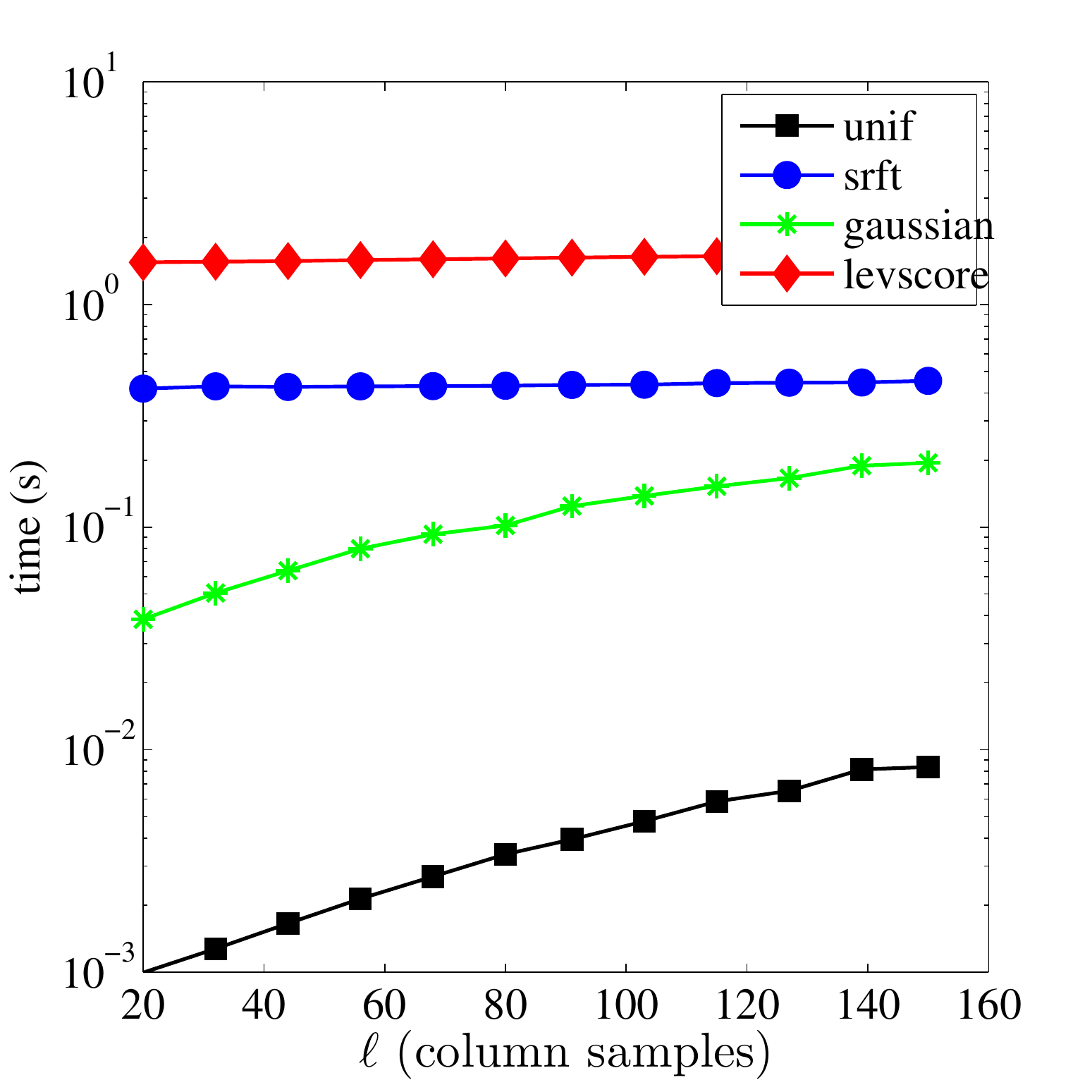}}%
 \subfigure[AbaloneS, $\sigma = 1, k = 20$]{\includegraphics[width=1.6in, keepaspectratio=true]{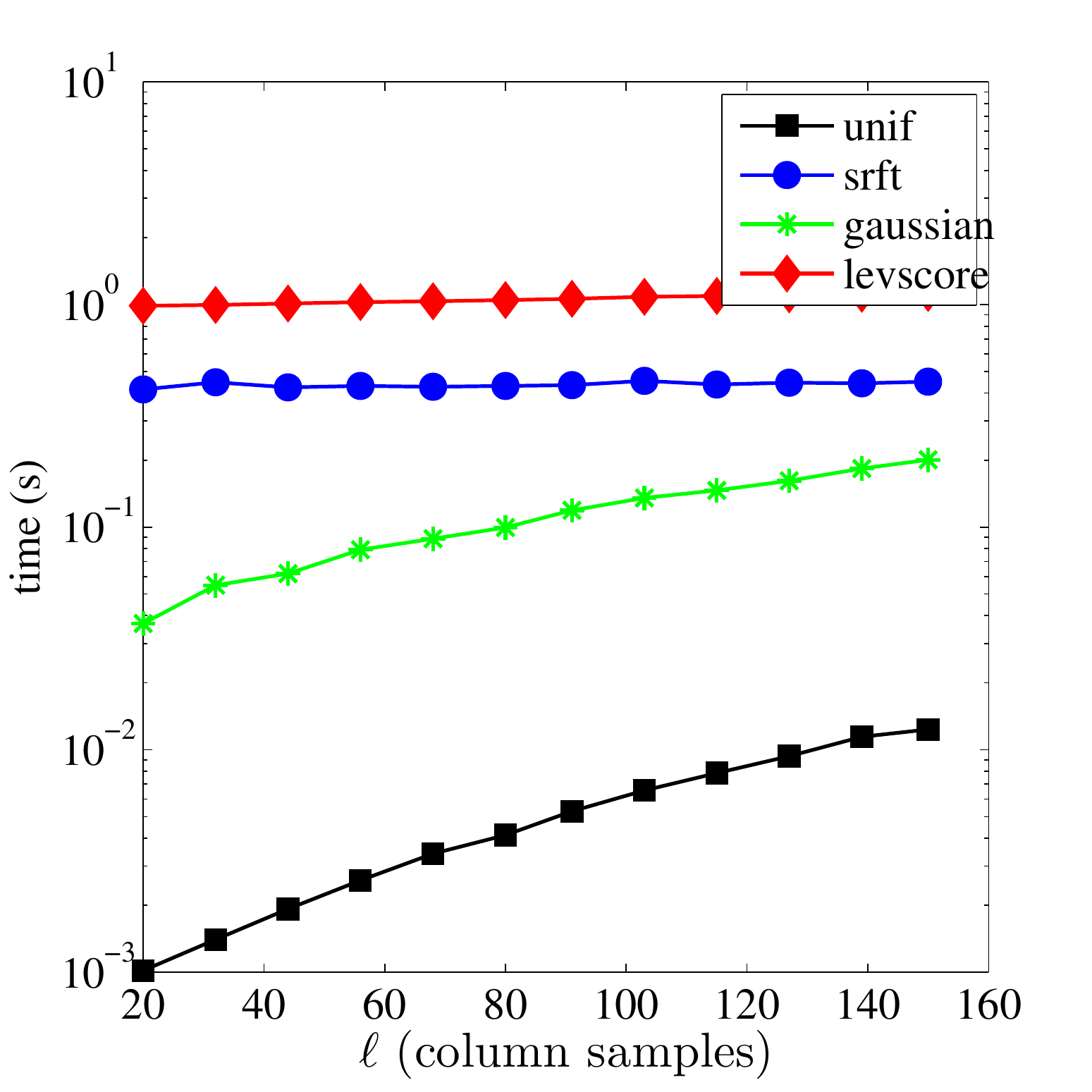}}%
 \subfigure[WineS, $\sigma = 1, k = 20$]{\includegraphics[width=1.6in, keepaspectratio=true]{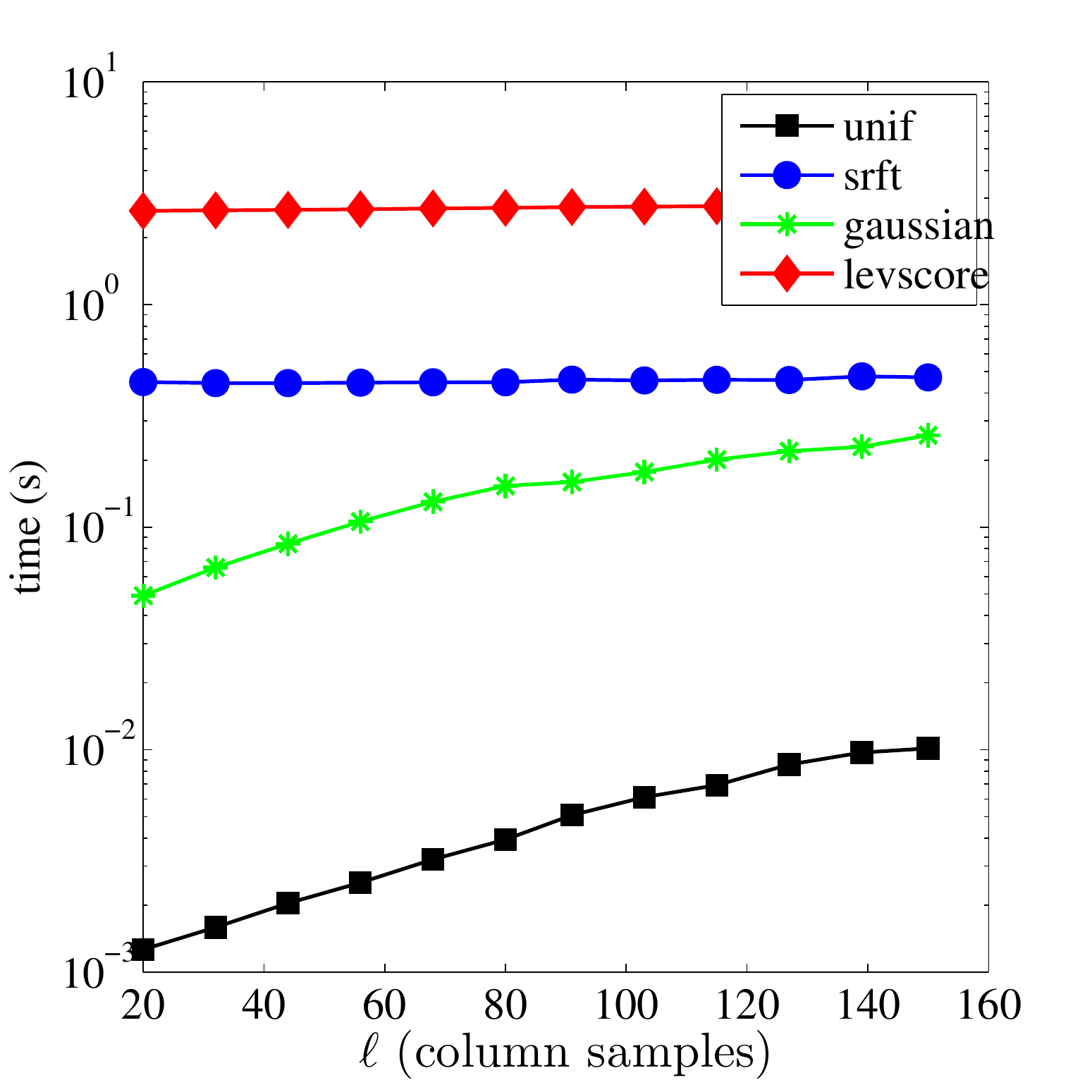}}%
 \subfigure[WineS, $\sigma = 2.1, k = 20$]{\includegraphics[width=1.6in, keepaspectratio=true]{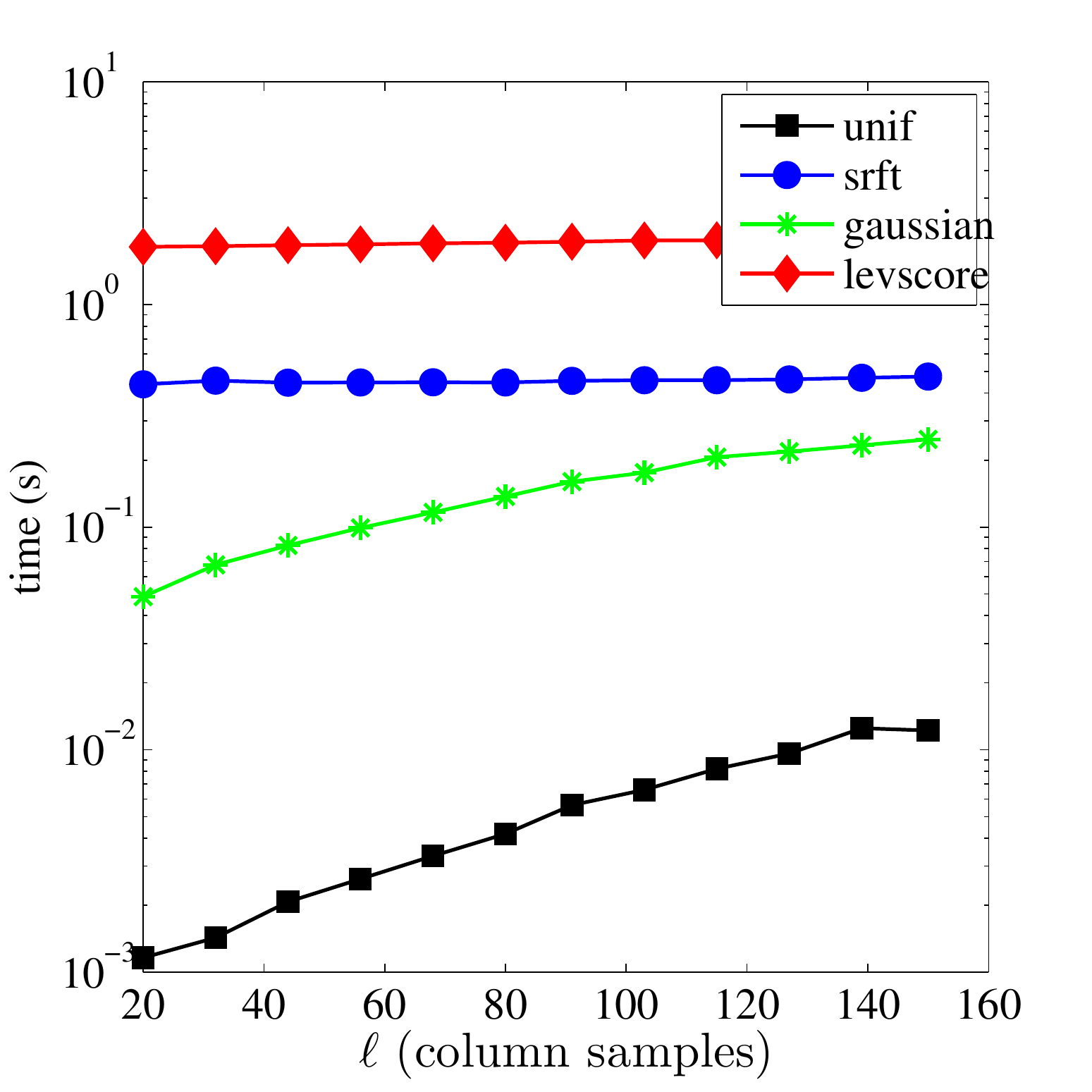}}%
 \caption{The times required to compute the (non-rank-restricted) SPSD sketches, as a function of the number of columns samples 
 $\ell$ for several data sets and two choices of the rank parameter $k$.  }%
 \label{fig:exact-computation-times}
\end{figure}

\begin{figure}[p]
 \centering
 \subfigure[GR, $k = 20$]{\includegraphics[width=1.6in, keepaspectratio=true]{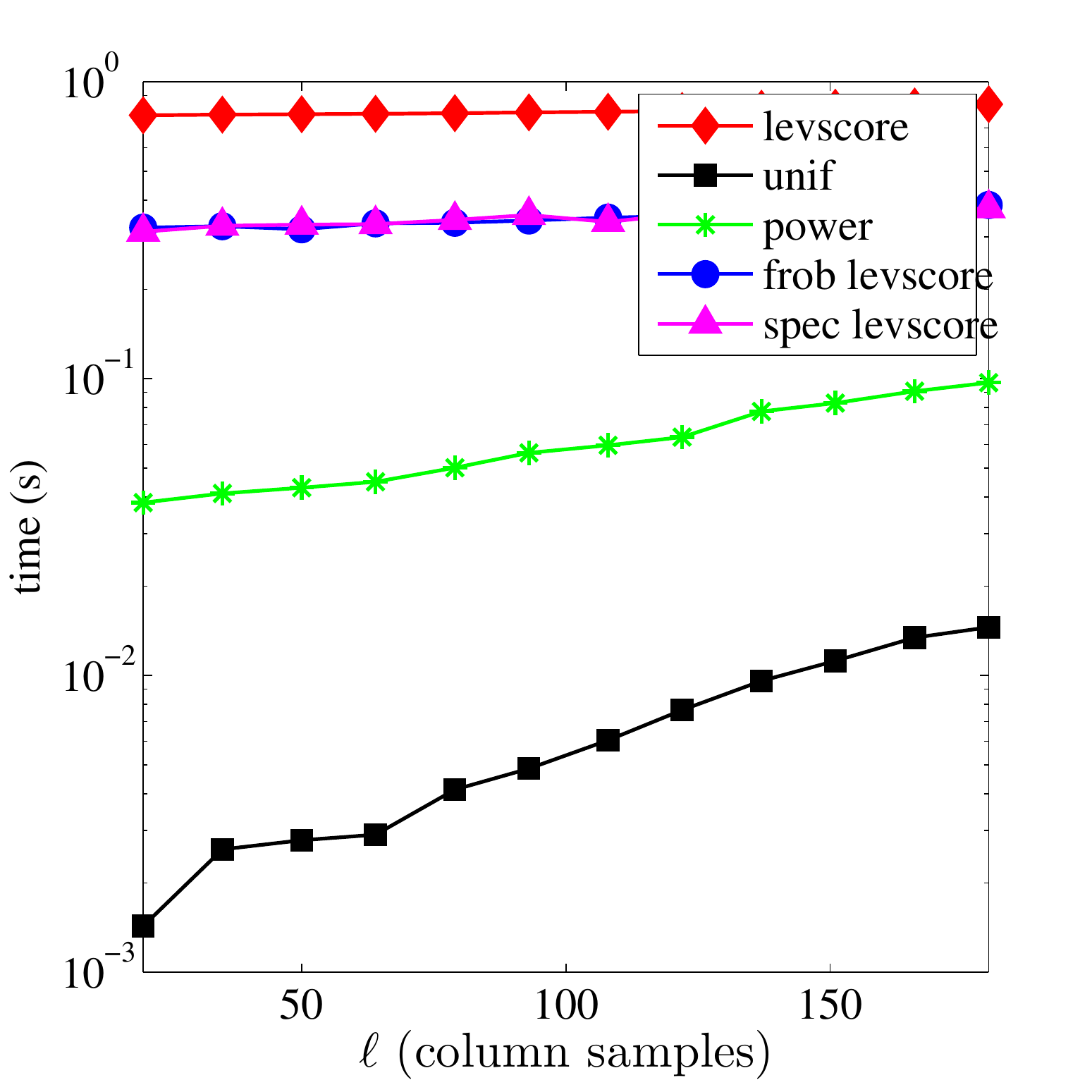}}%
 \subfigure[GR, $k = 60$]{\includegraphics[width=1.6in, keepaspectratio=true]{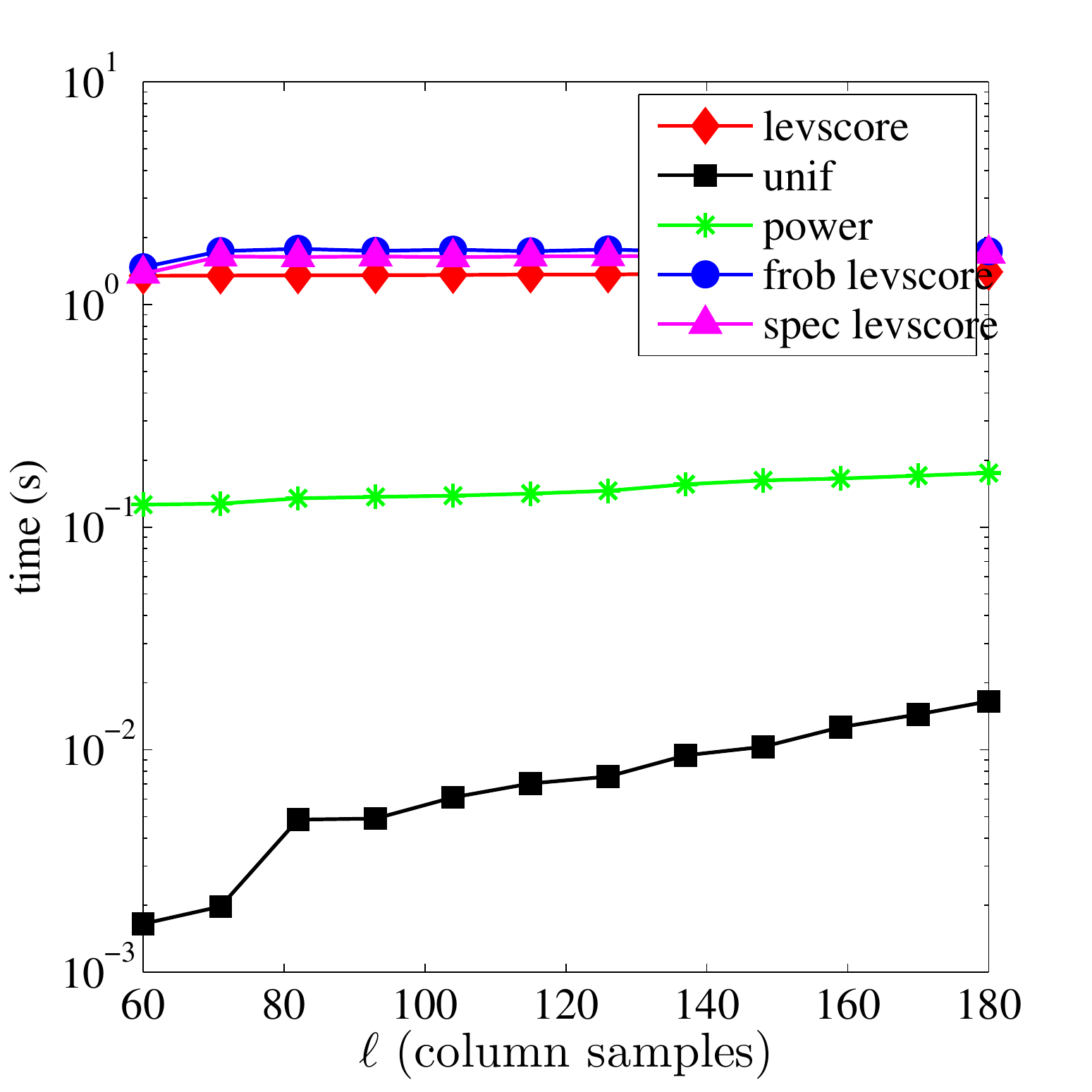}}%
 \subfigure[HEP, $k = 20$]{\includegraphics[width=1.6in, keepaspectratio=true]{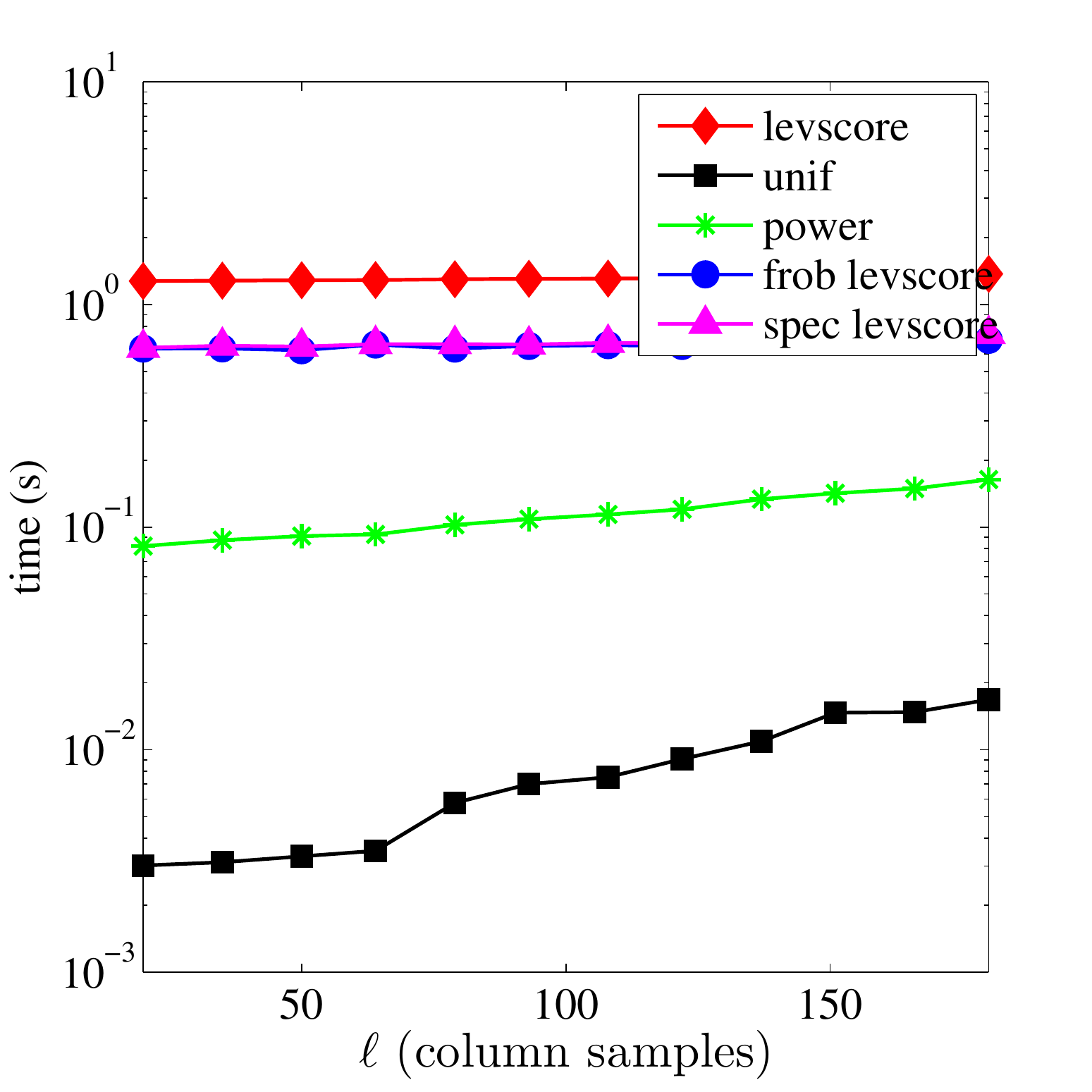}}%
 \subfigure[HEP, $k = 60$]{\includegraphics[width=1.6in, keepaspectratio=true]{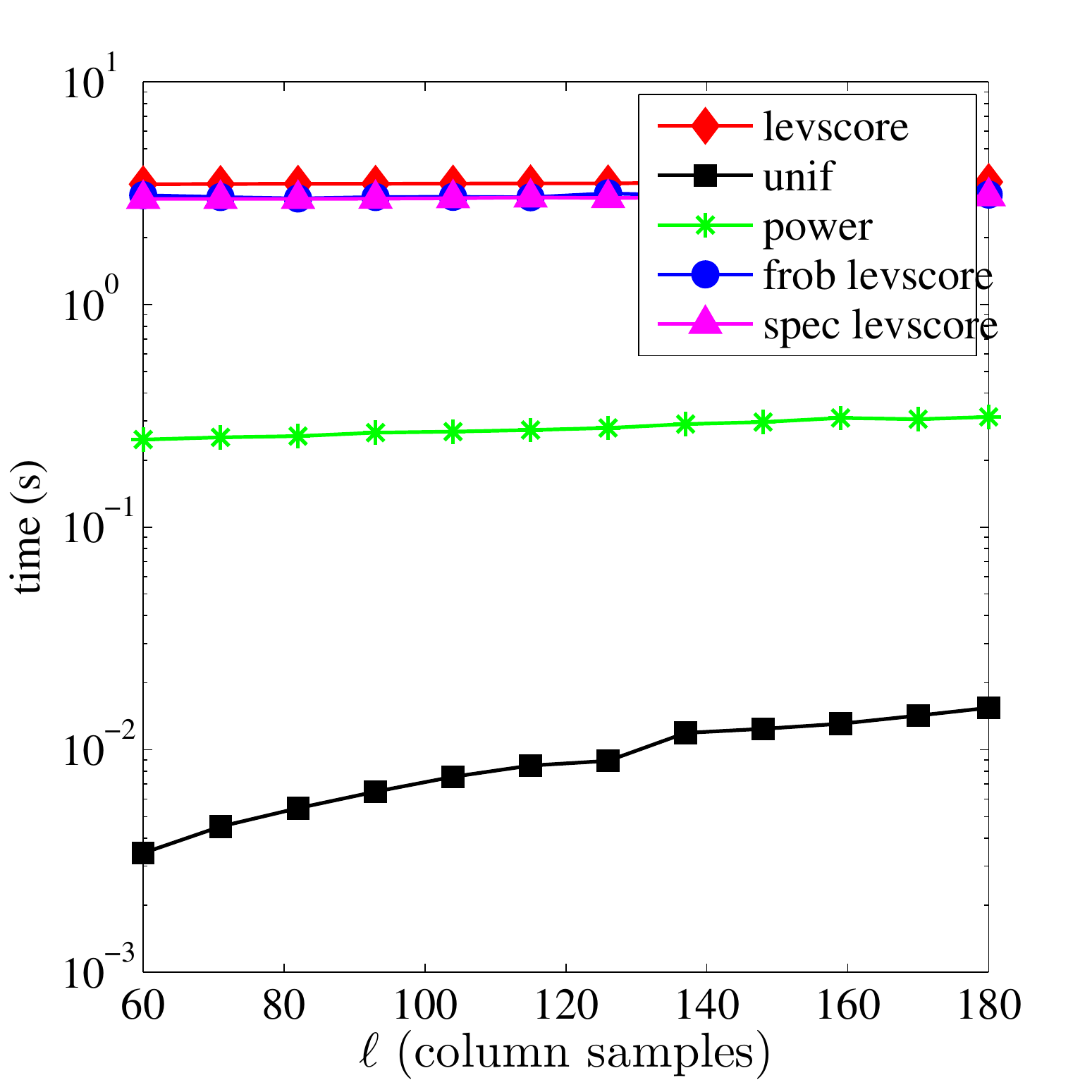}}\\%
 \subfigure[Dexter, $k = 8$]{\includegraphics[width=1.6in, keepaspectratio=true]{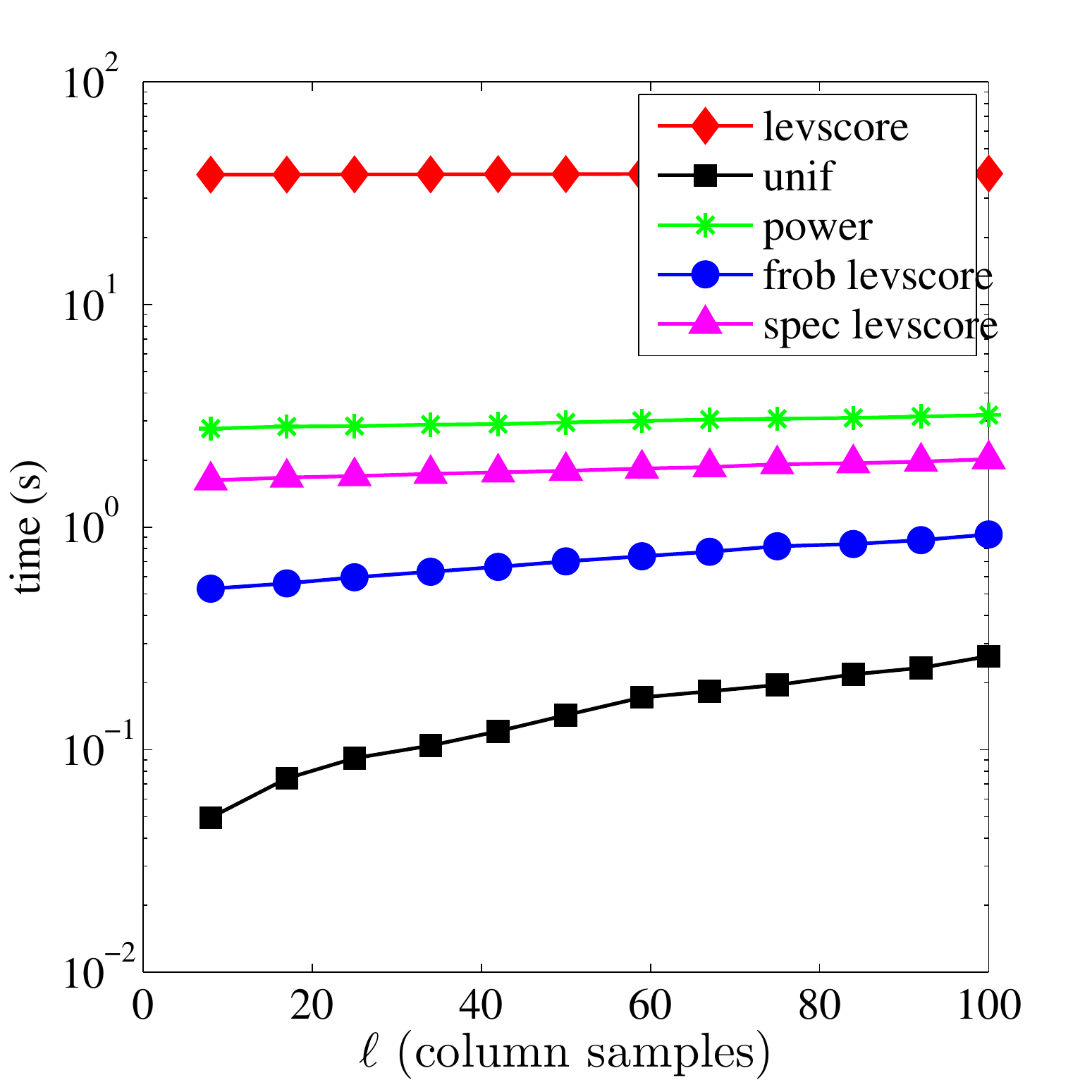}}%
 \subfigure[Protein, $k = 10$]{\includegraphics[width=1.6in, keepaspectratio=true]{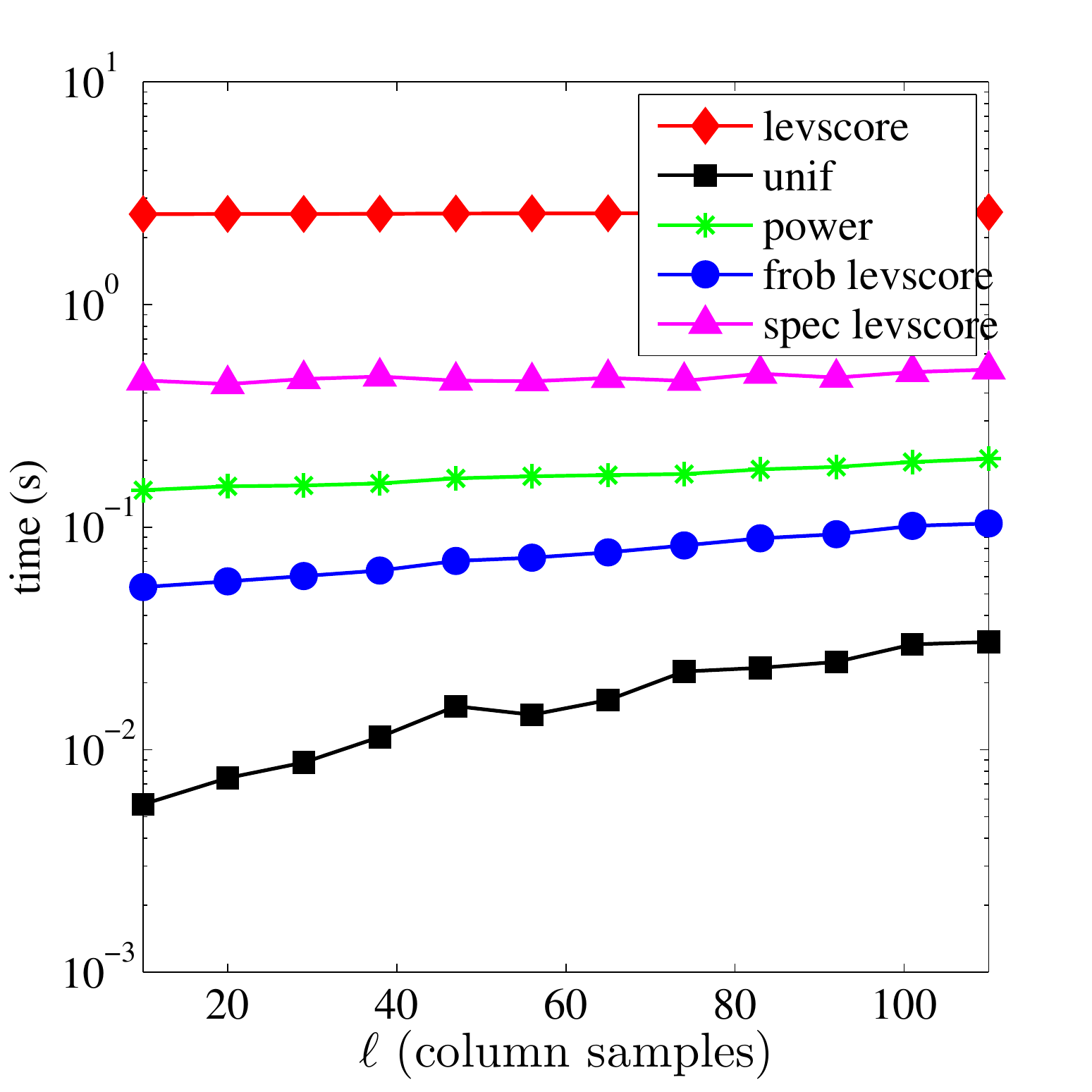}}%
 \subfigure[SNPs, $k = 5$]{\includegraphics[width=1.6in, keepaspectratio=true]{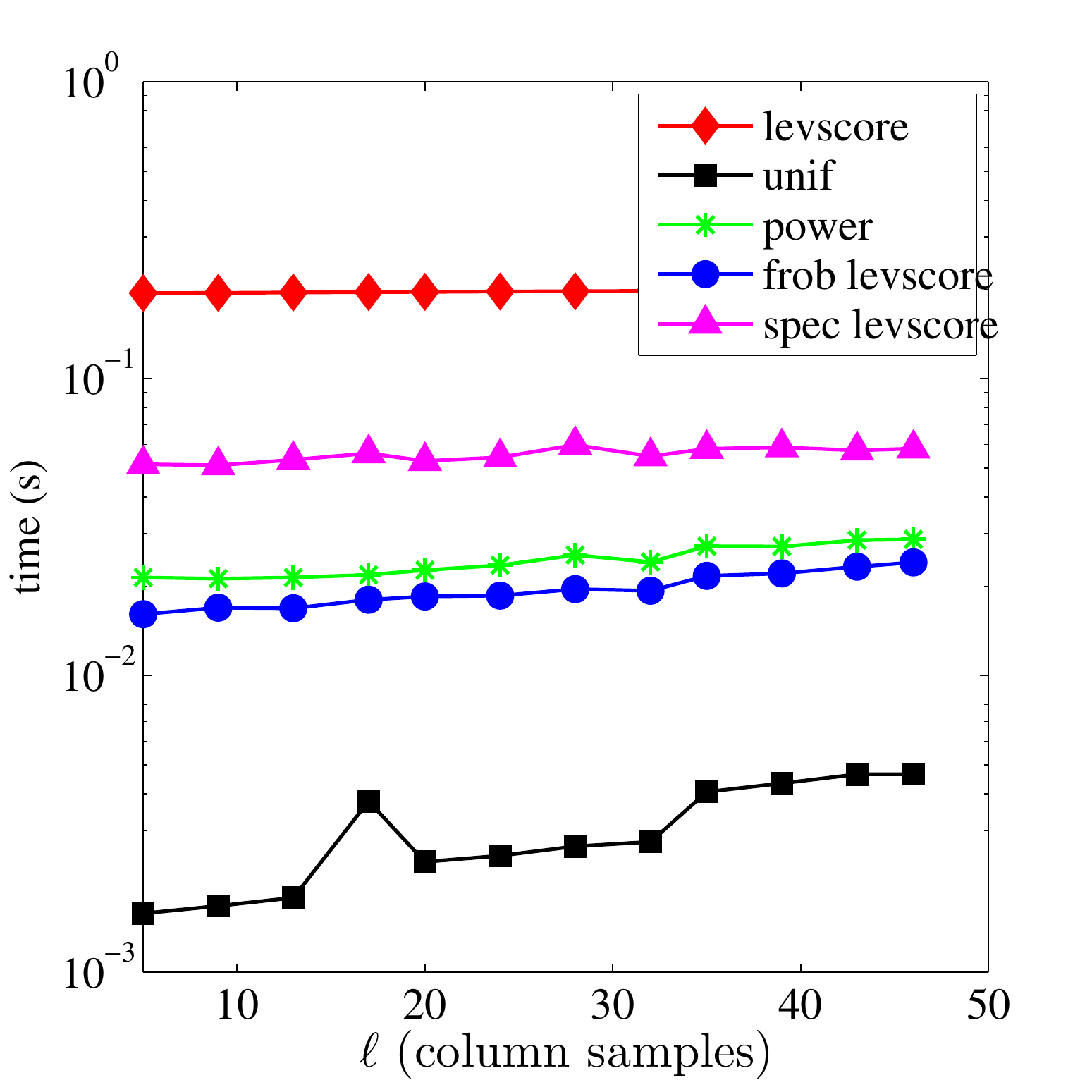}}%
 \subfigure[Gisette, $k = 12$]{\includegraphics[width=1.6in, keepaspectratio=true]{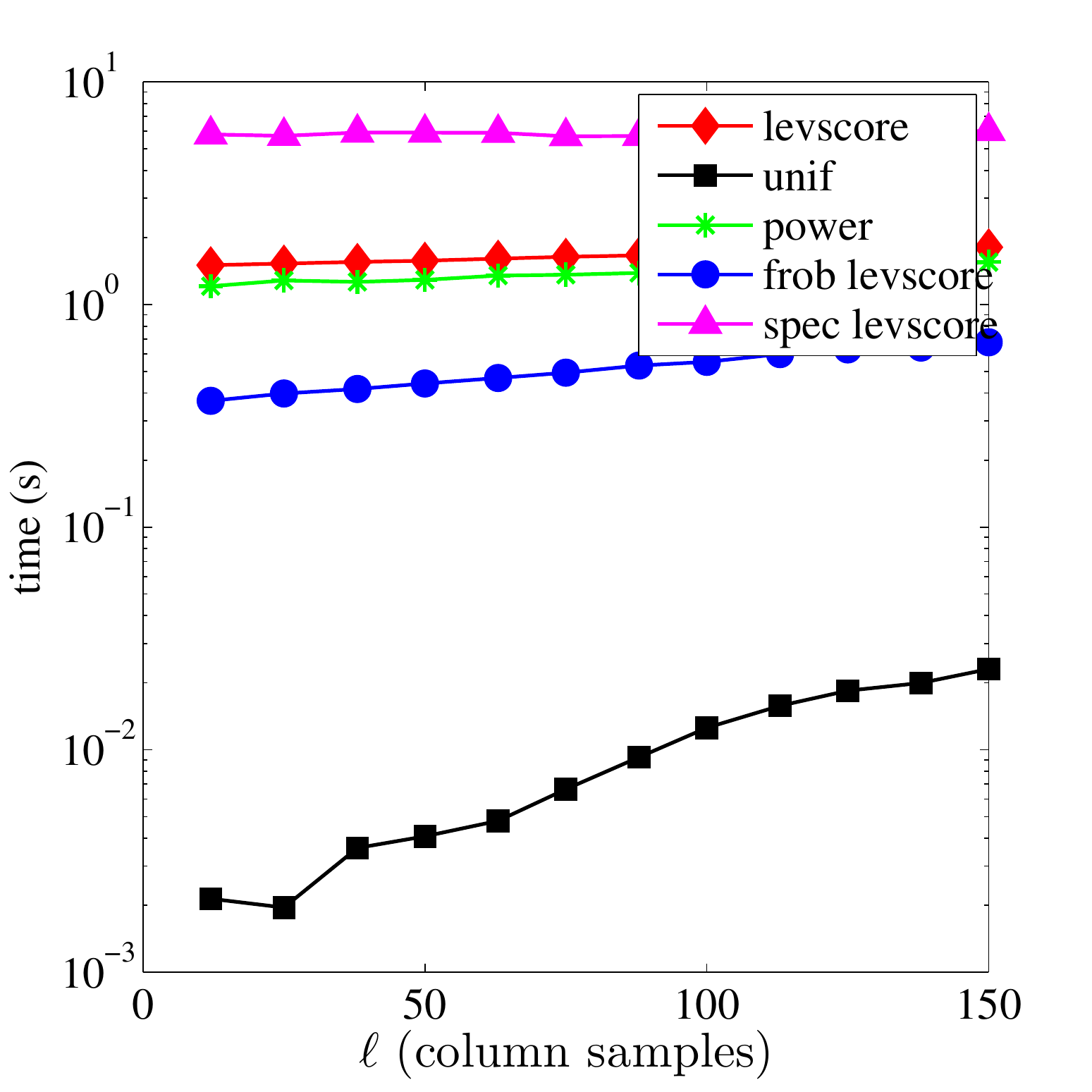}}\\%
 \subfigure[AbaloneD, $\sigma = .15, k = 20$]{\includegraphics[width=1.6in, keepaspectratio=true]{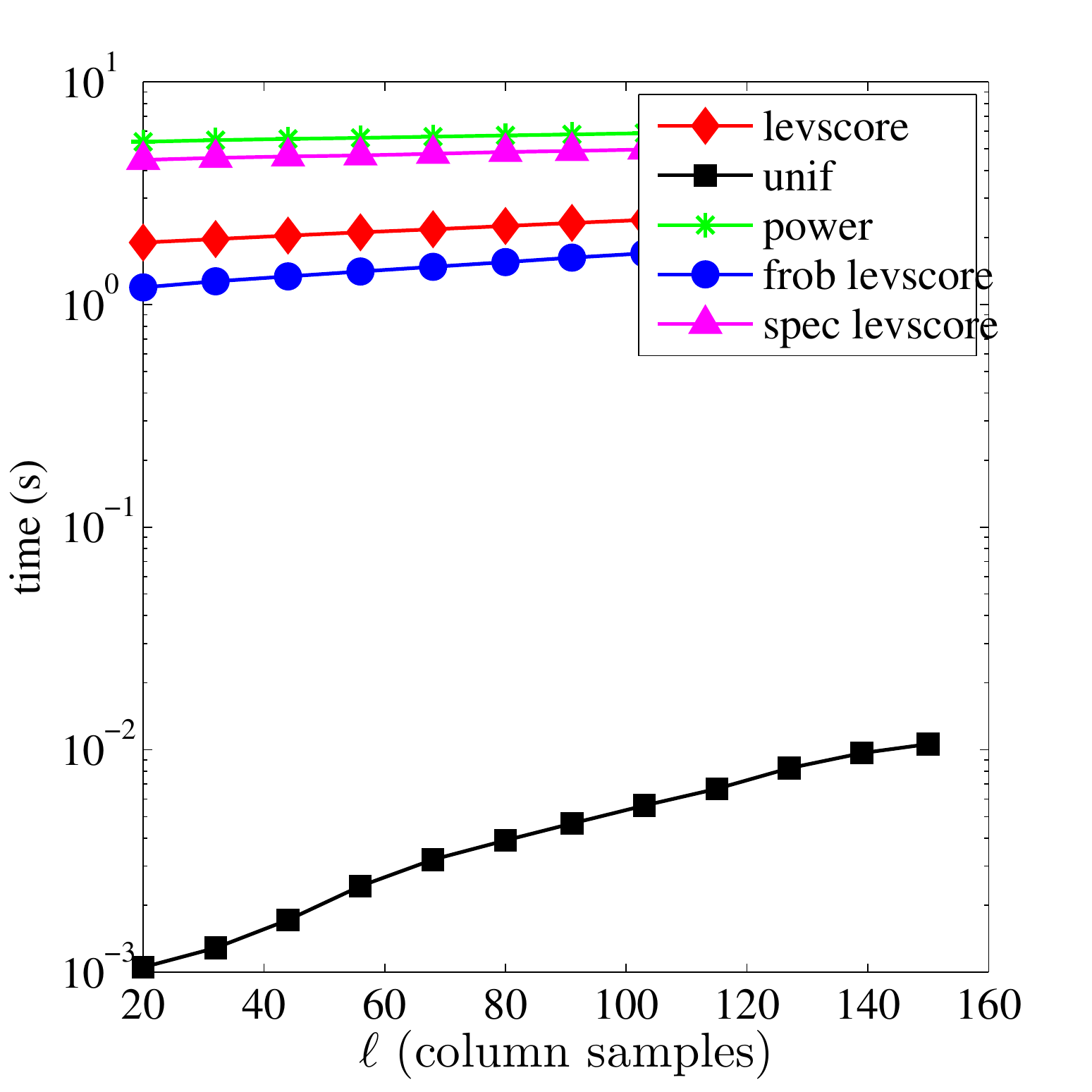}}%
 \subfigure[AbaloneD, $\sigma = 1, k = 20$]{\includegraphics[width=1.6in, keepaspectratio=true]{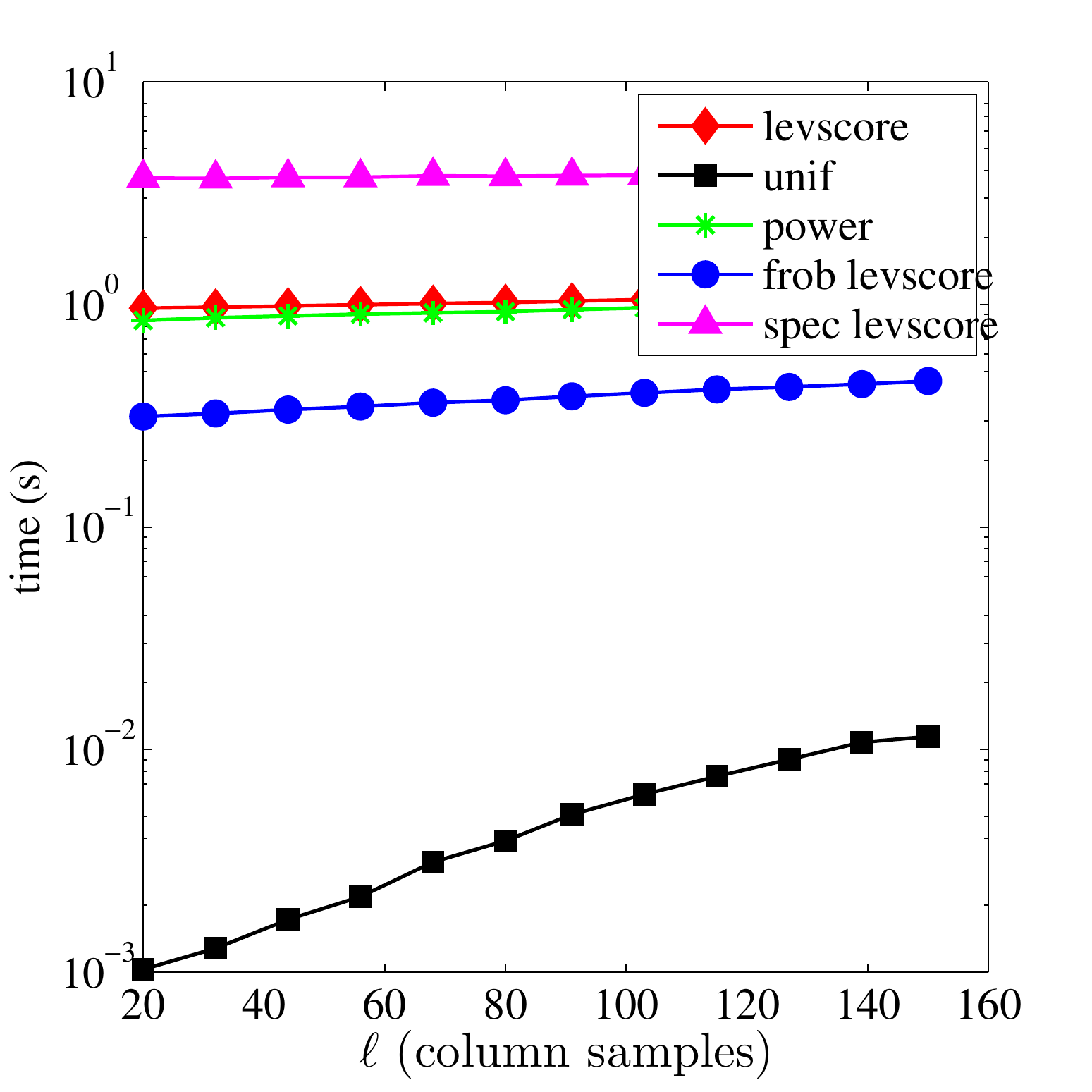}}%
 \subfigure[WineD, $\sigma = 1, k = 20$]{\includegraphics[width=1.6in, keepaspectratio=true]{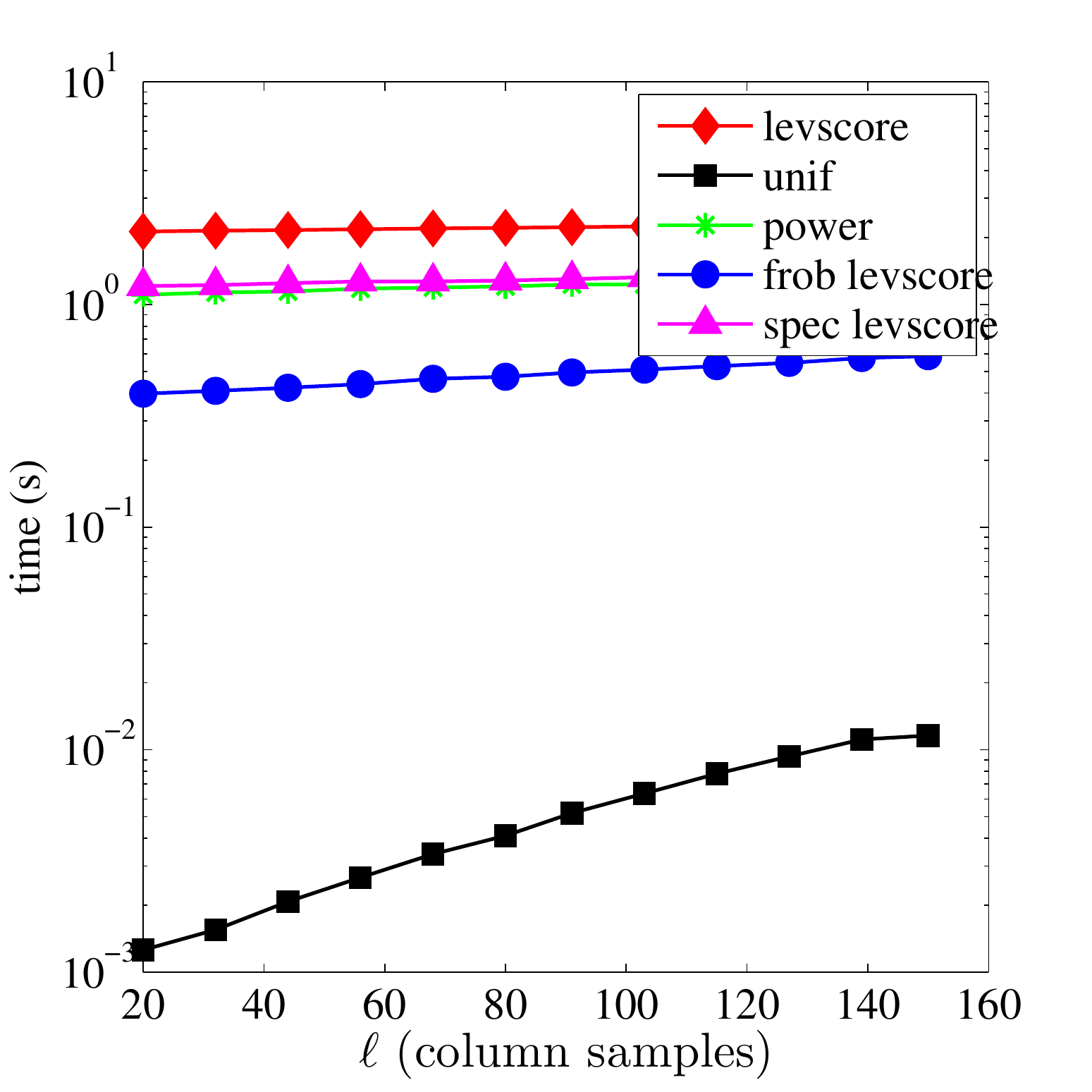}}%
 \subfigure[WineD, $\sigma = 2.1, k = 20$]{\includegraphics[width=1.6in, keepaspectratio=true]{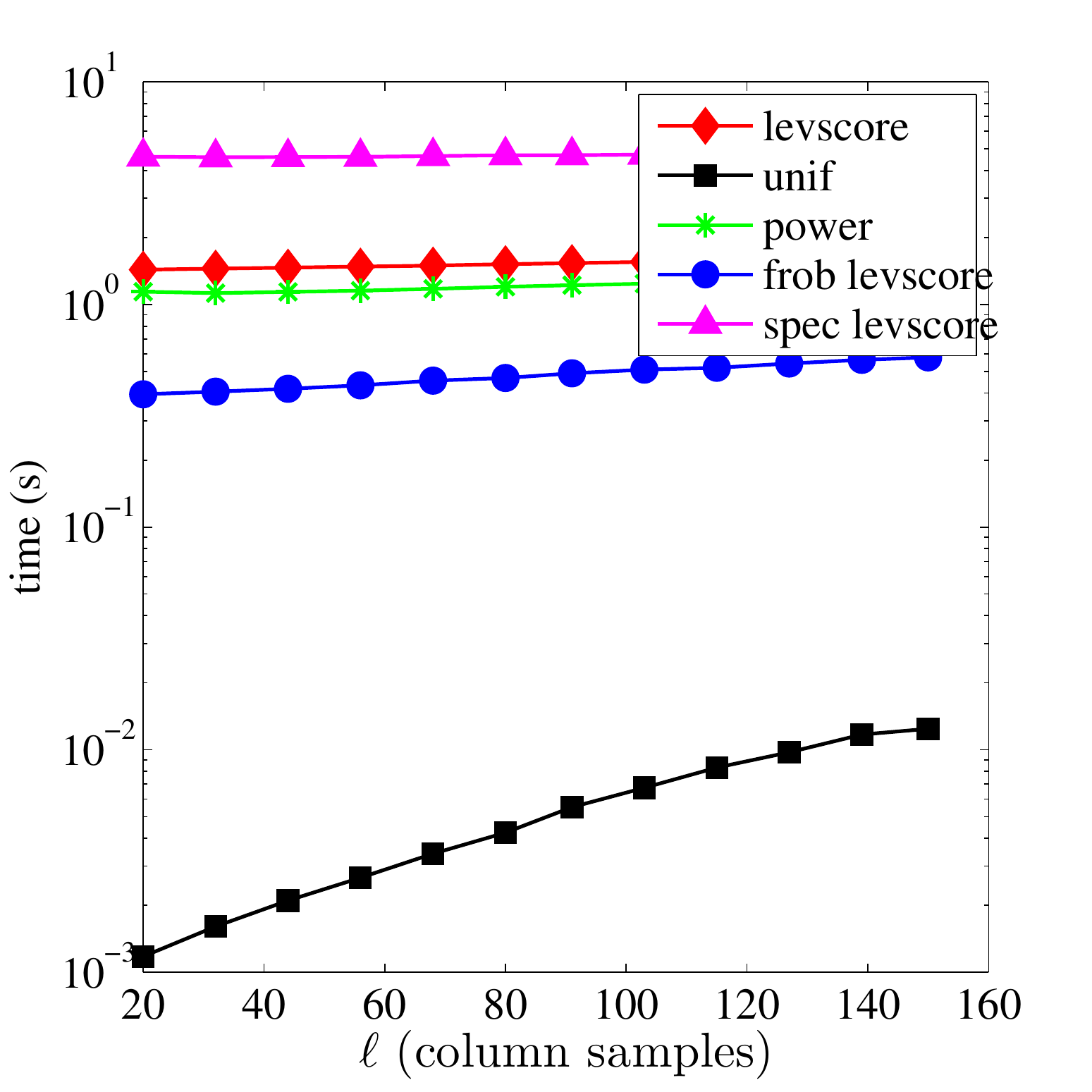}}\\%
 \subfigure[AbaloneS, $\sigma = .15, k = 20$]{\includegraphics[width=1.6in, keepaspectratio=true]{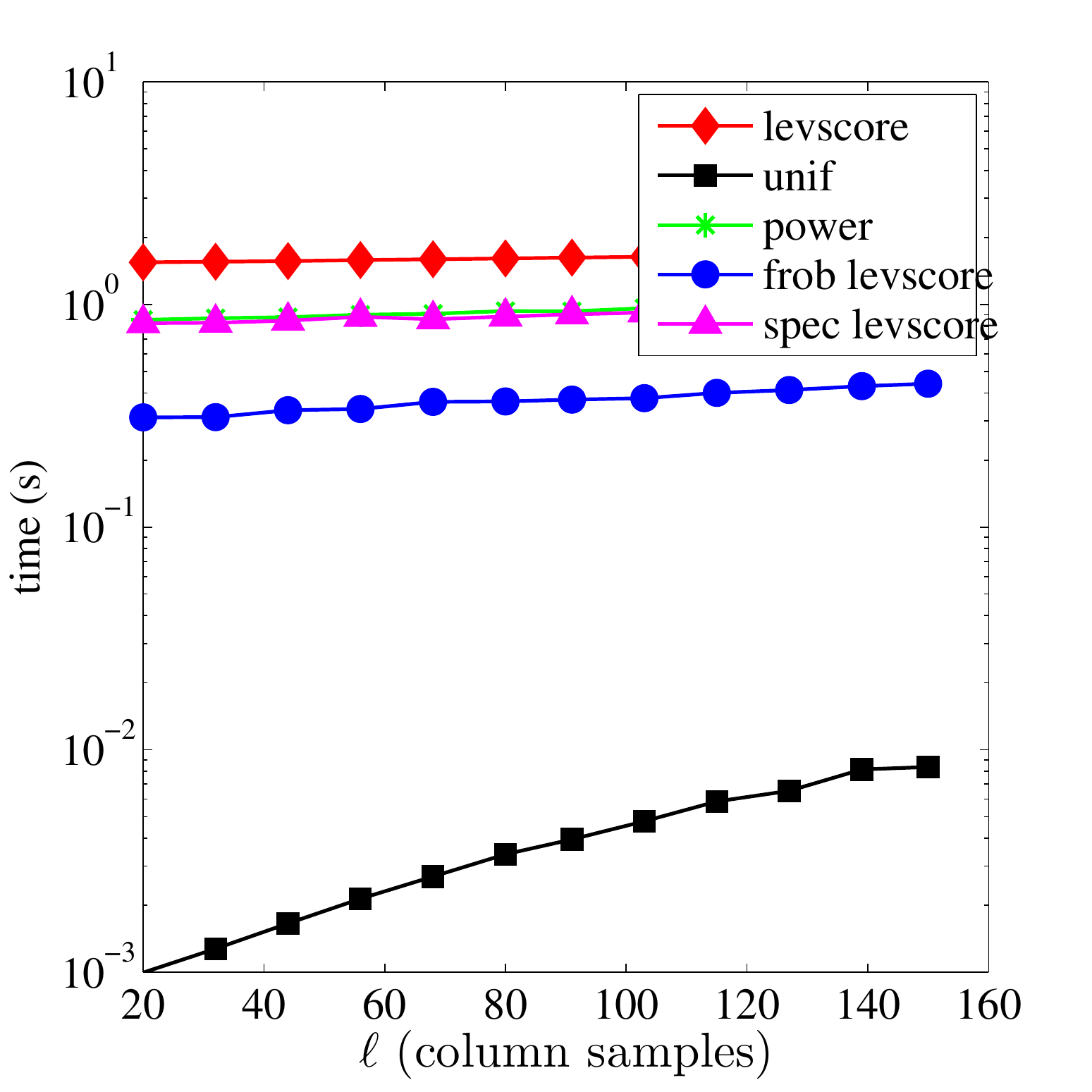}}%
 \subfigure[AbaloneS, $\sigma = 1, k = 20$]{\includegraphics[width=1.6in, keepaspectratio=true]{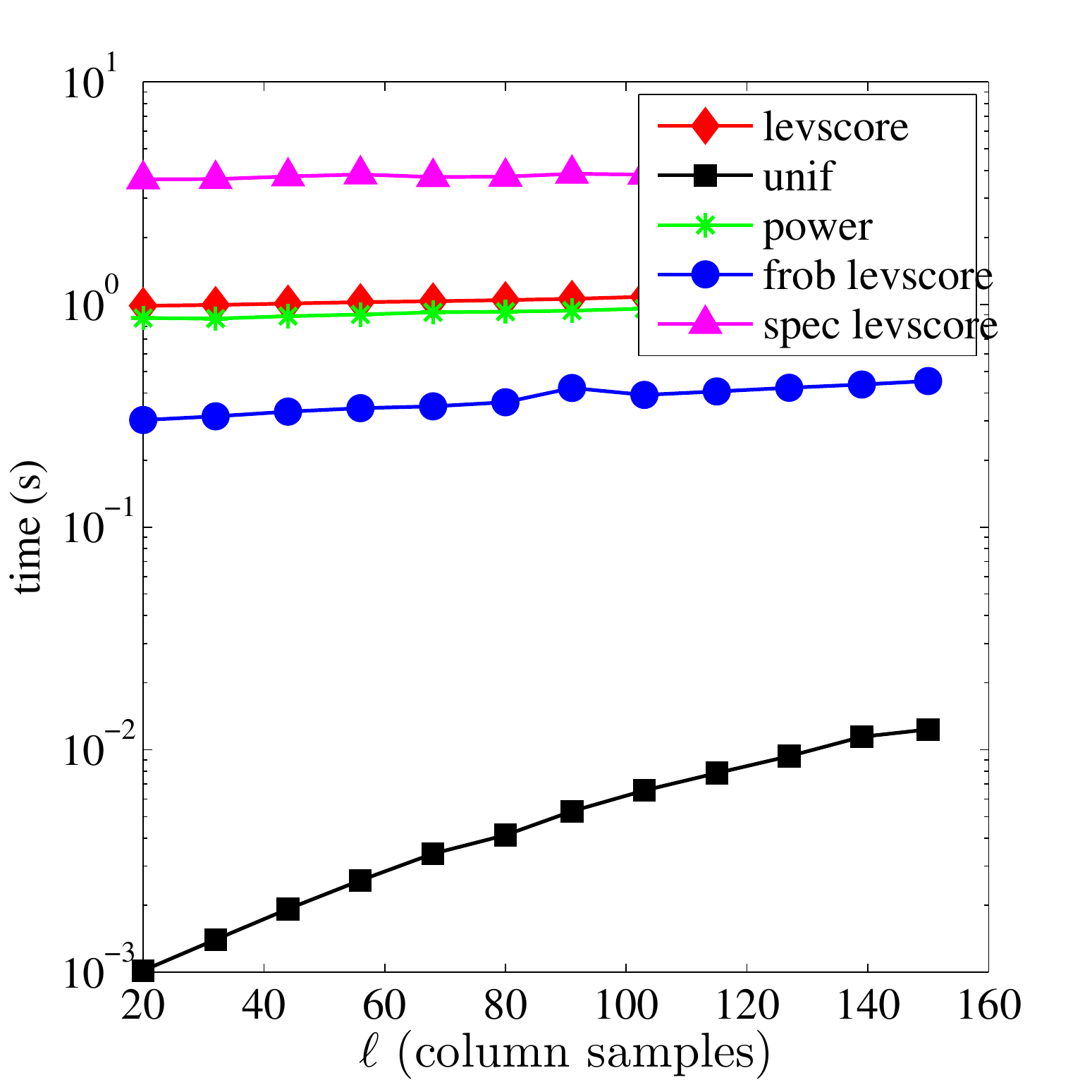}}%
 \subfigure[WineS, $\sigma = 1, k = 20$]{\includegraphics[width=1.6in, keepaspectratio=true]{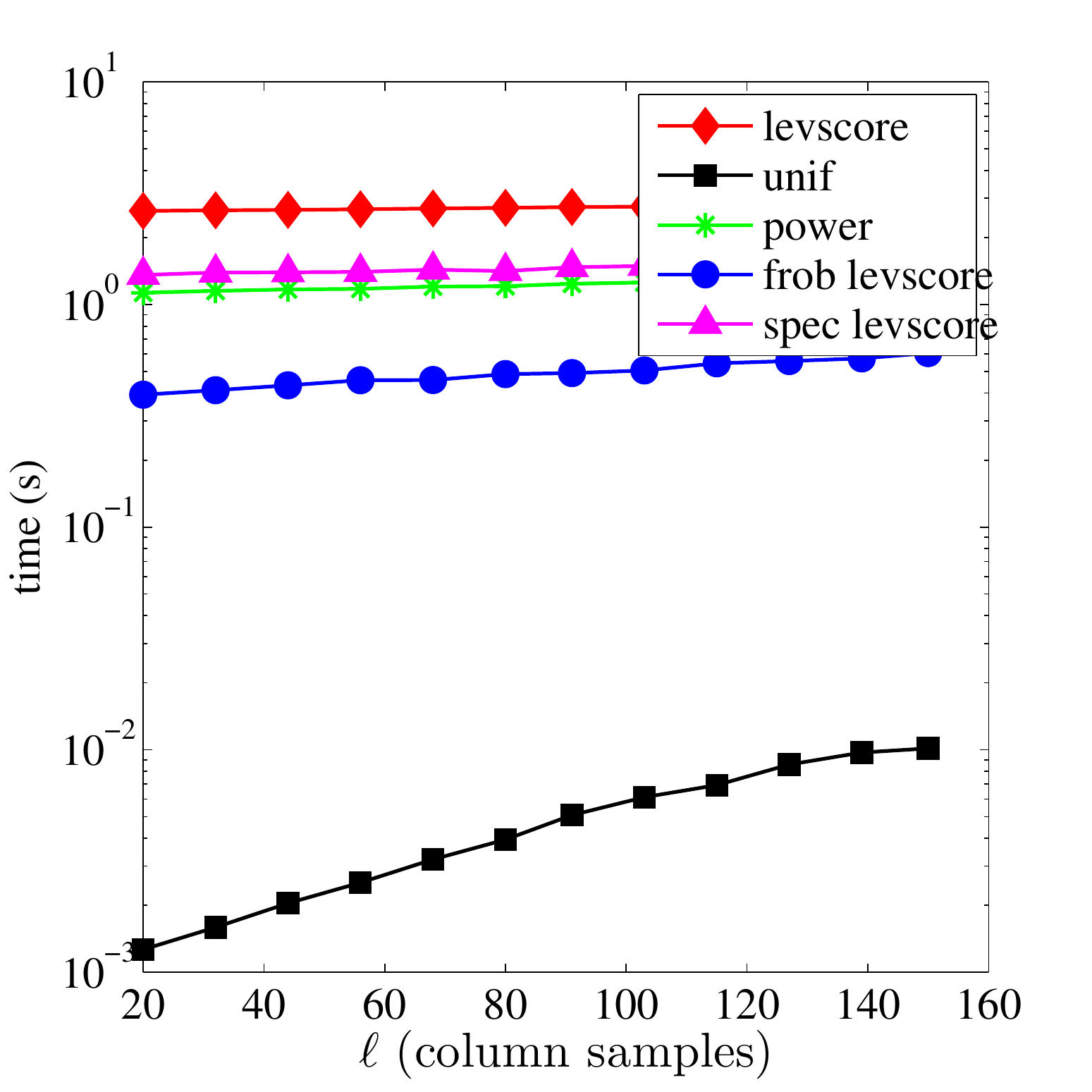}}%
 \subfigure[WineS, $\sigma = 2.1, k = 20$]{\includegraphics[width=1.6in, keepaspectratio=true]{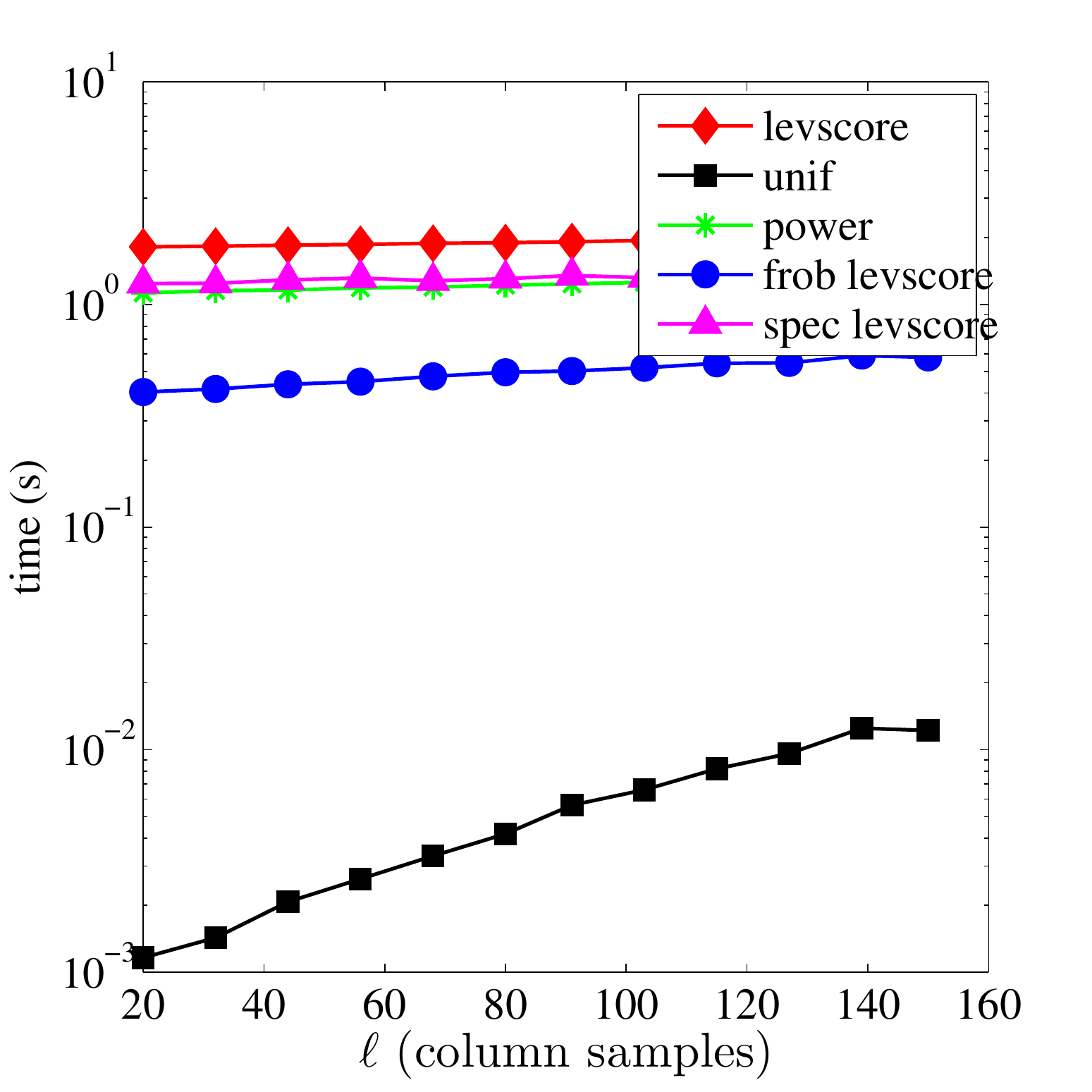}}%
 \caption{The times required to compute the (non-rank-restricted) approximate leverage score-based SPSD sketches, as a function of the number of columns samples $\ell$ for several data sets.
 }%
 \label{fig:inexact-computation-times}
\end{figure}

\begin{figure}[p]
  \centering
  \subfigure[Protein, $k=10$]{\includegraphics[width=1.6in, keepaspectratio=true]{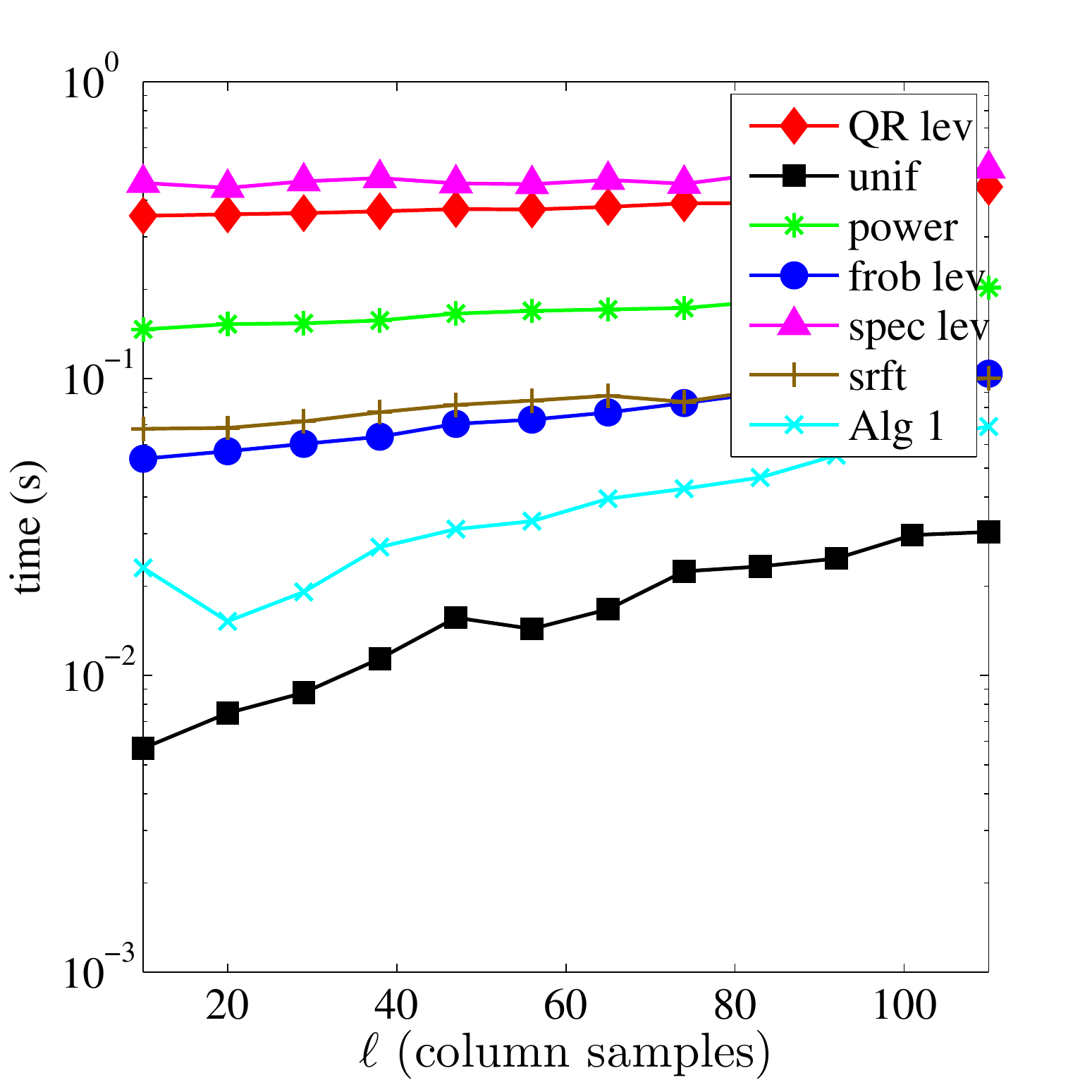}}%
  \subfigure[SNPs, $k=5$]{\includegraphics[width=1.6in, keepaspectratio=true]{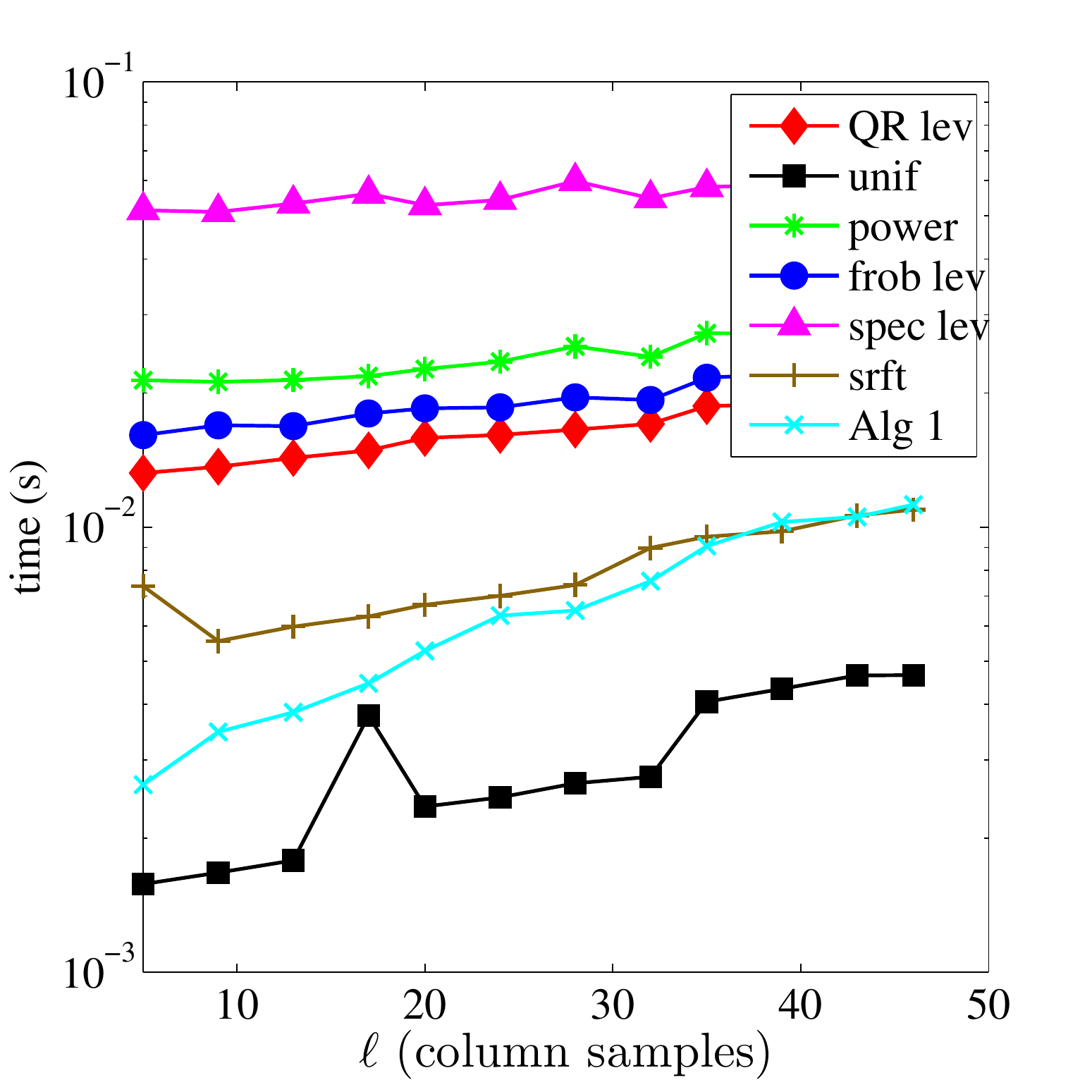}}%
  \caption{The running time of (non-rank-restricted) SPSD sketches computed using Algorithm~\ref{alg:tall_levscore_approx} 
   compared with that of other approximate leverage score-based SPSD sketches, as a function of the number of column samples $\ell$ for two
   Linear Kernel datasets. The parameters in Algorithm~\ref{alg:tall_levscore_approx} were taken to be 
   $r_1 = \epsilon^{-2} \ln(d \delta^{-1}) (\sqrt{d} + \sqrt{\ln(n \delta^{-1})})^2$ and $r_2 = \epsilon^{-2}(\ln n + \ln \delta^{-1} )$ with $\epsilon = 1$ and $\delta = 1/10.$
   }
   \label{fig:tallthin-computation-times}
\end{figure}

\begin{figure}[p]
 \centering
  \subfigure[\mbox{Protein, $k=10$,} \mbox{non-rank-restricted}]{\includegraphics[width=1.6in, keepaspectratio=true]{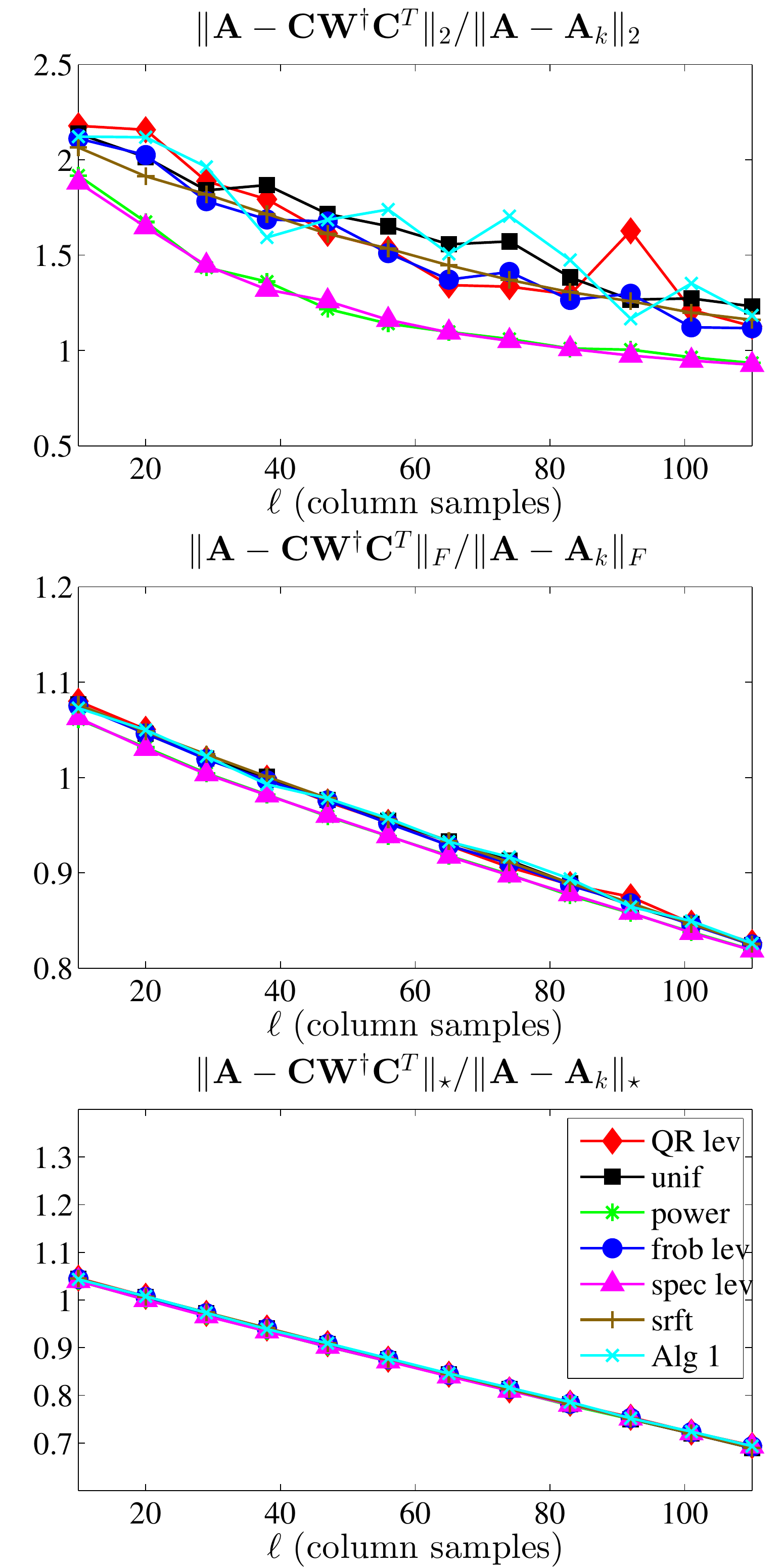}}%
  \subfigure[\mbox{Protein, $k = 10$,} \mbox{rank-restricted}]{\includegraphics[width=1.6in, keepaspectratio=true]{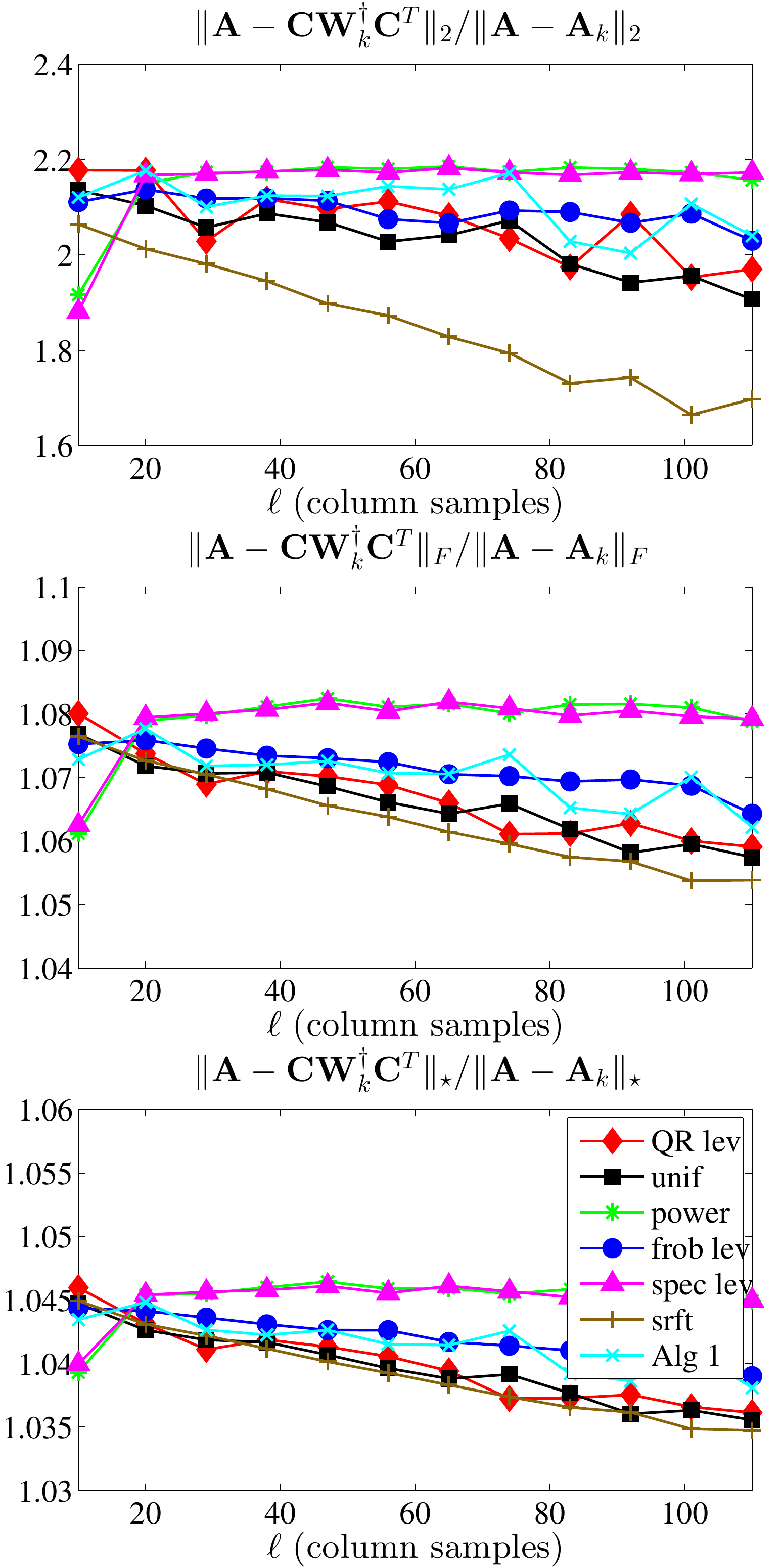}}%
  \subfigure[\mbox{SNPs, $k = 5$,} \mbox{non-rank-restricted}]{\includegraphics[width=1.6in, keepaspectratio=true]{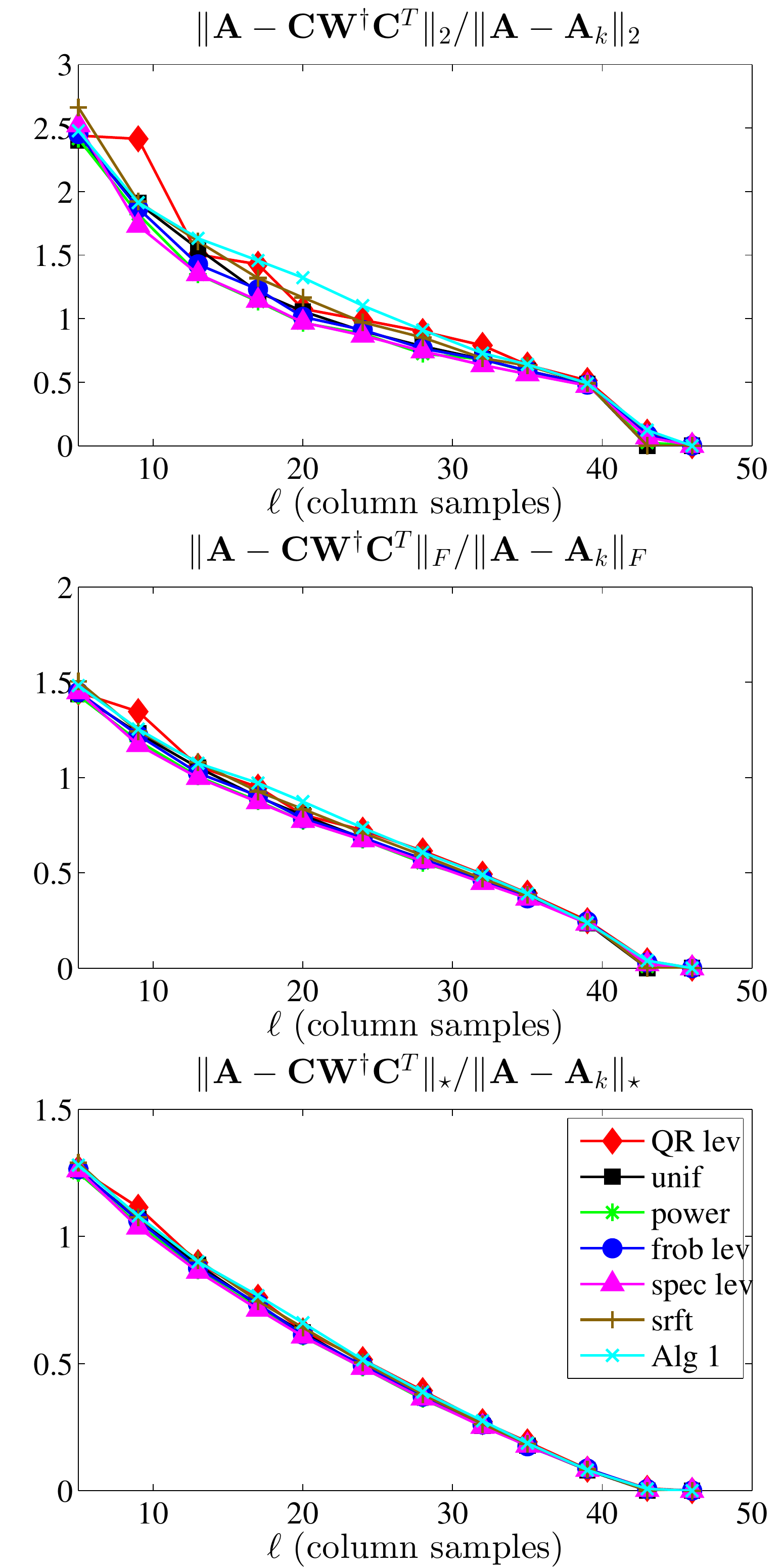}}%
  \subfigure[\mbox{SNPs, $k = 5$,} \mbox{rank-restricted}]{\includegraphics[width=1.6in, keepaspectratio=true]{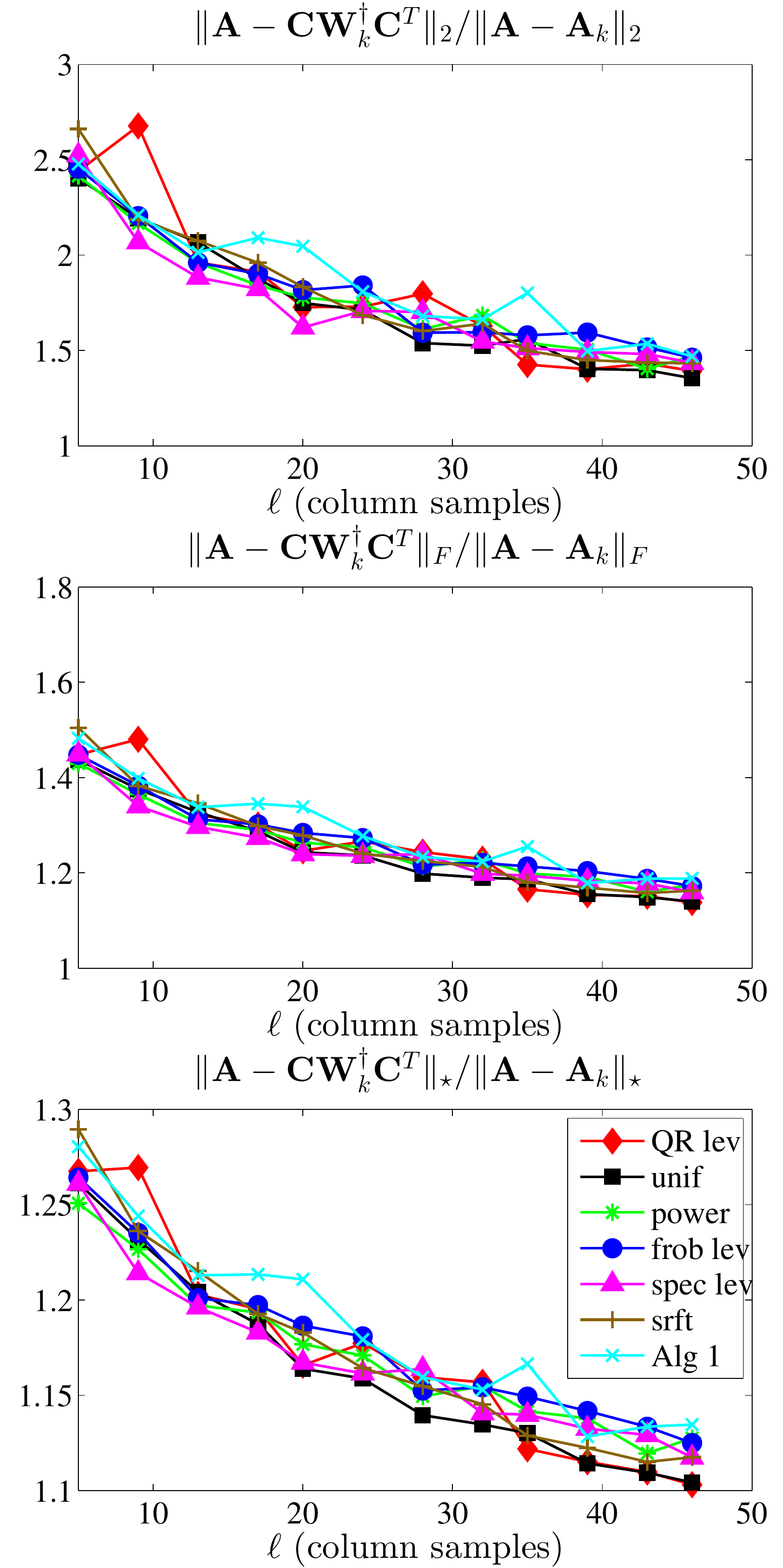}}%
  \caption{The spectral, Frobenius, and trace norm errors 
 (top to bottom, respectively, in each subfigure) of SPSD sketches computed using Algorithm~\ref{alg:tall_levscore_approx} 
 compared with those of other
 approximate leverage score-based sketching schemes, as a function of the number of columns samples $\ell$, 
 for two Linear Kernel data sets. The parameters in Algorithm~\ref{alg:tall_levscore_approx} were taken to be 
   $r_1 = \epsilon^{-2} \ln(d \delta^{-1}) (\sqrt{d} + \sqrt{\ln(n \delta^{-1})})^2$ and $r_2 = \epsilon^{-2}(\ln n + \ln \delta^{-1} )$ with $\epsilon = 1$ and $\delta = 1/10.$
 }%
 \label{fig:tallthin-errors}
\end{figure}

Here, we describe the performances of the various random sampling and 
random projection low-rank sketches considered in 
Section~\ref{sxn:emp-reconstruction} in terms of their running time, where 
the method that involves using the leverage scores to construct the 
importance sampling distribution is implemented both by computing the 
leverage scores ``exactly'' by calling a truncated SVD, as a black box, as 
well as computing them approximately by using one of several versions of 
Algorithm~\ref{alg:frob_levscore_approx}.
Our running time results are presented in 
Figure~\ref{fig:exact-computation-times} and
Figure~\ref{fig:inexact-computation-times}.

We start with the results described in
Figure~\ref{fig:exact-computation-times}, which shows the running times, as 
a function of $\ell$, for the low-rank approximations described 
in Section~\ref{sxn:emp-reconstruction}: \emph{i.e.}, for column sampling 
uniformly at random without replacement; for column sampling according to the 
exact nonuniform leverage score probabilities; and for sketching using 
Gaussian and SRFT mixtures of the columns.
Several observations are worth making about the results presented in this 
figure.
\begin{itemize}
\item
Uniform sampling is always less expensive and typically much less 
expensive than the other methods, while (with one minor exception) sampling 
according to the \emph{exact} leverage scores is always the most expensive
method.
\item
For most matrices, using the SRFT is nearly as expensive as exact leverage 
score sampling.
This is most true for the very sparse graph Laplacian Kernels, largely since 
the SRFT does not respect sparsity.
The main exception to this is for the dense and relatively well-behaved
Linear Kernels, where especially for large values of $\ell$ the SRFT is quite
fast and usually not too much more expensive than uniform sampling.
\item
The ``fast Fourier'' methods underlying the SRFT can take advantage of the
structure of the Linear Kernels to yield algorithms 
that are similar to Gaussian projections and much better than exact leverage 
score computation.
Note that the reason that SRFT is worse than Gaussians here is that the 
matrices we are considering are \emph{not} extremely large, and we are not 
considering very large values of the rank parameter. 
Extending in both those directions leads to Gaussian projections being
slower than SRFT, as the trends in the figures clearly indicate.
\item
Gaussian projections are not too much slower than uniform sampling for the
extremely sparse Laplacian Kernels---this is due to the sparsity of the 
Laplacian Kernels, since Gaussian projections can take advantage of the fast 
matrix-vector multiply, while the SRFT-based scheme cannot---but this 
advantage is lost for the (denser) Sparse RBF Kernels, to the extent that 
there is little running time improvement relative to the Dense RBF Kernels.
In addition, Gaussian projections are relatively slower, when compared to 
the SRFT and uniform sampling, for the Dense RBF Kernels than for the 
Linear Kernels, although both of those data sets are maximally~dense.
\end{itemize}

\noindent
We next turn to the results described in 
Figure~\ref{fig:inexact-computation-times}, which shows the running times, 
as a function of $\ell$, for several variants of approximate leverage-based
sampling. 
For ease of comparison, the timings for uniform sampling (``unif'') and 
exact leverage score sampling (``levscore'') are depicted in 
Figure~\ref{fig:inexact-computation-times} using the same shading as used in 
Figure~\ref{fig:exact-computation-times}. 
In addition to these two baselines, 
Figure~\ref{fig:inexact-computation-times} shows running time results for 
the following three variants of approximate leverage score sampling:
``frob levscore'' (which is Algorithm~\ref{alg:frob_levscore_approx} with $q=0$ 
and $r=2k$);
``spec levscore'' (Algorithm~\ref{alg:frob_levscore_approx} with $q=4$ and 
$r=2k$); and ``power''. 
The ``power'' scheme is a version of 
Algorithm~\ref{alg:frob_levscore_approx} where $r = k$ and $q$ is determined
by monitoring the convergence of the leverage scores of 
$\mat{A}^{2q+1} \boldsymbol{\Pi}$ and terminating when the change in the leverage 
scores between iterations, as measured in the infinity norm, is smaller than 
$10^{-2}.$ 
This is simply a version of subspace iteration with a convergence criterion 
appropriate for the task at hand.  
Since ``frob levscore'' requires one application of an SRFT, its timing results 
are depicted using the same shade as the SRFT timing results in 
Figure~\ref{fig:exact-computation-times}.
(There are no other correspondences between the shadings in the two figures.)
Several observations are worth making about the results presented in 
this~figure.

\begin{itemize}
\item
These approximate leverage score-based algorithms can be orders of magnitude 
faster than exact leverage score computation; but, especially for 
``spec levscore'' when $q$ is not prespecified to be $2$ or $3$, they can 
even be somewhat slower. 
Exactly which is the case depends upon the properties of the matrix and the 
parameters used in the approximation algorithm, including especially the 
number of power iterations.
\item
The ``frob levscore'' approximation method has running time comparable to the
running time of the SRFT, which is expected, given that the computation of 
the SRFT is the 
theoretical bottleneck for the running time of the ``frob levscore'' algorithm.
In particular, for larger values of $\ell$ for Linear Kernels, ``frob levscore'' is not much slower than uniform sampling.
\item
The ``spec levscore'' and ``power'' approximations with $q>0$ are more 
expensive than the $q=0$ ``frob lev'' approximation, which is a result of 
the relatively-expensive matrix-matrix multiplication.  
For the Linear Kernels, both are much better than the exact leverage score
computation, and for most other data at least ``power'' is somewhat less
expensive than the exact leverage score computation.
For example, this is particularly true for the Laplacian Kernels.
\end{itemize}

\noindent
Recall that the cost associated with these SPSD sketches
is two-fold:
first, the cost to construct the sample---by sampling columns uniformly at 
random, by computing a nonuniform importance sampling distribution, or by 
performing a random projection to uniformize the leverage scores; and 
second, the cost to construct the low-rank approximation from the 
sample.
For uniform sampling, the latter step dominates the cost, while for more 
sophisticated methods the former step typically dominates the cost.
The approximate leverage score sampling methods
are still sufficiently expensive that the cost of computing the 
sampling probabilities still dominates the cost to construct the low-rank 
approximation.

Finally, 
Algorithm~\ref{alg:tall_levscore_approx} can be used to approximate quickly 
the leverage scores of matrices of the form $\matA = \matX\matX^\transp$, 
when $\matX \in \R^{n \times d}$ is a rectangular matrix of sufficent aspect 
ratio, and in such cases it is faster than 
Algorithm~\ref{alg:frob_levscore_approx}. 
Specifically, for the first dimensional reduction step in 
Algorithm~\ref{alg:tall_levscore_approx} to be beneficial (\emph{i.e.}, to 
ensure $r_1 < n$), the condition $n = \Omega(d \ln d)$ is necessary; 
for the second dimensional reduction step to be beneficial (\emph{i.e.}, to 
ensure $r_2 < d$), the condition $d = \Omega(\ln n)$ must be satisfied. 
Figure~\ref{fig:tallthin-computation-times} summarizes our main results for
the run time of Algorithm~\ref{alg:tall_levscore_approx} applied to 
rectangular matrices with $n \gg d$.
Among other things, 
Figure~\ref{fig:tallthin-computation-times} illustrates, using the Linear 
Kernel datasets Protein and SNPs (which satisfy these constraints), two 
points.
\begin{itemize}
\item
Most importantly, the running time of 
Algorithm~\ref{alg:tall_levscore_approx} on these rectangular matrices is
faster than performing a QR decomposition on $\matA$ and is comparable to 
applying a SRFT to $\matA$. 
This is expected, since the running time bottleneck for 
Algorithm~\ref{alg:tall_levscore_approx} is the application of the SRFT.
\item
In addition, the running time of Algorithm~\ref{alg:tall_levscore_approx} 
is significantly faster than the other approximate leverage score algorithms. 
This too is expected, since these other algorithms are applied to $\matA$
and ignore the rectangular structure of $\matX$.
\end{itemize}
Figure~\ref{fig:tallthin-errors} shows that these improved running time 
gains for Algorithm~\ref{alg:tall_levscore_approx} can come at the cost of a 
slight loss in the reconstruction accuracy (relative to the
exact computation of the leverage scores) of the low-rank approximations; 
the accuracy of the other approximate leverage score algorithms is discussed 
in the following subsection.

\subsubsection{Reconstruction Accuracy Results}

Here, we describe the performances of the various low-rank 
approximations that use approximate leverage scores in terms of reconstruction 
accuracy for the data sets described in Section~\ref{sxn:emp-datasets}.
The results are presented in 
Figure~\ref{fig:laplacian-inexact-nonfiltered-errors}
through Figure~\ref{fig:sparserbf-inexact-nonfiltered-errors}.
The setup for these results parallels that for the 
low-rank approximation results described in 
Section~\ref{sxn:emp-reconstruction}, and these figures parallel 
Figure~\ref{fig:laplacian-exact-errors-1} through
Figure~\ref{fig:sparserbf-exact-errors}.
To provide a baseline for the comparison, we also plot the previous 
reconstruction errors for sampling with the exact leverage scores as well
as the uniform column sampling sketch.
Several observations are worth making about the results presented in 
these~figures.

\begin{figure}[p]
 \centering
 \subfigure[GR, $k = 20$]{\includegraphics[width=1.6in, keepaspectratio=true]{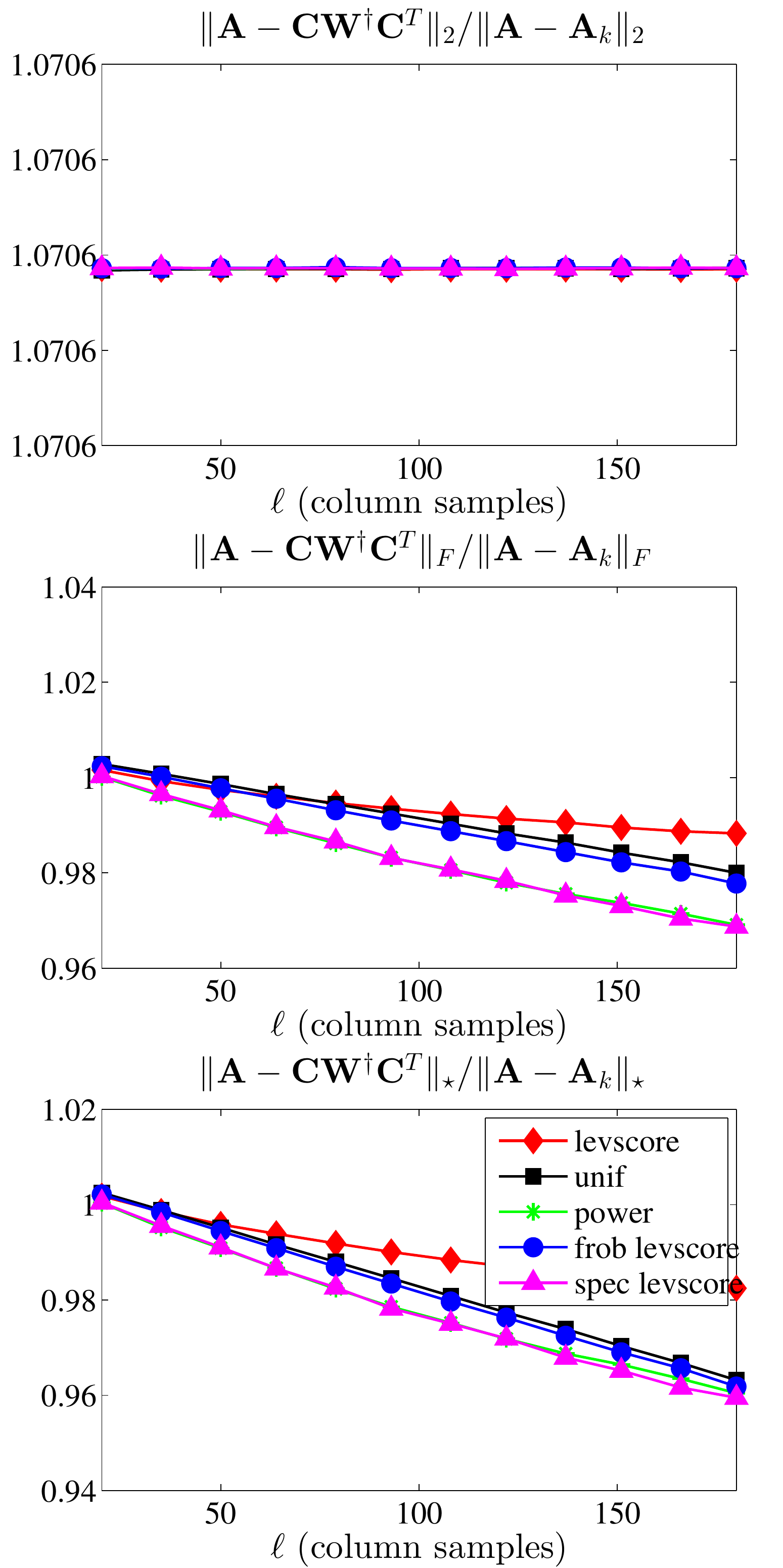}}%
 \subfigure[GR, $k = 60$]{\includegraphics[width=1.6in, keepaspectratio=true]{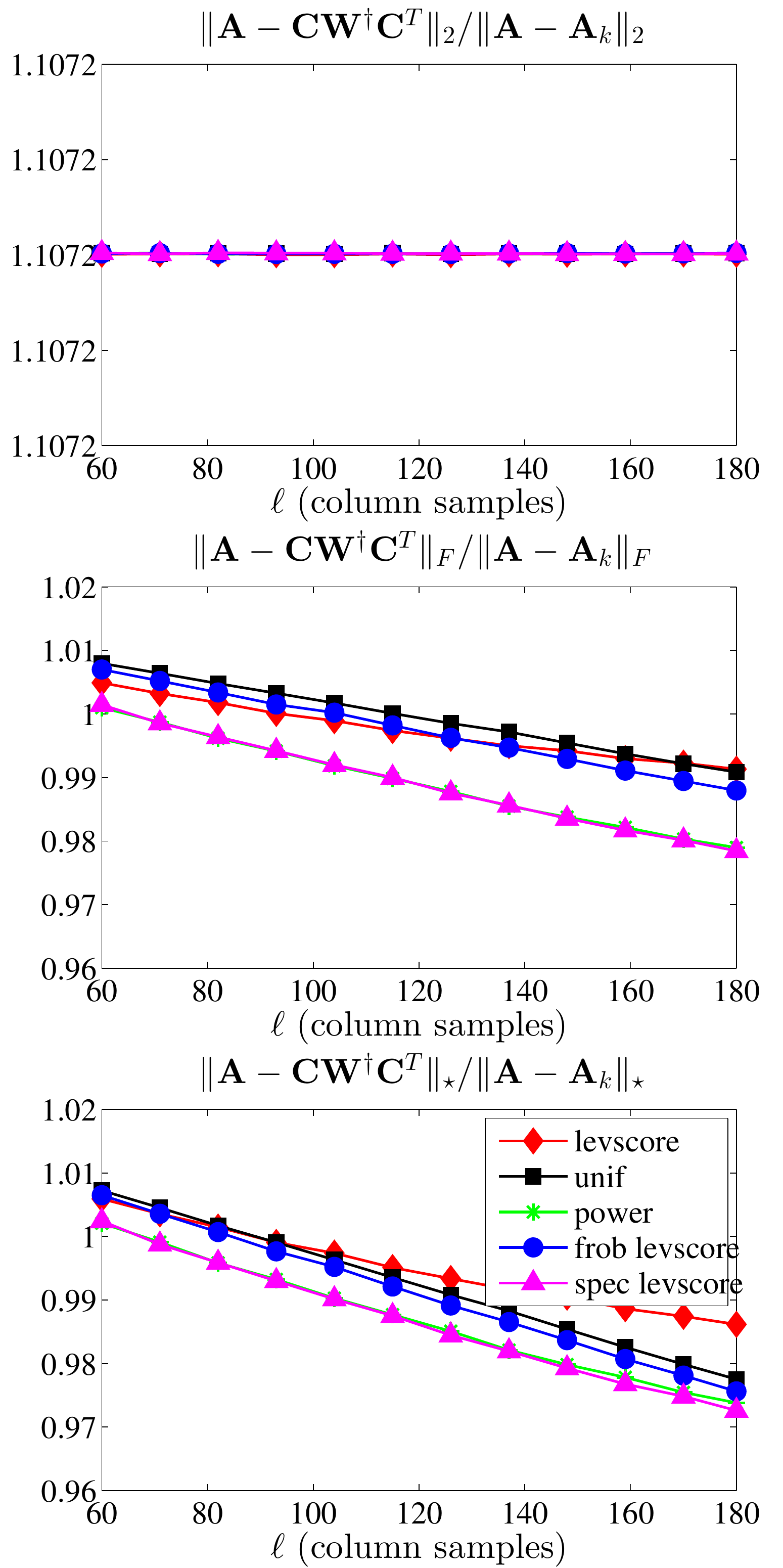}}%
 \subfigure[HEP, $k = 20$]{\includegraphics[width=1.6in, keepaspectratio=true]{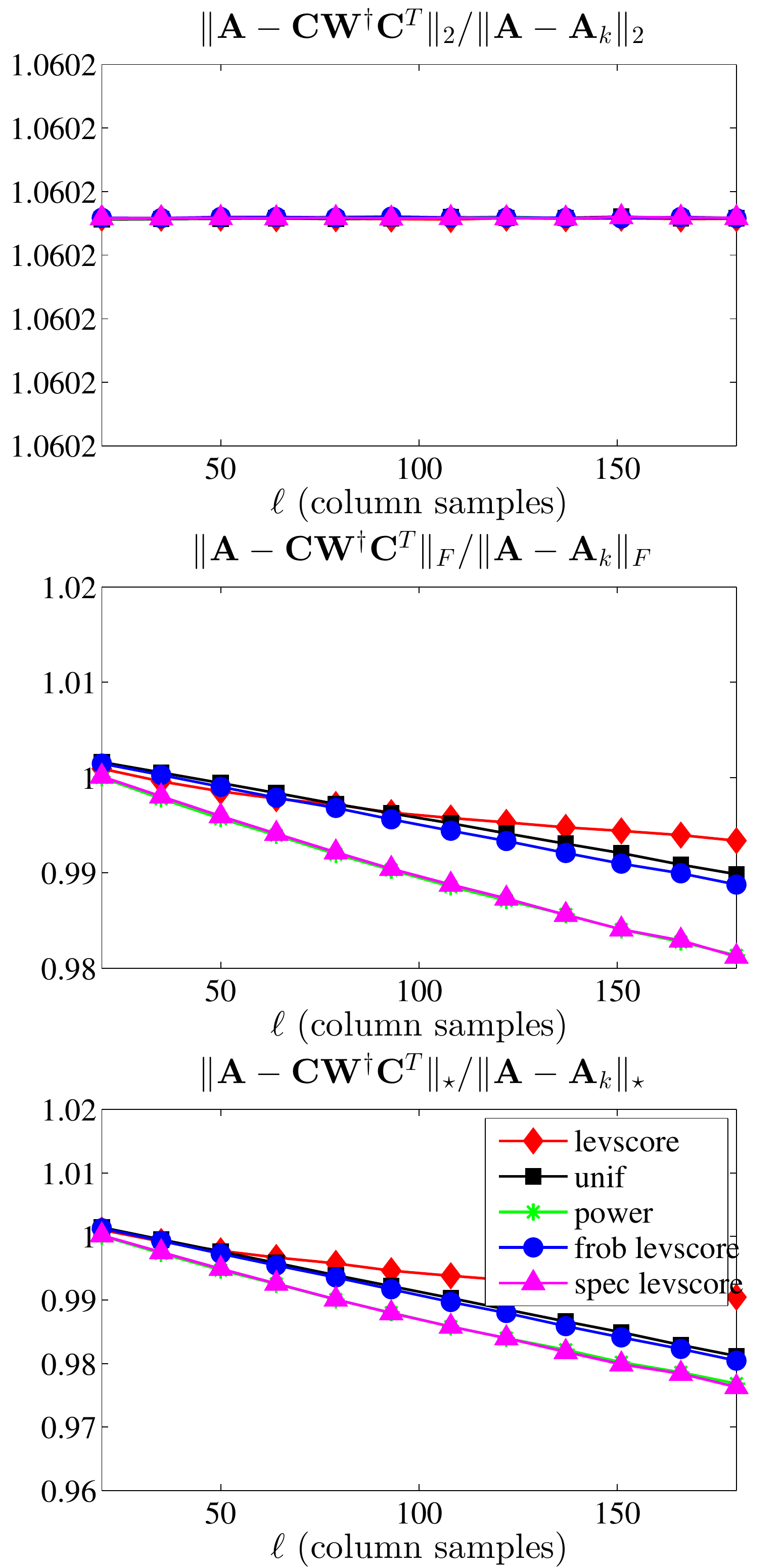}}%
 \subfigure[HEP, $k = 60$]{\includegraphics[width=1.6in, keepaspectratio=true]{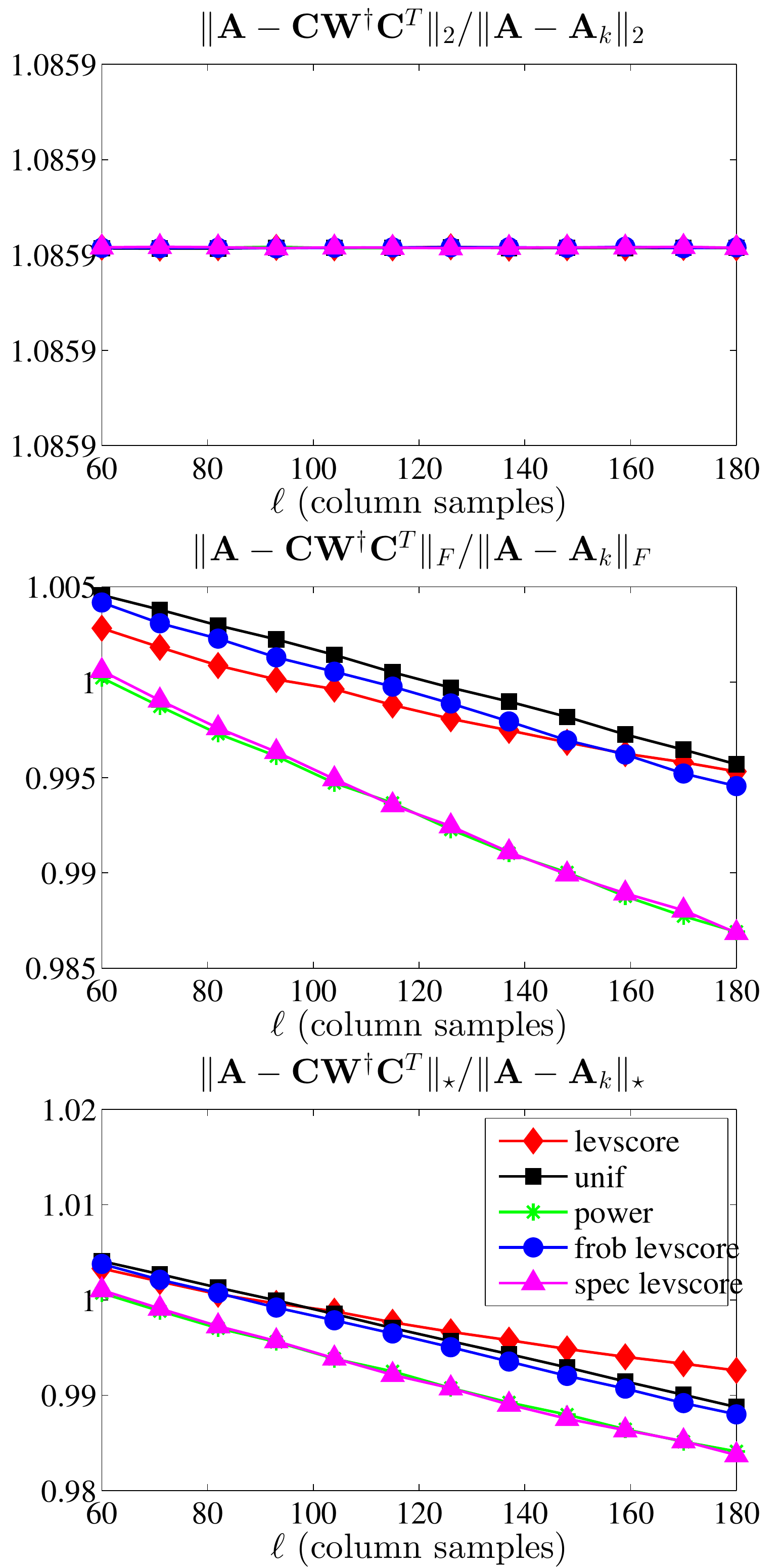}}%
 \\%
 \subfigure[GR, $k = 20$]{\includegraphics[width=1.6in, keepaspectratio=true]{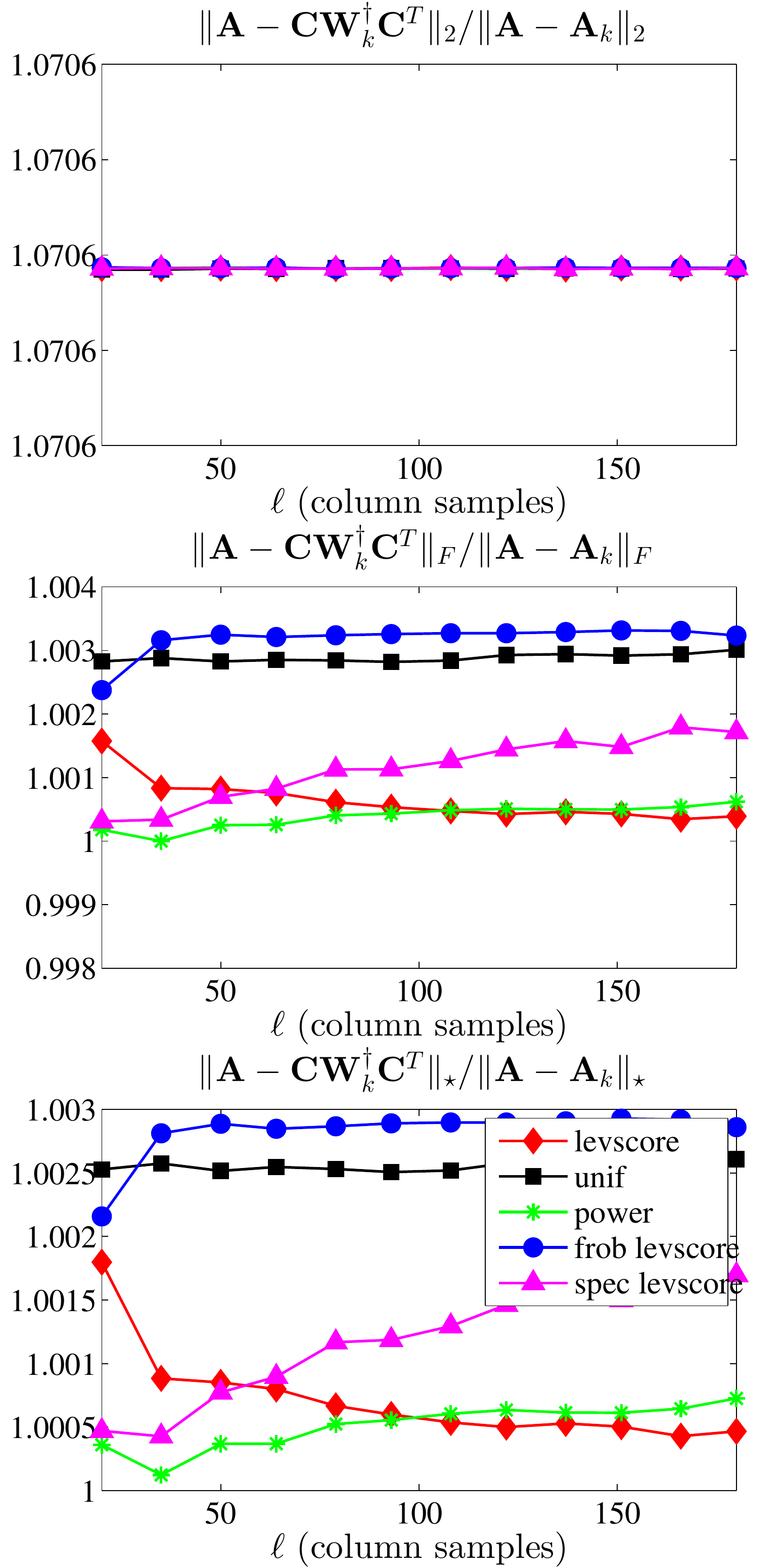}}%
 \subfigure[GR, $k = 60$]{\includegraphics[width=1.6in, keepaspectratio=true]{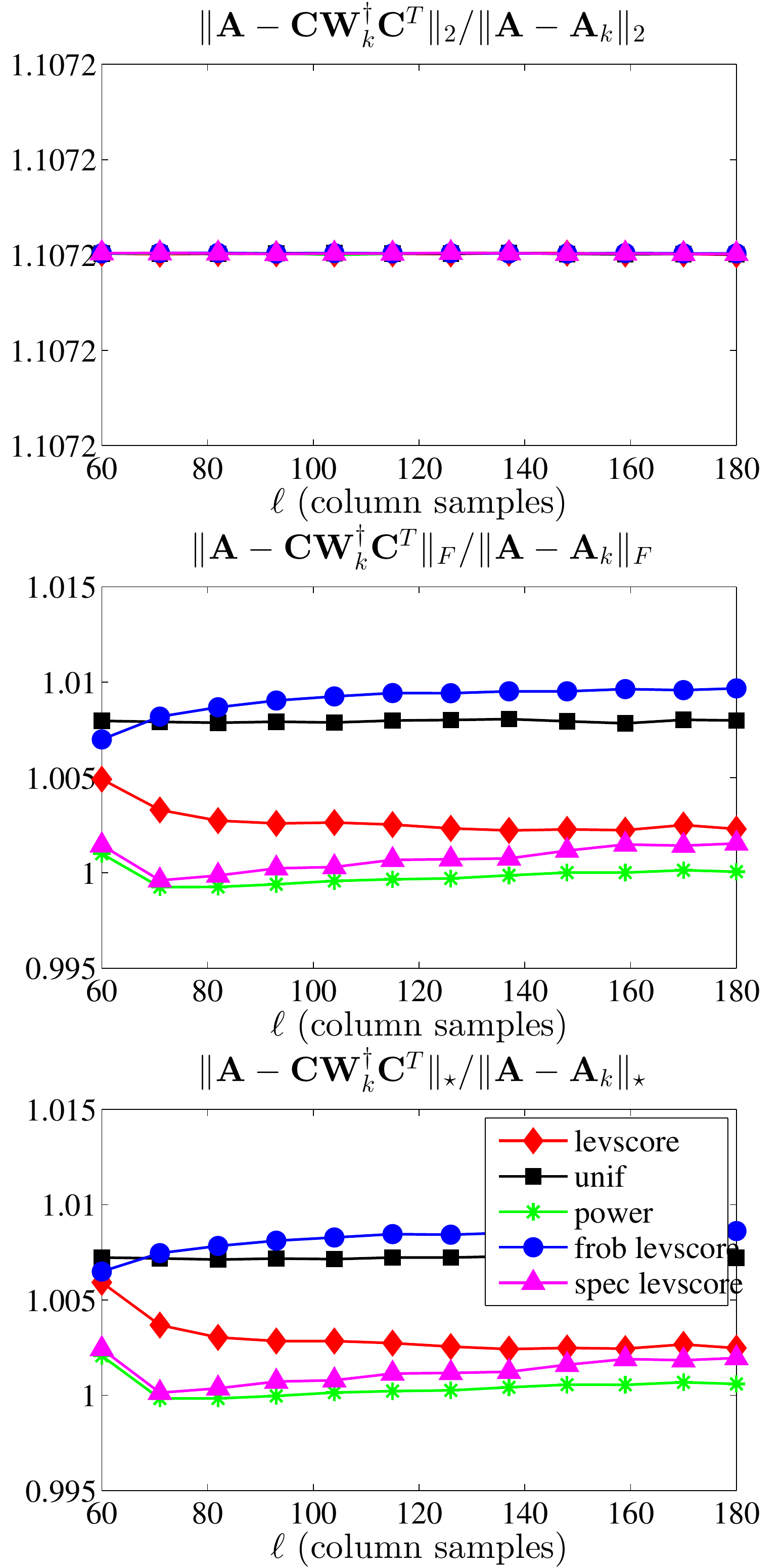}}%
 \subfigure[HEP, $k = 20$]{\includegraphics[width=1.6in, keepaspectratio=true]{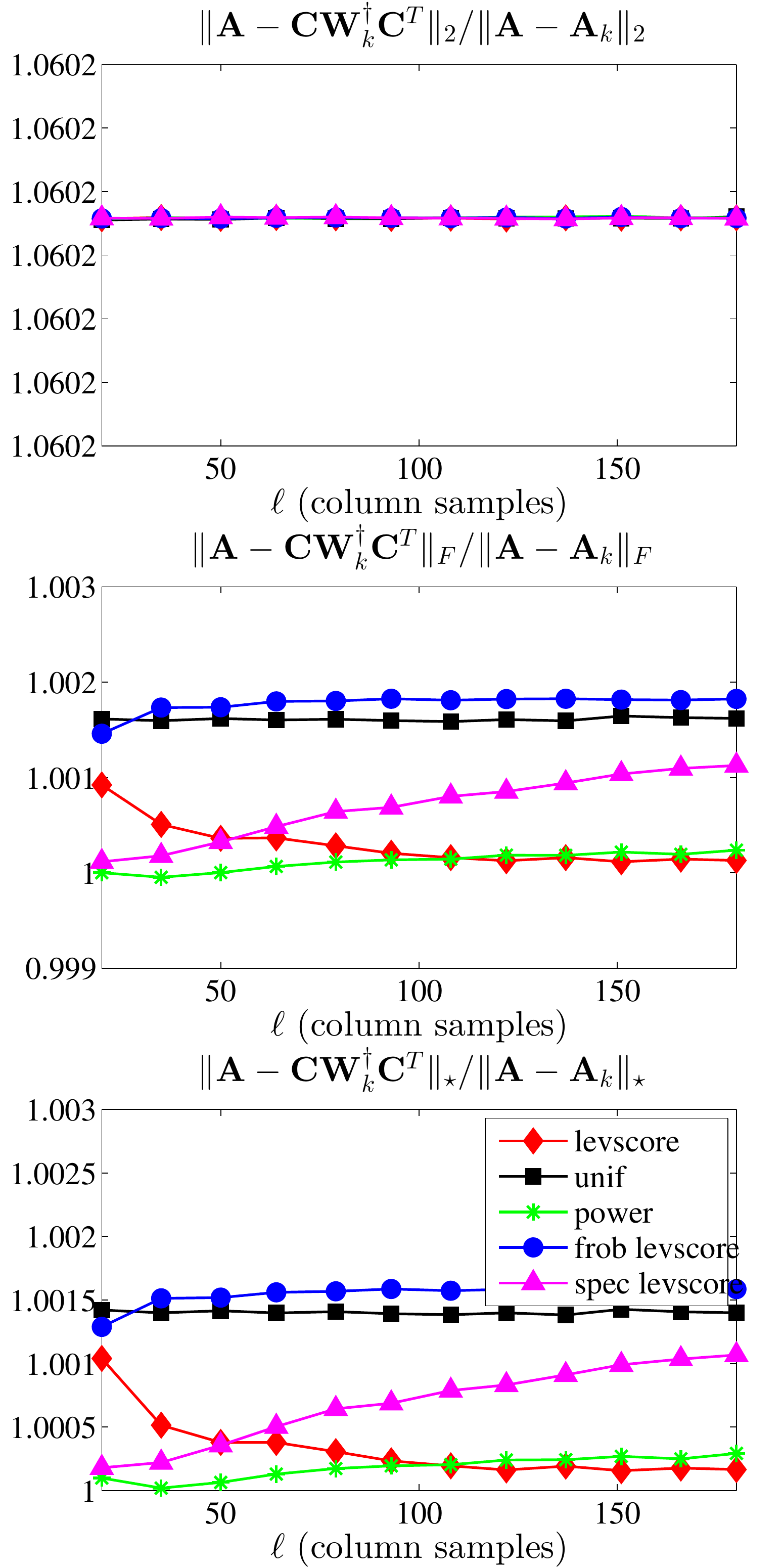}}%
 \subfigure[HEP, $k = 60$]{\includegraphics[width=1.6in, keepaspectratio=true]{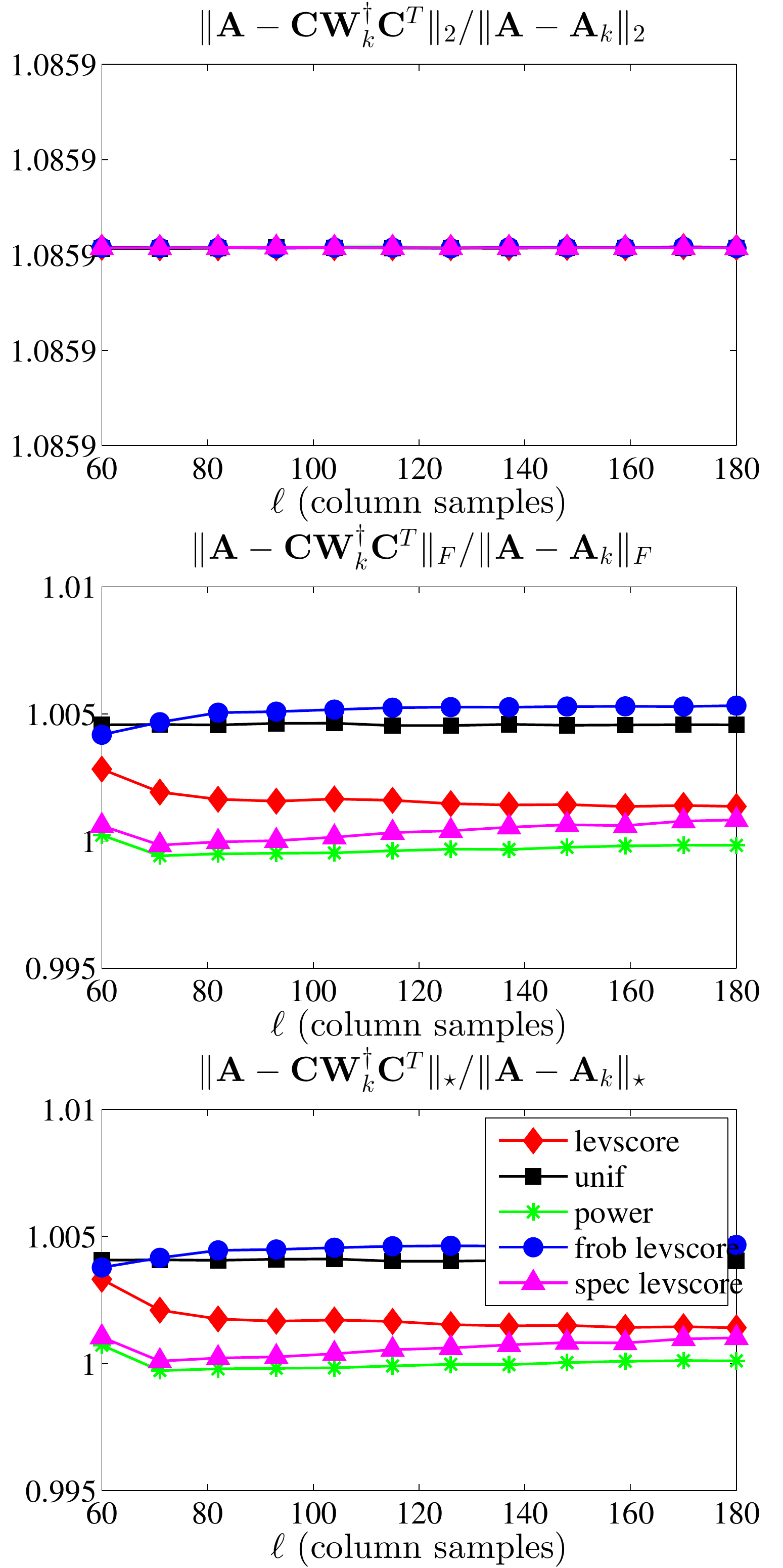}}%
 \caption{The spectral, Frobenius, and trace norm errors 
 (top to bottom, respectively, in each subfigure) of several
 (non-rank-restricted in top panels and rank-restricted in bottom panels)
 approximate leverage score-based SPSD sketches, as a function of the number of columns samples $\ell$, 
 for the GR and HEP Laplacian data sets, with two choices of the rank parameter $k$.}%
 \label{fig:laplacian-inexact-nonfiltered-errors}
\end{figure}

\begin{figure}[p]
 \centering
 \subfigure[Enron, $k = 20$]{\includegraphics[width=1.6in, keepaspectratio=true]{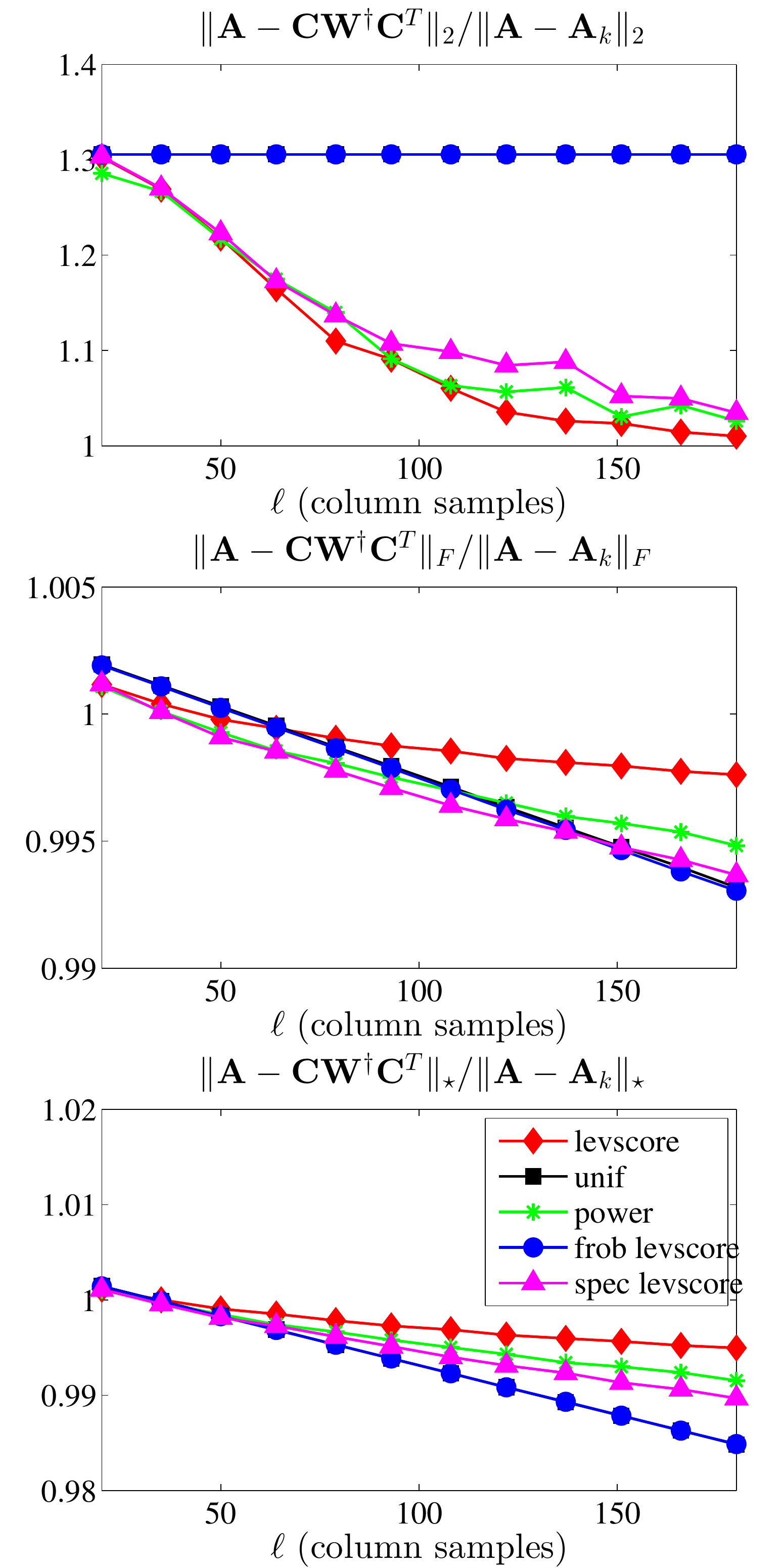}}%
 \subfigure[Enron, $k = 60$]{\includegraphics[width=1.6in, keepaspectratio=true]{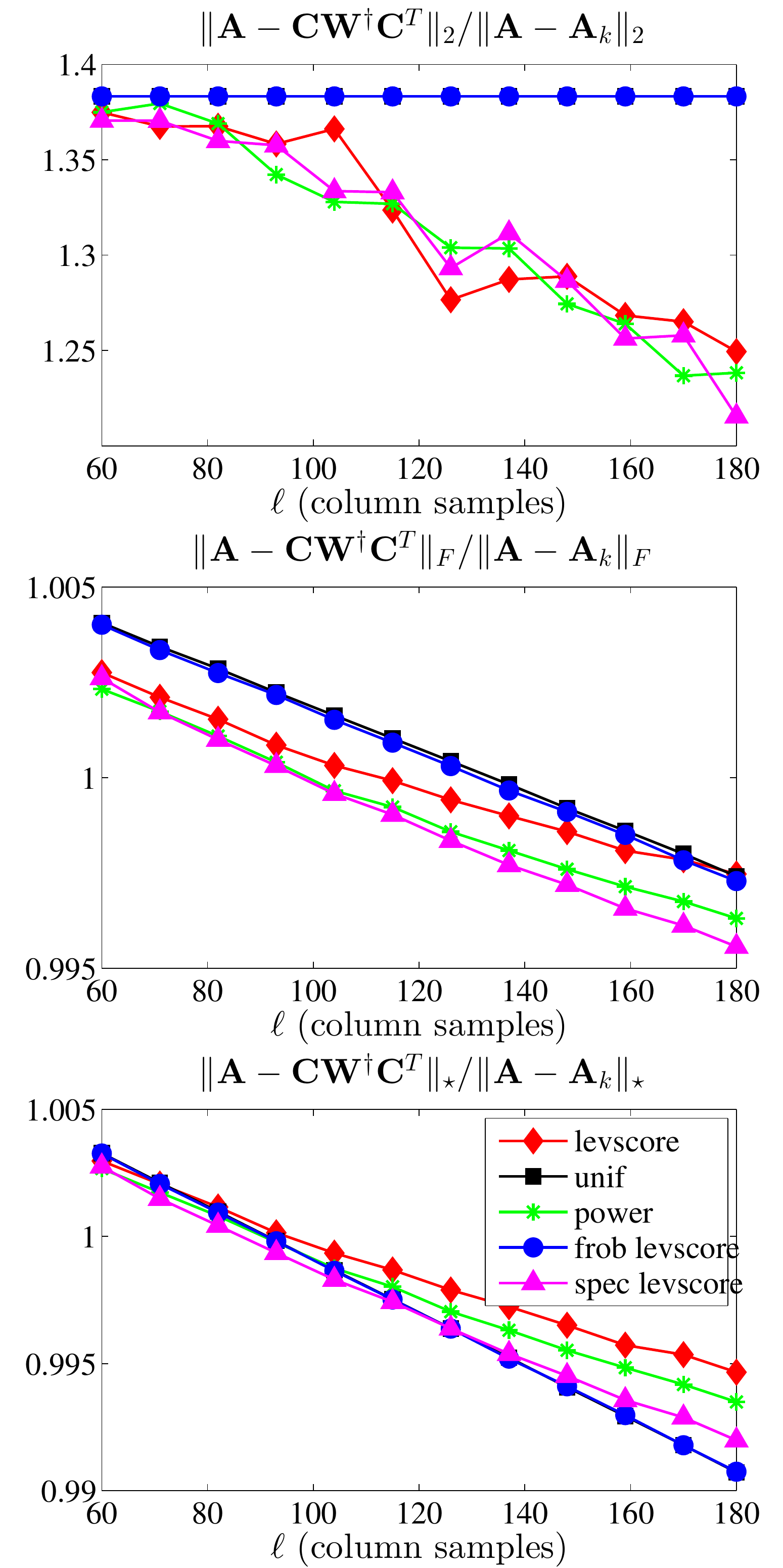}}%
 \subfigure[Gnutella, $k = 20$]{\includegraphics[width=1.6in, keepaspectratio=true]{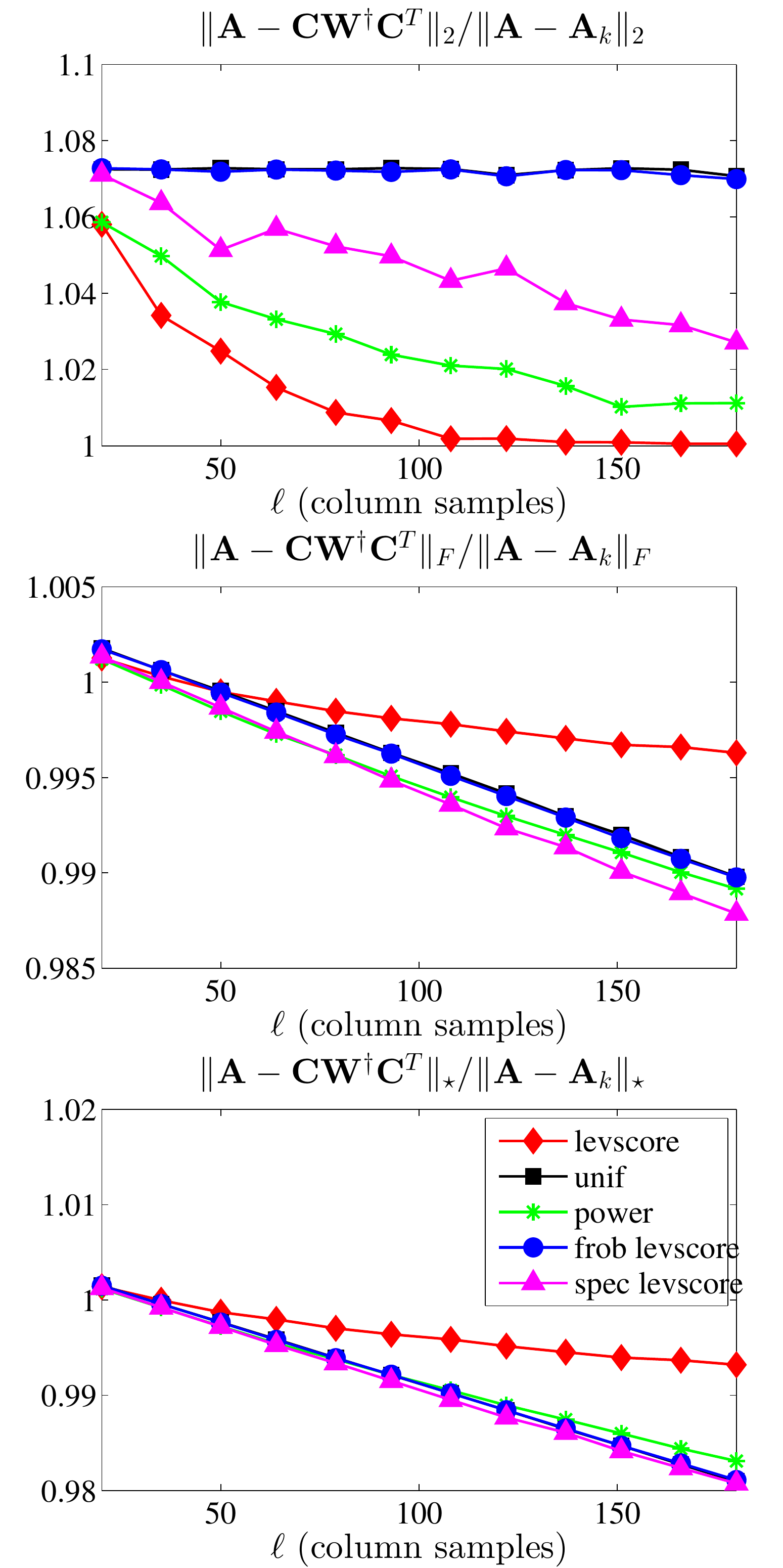}}%
 \subfigure[Gnutella, $k = 60$]{\includegraphics[width=1.6in, keepaspectratio=true]{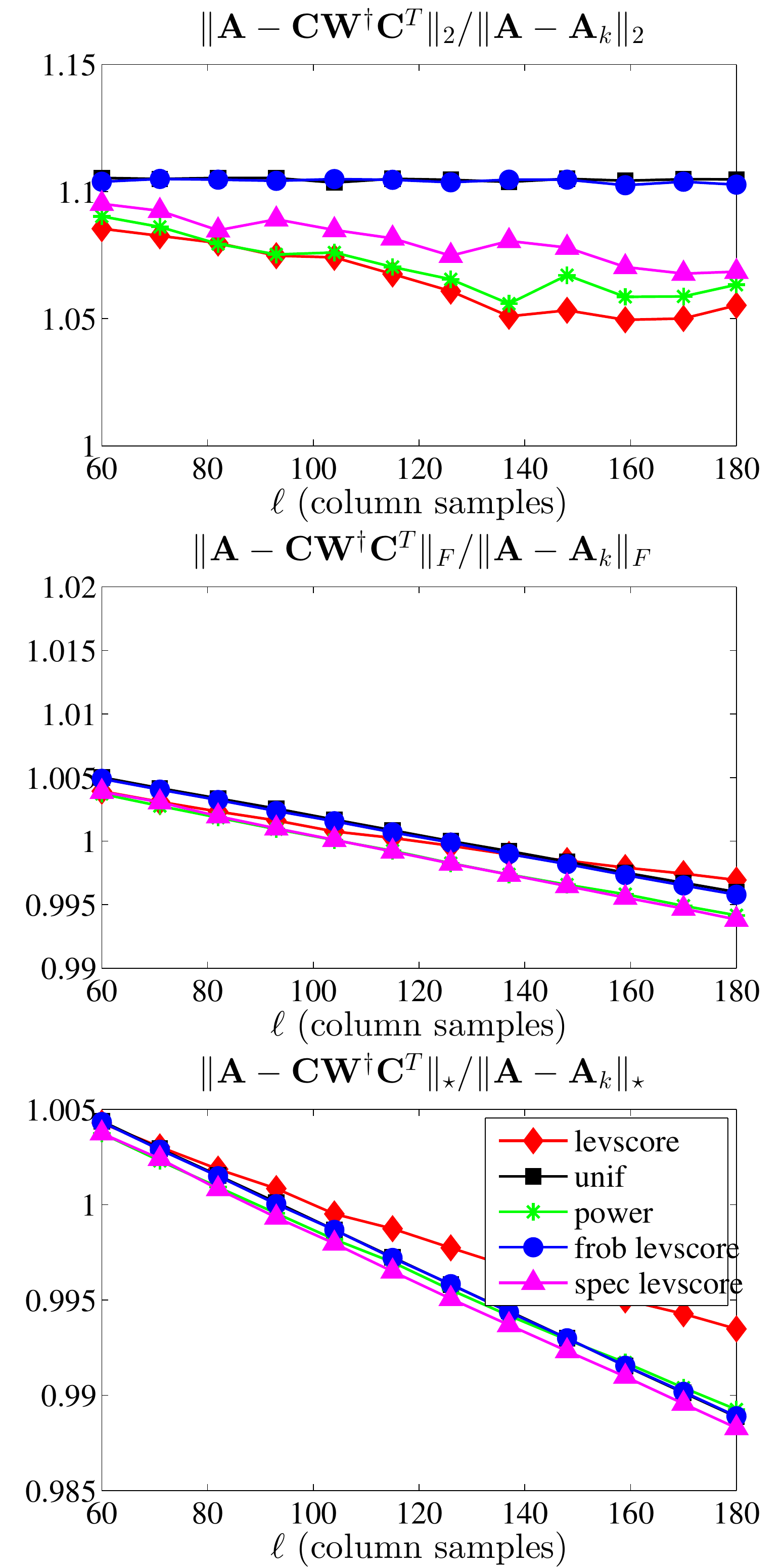}}%
 \\%
 \subfigure[Enron, $k = 20$]{\includegraphics[width=1.6in, keepaspectratio=true]{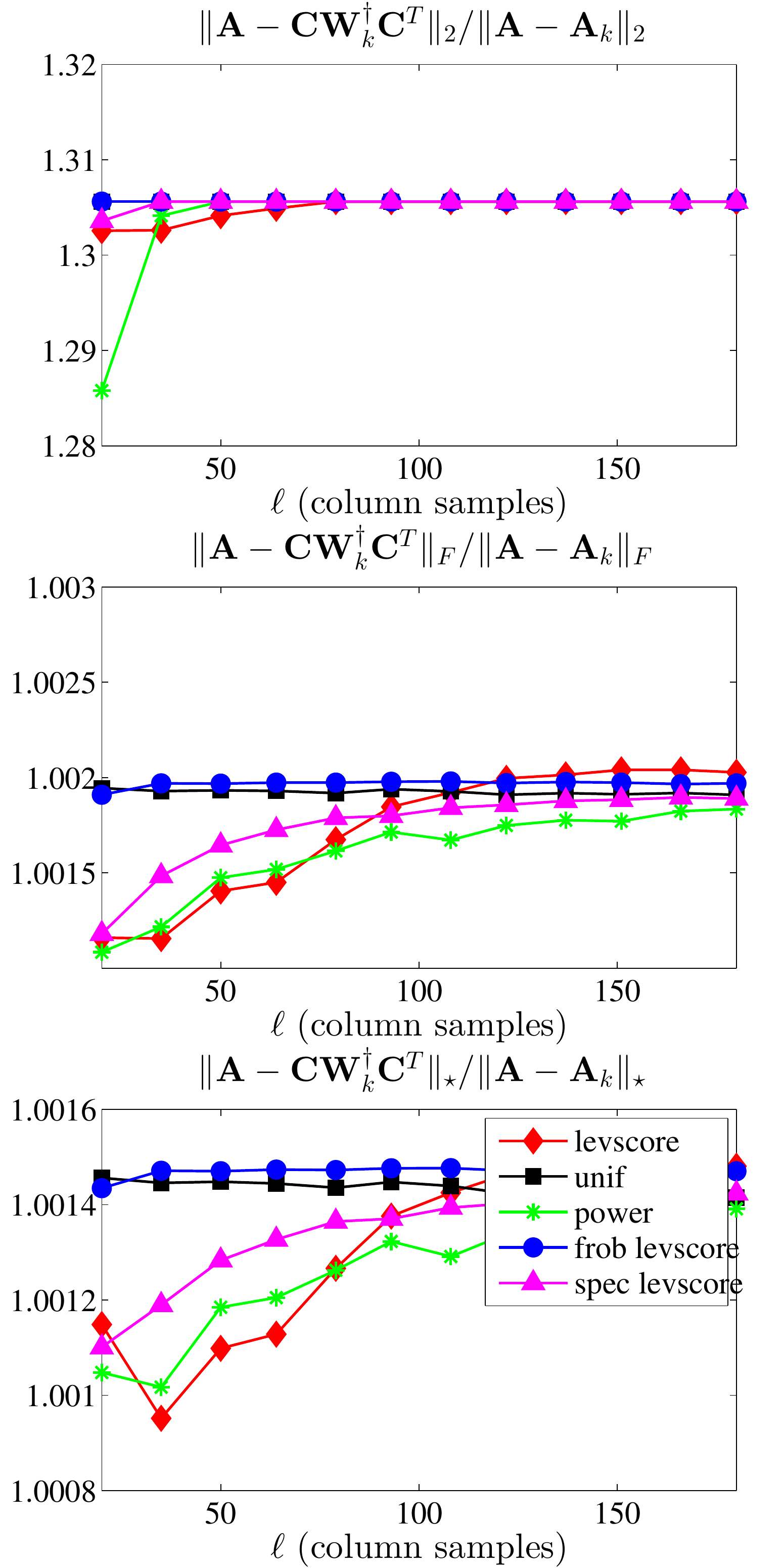}}%
 \subfigure[Enron, $k = 60$]{\includegraphics[width=1.6in, keepaspectratio=true]{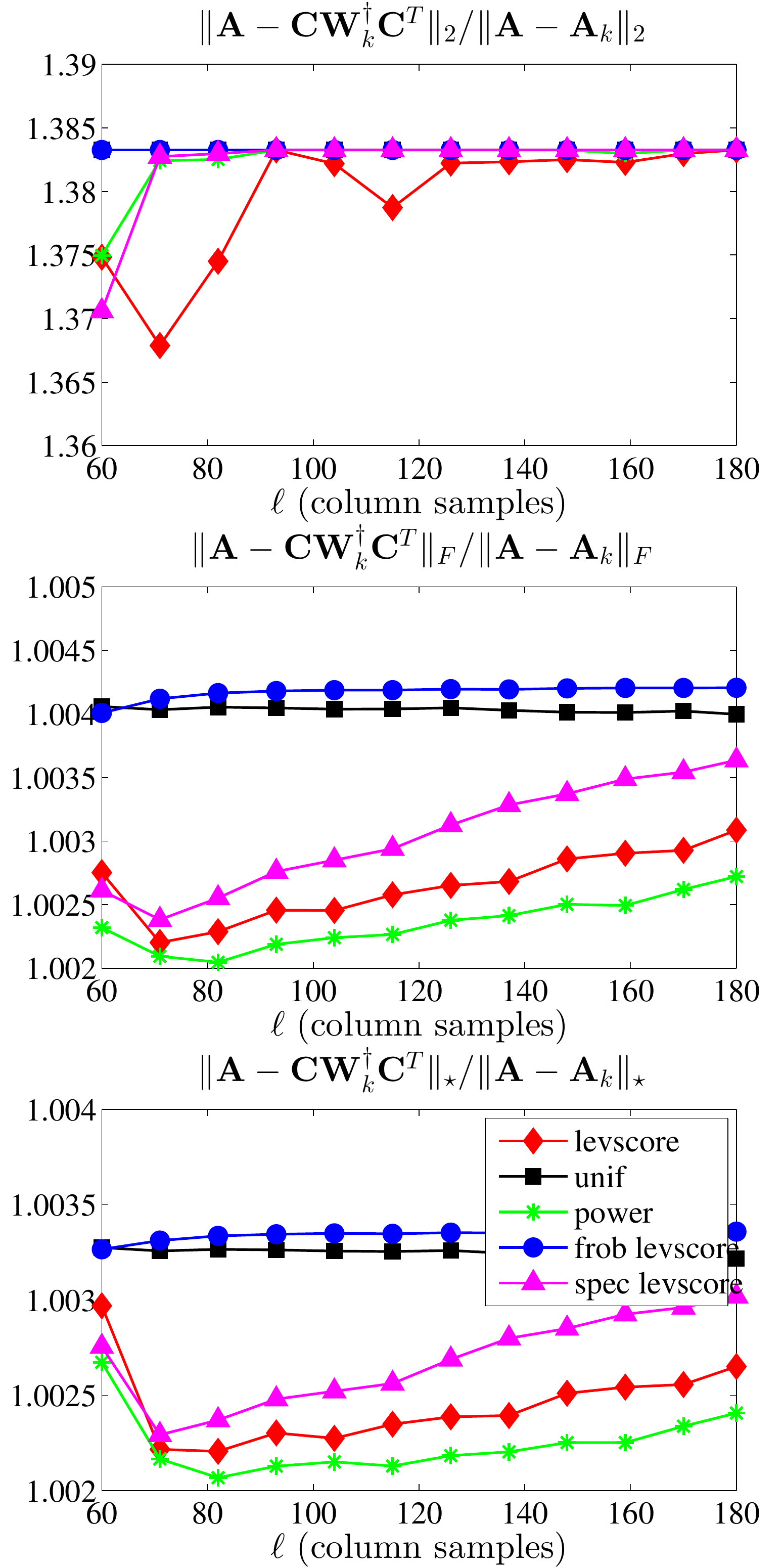}}%
 \subfigure[Gnutella, $k = 20$]{\includegraphics[width=1.6in, keepaspectratio=true]{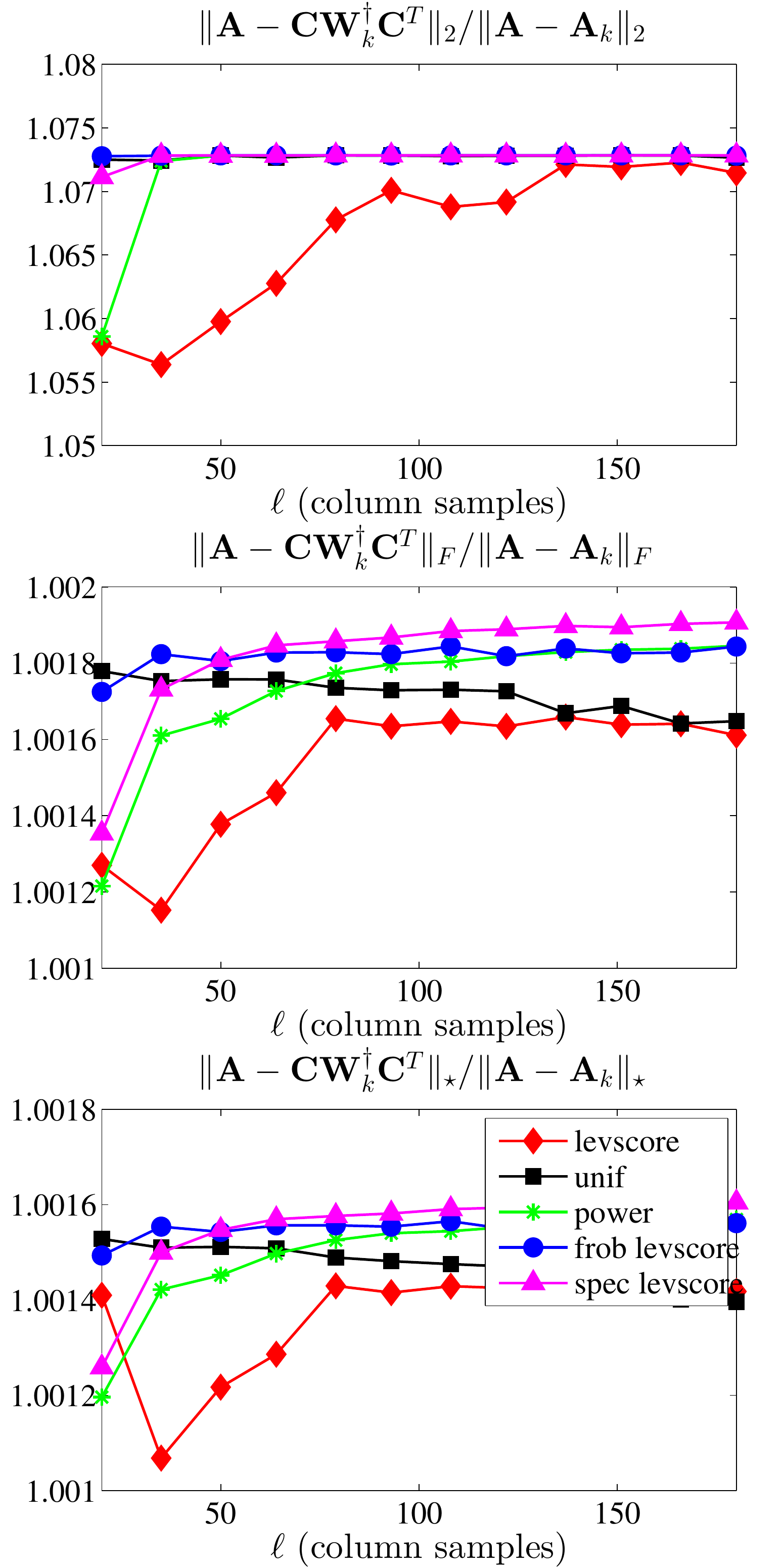}}%
 \subfigure[Gnutella, $k = 60$]{\includegraphics[width=1.6in, keepaspectratio=true]{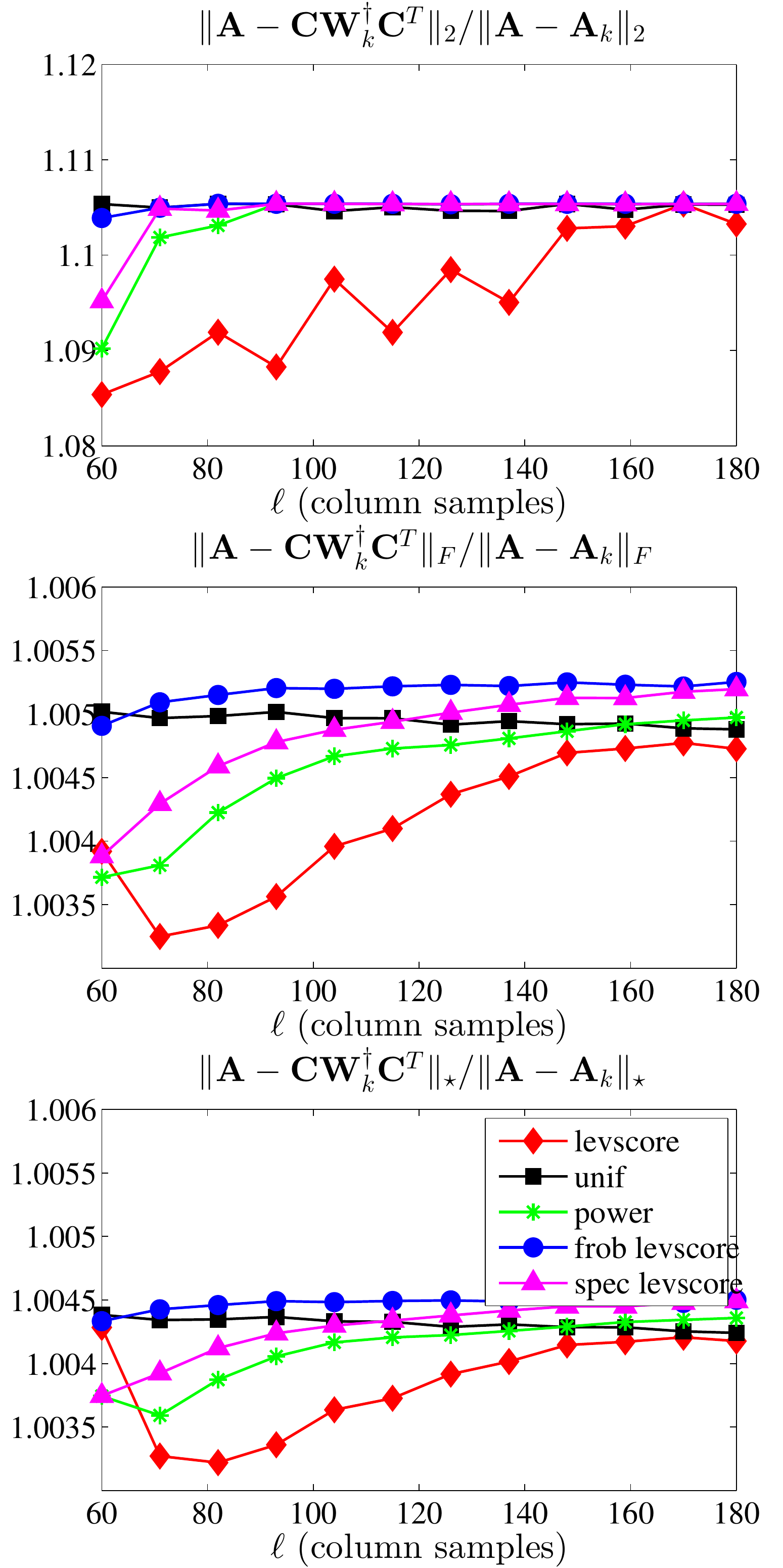}}%
 \caption{The spectral, Frobenius, and trace norm errors 
 (top to bottom, respectively, in each subfigure) of several
 (non-rank-restricted in top panels and rank-restricted in bottom panels)
 approximate leverage score-based SPSD sketches, as a function of the number of columns samples $\ell$, 
 for the Enron and Gnutella Laplacian data sets, with two choices of the rank parameter $k$.}%
 \label{fig:laplacian-inexact-nonfiltered-errors-B}
\end{figure}

\begin{figure}[p]
 \centering
 \subfigure[Dexter, $k = 8$]{\includegraphics[width=1.6in, keepaspectratio=true]{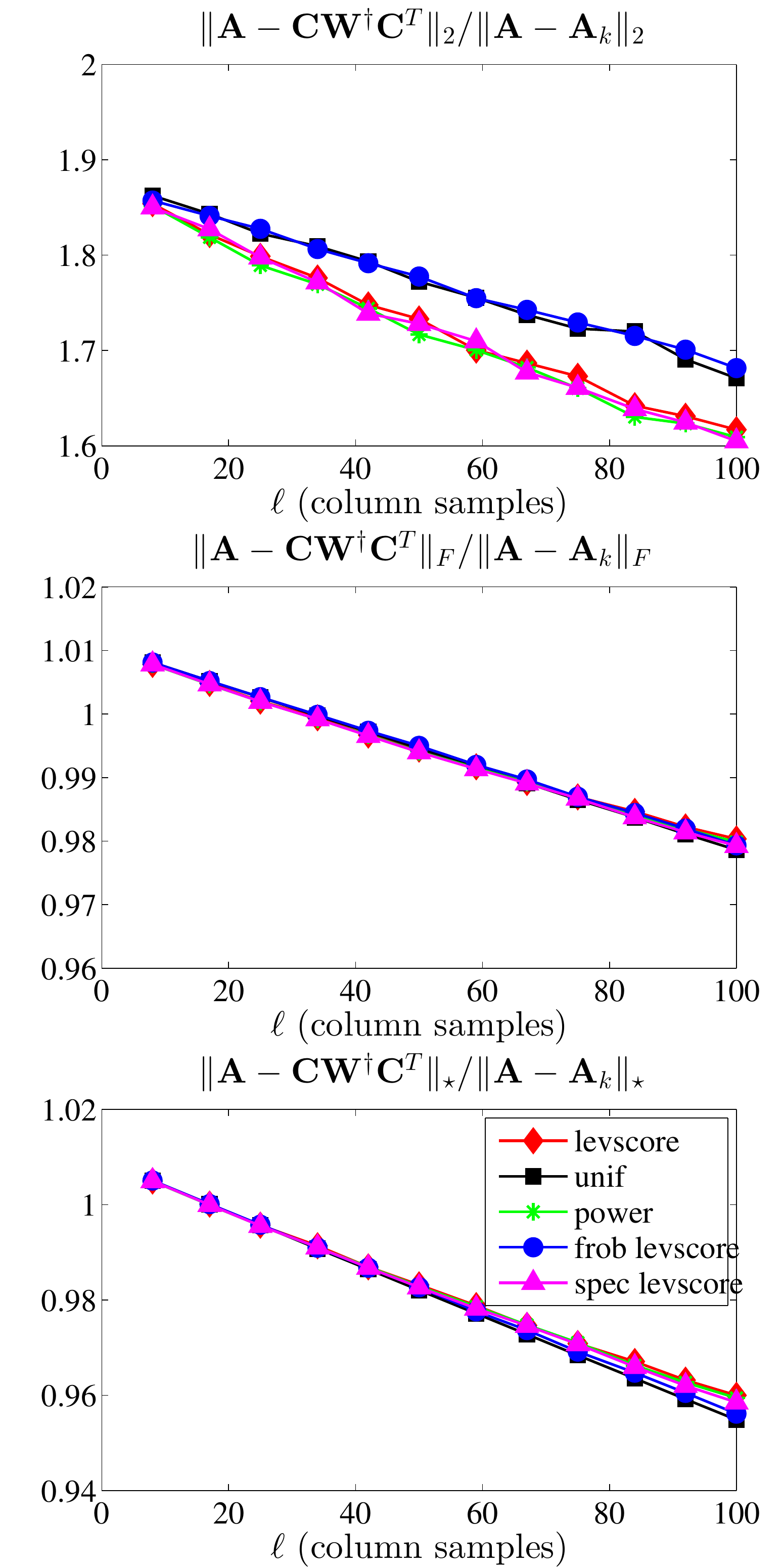}}%
 \subfigure[Protein, $k = 10$]{\includegraphics[width=1.6in, keepaspectratio=true]{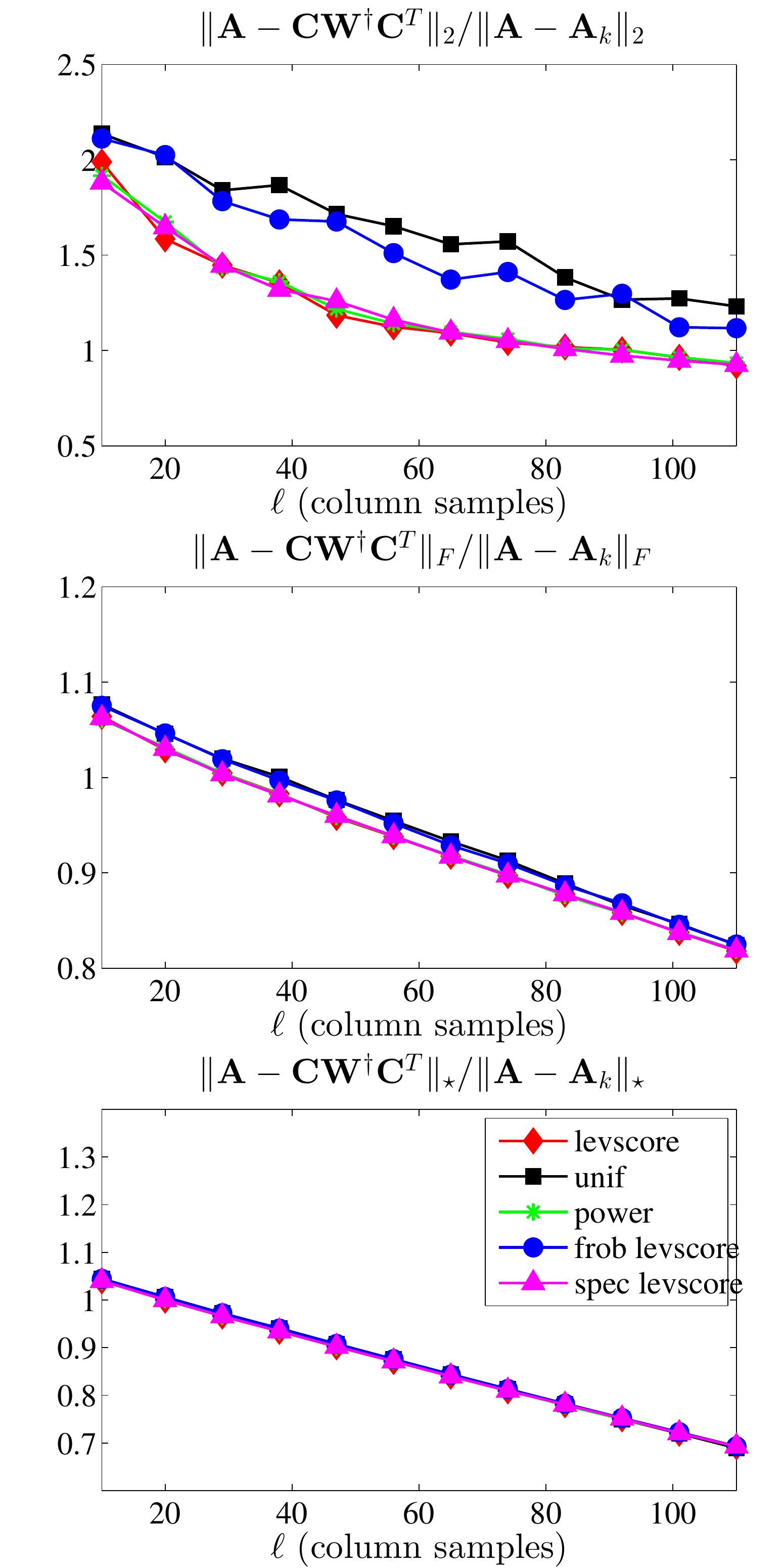}}%
 \subfigure[SNPs, $k = 5$]{\includegraphics[width=1.6in, keepaspectratio=true]{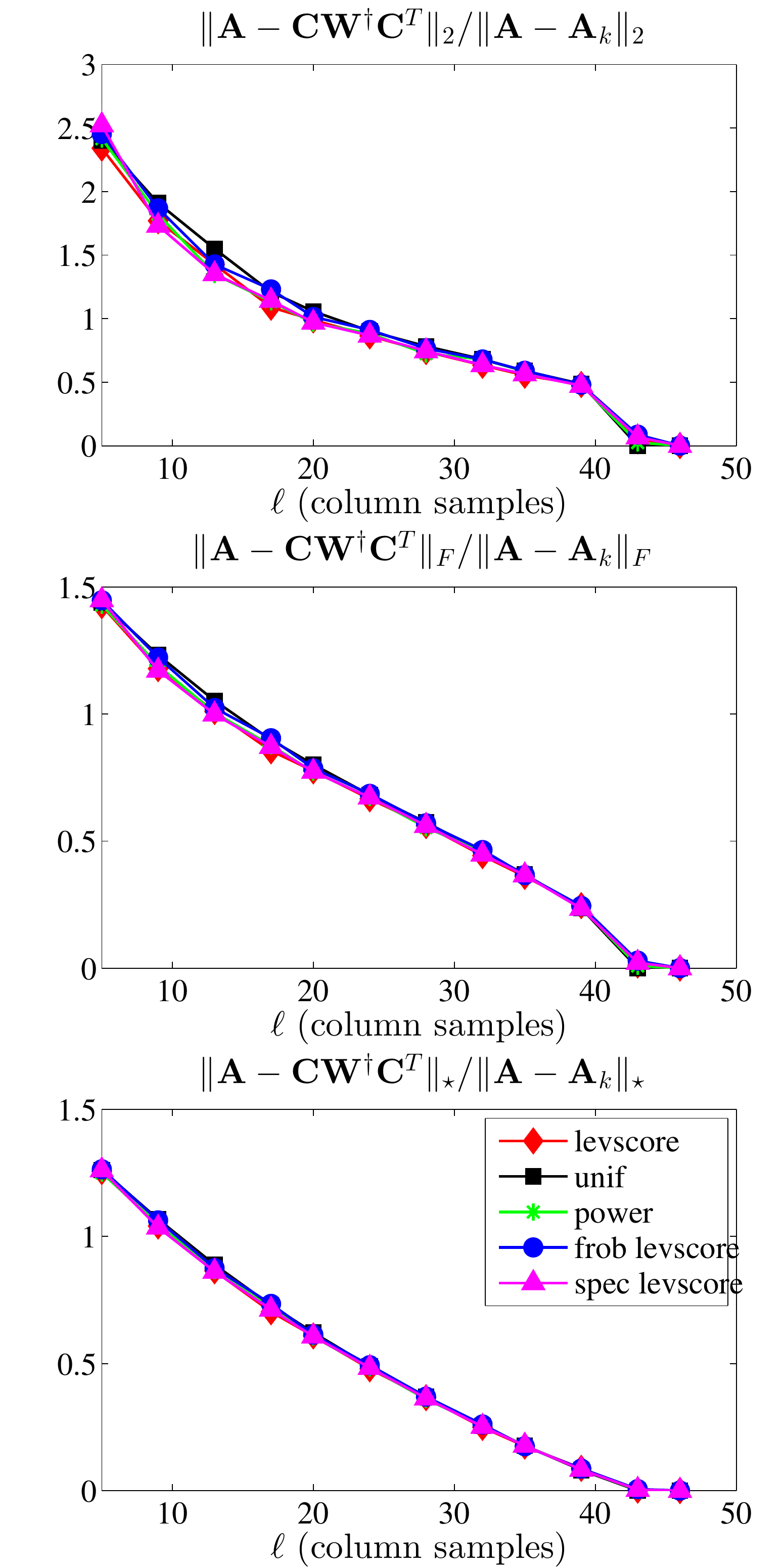}}%
 \subfigure[Gisette, $k = 12$]{\includegraphics[width=1.6in, keepaspectratio=true]{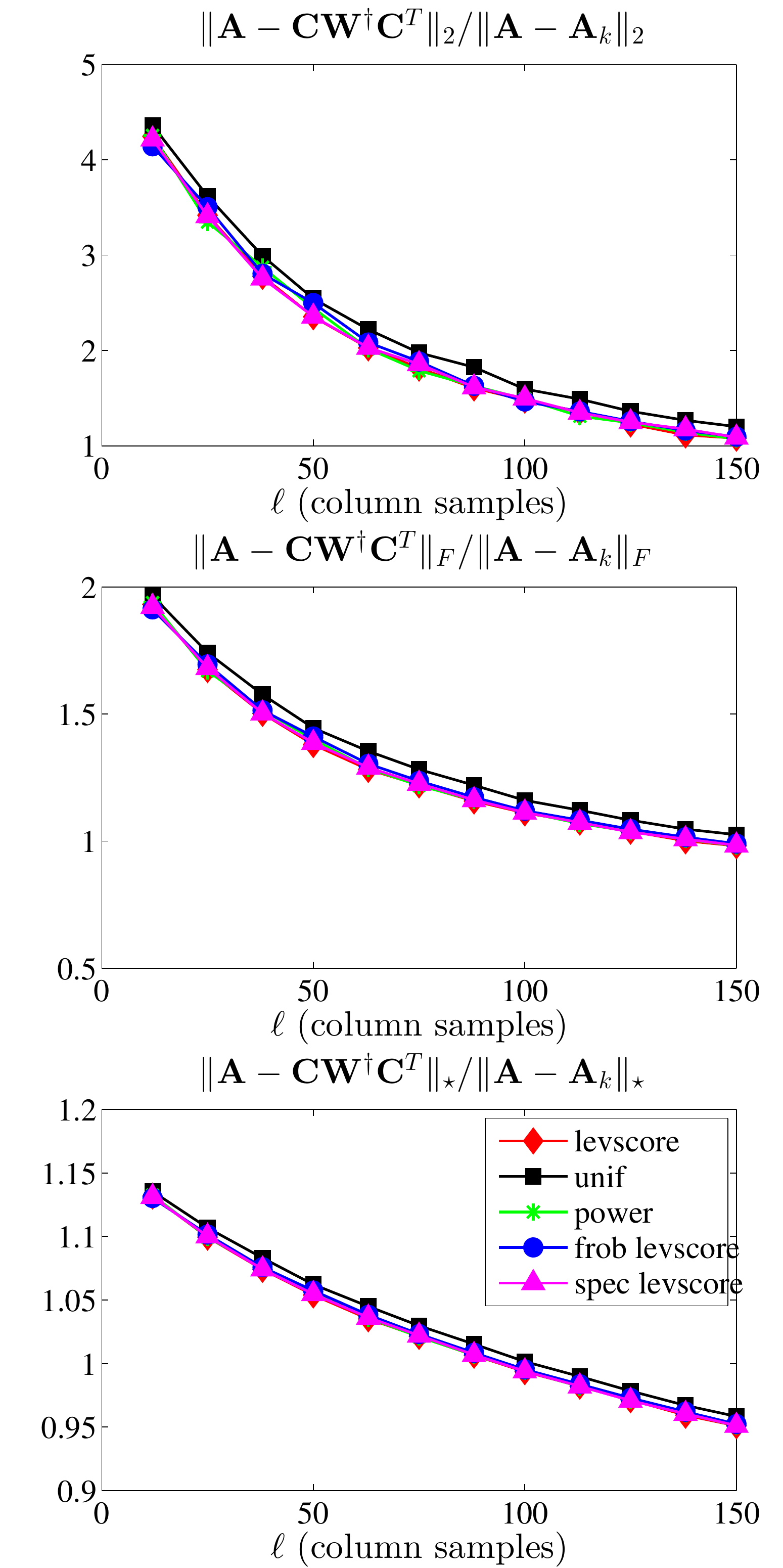}}%
 \\%
 \subfigure[Dexter, $k = 8$]{\includegraphics[width=1.6in, keepaspectratio=true]{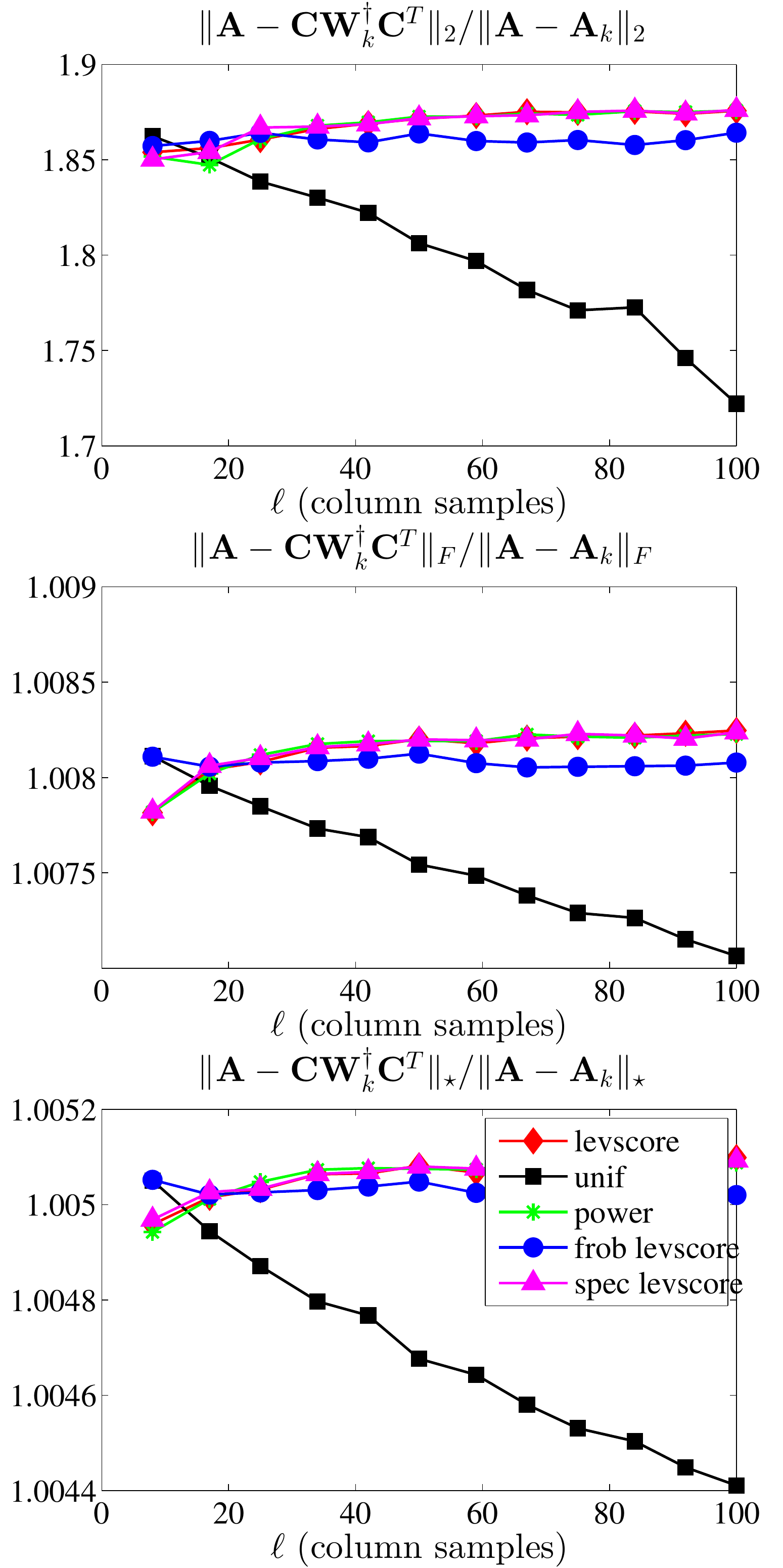}}%
 \subfigure[Protein, $k = 10$]{\includegraphics[width=1.6in, keepaspectratio=true]{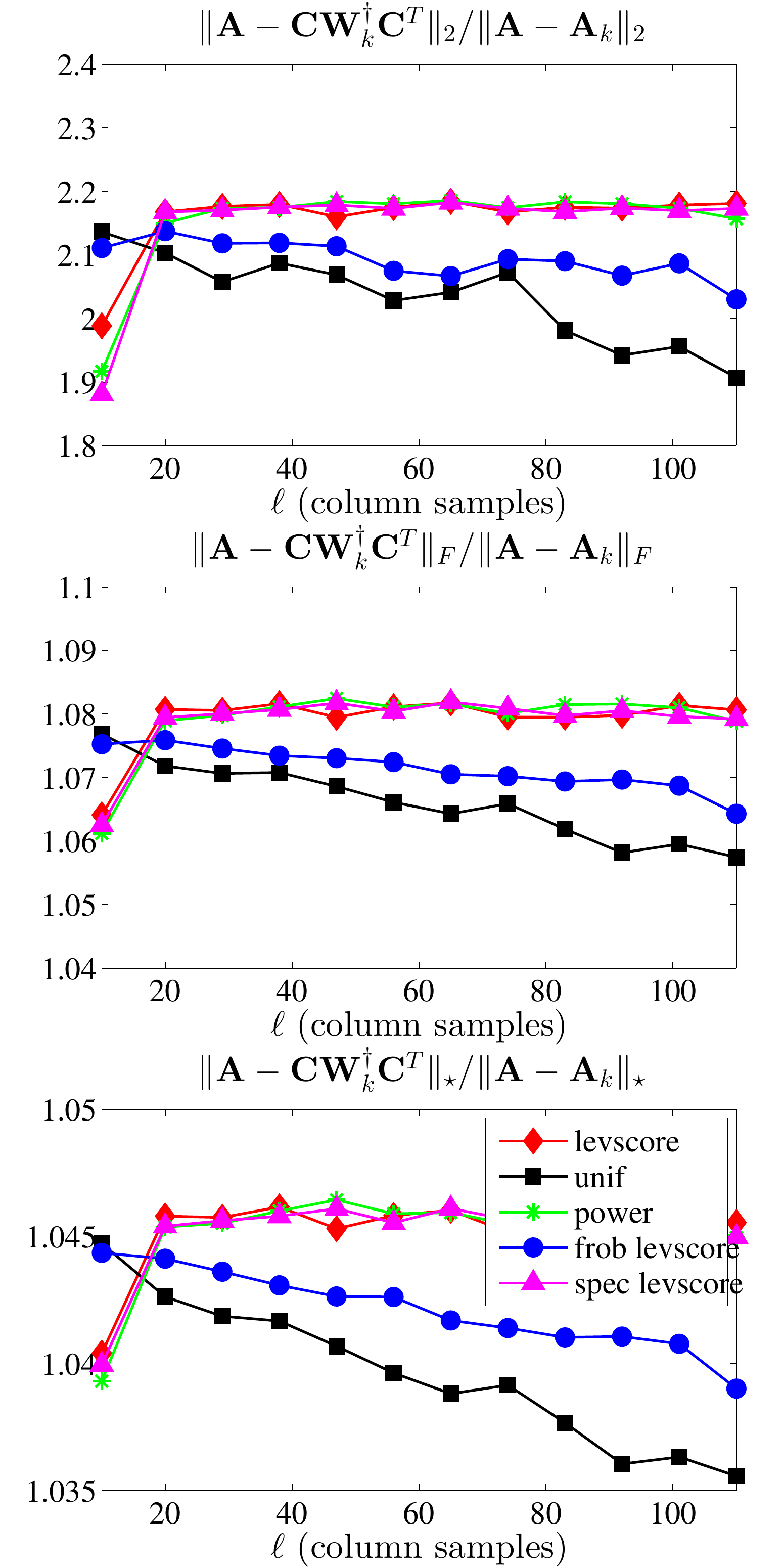}}%
 \subfigure[SNPs, $k = 5$]{\includegraphics[width=1.6in, keepaspectratio=true]{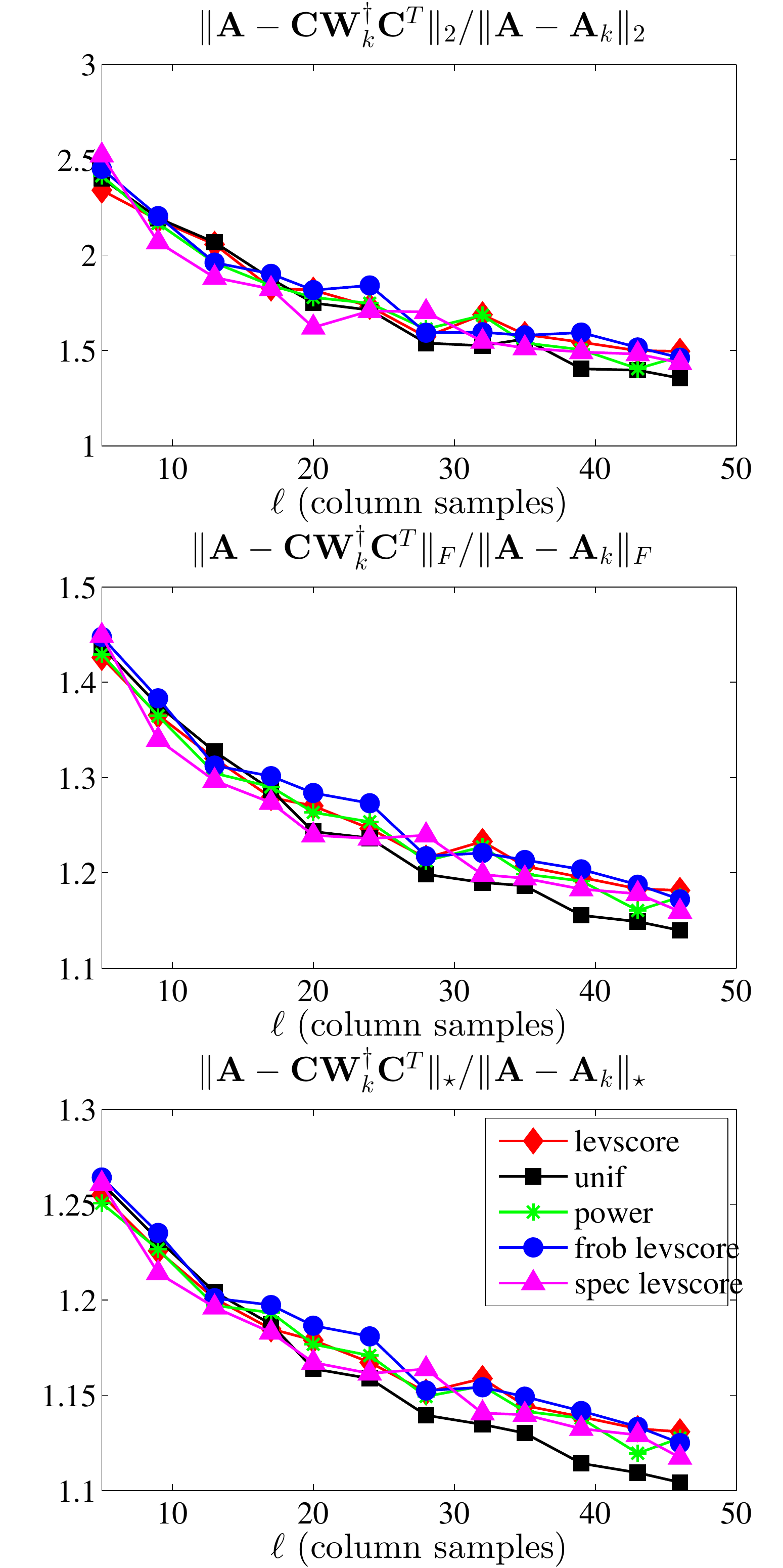}}%
 \subfigure[Gisette, $k = 12$]{\includegraphics[width=1.6in, keepaspectratio=true]{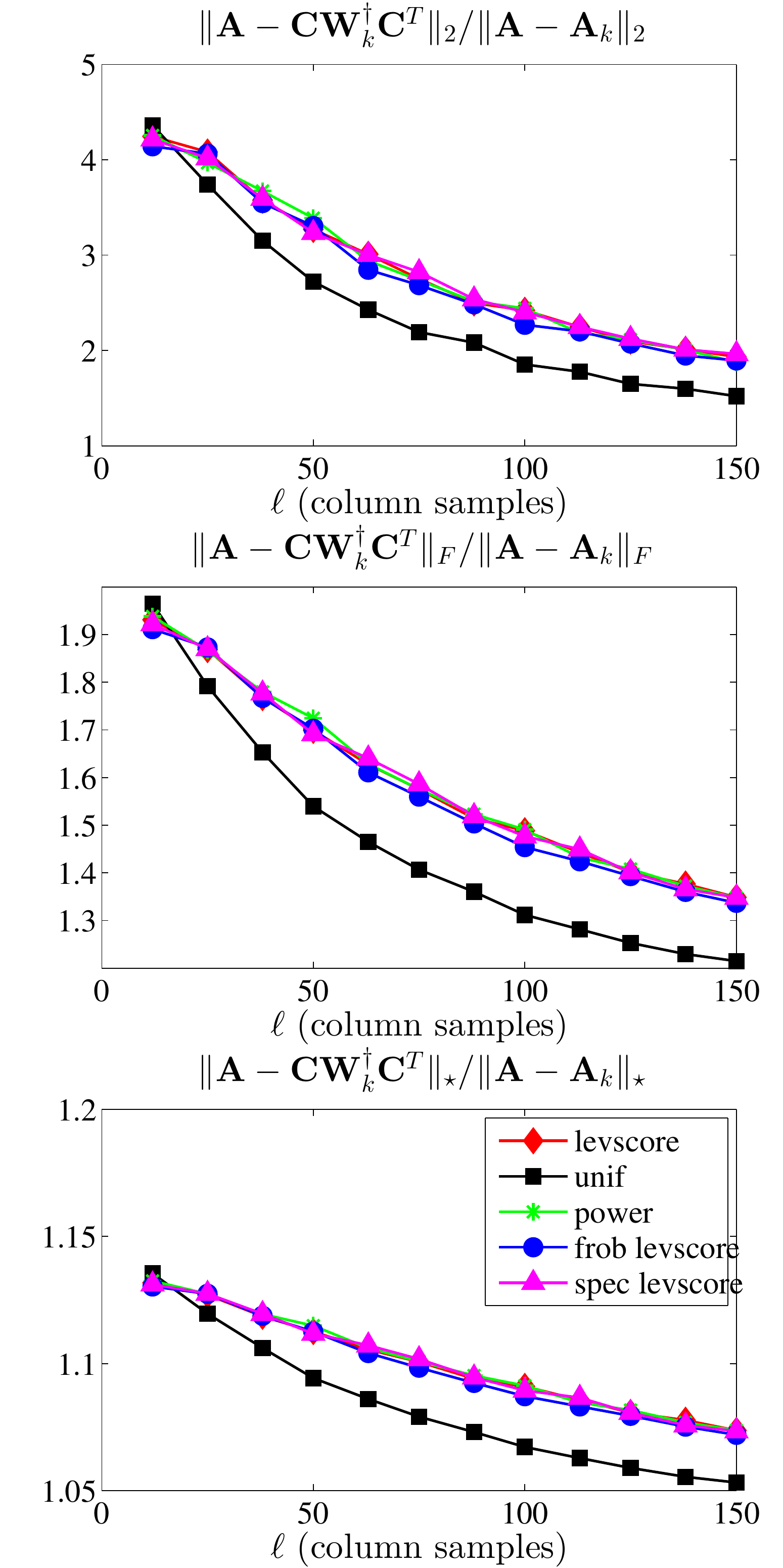}}%
 \caption{The spectral, Frobenius, and trace norm errors 
 (top to bottom, respectively, in each subfigure) of several
 (non-rank-restricted in top panels and rank-restricted in bottom panels)
 approximate leverage score-based SPSD sketches, as a function of the number of columns samples $\ell$, 
 for the Linear Kernel data sets}%
 \label{fig:linearkernel-inexact-errors}
\end{figure}

\begin{figure}[p]
 \centering
 \subfigure[AbaloneD, $\sigma = .15, k = 20$]{\includegraphics[width=1.6in, keepaspectratio=true]{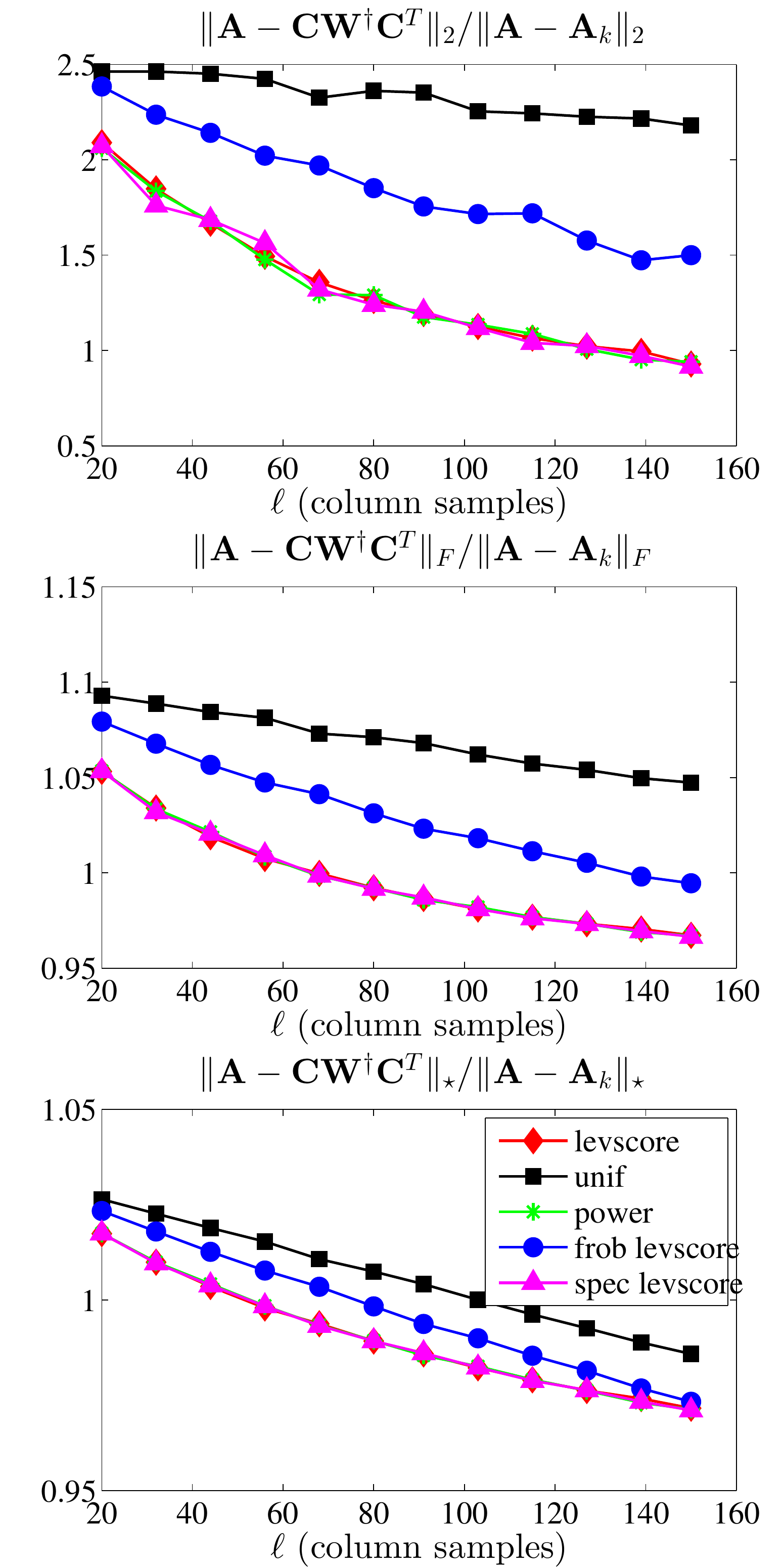}}%
 \subfigure[AbaloneD, $\sigma = 1, k = 20$]{\includegraphics[width=1.6in, keepaspectratio=true]{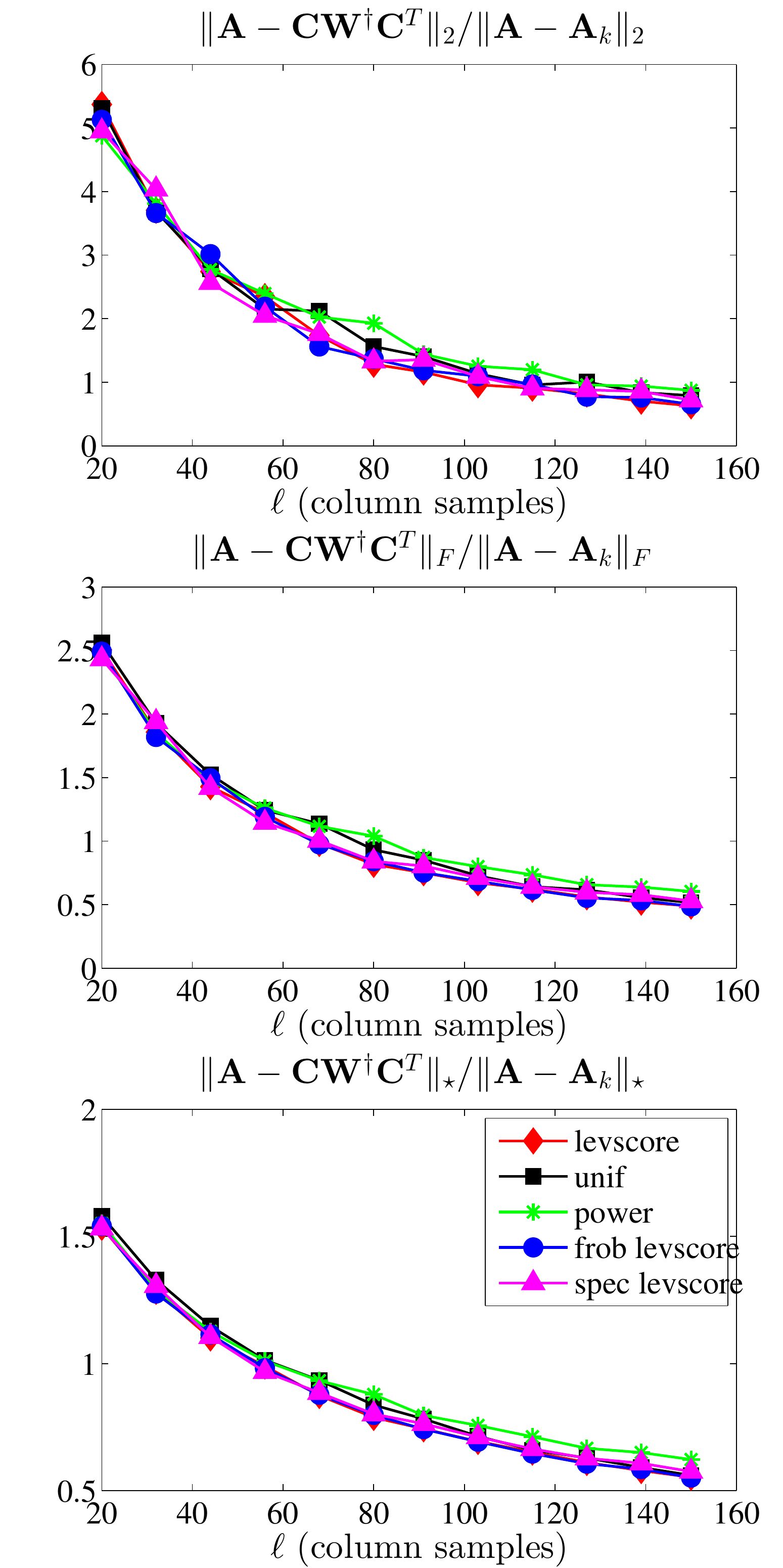}}%
 \subfigure[WineD, $\sigma = 1, k = 20$]{\includegraphics[width=1.6in, keepaspectratio=true]{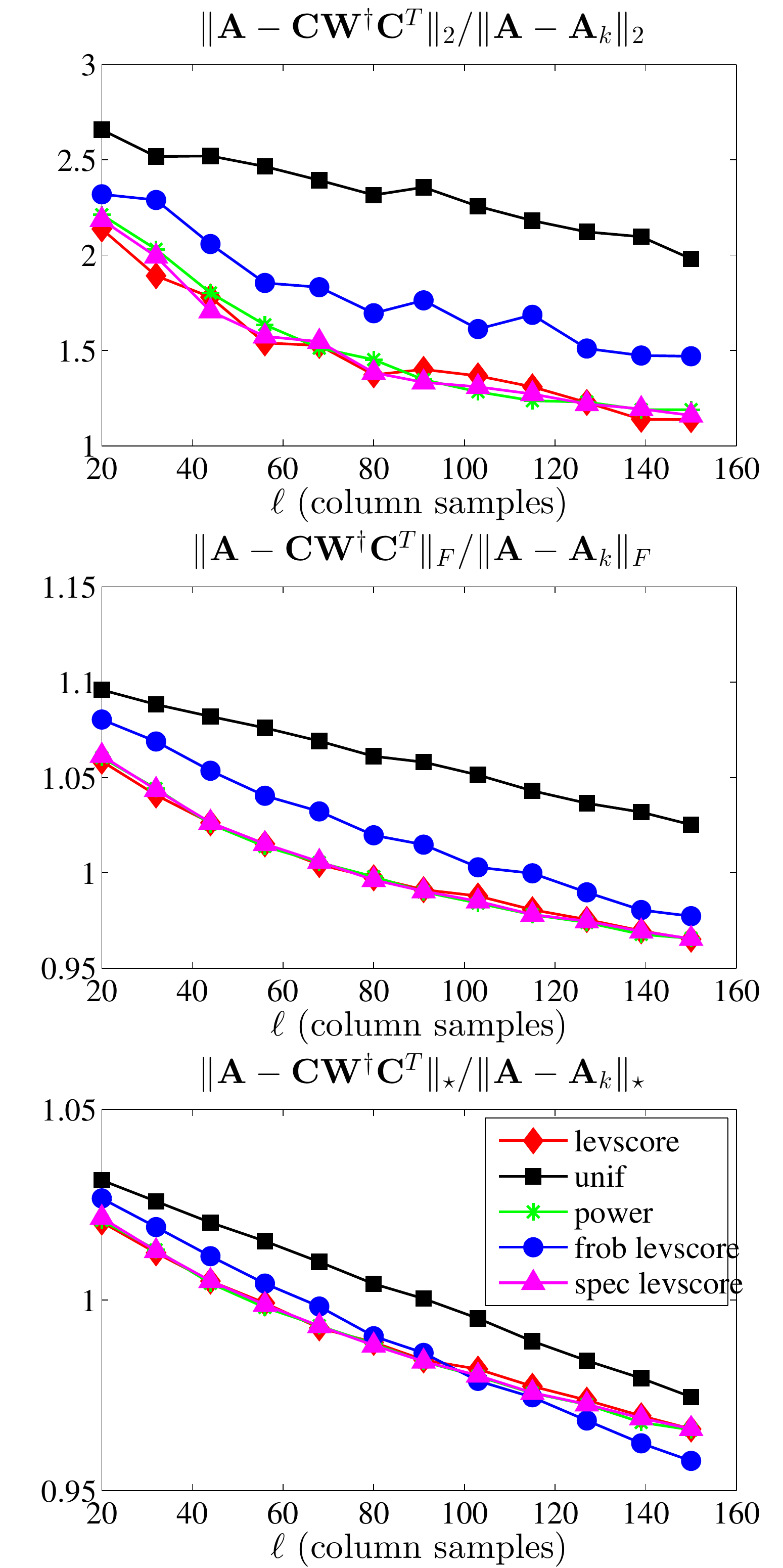}}%
 \subfigure[WineD, $\sigma = 2.1, k = 20$]{\includegraphics[width=1.6in, keepaspectratio=true]{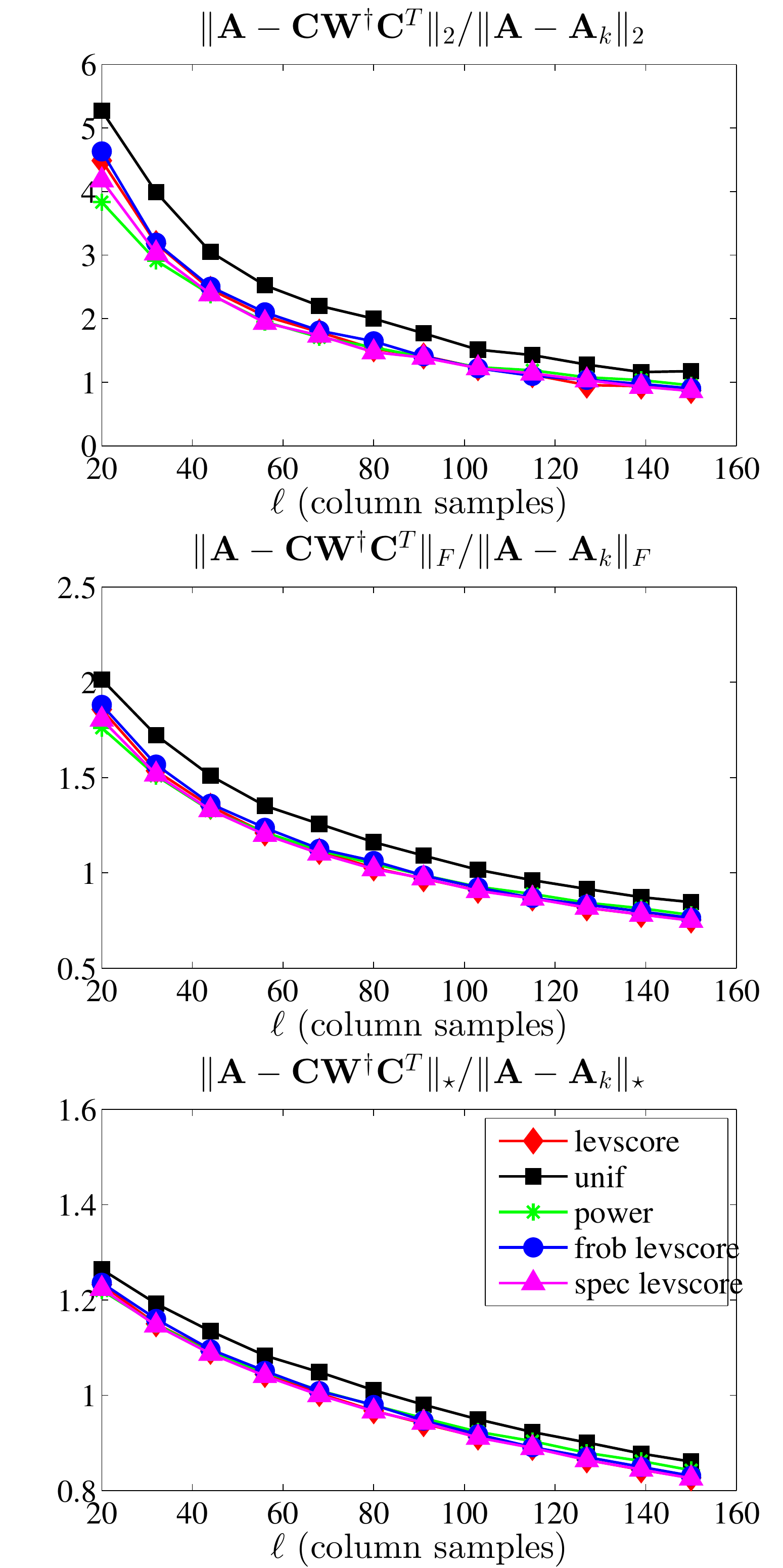}}%
 \\%
 \subfigure[AbaloneD, $\sigma = .15, k = 20$]{\includegraphics[width=1.6in, keepaspectratio=true]{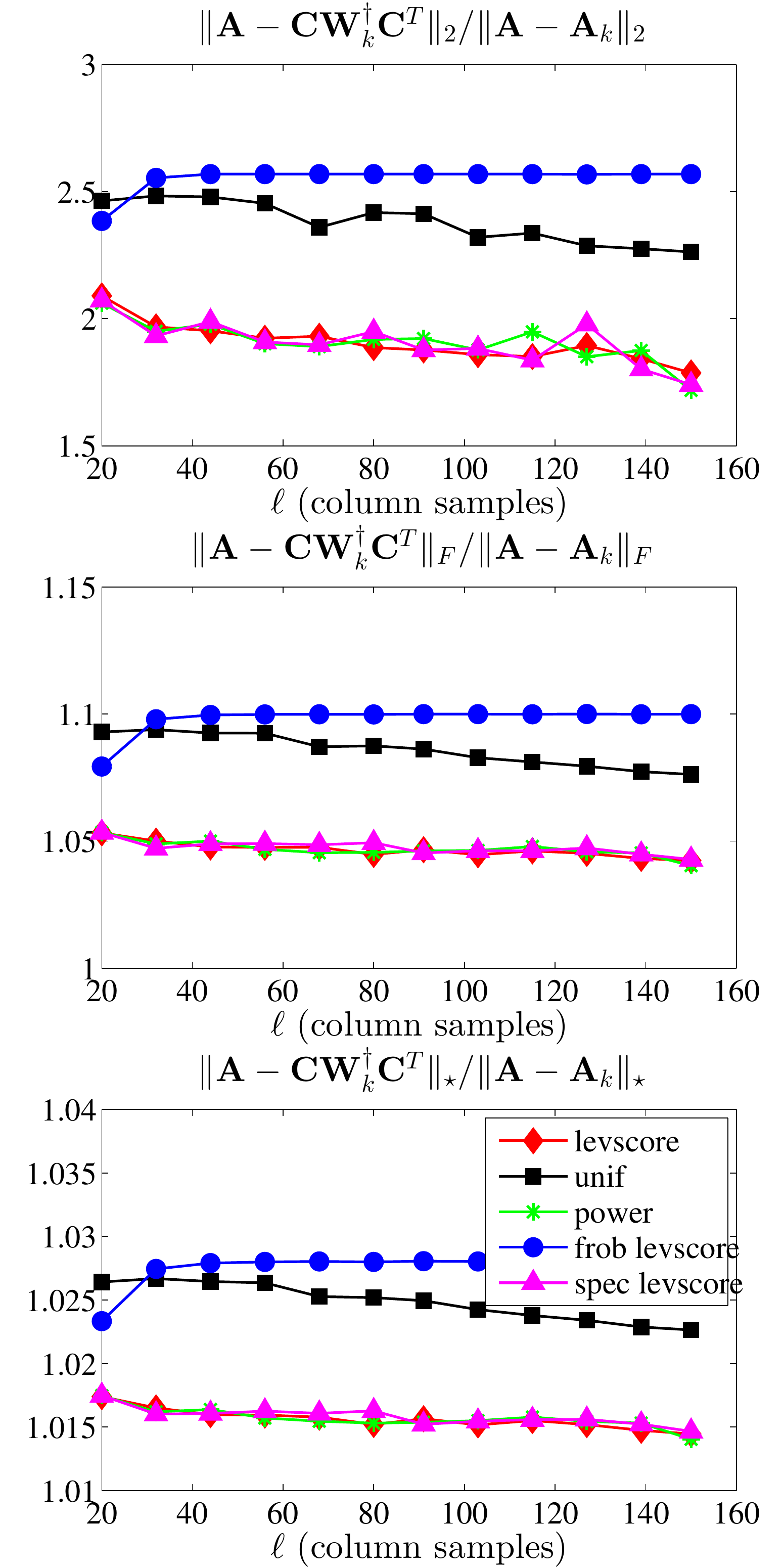}}%
 \subfigure[AbaloneD, $\sigma = 1, k = 20$]{\includegraphics[width=1.6in, keepaspectratio=true]{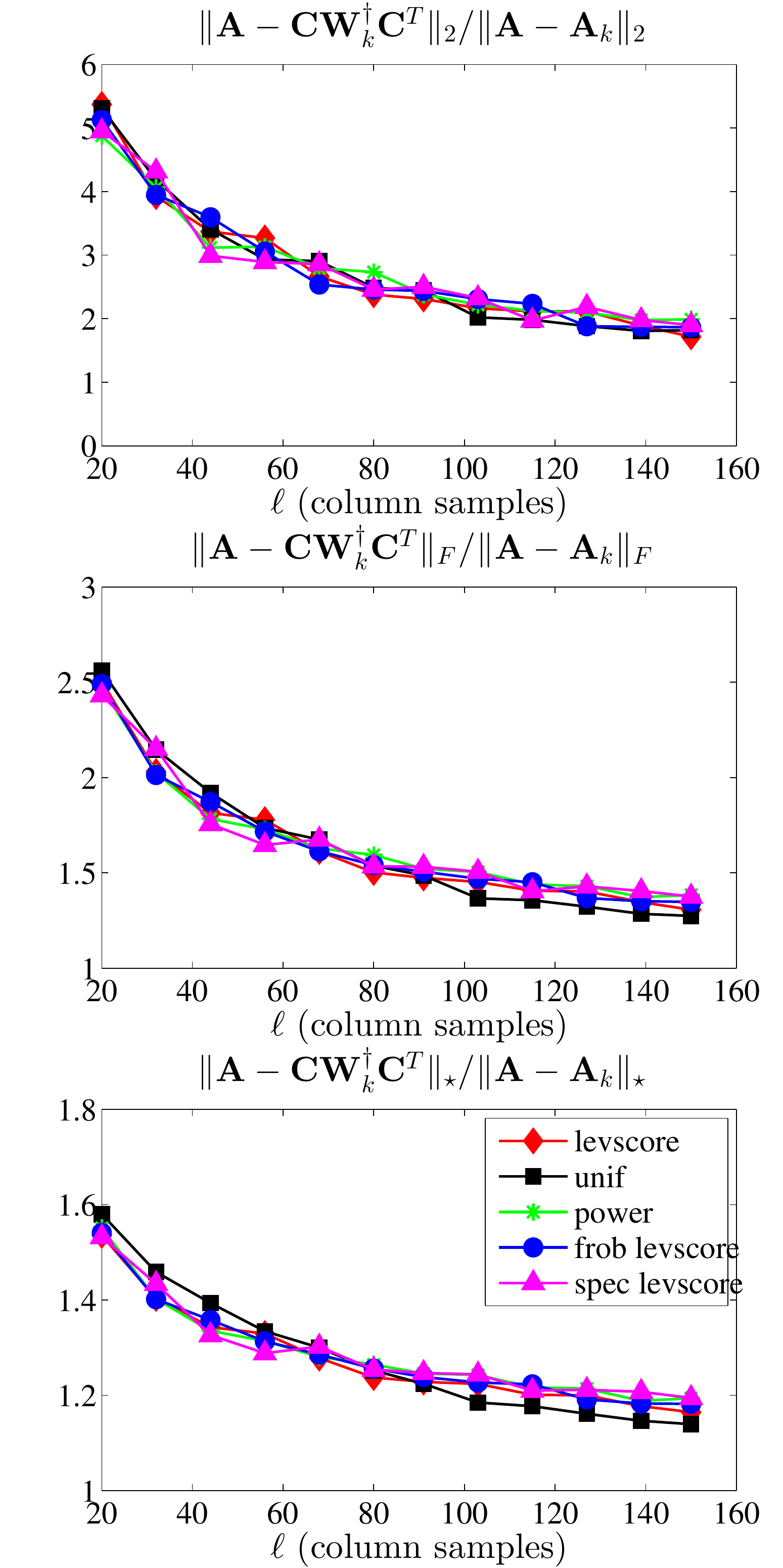}}%
 \subfigure[WineD, $\sigma = 1, k = 20$]{\includegraphics[width=1.6in, keepaspectratio=true]{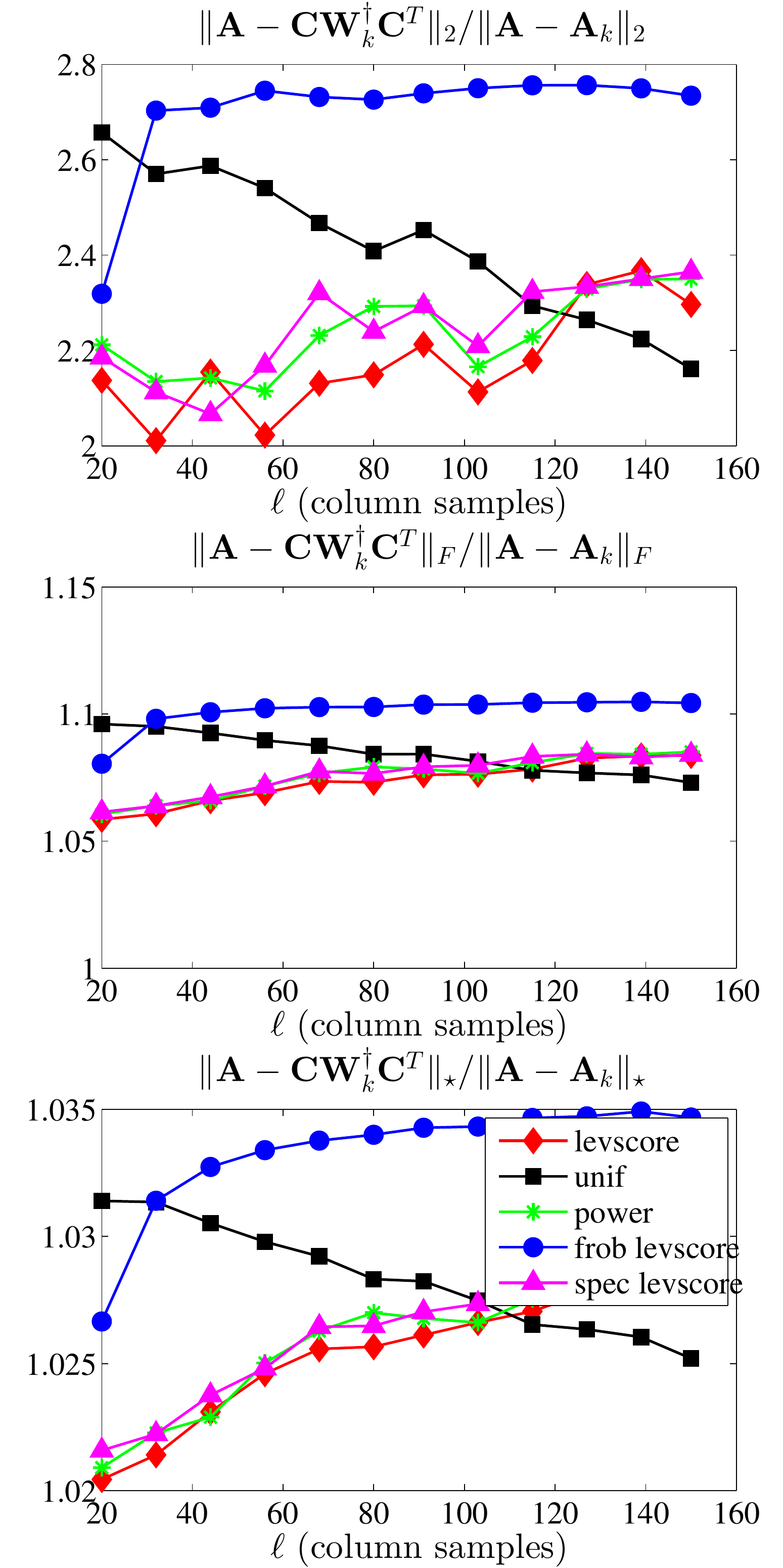}}%
 \subfigure[WineD, $\sigma = 2.1, k = 20$]{\includegraphics[width=1.6in, keepaspectratio=true]{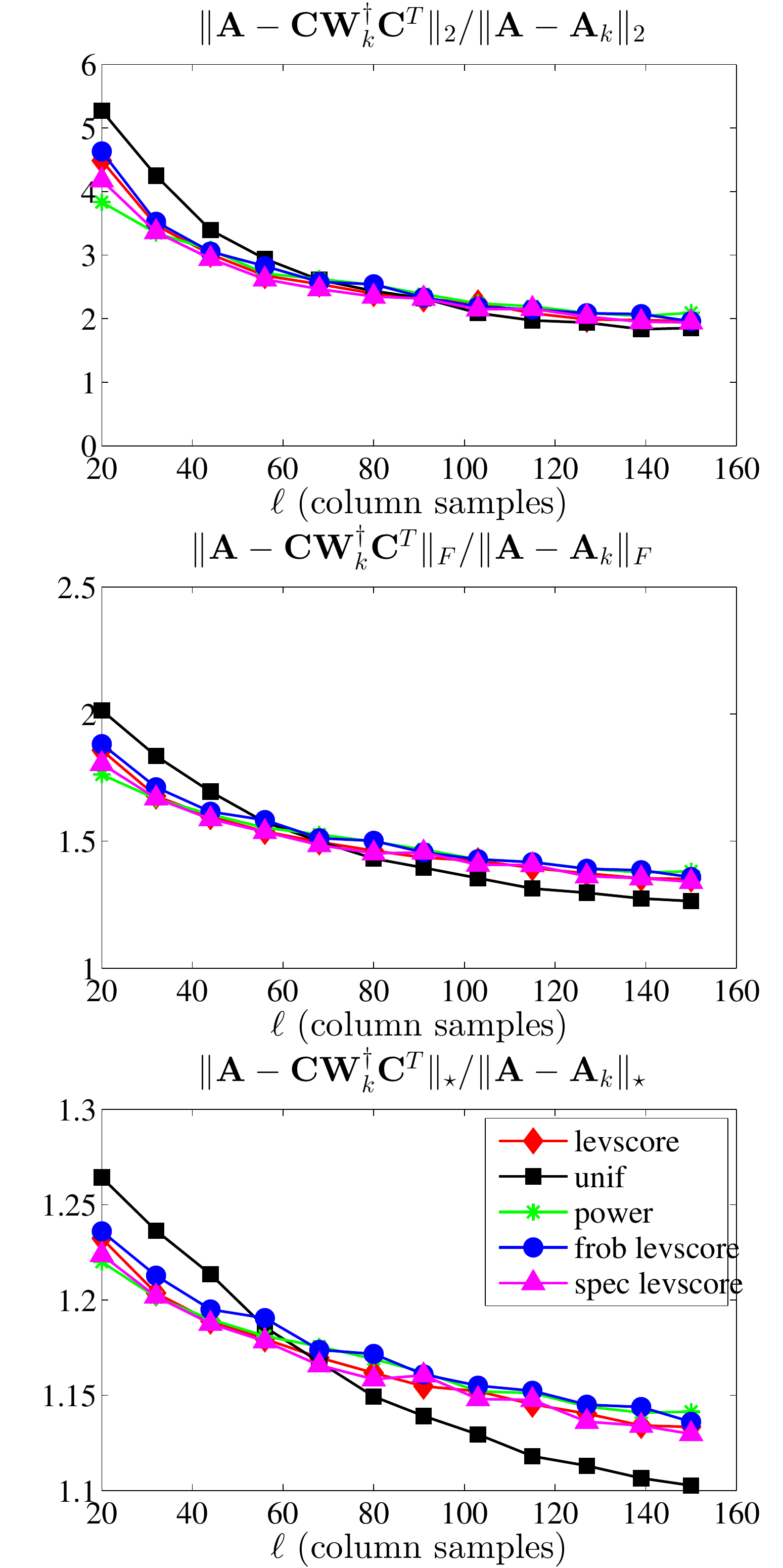}}%
 \caption{The spectral, Frobenius, and trace norm errors 
 (top to bottom, respectively, in each subfigure) of several
 (non-rank-restricted in top panels and rank-restricted in bottom panels)
 approximate leverage score-based SPSD sketches, as a function of the number of columns samples $\ell$, 
 for several dense RBF data sets.}%
 \label{fig:denserbf-inexact-nonfiltered-errors}
\end{figure}

\begin{figure}[p]
 \centering
 \subfigure[AbaloneS, $\sigma = .15, k = 20$]{\includegraphics[width=1.6in, keepaspectratio=true]{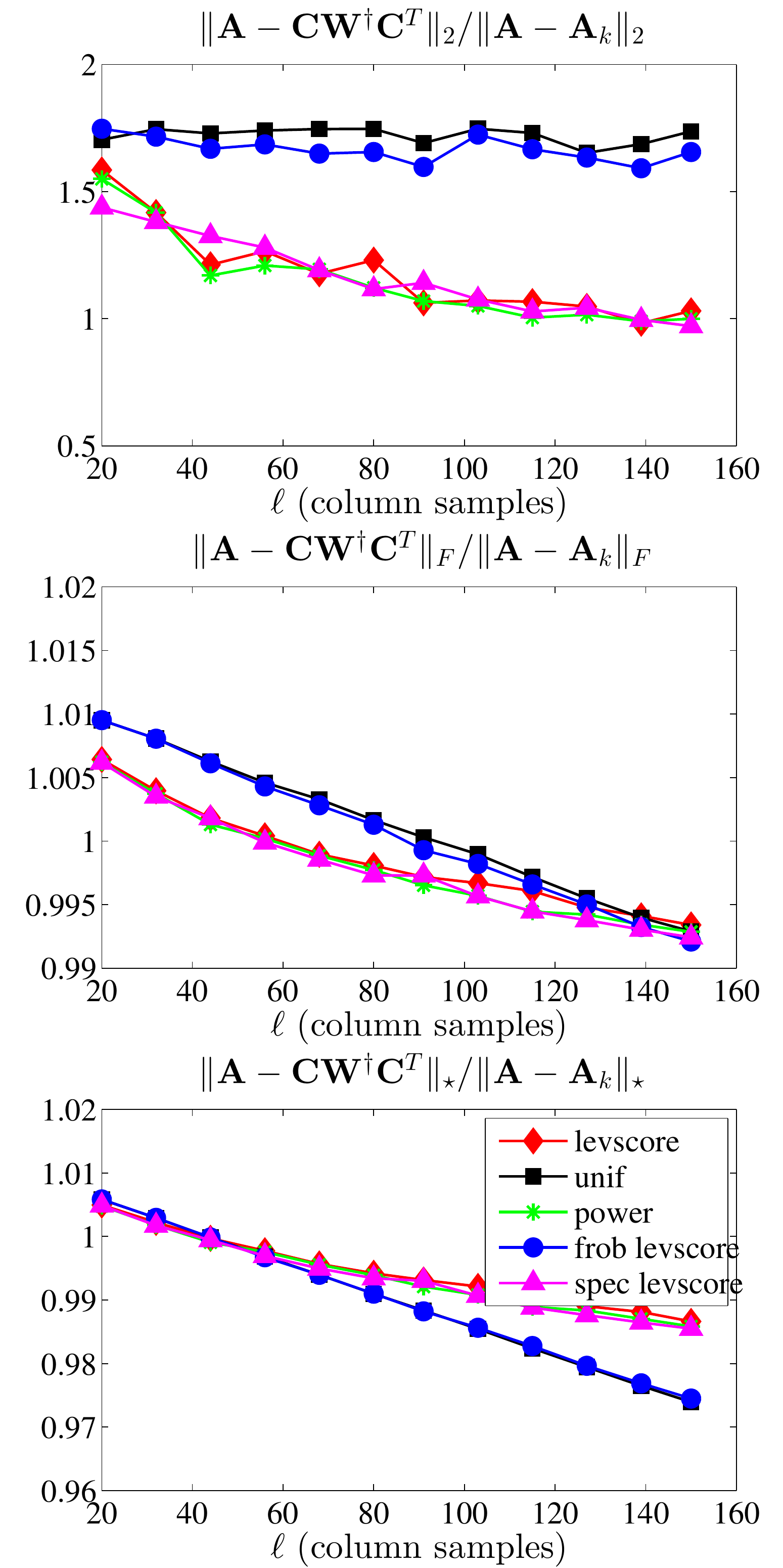}}%
 \subfigure[AbaloneS, $\sigma = 1, k = 20$]{\includegraphics[width=1.6in, keepaspectratio=true]{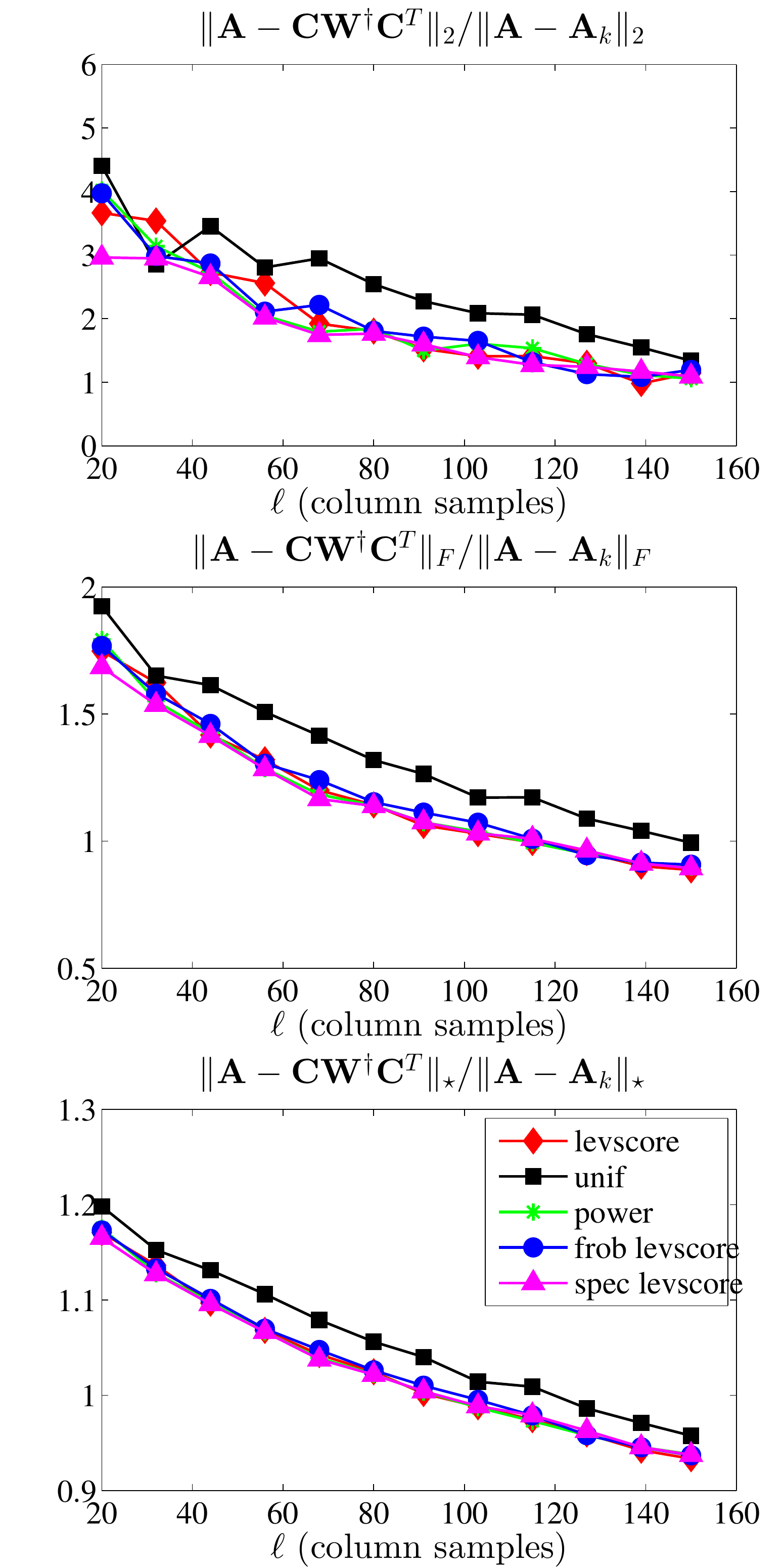}}%
 \subfigure[WineS, $\sigma = 1, k = 20$]{\includegraphics[width=1.6in, keepaspectratio=true]{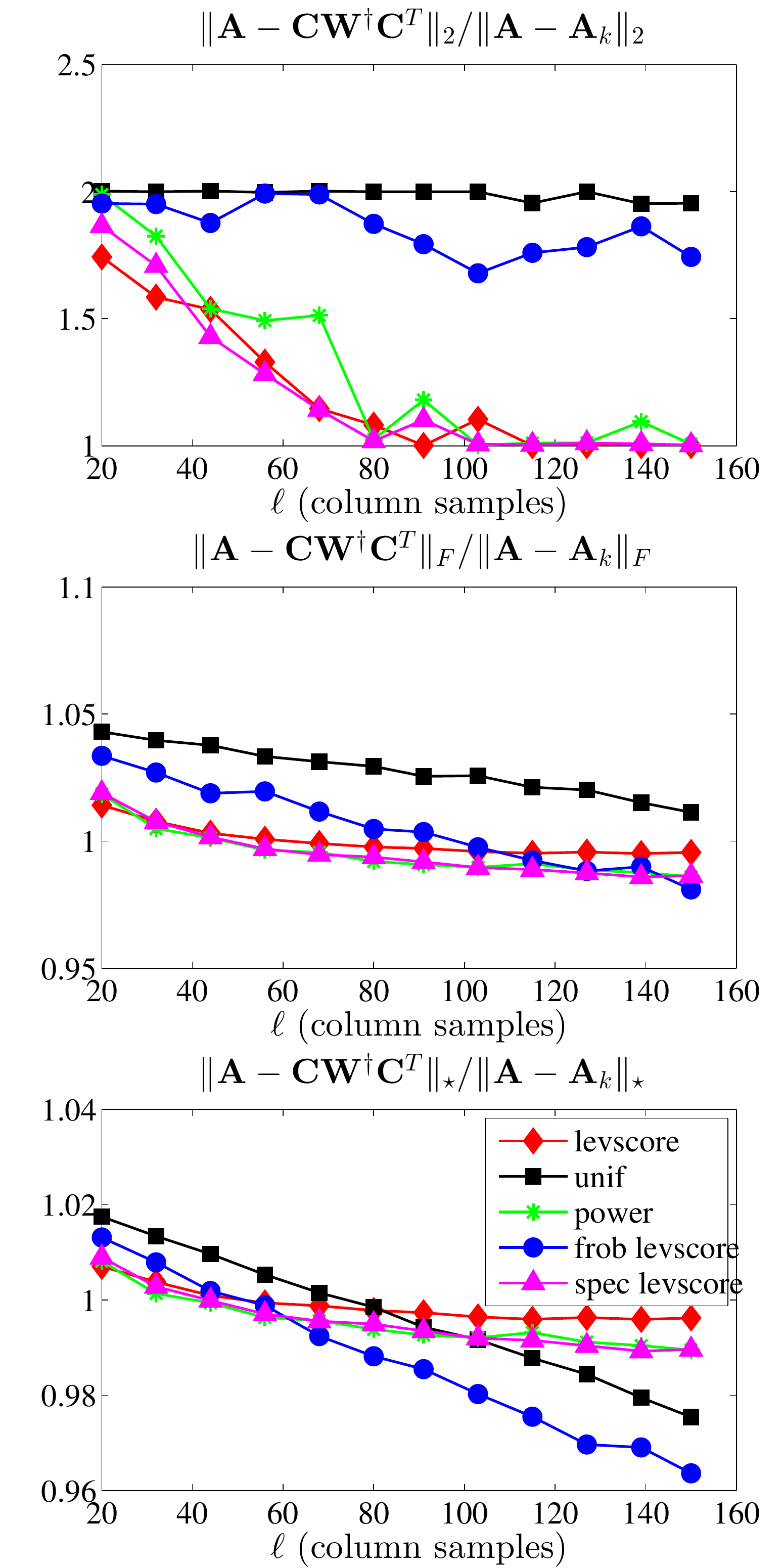}}%
 \subfigure[WineS, $\sigma = 2.1, k = 20$]{\includegraphics[width=1.6in, keepaspectratio=true]{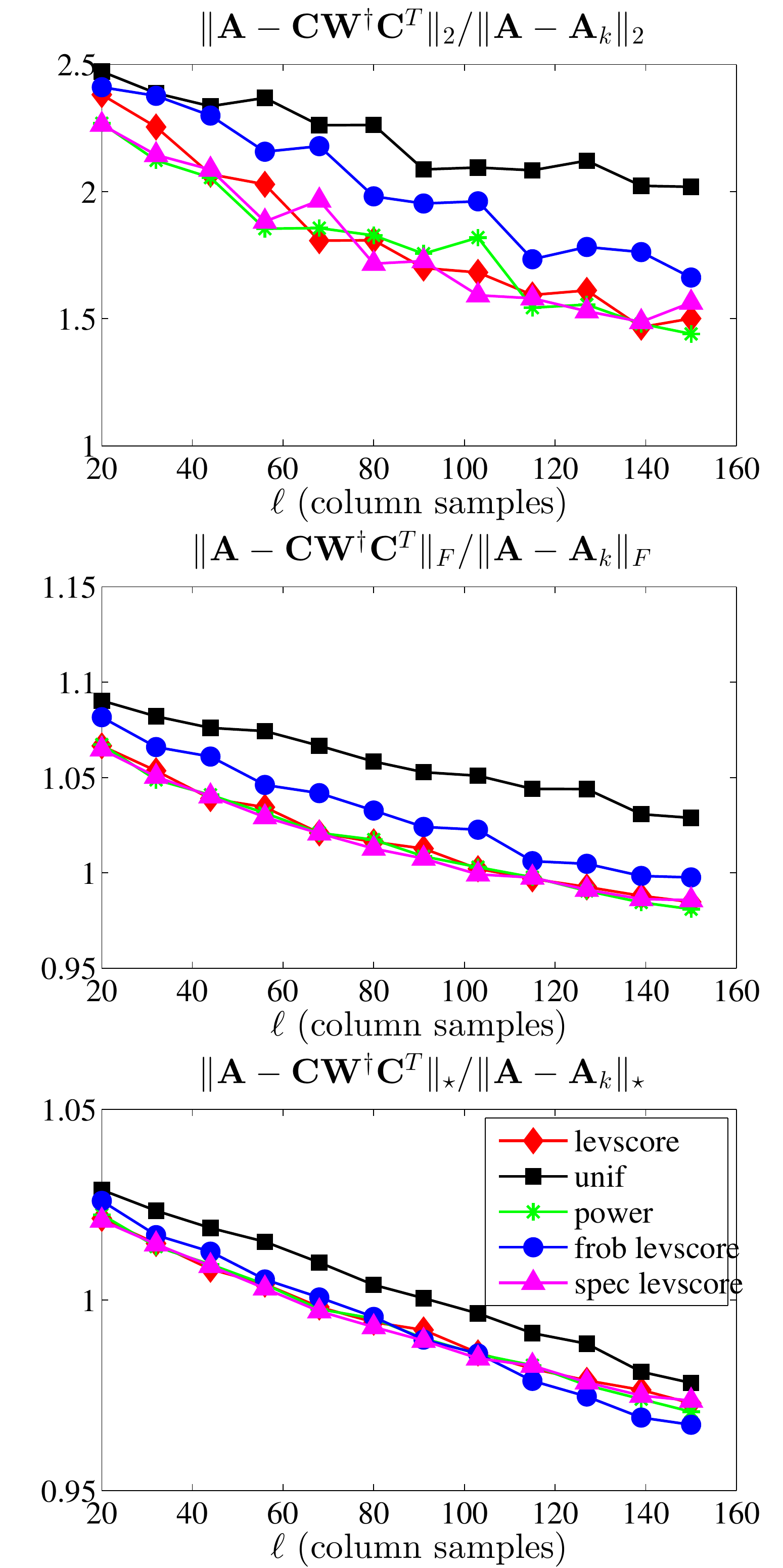}}%
 \\%
 \subfigure[AbaloneS, $\sigma = .15, k = 20$]{\includegraphics[width=1.6in, keepaspectratio=true]{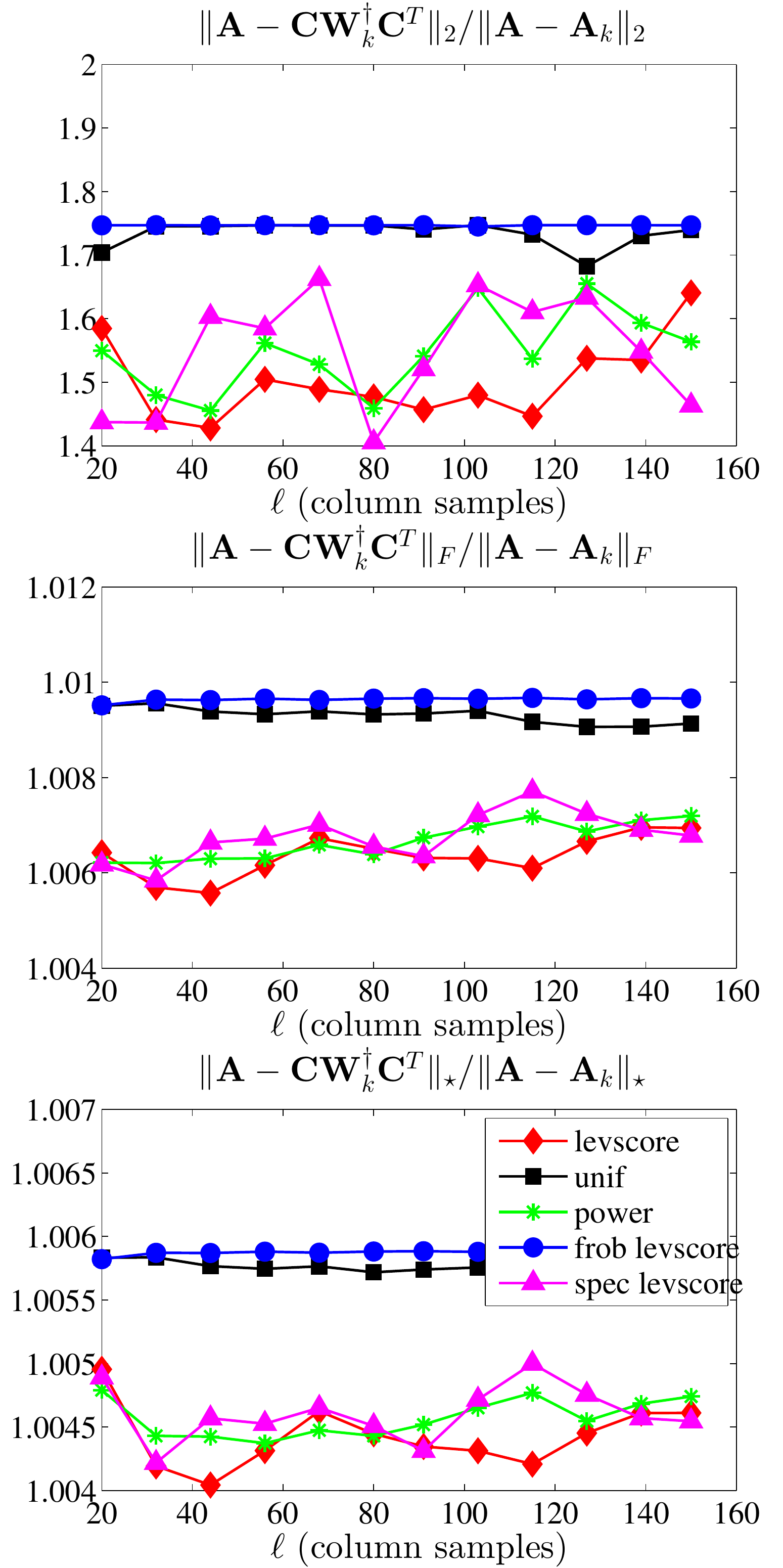}}%
 \subfigure[AbaloneS, $\sigma = 1, k = 20$]{\includegraphics[width=1.6in, keepaspectratio=true]{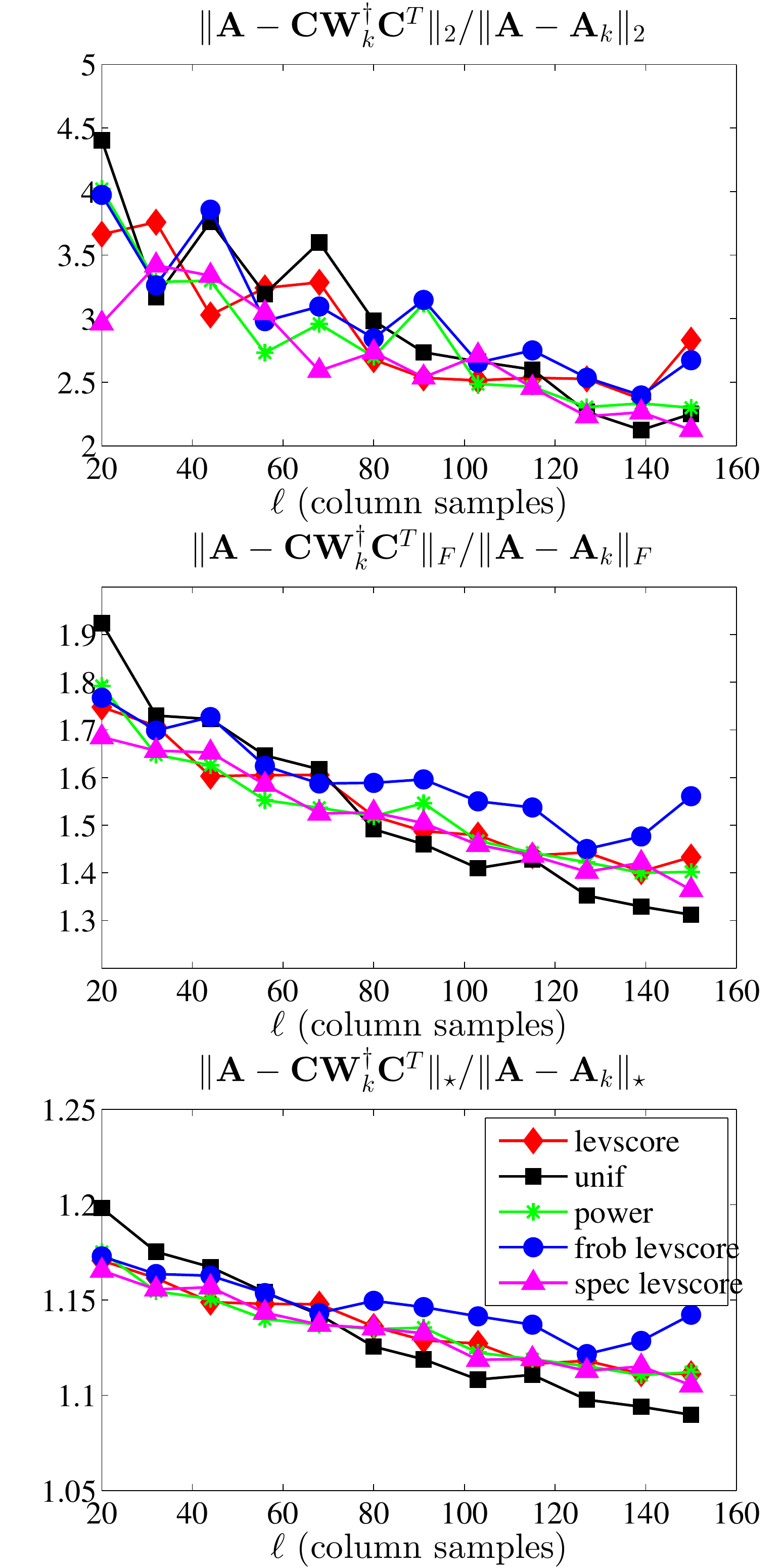}}%
 \subfigure[WineS, $\sigma = 1, k = 20$]{\includegraphics[width=1.6in, keepaspectratio=true]{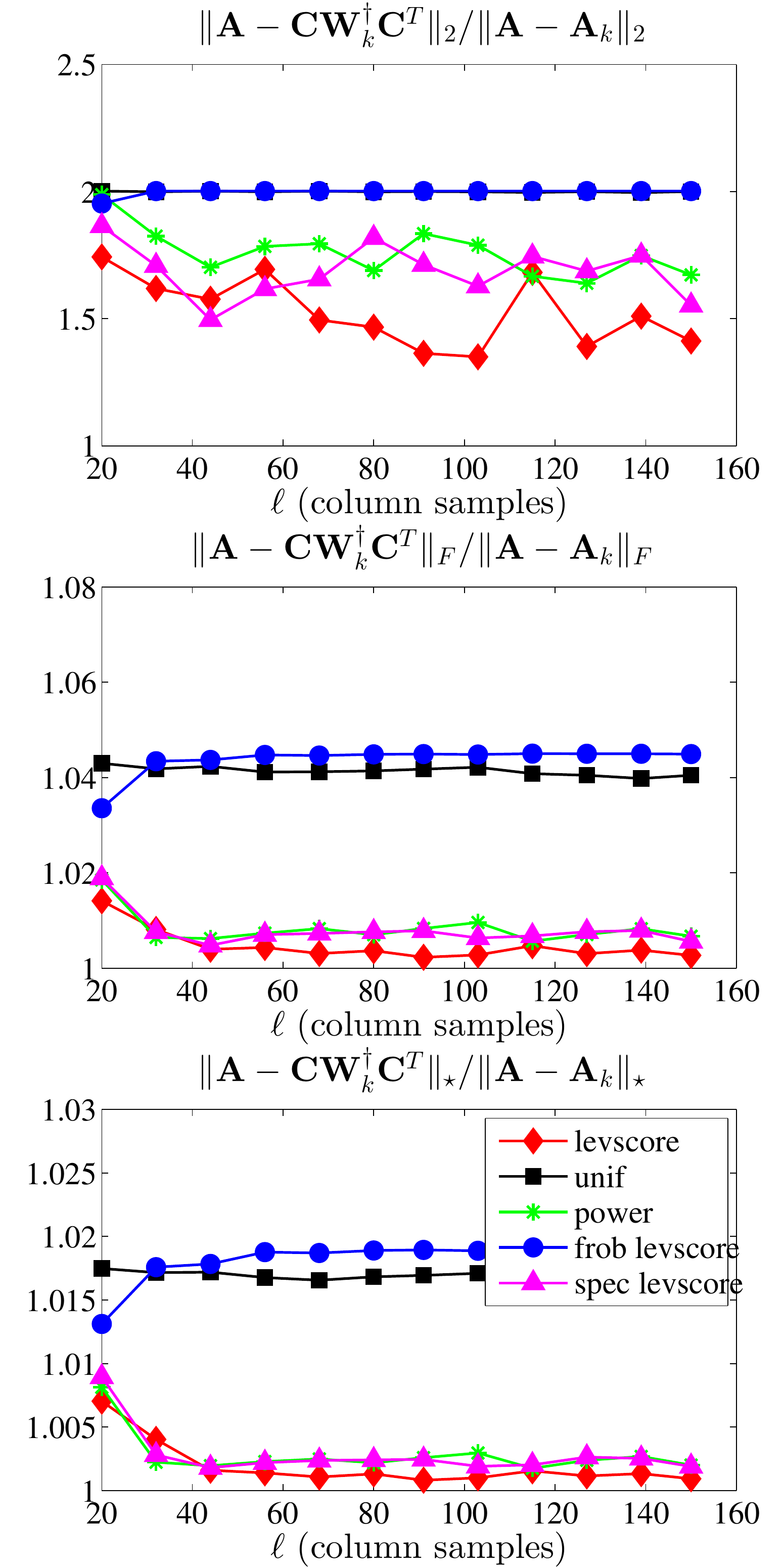}}%
 \subfigure[WineS, $\sigma = 2.1, k = 20$]{\includegraphics[width=1.6in, keepaspectratio=true]{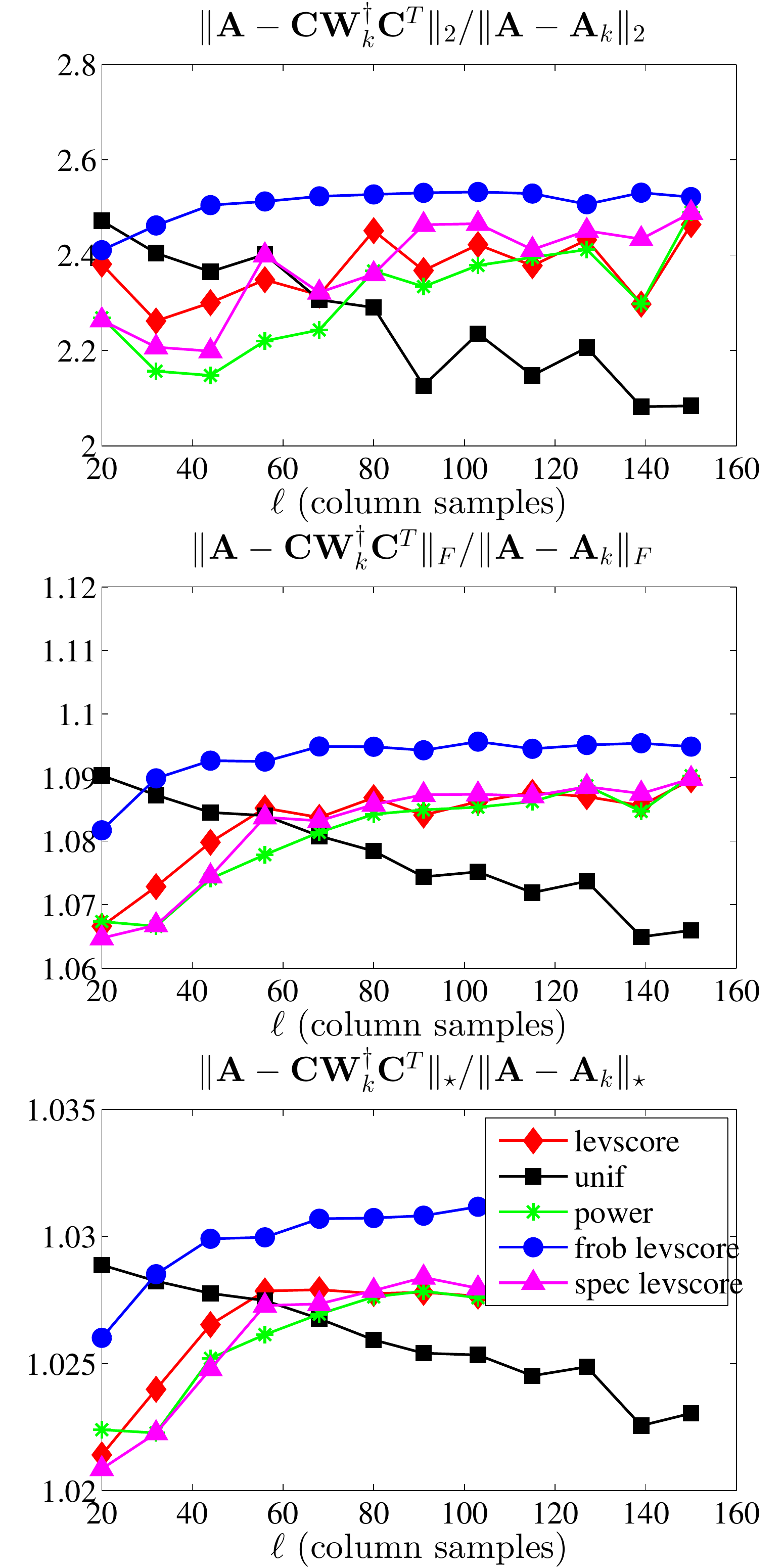}}%
 \caption{The spectral, Frobenius, and trace norm errors 
 (top to bottom, respectively, in each subfigure) of several
 (non-rank-restricted in top panels and rank-restricted in bottom panels)
 approximate leverage score-based SPSD sketches, as a function of the number of columns samples $\ell$, 
 for several sparse RBF data sets.}%
 \label{fig:sparserbf-inexact-nonfiltered-errors}
\end{figure}
\begin{itemize}
\item
For Laplacian Kernels, for the non-rank-restricted results, 
``frob levscore'' is only slightly better than uniform sampling, 
while ``power'' and ``spec levscore'' are substantially better than uniform 
sampling.
All of those methods also lead to \emph{better} reconstruction 
results even than using the exact leverage scores (suggesting
that some form of implicit regularization is taking place): the reconstruction
quality is higher for a given $\ell$ and, also, using approximate leverage scores does not lead to the 
saturation effect observed when using the exact leverage scores.
\item
For Laplacian Kernels, for the rank-restricted results, the ``frob levscore''
results are similar to the exact leverage score results for $\ell=k$, but 
the quality degrades considerably as $\ell$ increases.
On the other hand, ``power'' and ``spec levscore'' are much better than 
using the exact leverage scores when $\ell=k$, and are slightly better or 
only slightly worse as $\ell$ increases.
\item
For the Linear Kernels, all the methods perform similarly in the 
non-rank-restricted case; while in the rank-restricted case, the methods that use
approximate leverage scores tend to parallel the exact leverage score 
results, both when those get better and when those get worse with increasing 
$\ell$.
\item
For both the dense and the sparse RBF data sets, for the non-rank-restricted 
case, the approximate leverage score algorithms tend to parallel the exact 
leverage score algorithm, and they are not substantially better.
In particular, both ``power'' and ``spec levscore'' tend to saturate when 
the exact method saturates, but in those cases ``frob levscore'' tends not to 
saturate.
\item
For both the dense and sparse RBF data sets, for the rank-restricted case, 
the results depend on the value of the $\sigma$ width parameter.
When $\sigma$ is larger and the matrices are more homogeneous, all the 
methods tend to parallel each other (although WineS is an exception).
When $\sigma$ is smaller, ``frob levscore'' is generally better than uniform 
sampling but worse than the other methods for $\ell=k$, but it degrades
with increasing $\ell$; while both ``power'' and ``spec levscore'' tend to 
parallel the results for the exact leverage scores.
\end{itemize}

\noindent\todo[inline]{This paragraph was not clear to Ilse; clarify}
Note that the difference between different approximate leverage score 
algorithms often corresponds to a difference in the spectral gaps of the 
corresponding matrices.
From Table~\ref{table:datasets_stats}, if we fix $k$ and use the approximate 
leverage scores filtered through rank $k$ to form a Nystr\"om approximation 
to $\mat{A}$, the accuracy of that approximation has a strong dependence on 
the spectral gap of $\mat{A}$ at rank $k,$ as measured by 
$\frac{\lambda_k}{\lambda_{k+1}}.$ 
In general, the larger the spectral gap, the more accurate the approximation. 
This phenomena can also be understood in terms of the convergence of the 
approximate leverage scores: the approximation algorithms 
(in Algorithm~\ref{alg:spectral_levscore_approx}
and Algorithm~\ref{alg:frob_levscore_approx}) are essentially 
truncated versions of the subspace iteration method for computing the top 
$k$ eigenvectors of $\mat{A}.$ 
It is a classical result that the spectral gap determines the rate of 
convergence of the subspace iteration process to the desired eigenvectors: 
the larger it is, the fewer iterations of the process are required to get 
accurate approximations of the top eigenvectors. 
It follows immediately that the larger the spectral gap, the more accurate 
the approximate leverage scores generated by these approximation algorithms 
are.
Our empirical results illustrate the complexities and subtle consequences 
of these properties in realistic machine learning applications of even modestly-large size.

\subsubsection{Summary of Leverage Score Approximation Algorithms}

Before proceeding, there are several summary observations that we can make 
about the running time and reconstruction quality of approximate leverage 
score sampling algorithms for the data sets we have considered.
\begin{itemize}
\item
The running time of computing the exact leverage scores is generally much 
worse than that of uniform sampling and both SRFT-based and Gaussian-based 
random projection methods.
\item
The running time of computing approximations to the leverage scores can, 
with appropriate choice of parameters, be much faster than the exact 
computation of the leverage scores; and, especially for ``frob levscore,'' 
can be comparable to the running time of the random projection (SRFT or Gaussian) 
used in the leverage score approximation algorithm.
For the methods that involve $q>0$ iterations to compute stronger 
approximations to the leverage scores, the running time can vary 
considerably depending on details of the stopping condition.
\item
The leverage scores computed by the ``frob levscore'' procedure are typically very
different than the ``exact'' leverage scores, but they are leverage scores
for a low-rank space that is near the best rank-$k$ approximation to the
matrix.
This is often sufficient for good low-rank approximation, although the 
reconstruction accuracy can degrade in the rank-restricted cases as $\ell$ 
is increased.
\item
The approximate leverage scores computed from ``power'' and ``spec levscore'' 
approach those of the exact leverage scores, as $q$ is increased; and they 
obtain reconstruction accuracy that is no worse, and in many cases is 
better, than that obtained by the exact leverage scores.
This suggests that, by not fitting exactly to the empirical statistical 
leverage scores, we are observing a form of implicit regularization.
\item
The running time of Algorithm~\ref{alg:tall_levscore_approx}, when applied 
to ``tall'' matrices for which $n \gg d$, is faster than the running time of
performing a QR decomposition of the matrix $\mat{A}$; and it is comparable
to the running time of applying a random projection to $\mat{A}$ (which is
the computational bottleneck of applying 
Algorithm~\ref{alg:tall_levscore_approx}).
Thus, in particular, one could use this algorithm to compute approximations 
to the leverage scores to obtain a sketch that provides a relative-error 
approximation to a least-squares problem involving 
$\mat{A}$~\cite{DMM08_CURtheory_JRNL,DMMS07_FastL2_NM10,Mah-mat-rev_BOOK}; or
one could use the sketch thereby obtained as a preconditioner to an 
iterative method to solve the least-squares problem, in a manner analogous
to how Blendenpik or LSRN do so with a random projection~\cite{AMT10,MSM11_TR}.
\end{itemize}

\noindent
Previous work has showed that one can implement random projection 
algorithms to provide low-rank approximations with error comparable to that 
of the SVD in less time than state-of-the art Krylov solvers and other 
``exact'' numerical methods~\cite{HMT09_SIREV,Mah-mat-rev_BOOK}.
Our empirical results show that these random projection algorithms can be 
used in two complementary ways to approximate SPSD matrices of interest in 
machine learning:
first, they can be used directly to compute a projection-based low-rank 
approximation; and
second, they can be used to compute approximations to the leverage scores, 
which can be used to compute a sampling-based low-rank approximation.
With the right choice of parameters, the two complementary approaches have
roughly comparable running times, and neither one dominates the other in
terms of reconstruction accuracy.

\subsection{Projection-based Sketches}

\begin{figure}[t]
 \centering
 \subfigure[Gnutella, $k=20$]{\includegraphics[width=1.6in, keepaspectratio=true]{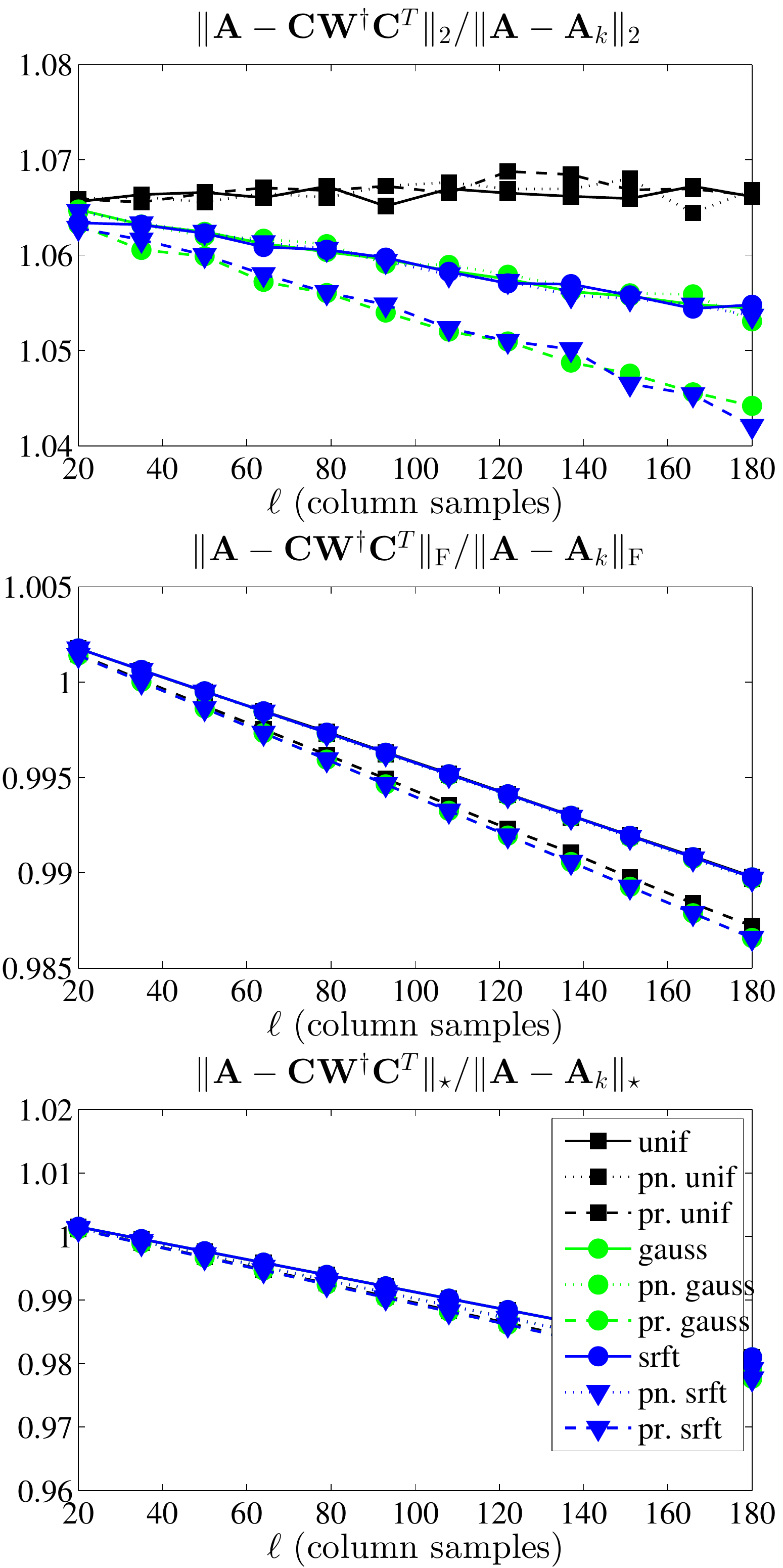}}%
 \subfigure[Dexter, $k=8$]{\includegraphics[width=1.6in, keepaspectratio=true]{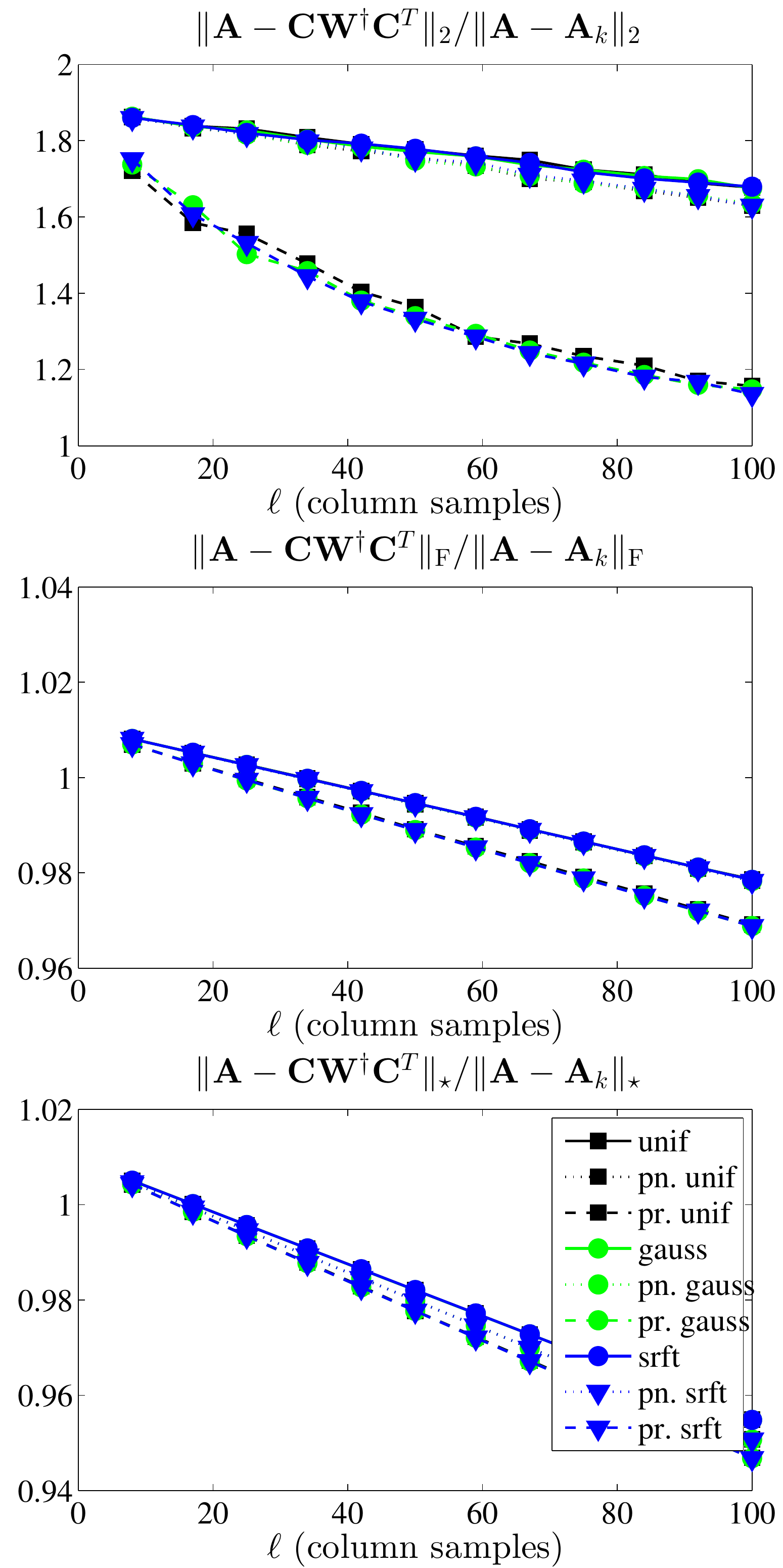}}%
 \subfigure[AbaloneD, $\sigma = .15, k = 20$]{\includegraphics[width=1.6in, keepaspectratio=true]{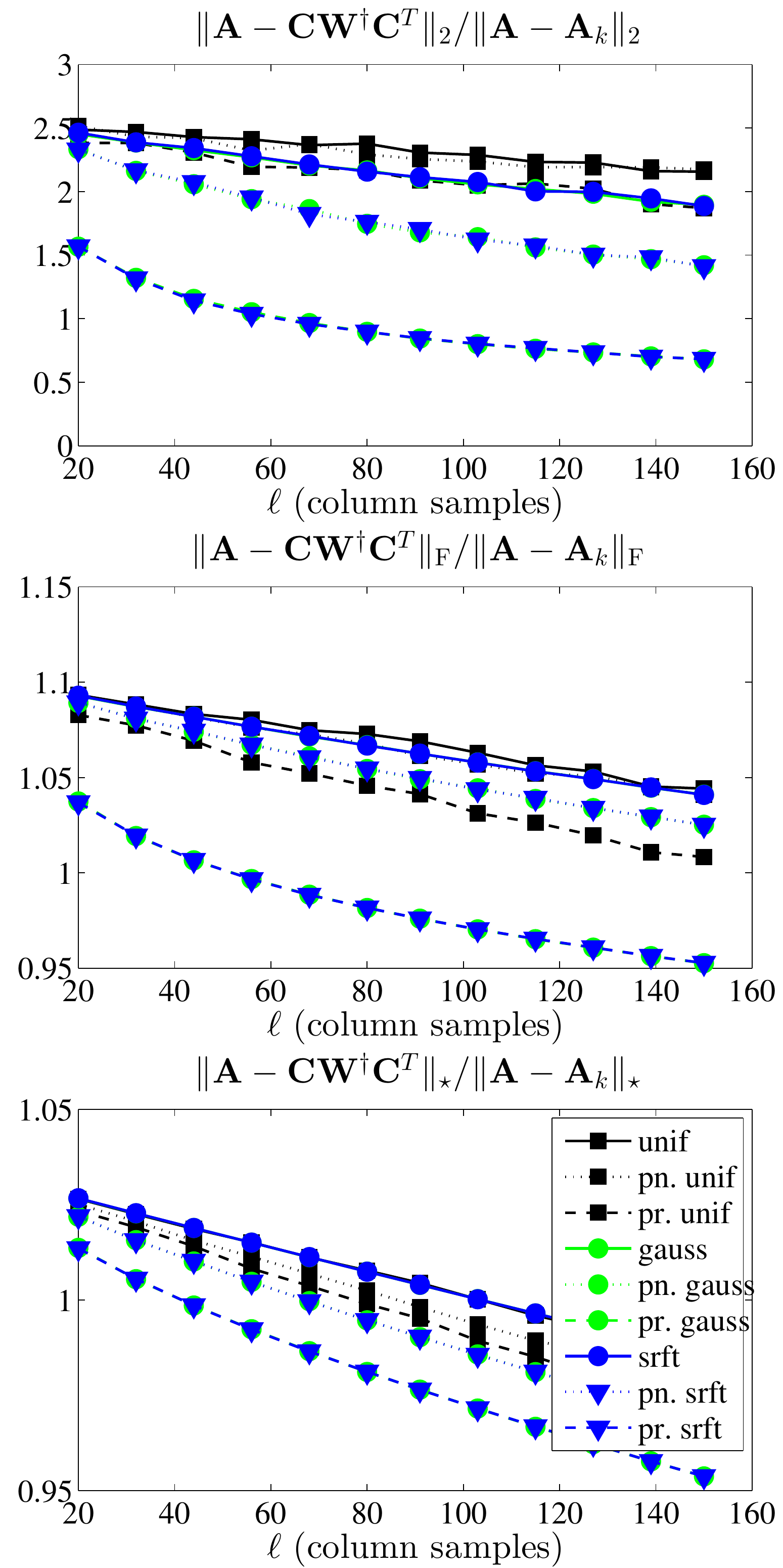}}%
 \subfigure[WineS, $\sigma=1, k=20$]{\includegraphics[width=1.6in, keepaspectratio=true]{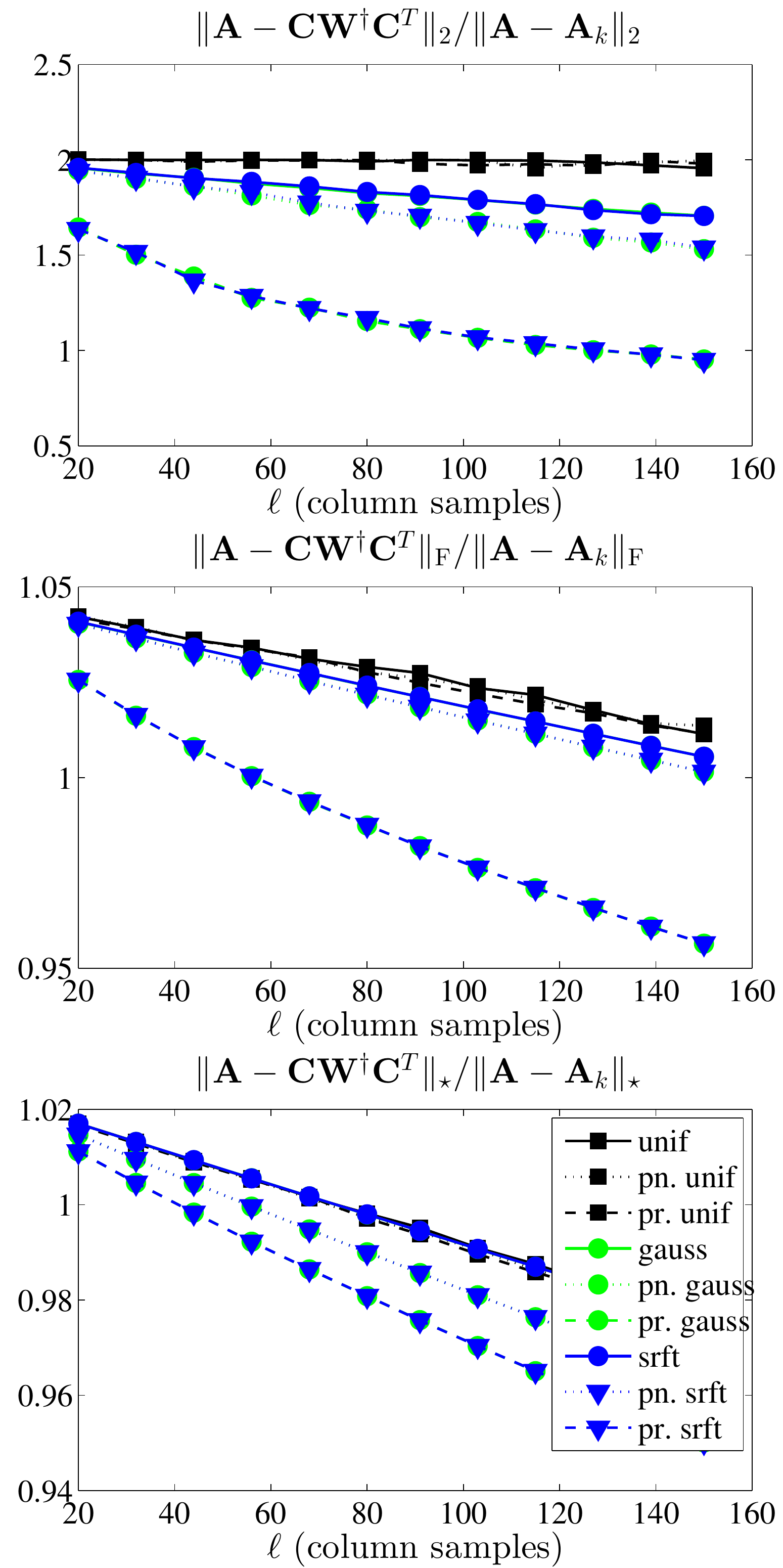}}%
 \caption{The spectral, Frobenius, and trace norm errors 
 (top to bottom, respectively, in each subfigure) of several
 non-rank-restricted SPSD sketches, including the pinched and prolonged sketches, as a function of the number of columns samples $\ell$, 
 for several datasets. Pinched and prolonged sketches, respectively indicated by ``pn.'' and ``pr.'', are defined in 
 Equations~\eqref{eqn:pinchedsketchdef} and~\eqref{eqn:prolongedsketchdef}. }%
 \label{fig:pinched-and-prolonged-sketches}
\end{figure}

Finally, for completeness, we consider the performance of the two 
projection-based SPSD sketches proposed in~\cite{HMT09_SIREV}, and we show 
how they perform when compared with the sketches we have considered. 
Recall that the idea of these sketches is to 
construct low-rank approximations by forming an approximate basis $\matQ$ for 
the top eigenspace of $\matA$ and then restricting $\matA$ to that eigenspace. 
In more detail, given a sketching matrix $\matS,$ form the matrix 
$\matY = \matA \matS$ and take the QR decomposition of $\matY$ to obtain 
$\matQ,$ a matrix with orthonormal columns. The first sketch, which we
eponymously refer to as the \emph{pinched} sketch, is simply $\matA$ pinched 
to the space spanned by $\matQ:$
\begin{equation}
\matQ (\matQ^\transp \matA \matQ) \matQ^\transp.
\label{eqn:pinchedsketchdef}
\end{equation}
The second sketch, which 
we refer to as the \emph{prolonged} sketch, is 
\begin{equation}
 \matA \matQ (\matQ^\transp \matA \matQ)^\pinv \matQ^\transp \matA. 
 \label{eqn:prolongedsketchdef}
\end{equation}

It is clear that the prolonged sketch can be constructed using our SPSD 
Sketching Model by taking $\matQ$ as the sketching matrix. In fact, a stronger
statement can be made. As shown in~\cite{Gittens12_TR}, and as stated in 
Lemma~\ref{lemma:sketchidentity} below, it 
is the case, for any sketching matrix $\matX,$ that when 
$\matC = \matA \matX$ and $\matW = \matX^\transp \matA \matX,$
\[
 \matC \matW^\pinv \matC^\transp = 
 \matA^{1/2} \matP_{\matA^{1/2} \matX} \matA^{1/2}.
\]
By considering the two choices $\matX = \matA \matS$ and $\matX = \matQ$, we 
see that in fact the prolonged sketch is exactly the sketch obtained by 
applying the power method with $q = 2:$
\begin{align*}
 \matA \matQ ( \matQ^\transp \matA \matQ)^\pinv \matQ^\transp \matA & = 
 \matA^{1/2} \matP_{\matA^{1/2} \matQ} \matA^{1/2} \\
 & = \matA^{1/2} \matP_{\matA^{1/2} (\matA \matS) } \matA^{1/2} \\
 & = \matA^2 \matS (\matS^\transp \matA^3 \matS)^\pinv\matS^\transp \matA^2.
\end{align*}
It follows that the bounds we provide in Section~\ref{sxn:theory} on the 
performance of sketches obtained using the power method pertain also to 
prolonged sketches.

In Figure~\ref{fig:pinched-and-prolonged-sketches}, we compare the empirical 
performances of several of the SPSD sketches considered earlier with their 
pinched and prolonged variants. Specifically, we plot the errors of
pinched and prolonged sketches for several choices of sketching
matrices---corresponding to uniform column sampling, gaussian column 
mixtures, and SRFT-based column mixtures---along with the errors of non-pinched, 
non-prolonged sketches constructing using the same choices of $\matS.$ In the
interest of brevity, we provide results only for several of the datasets 
listed in Table~\ref{table:datasets}, and we consider only the nonfixed-rank 
variants of the sketches.

Some trends are clear from Figure~\ref{fig:pinched-and-prolonged-sketches}.
\begin{itemize}
\item In the spectral norm, the prolonged sketches are considerably more accurate 
than the pinched and standard sketches for all the datasets considered. Without 
exception, the prolonged Gaussian and SRFT column-mixture sketches are the most accurate
in the spectral norm, of all the sketches considered. Only in the case of the 
Dexter Linear Kernel is the prolonged uniformly column-sampled sketch nearly as
accurate in the spectral norm as the prolonged Gaussian and SRFT sketches. To
 a lesser extent, the prolonged sketches are also more accurate in the Frobenius
 and trace norms than the other sketches considered. The increased Frobenius and
 trace norm accuracy is particularly notable for the two RBF Kernel datasets;
 again, the prolonged Gaussian and SRFT sketches are considerably more accurate
 than the prolonged uniformly column-sampled sketches.  
 \item After the 
 prolonged sketches, the pinched Gaussian and SRFT column-mixture sketches 
 exhibit the least  spectral, Frobenius, and trace norm 
 errors. Again, however, we see that the pinched uniformly column-sampled 
 sketches are considerably less accurate than the
 pinched Gaussian and SRFT column-mixture sketches. Particularly in the 
 spectral and Frobenius norms, the pinched uniformly column-sampled sketches
 are not any more accurate than the basic uniformly column-sampled sketches.
 \end{itemize}
From these considerations, it seems evident that the benefits of pinched and
prolonged sketches are most prominent when the spectral norm is the error metric,
or when the dataset is an RBF Kernel. In particular, pinched and prolonged
sketches are not significantly more accurate (than the sketches considered 
in the previous subsections) in the Frobenius and trace norms
for any of the datasets considered. 

It is also evident from Figure~\ref{fig:pinched-and-prolonged-sketches} that the pinched sketches
often have a much slighter increase in accuracy over the basic sketches than do the
prolonged sketches. To understand why the pinched sketches are less accurate than the
prolonged sketches, observe that the pinched sketches satisfy
\begin{align*}
 \matQ(\matQ^\transp\matA\matQ) \matQ^\transp 
 & = \matP_{\matA\matS} \matA \matP_{\matA\matS} \\
 & = (\matP_{\matA \matS} \matA^{1/2}) 
     (\matA^{1/2} \matP_{\matA \matS} ),
\end{align*}
while, as noted above, the prolonged sketches can be written in the form
\begin{align*}
 \matA \matQ(\matQ^\transp \matA \matQ)^\pinv \matQ^\transp \matA 
  = (\matA^{1/2} \matP_{\matA^{3/2} \matS}) (\matP_{\matA^{3/2} \matS} \matA^{1/2}).
\end{align*}
Thus, pinched and prolonged sketches approximate the square root of $\matA$ by projecting, respectively, onto the ranges of
$\matA \matS$ and $\matA^{3/2} \matS.$ The spectral decay present in $\matA$ is increased when $\matA$ is raised to a
power larger than one; consequently, the range of $\matA^{3/2} \matS$ is more biased towards the top $k$-dimensional
invariant subspace of $\matA$ than is the range of $\matA \matS.$ It follows that the approximate square root used to construct the
prolonged sketches more accurately captures the top $k$-dimensional subspace of $\matA$ than does that used to construct 
the pinched sketches.


\section{Theoretical Aspects of SPSD Low-rank Approximation}
\label{sxn:theory}

In this section, we present our main theoretical results, which consist
of a suite of bounds on the quality of low-rank approximation under several 
different sketching methods.
As mentioned above, these were motivated by our empirical observation that
\emph{all} of the sampling and projection methods we considered perform 
\emph{much} better on the SPSD matrices we considered than previous 
worst-case bounds (\emph{e.g.},~\cite{dm_kernel_JRNL,KMT12,Gittens12_TR}) 
would suggest.
We start in Section~\ref{sxn:theory-det} with deterministic structural 
conditions for the spectral, Frobenius, and trace norms;
and then in Section~\ref{sxn:theory-rand} we use these results to 
provide our bounds for several random sampling and random projection 
procedures.

\subsection{Deterministic Error Bounds for Low-rank SPSD Approximation} 
\label{sxn:theory-det}

In this section, we present three theorems that provide error bounds for 
the spectral, Frobenius, and trace norm approximation errors under the SPSD 
Sketching Model of Section~\ref{sxn:prelim-prelim}.
These are provided in Sections~\ref{sxn:theory-det-spectral},
\ref{sxn:theory-det-frobenius}, and~\ref{sxn:theory-det-trace}, 
respectively, and they are followed by several more general remarks in 
Section~\ref{sxn:theory-det-comments}.
Note that these bounds hold for \emph{any}, \emph{e.g.}, deterministic or 
randomized, sketching matrix $\matS$.
Thus, \emph{e.g.}, one could use them to check, in an \emph{a posteriori} 
manner, the quality of a sketching method for which one cannot establish an 
\emph{a priori} bound.
Rather than doing this, we use these results (in 
Section~\ref{sxn:theory-rand} below) to derive \emph{a priori} bounds for 
when the sketching operation consists of common random sampling and random 
projection algorithms.

Our results are based on the fact, established in~\cite{Gittens12_TR}, that
approximations which satisfy our SPSD Sketching Model can be written in terms of 
a projection onto a subspace of the range of the square root of the matrix
being approximated. The following fact appears in the proof of Proposition 1 in~\cite{Gittens12_TR}.

\todo[inline]{Update this reference to unreleased version in anticipation of changing on arxiv}
\begin{lemma}[\protect{\cite{Gittens12_TR}}]
 Let $\matA$ be an SPSD matrix and $\matS$
be a conformal sketching matrix. Then when
$\matC = \matA \matS$ and $\matW = \matS^\transp \matA \matS,$
the corresponding low-rank SPSD approximation satisfies
\[
 \matC\matW^\pinv\matC^\transp = \matA^{1/2} \matP_{\matA^{1/2} \matS} \matA^{1/2}.
\]
\label{lemma:sketchidentity}
\end{lemma}

\subsubsection{Spectral Norm Bounds}
\label{sxn:theory-det-spectral}


We start with a bound on the spectral norm of the residual error.
Although this result is trivial to prove, given prior work, it
highlights several properties that we use in the analysis of our 
subsequent results.

\begin{theorem}
Let $\matA$ be an $n \times n$ SPSD matrix with eigenvalue decomposition 
partitioned as in Equation~\eqref{eqn:eigendecompositionpartition}, $\matS$ 
be a sketching matrix of size $n \times \ell$, $q$ be a positive integer, and
$\matOmega_1$ and $\matOmega_2$ be as defined in Equation~\eqref{eqn:interactiontermsdefinition}. 
Then when $\matC = \matA^q \matS$ and $\matW = \matS^\transp \matA^{2q-1} \matS$, 
the corresponding low-rank SPSD approximation satisfies
\begin{equation*}
\TNorm{\matA - \matC \matW^\pinv \matC^\transp} 
   \leq \TNorm{\matSig_2} + 
   \TNorm{\matSig_2^{q-1/2} \matOmega_2 \matOmega_1^\pinv}^{2/(2q-1)},
\end{equation*}
assuming $\matOmega_1$ has full row rank. 
\label{thm:spectral-deterministic-error}
\end{theorem}
\begin{proof}
Apply Lemma~\ref{lemma:sketchidentity} with the sampling matrix $\matS^\prime = \matA^{q-1} \matS$ (where, recall, $q \ge 1$) to see that
\[
 \matC \matW^\pinv \matC^\transp = \matA^{1/2} \matP_{\matA^{q-1/2} \matS} \matA^{1/2}.
\]
It follows that 
\begin{equation}
  \TNorm{\matA - \matC \matW^\pinv \matC^\transp} = 
  \TNormS{\matA^{1/2}\left(\mat{I} - 
             \matP_{\left(\matA^{1/2}\right)^{2q-1}\matS}\right)\matA^{1/2}}.
\label{eqn:spectral-pf1}
\end{equation}
Next, recall that $\matOmega_i = \matU_i^\transp \matS$ and that 
$\matA^{1/2}$ has eigenvalue decomposition 
$\matA = \matU \matSig^{1/2} \matU^\transp,$ where
\[
 \matU = \begin{pmatrix} \matU_1 & \matU_2 \end{pmatrix} 
 \text{ and } 
 \matSig^{1/2} = \begin{pmatrix} \matSig_1^{1/2} & \, \\ 
                                 \,          & \matSig_2^{1/2} \end{pmatrix}.
\]
It can be shown (\cite[Theorems 9.1 and 9.2]{HMT09_SIREV}) that, because $\matOmega_1$ has full row
rank,
\begin{equation}
  \TNormS{\matA^{1/2}\left(\matI - 
   \matP_{\left(\matA^{1/2}\right)^{2q-1} \matS}\right) \matA^{1/2}} 
  \leq \left(\TNormS{\left(\matSig_2^{1/2}\right)^{2q-1}} + 
             \TNormS{\left(\matSig_2^{1/2}\right)^{2q-1} 
                     \matOmega_2 \matOmega_1^\pinv}
        \right)^{1/(2q-1)}.
\label{eqn:spectral-pf2}
\end{equation}
Equations~\eqref{eqn:spectral-pf1} and~\eqref{eqn:spectral-pf2} imply that
\begin{align*}
 \TNorm{\matA - \matC \matW^\pinv \matC^\transp} & \leq 
    \left(\TNormS{\matSig_2^{q-1/2}} + 
          \TNormS{\matSig_2^{q-1/2} \matOmega_2 \matOmega_1^\pinv}
    \right)^{1/(2q-1)} \\
    & \leq \TNorm{\matSig_2} + \TNorm{\matSig_2^{q-1/2} \matOmega_2 \matOmega_1^\pinv}^{2/(2q-1)}
\end{align*}
The latter inequality follows from the fact that the $2q-1$ radical function is subadditive when $q \geq 1$ and the identity 
$\TNormS{\matSig_2^{q-1/2}} = \TNorm{\matSig_2}^{2q-1}.$ This establishes the stated bound.
\end{proof}

\noindent
\textbf{Remark.}
The assumption that $\matOmega_1$ has full row rank is very non-trivial. 
It is, however, satisfied by our algorithms below.
See Section~\ref{sxn:theory-det-comments} for more details on this point.

\noindent
\textbf{Remark.}
The proof of Theorem~\ref{thm:spectral-deterministic-error} proceeds in two 
steps.
The first step relates low-rank approximation of an SPSD matrix $\matA$ 
under the SPSD Sketching Model of Section~\ref{sxn:prelim-prelim} to 
column sketching (\emph{e.g.}, sampling or projecting) from the square-root 
of $\matA$. 
A weaker relation of this type was used in~\cite{dm_kernel_JRNL}, but the 
stronger form that we use here in Equation~(\ref{eqn:spectral-pf1}) was 
first proved in~\cite{Gittens12_TR}.
The second step is to use a deterministic structural result that holds for 
sampling/projecting from an arbitrary matrix.
The structural bound of the form of Equation~(\ref{eqn:spectral-pf2}) was originally 
proven for $q=1$ in~\cite{BMD09_CSSP_SODA}, where it was applied to the 
Column Subset Selection Problem. The bound was subsequently improved 
in~\cite{HMT09_SIREV}, where it was applied to a random projection algorithm 
and extended to apply when $q > 1$. Although the analyses of our next two 
results are more complicated, they follow the same high-level two-step approach.

Before proceeding with the analogous Frobenius and trace norm bounds, we 
pause to describe a geometric interpretation of this result.

\todo[inline]{Explain interpretation of $\matOmega_2 \matOmega_1^\pinv$, how it
shows up in convergence of power method, and how it being small implies 
$\matS$ is an effective sketching matrix.}

\subsubsection{A geometric interpretation of the sketching interaction matrix}
It is evident from the bound in Theorem~\ref{thm:spectral-deterministic-error}
that the smaller the spectral norm of the \emph{sketching interaction matrix}
$\matOmega_2 \matOmega_1^\pinv,$ the more effective 
$\matS$ is as a sketching matrix. If, additionally, the columns of $\matS$ 
are orthonormal, we can give this norm a natural geometric
interpretation as the tangent of the largest angle between the spaces spanned
by $\matS$ and $\matU_1.$ 

To verify our claim, we first recall the definition of the sine between the range spaces
of two matrices $\matM_1$ and $\matM_2:$
\[
 \sin^2(\matM_1, \matM_2) = 
 \TNormS{(\matI - \matP_{\matM_1}) \matP_{\matM_2})}.
\]
Note that this quantity is \emph{not} symmetric: it measures how well the 
range of $\matM_1$ captures that of $\matM_2$~\cite[Chapter 12]{GL96}. Since $\matU_1$ and $\matS$ (by assumption here)
have orthonormal columns, we see that
\begin{align*}
 \sin^2(\matS, \matU_1) & = 
 \TNormS{(\matI - \matS \matS^\transp) \matU_1 \matU_1^\transp} \\
 & = \TNorm{\matU_1^\transp (\matI - \matS \matS^\transp) \matU_1} \\
 & = \TNorm{\matI - \matU_1^\transp \matS \matS^\transp \matU_1} \\
 & = 1 - \lambda_k(\matU_1^\transp \matS \matS^\transp \matU_1) \\
 & = 1 - \TNorm{\matOmega_1^\pinv}^{-2}.
\end{align*}
The second to last equality holds because $\matU_1^\transp \matS$ has $k$ 
rows and we assumed it has full row rank. Accordingly, it follows that
\[
 \tan^2(\matS, \matU_1) = 
  \frac{\sin^2(\matS, \matU_1)}{1 - \sin^2(\matS, \matU_1)} 
  = \TNormS{\matOmega_1^\pinv} - 1.
\]
Now observe that 
\begin{align*}
\TNormS{\matOmega_2 \matOmega_1^\pinv} & = 
\TNorm{(\matS^\transp \matU_1)^\pinv 
       \matS^\transp \matU_2 \matU_2^\transp \matS
       (\matU_1^\transp \matS)^\pinv} \\
& = \TNorm{(\matS^\transp \matU_1)^\pinv 
       (\matI - \matS^\transp \matU_1 \matU_1^\transp \matS)
       (\matU_1^\transp \matS)^\pinv} \\
& = \TNormS{(\matS^\transp \matU_1)^\dagger} - 1 \\
& = \tan^2(\matS, \matU_1).
\end{align*}
The second to last equality holds because of the fact that, for any matrix $\matM,$
\[
 \TNorm{\matM^\pinv( \matI - \matM \matM^\transp) (\matM^\transp)^\pinv} =
 \TNormS{\matM^\pinv} - 1;
\]
this identity can be established with a routine SVD argument.

Thus, when $\matS$ has orthonormal columns and $\matU_1^\transp \matS$ has full
row-rank, $\TNorm{\matOmega_2 \matOmega_1}$ is the tangent of the
largest angle between the range of $\matS$ and the eigenspace spanned by 
$\matU_1.$ If $\matU_1^\transp \matS$ does not have full row-rank, then our derivation
above shows that $\sin^2(\matS, \matU_1) = 1,$ meaning that there is a vector
in the eigenspace spanned by $\matU_1$ which has no component in the space
spanned by the sketching matrix $\matS.$ 

We note that $\tan(\matS, \matU_1)$ also arises in the classical 
bounds on the convergence of the orthogonal iteration algorithm for approximating
invariant subspaces of a matrix (see, e.g.~\cite[Theorem 8.2.2]{GL96}).


\subsubsection{Frobenius Norm Bounds}
\label{sxn:theory-det-frobenius}

Next, we state and prove the following bound on the Frobenius norm of the 
residual error.
The proof parallels that for the spectral norm bound, in that we divide it 
into two analogous parts, but the analysis is somewhat more complex. 

The multiplicative eigengap
$\gamma = \lambda_{k+1}(\matA)/\lambda_k(\matA)$ that appears in the statement of this theorem predicts the effect of
using the power method when constructing sketches. Specifically, the 
additional errors of sketches constructed using
$\matC = \matA^q \matS$ are at least a factor of $\gamma^{q-1}$ times smaller 
than those constructed using $\matC = \matA \matS.$ 

\begin{theorem}
Let $\matA$ be an $n \times n$ SPSD matrix with eigenvalue decomposition 
partitioned as in Equation~\eqref{eqn:eigendecompositionpartition}, $\matS$
be a sketching matrix of size $n \times \ell$, $q$ be a positive integer,
$\matOmega_1$ and $\matOmega_2$ be as defined in 
Equation~\eqref{eqn:interactiontermsdefinition}, and define
\[
 \gamma = \frac{\lambda_{k+1}(\matA)}{\lambda_k(\matA)}.
\]

Then when $\matC = \matA^q \matS$ and $\matW = \matS^\transp \matA^{2q-1} \matS$, 
the corresponding low-rank SPSD approximation satisfies
\begin{equation*}
\FNorm{\matA - \matC \matW^\pinv \matC^\transp } \leq \FNorm{\matSig_2} + 
\gamma^{q-1} \TNorm{\matSig_2^{1/2} \matOmega_2 \matOmega_1^\pinv}
  \cdot \left( \sqrt{ 2 \tr{\matSig_2} } + 
           \gamma^{q-1} \FNorm{\matSig_2^{1/2} \matOmega_2 \matOmega_1^\pinv} 
        \right),
\end{equation*}
assuming $\matOmega_1$ has full row rank. 
\label{thm:frobenius-deterministic-error}
\end{theorem}
\begin{proof}
Apply Lemma~\ref{lemma:sketchidentity} with the sampling matrix $\matS^\prime = 
\matA^{q-1}\matS$ to see that 
\[
 \matC \matW^\pinv \matC^\transp = 
  \matA^{1/2} \matP_{\matA^{q-1/2} \matS} \matA^{1/2}.
\]
It follows that
$$
\FNorm{\matA - \matC \matW^\pinv \matC^\transp} 
   = \FNorm{\matA^{1/2}\left(\matI - 
   \matP_{\matA^{q-1/2} \matS}\right)\matA^{1/2}}  .
$$
To bound this quantity, we first use the unitary invariance of the Frobenius norm and 
the fact that 
\[
 \matP_{\matA^{q-1/2}\matS} = 
 \matU \matP_{\matSig^{q-1/2} \matU^\transp \matS} \matU^\transp
\]
to obtain
\[
 E := \FNorm{\matA^{1/2}\left(\matI - 
           \matP_{\matA^{q-1/2} \matS}\right)\matA^{1/2}} 
    = \FNormS{\matSig^{1/2} \left(\matI - 
       \matP_{\matSig^{q-1/2} \matU^\transp \matS}\right) \matSig^{1/2}}.
\]
Then we take
\begin{equation}
 \matZ = \matSig^{q-1/2} \matU^\transp \matS \matOmega_1^\pinv \matSig_1^{-(q-1/2)}
  = \begin{pmatrix} \matI \\ \matF \end{pmatrix},
 \label{eqn:defz}
\end{equation}
where $\matI \in \R^{k \times k}$ and $\matF \in \R^{n-k \times k}$ is
given by $\matF = 
\matSig_2^{q-1/2} \matOmega_2 \matOmega_1^\pinv \matSig_1^{-(q-1/2)}.$
The latter equality in Equation~\eqref{eqn:defz} holds because of our assumption that $\matOmega_1$ 
has full row rank.
Since the range of $\matZ$ is contained in the range of 
$\matSig^{q-1/2} \matU^\transp \matS,$
\[
 E \leq \FNormS{\matSig^{1/2} (\matI - \matP_{\matZ}) \matSig^{1/2}}.
\]
By construction, $\matZ$ has full column rank, thus $\matZ(\matZ^\transp\matZ)^{-1/2}$
is an orthonormal basis for the span of $\matZ$, and
\begin{align}
 \mat{I} - \matP_\matZ & = \mat{I} - \matZ(\matZ^\transp\matZ)^{-1}\matZ^\transp
   = \mat{I} - \begin{pmatrix} \matI \\ \matF \end{pmatrix}
  (\matI + \matF^\transp\matF)^{-1} \begin{pmatrix} \matI & \matF^\transp \end{pmatrix} \notag\\
  & = \begin{pmatrix} 
  \matI - (\matI + \matF^\transp \matF)^{-1} & -(\matI + \matF^\transp \matF)^{-1} \matF^\transp \\
   -\matF (\matI + \matF^\transp \matF)^{-1} & \matI - \matF (\matI + \matF^\transp \matF)^{-1} \matF^\transp
    \end{pmatrix}.
    \label{eqn:pzestimate}
\end{align}
This implies that
\begin{equation}
\label{ch4:eqn:frobenius-partitioned-estimate}
\begin{split}
 E & \leq \FNormS{\matSig^{1/2} \begin{pmatrix} 
			\matI - (\matI + \matF^\transp \matF)^{-1} & -(\matI + \matF^\transp \matF)^{-1} \matF^\transp \\
                        -\matF (\matI + \matF^\transp \matF)^{-1} & \matI - \matF (\matI + \matF^\transp \matF)^{-1} \matF^\transp
                       \end{pmatrix}
                \matSig^{1/2}} \\
 & = \FNormS{\matSig_1^{1/2} \big(\matI - (\matI + \matF^\transp \matF)^{-1}\big) \matSig_1^{1/2}} + 
     2 \FNormS{\matSig_1^{1/2} (\matI + \matF^\transp \matF)^{-1} \matF^\transp \matSig_2^{1/2}} \\
   & \quad + \FNormS{\matSig_2^{1/2} \big(\matI - \matF (\matI + \matF^\transp \matF)^{-1} \matF^\transp \big) \matSig_2^{1/2}} \\
 & := T_1 + T_2 + T_3.
\end{split}
\end{equation}
Next, we provide bounds for $T_1$, $T_2$, and $T_3$.
Using the fact that 
$\mat{0} \preceq \matI - \matF(\matI + \matF^\transp \matF)^{-1}\matF^\transp \preceq \matI,$ 
we can bound $T_3$ with
\[
 T_3 \leq \FNormS{\matSig_2}.
\]
Likewise, the fact that
$\matI - (\matI + \matF^\transp\matF)^{-1} \preceq \matF^\transp\matF$ 
(easily seen with an SVD) implies that we can bound $T_1$~as
\begin{align*}
 T_1 & \leq \FNormS{\matSig_1^{1/2} \matF^\transp \matF \matSig_1^{1/2}} 
       \leq \TNormS{\matF \matSig_1^{1/2}} \FNormS{\matF \matSig_1^{1/2}} \\
     & = \TNormS{\matSig_2^{q-1/2} 
                    \matOmega_2 \matOmega_1^\pinv \matSig_1^{-(q-1)}}
            \FNormS{\matSig_2^{q-1/2} 
                    \matOmega_2 \matOmega_1^\pinv \matSig_1^{-(q-1)}} \\
     & \leq \TNorm{\matSig_2^{q-1}}^4 \TNorm{\matSig_1^{-(q-1)}}^4
                   \TNormS{\matSig_2^{1/2} \matOmega_2 \matOmega_1^\pinv}
                   \FNormS{\matSig_2^{1/2} \matOmega_2 \matOmega_1^\pinv} \\           
     & = (\TNorm{\matSig_2} \TNorm{\matSig_1^{-1}})^{4(q-1)}
                   \TNormS{\matSig_2^{1/2} \matOmega_2 \matOmega_1^\pinv}
                   \FNormS{\matSig_2^{1/2} \matOmega_2 \matOmega_1^\pinv} \\
     & = \left(\frac{\lambda_{k+1}(\matA)}{\lambda_k(\matA)}\right)^{4(q-1)}
            \TNormS{\matSig_2^{1/2} \matOmega_2 \matOmega_1^\pinv}
            \FNormS{\matSig_2^{1/2} \matOmega_2 \matOmega_1^\pinv}.
\end{align*}
We proceed to bound $T_2$ by using the estimate
\begin{equation}
\label{eqn:t2-estimate}
 T_2 \leq 2 \TNormS{\matSig_1^{1/2} 
                    (\matI + \matF^\transp \matF)^{-1} \matF^\transp} 
            \FNormS{\matSig_2^{1/2}}.
\end{equation}
To develop the term involving a spectral norm, observe that for any SPSD 
matrix $\mat{M}$ with eigenvalue decomposition 
$\matM = \matV \matD \matV^\transp,$
\begin{align*}
 (\matI + \matM)^{-1} \matM (\matI + \matM)^{-1} & = 
 (\matV\matV^\transp + \matV \matD \matV^\transp)^{-1} \matV \matD \matV^\transp
 (\matV\matV^\transp + \matV \matD \matV^\transp)^{-1} \\
 & = \matV(\matI + \matD)^{-1} \matD (\matI + \matD)^{-1} \matV^\transp \\
 & \preceq \matV \matD \matV^\transp = \matM.
\end{align*}
It follows that
\begin{align*}
 \TNormS{\matSig_1^{1/2} 
         (\matI + \matF^\transp \matF)^{-1} \matF^\transp} & = 
 \TNorm{\matSig_1^{1/2} (\matI + \matF^\transp \matF)^{-1} 
        \matF^\transp \matF 
        (\matI + \matF^\transp \matF)^{-1} \matSig_1^{1/2}} \\
 & \leq \TNorm{\matSig_1^{1/2} \matF^\transp\matF \matSig_1^{1/2}} 
   =    \TNormS{\matF \matSig_1^{1/2}} \\
 & = \TNormS{\matSig_2^{q-1/2} 
             \matOmega_2 \matOmega_1^\pinv \matSig_1^{-(q-1)} } \\
 & \leq \TNormS{\matSig_2^{q-1}} \TNormS{\matSig_1^{-(q-1)}} 
        \TNormS{\matSig_2^{1/2} \matOmega_2 \matOmega_1^\pinv} \\
 & = \left( \frac{\lambda_{k+1}(\matA)}{\lambda_k(\matA)} \right)^{2(q-1)} 
     \TNormS{\matSig_2^{1/2} \matOmega_2 \matOmega_1^\pinv}.
\end{align*}
Using this estimate in Equation~\eqref{eqn:t2-estimate}, we conclude that
\[
 T_2 \leq 2 \left( \frac{\lambda_{k+1}(\matA)}{\lambda_k(\matA)} \right)^{2(q-1)} 
 \TNormS{\matSig_2^{1/2} \matOmega_2 \matOmega_1^\pinv} \FNormS{\matSig_2^{1/2}}.
\]

Combining our estimates for $T_1,$ $T_2,$ and $T_3$ with
Equation~\eqref{ch4:eqn:frobenius-partitioned-estimate} gives
\begin{equation*}
\begin{split}
E & = \FNorm{\matA^{1/2}\left(\matI - 
           \matP_{\matA^{q-1/2}\matS}\right)\matA^{1/2}} \\
  & \leq \FNormS{\matSig_2} + \left( \frac{\lambda_{k+1}(\matA)}{\lambda_k(\matA)} \right)^{2(q-1)} 
 \TNormS{\matSig_2^{1/2} \matOmega_2 \matOmega_1^\pinv} \cdot \left( 2\FNormS{\matSig_2^{1/2}}
 + \left(\frac{\lambda_{k+1}(\matA)}{\lambda_k(\matA)}\right)^{2(q-1)}
            \FNormS{\matSig_2^{1/2} \matOmega_2 \matOmega_1^\pinv}\right).
\end{split}
\end{equation*}
The claimed bound follows by identifying $\gamma$ and applying the subadditivity of the square-root function:
\[
 E \leq \FNorm{\matSig_2} +  \gamma^{q-1} \TNorm{\matSig_2^{1/2} \matOmega_2 \matOmega_1^\pinv}
  \cdot\left( \sqrt{ 2 \tr{\matSig_2} } + \gamma^{q-1} \FNorm{\matSig_2^{1/2} \matOmega_2 \matOmega_1^\pinv} \right).
\]
\end{proof}

\noindent
\textbf{Remark.}
The quality of approximation guarantee provided by 
Theorem~\ref{thm:frobenius-deterministic-error} depends on the
quantities $\TNorm{\matSig_2^{1/2} \matOmega_2 \matOmega_1^\pinv}$ and
$\FNorm{\matSig_2^{1/2} \matOmega_2 \matOmega_1^\pinv}$; these quantities
reflect the extent to which the sketching matrix is aligned with the 
eigenspaces of $\matA.$ As we will see in Section~\ref{sxn:theory-rand}, the degree
to which we can bound each of these for different sketching procedures is 
slightly different. The dependence on $\gamma$ captures the facts that the
power method is effective only when there is spectral decay, and that larger
gaps between the $k$ and $k+1$ eigenvalues lead to smaller
errors when the power method is used.

As before, we 
pause to describe a geometric interpretation of this result.

\subsubsection{Another geometric interpretation of the sketching interaction matrix}

Just as the spectral norm of the spectral interaction matrix is the tangent
of the largest angle between the range of the sketching matrix and the 
dominant $k$-dimensional eigenspace of $\matA$, the Frobenius norm of the
spectral interaction matrix has a geometric interpretation. 
To see this, recall that the \emph{principal angles} between the ranges of the matrices
$\matS$ and $\matU_1$ are defined recursively by
\[
 \cos(\theta_i) = \vec{u}_i^\transp \vec{v}_i
  = \max_{\substack{\VTNorm{\vec{u}} = 1\\
		    \vec{u} \in \mathcal{R}(\matS)\\
                    \vec{u}_i \perp \vec{u}_1, \ldots, \vec{u}_{i-1}}}
  \max_{\substack{\VTNorm{\vec{v}} = 1 \\
                    \vec{v} \in \mathcal{R}(\matU_1) \\
                    \vec{v}_i \perp \vec{v}_1, \ldots, \vec{v}_{i-1}}}
    \vec{v}^\transp \vec{u};
\]
satisfy $0 \leq \theta_1 \leq \ldots \leq \theta_k \leq \pi/2;$ and further, since
$\matS$ and $\matU_1$ have orthonormal columns, 
$\cos(\theta_i) = \sigma_i(\matU_1^\transp \matS)$~\cite[Chapter 12]{GL96}.
Specifically, 
when $\matS$ has orthonormal columns and $\matU_1^\transp \matS$ has full
row-rank, $\|\matOmega_2 \matOmega_1^\pinv\|_{\mathrm{F}}$ is the sum of the squared tangents
of the principal angles between the range of the sketching matrix 
and the dominant $k$-dimensional eigenspace of $\matA$. Thus, this quantity is 
a more stringent measure of how well the sketching matrix captures the 
dominant eigenspaces of $\matA.$

The proof of this claim hinges on the fact that, for any matrix $\matM,$
\[
 \tr{\matM^\pinv( \matI - \matM\matM^\transp) (\matM^\transp)^\pinv} 
 = \tr{(\matM^\transp \matM)^\pinv} - \mbox{rank}(\matM),
\]
as can be readily verified with an SVD argument. From this observation, we
have that
\begin{align*}
 \FNormS{\matOmega_2 \matOmega_1^\pinv} & = 
  \tr{(\matS^\transp \matU_1)^\pinv \matS^\transp 
      \matU_2 \matU_2^\transp \matS (\matU_1^\transp \matS)^\pinv} \\
  & = \tr{(\matS^\transp \matU_1)^\pinv ( \matI - \matS^\transp 
      \matU_1 \matU_1^\transp \matS ) (\matU_1^\transp \matS)^\pinv} \\
  & = \tr{(\matOmega_1 \matOmega_1^\transp)^\pinv} - k 
    = \FNormS{\matOmega_1^\pinv} - k \\
  & = \sum\nolimits_{i=1}^k (\sigma_i^{-2}(\matU_1^\transp \matS) - 1)
    = \sum\nolimits_{i=1}^k (\cos^{-2}(\theta_i) - 1) \\
  & = \sum\nolimits_{i=1}^k \tan^2(\theta_i).
\end{align*}

\noindent
\textbf{Remark.}
To obtain a greater understanding of the additional error term in 
Theorem~\ref{thm:frobenius-deterministic-error}, assume that $\matS$ is a 
particularly effective sketching matrix, so that 
$\TNorm{\matOmega_2 \matOmega_1^\pinv} = \const{O}(1).$ Then
\[ 
\TNorm{\matSig_2^{1/2} \matOmega_2 \matOmega_1^\pinv} =
\const{O}\left(\TNorm{\matSig_2}^{1/2}\right) \quad \text{ and } \quad
\sqrt{2 \tr{\matSig_2}} + 
\FNorm{\matSig_2^{1/2} \matOmega_2 \matOmega_1^\pinv} = 
\const{O}\left(\tracenorm{\matSig_2}^{1/2}\right),
\]
so the theorem guarantees that the additional error is on the order
of $\sqrt{\TNorm{\matSig_2} \tracenorm{\matSig_2}}.$ This
is an upper bound on the optimal Frobenius error:
\[
 \FNorm{\matSig_2} \leq \sqrt{\TNorm{\matSig_2} \tracenorm{\matSig_2}}.
\]
We see, in particular, that if the residual spectrum is flat,
i.e. $\lambda_{k+1}(\matA) = \cdots = \lambda_n(\matA),$ then equality holds 
and the additional error is on the scale of the optimal error.

\todo[inline]{It's worth asking if this bound can be replaced with the optimal 
error itself.}

\subsubsection{Trace Norm Bounds}
\label{sxn:theory-det-trace}

Finally, we state and prove the following bound on the trace norm of the 
residual error.
The proof method is analogous to that for the spectral and Frobenius norm 
bounds.

As in the case of the Frobenius norm error, we see that the multiplicative
eigengap $\gamma = \lambda_{k+1}(\matA)/ \lambda_k(\matA)$ predicts the 
effect of using the power method when constructing sketches: the additional
errors of sketches constructed using $\matC = \matA^q \matS$ are a factor
of $\gamma^{2(q-1)}$ times smaller than the additional errrors of those 
constructed using $\matC = \matA \matS.$

\begin{theorem}
Let $\matA$ be an $n \times n$ SPSD matrix with eigenvalue decomposition 
partitioned as in Equation~\eqref{eqn:eigendecompositionpartition}, $\matS$ 
be a sketching matrix of size $n \times \ell$, $q$ be a positive integer,
$\matOmega_1$ and
$\matOmega_2$ be as defined in 
Equation~\eqref{eqn:interactiontermsdefinition}, and define
\[
 \gamma = \frac{\lambda_{k+1}(\matA)}{\lambda_k(\matA)}.
\]
Then when $\matC = \matA^q \matS$ and $\matW = \matS^\transp \matA^{2q-1} \matS$, 
the corresponding low-rank SPSD approximation satisfies
\begin{align*}
 \tracenorm{\matA - \matC \matW^\pinv \matC^\transp} \leq \tr{\matSig_2} + 
       \gamma^{2(q-1)} \FNormS{\matSig_2^{1/2} \matOmega_2 \matOmega_1^\pinv} ,
\end{align*}
assuming $\matOmega_1$ has full row rank. 
\label{thm:trace-deterministic-error}
\end{theorem}

\begin{proof}
Since $\matA - \matC \matW^\pinv \matC^\transp = 
\matA^{1/2} (\matI - 
\matP_{\matA^{q-1/2}\matS}) \matA^{1/2} \succeq \mat{0},$ 
its trace norm simplifies to its trace. Thus
\begin{align*}
\tracenorm{\matA - \matC \matW^\pinv \matC^\transp} & = 
  \tr{\matA - \matC \matW^\pinv \matC^\transp} 
     = \tr{\matSig^{1/2}
       \left(\matI - \matP_{\matSig^{q-1/2}\matS}\right) \matSig^{1/2} } \\
   & \leq \tr{\matSig^{1/2}(\matI - \matP_{\mat{Z}}) \matSig^{1/2}},
\end{align*}
where $\matZ = \begin{pmatrix} \matI \\ \matF \end{pmatrix}$ is defined in 
Equation~\eqref{eqn:defz}. The expression for $\matI - \matP_{\matZ}$ given in
Equation~\eqref{eqn:pzestimate} implies~that 
\[
   \tr{\matSig^{1/2}(\matI - \matP_{\mat{Z}}) \matSig^{1/2}} =
   \tr{\matSig_1^{1/2} 
        (\matI - (\matI + \matF^\transp \matF)^{-1})
       \matSig_1^{1/2}} + 
   \tr{\matSig_2^{1/2} 
        (\matI - \matF (\matI + \matF^\transp \matF)^{-1} \matF^\transp) 
       \matSig_2^{1/2}}.
\]
Recall the estimate 
$\matI - (\matI + \matF^\transp \matF)^{-1} \preceq \matF^\transp \matF$
and the basic estimate 
$\matI - 
 \matF (\matI + \matF^\transp \matF)^{-1} \matF^\transp \preceq \matI.$ 
Together these imply that 
\begin{align*}
\tr{\matSig^{1/2}(\matI - \matP_{\mat{Z}}) \matSig^{1/2}} 
   & \leq \tr{\matSig_1^{1/2} \mat{F}^\transp \matF \matSig_1^{1/2} } + 
     \tr{\matSig_2} \\
   & = \tr{\matSig_2} + 
       \FNormS{\matSig_2^{q-1/2} 
               \matOmega_2 \matOmega_1^\pinv \matSig_1^{-(q-1)}} \\
   & \leq \tr{\matSig_2} + 
          \TNormS{\matSig_2^{q-1}} \TNormS{\matSig_1^{-(q-1)}} 
          \FNormS{\matSig_2^{1/2} \matOmega_2 \matOmega_1^\pinv} \\
   & = \tr{\matSig_2} + 
       \gamma^{2(q-1)} \FNormS{\matSig_2^{1/2} \matOmega_2 \matOmega_1^\pinv}.
\end{align*}
The first equality follows from substituting the definition of $\mat{F}$ and 
identifying the squared Frobenius norm. The last equality follows from 
identifying $\gamma.$ We have established the claimed bound.

\end{proof}

\noindent
\textbf{Remark.}
Since the identity $\FNormS{\mat{X}} = \tracenorm{\mat{X}\mat{X}^\transp}$ 
holds for any matrix $\mat{X},$ the squared Frobenius norm term present in 
the deterministic error bound for the trace norm error is on the scale of 
$\tracenorm{\matSig_2}$ when $\big\|\matOmega_2 \matOmega_1^\pinv\big\|_2$ is
$\const{O}(1).$

\subsubsection{Additional Remarks on Our Deterministic Structural Results.}
\label{sxn:theory-det-comments}

Before applying these deterministic structural results in particular 
randomized algorithmic settings, we pause to make several additional remarks 
about these three theorems.

First, for some randomized sampling schemes, it may be difficult to obtain a sharp 
bound on $\big\|\matOmega_2 \matOmega_1^\pinv\big\|_\xi$ for $\xi = 2, F$. 
In these situations, the bounds on the excess error supplied by 
Theorems~\ref{thm:spectral-deterministic-error},~\ref{thm:frobenius-deterministic-error}, 
and~\ref{thm:trace-deterministic-error} may be quite pessimistic. 
On the other hand, since 
$\matA - \matC \matW^\pinv \matC^\transp 
   = \matA^{1/2}(\matI - \matP_{\left(\matA^{1/2}\right)^{2q-1} \matS})\matA^{1/2}$, 
it follows that 
$ \mat{0} \preceq \matA - \matC \matW^\pinv \matC^\transp \preceq \matA$. 
This implies that the errors of \emph{any} approximation generated used the 
SPSD Sketching Model, deterministic or randomized, satisfy at least the crude 
bound $\xiNorm{\matA - \matC \matW^\pinv \matC^\transp} \leq \xiNorm{\matA}$.

Second, we emphasize that these theorems are deterministic structural results that 
bound the additional error (beyond that of the optimal rank-$k$ approximation) 
of low-rank approximations which follow our SPSD sketching model. 
That is, there is no randomness in their statement or analysis.
In particular, these bounds hold for deterministic as well as randomized 
sketching matrices $\matS$. In the latter case, the randomness enters only through $\matS$, and one 
needs to show that the condition that $\matOmega_1$ has full row rank is 
satisfied with high probability; conditioned on this, the quality of the 
bound is determined by terms that depend on how the sketching matrix 
interacts with the subspace structure of the matrix~$\matA$.

In particular, we remind the reader that (although it is beyond the scope of 
this paper to explore this point in detail) these deterministic structural results 
could be used to check, in an \emph{a posteriori} manner, the quality of a 
sketching method for which one cannot establish an \emph{a priori} bound.

Third, we also emphasize that the assumption that $\matOmega_1$ has full row rank (equivalently, 
that $\tan(\matS, \matU_1) < \infty$) is very non-trivial; and that it is false, in worst-case at least and for 
non-trivial parameter values, for common sketching methods such as uniform 
sampling.
To see that some version of leverage-based sampling is needed to ensure 
this condition, recall that $\matU_1^T\matU_1 = \matI$ and thus that 
$\matOmega_1\matOmega_1^T = \matU_1^T\matS\matS^T\matU_1$ can be viewed 
as approximating $\matI$ with a small number of rank-$1$ components of 
$\matU_1^T\matU_1$.
The condition that $\matOmega_1$ has full row rank is equivalent to 
$\TNorm{ \matU_1^T\matU_1 - \matU_1^T\matS\matS^T\matU_1 } < 1 $.
Work on approximating the product of matrices by random sampling shows that 
to obtain non-trivial bounds one must sample with respect to the norm of the 
rank-$1$ components~\cite{dkm_matrix1}, which here (since we are approximating 
the product of two orthogonal matrices) equal the statistical leverage 
scores.
From this perspective, random projections satisfy this condition 
since (informally) they rotate to a random basis where the leverage scores 
of the rotated matrix are approximately uniform and thus where uniform 
sampling is appropriate~\cite{DMMS07_FastL2_NM10,Mah-mat-rev_BOOK}.

Finally, as observed recently~\cite{Bac12_TR}, methods that use knowledge of a 
matrix square root $\Phi$ (\emph{i.e.}, a $\Phi$ such that 
$\matA=\Phi\Phi^T$) typically lead to $\Omega(n^2)$ complexity.
An important feature of our approach is that we only use the matrix square 
root implicitly---that is, inside the analysis, and not in the statement
of the algorithm---and thus we do \emph{not} incur any such cost.

%

\subsection{Stochastic Error Bounds for Low-rank SPSD Approximation} 
\label{sxn:theory-rand}

In this section, we apply 
the three theorems from Section~\ref{sxn:theory-det}
to bound the reconstruction errors for several random sampling and random 
projection methods that conform to our SPSD Sketching Model. 
In particular, we consider two variants of random sampling and two variants
of random projections:
sampling columns according to an importance sampling distribution that 
depends on the statistical leverage scores
(in Section~\ref{sxn:theory-rand-levscore});
randomly projecting by using subsampled randomized Fourier transformations
(in Section~\ref{sxn:theory-rand-fourier});
randomly projecting by uniformly sampling from Gaussian mixtures of the 
columns 
(in Section~\ref{sxn:theory-rand-gaussian}); and, 
finally, sampling columns uniformly at random 
(in Section~\ref{sxn:theory-rand-unif}).

The results are presented for the general case of SPSD sketches constructed using
the power method, {\emph i.e.,}\  sketches constructed using $\matC = \matA^q \matS$ for
a positive integer $q > 1.$ The additive errors of these sketches decrease proportionally to the number
of iterations $q,$ where the constant of proportionality is given by the 
multiplicative eigengap $\gamma = \lambda_{k+1}(\matA)/\lambda_k(\matA).$ 
Accordingly, the bounds involve the terms $\gamma^{q-1}$ and $\gamma^{2(q-1)}.$ 
The bounds simplify considerably when $q = 1$ (\emph{i.e.}, when there are no 
additional iterations) or $\gamma = 1$ (\emph{i.e.}, when there is no eigengap).
In either of these cases, the terms $\gamma^{q-1}$
and $\gamma^{2(q-1)}$ all become the constant 1.

Before establishing these results, we pause here to provide a brief 
review of running time issues, some of which were addressed empirically in 
Section~\ref{sxn:emp}.
The computational bottleneck for random sampling algorithms (except for 
uniform sampling that we address in Section~\ref{sxn:theory-rand-unif}, which 
is trivial to implement) is often the exact or approximate computation of 
the importance sampling distribution with respect to which one samples; 
and the computational bottleneck for random projection methods is often the 
implementation of the random projection.
For example, if the sketching matrix $\matS$ is a random projection 
constructed as an $n \times \ell$ matrix of i.i.d. Gaussian random 
variables, as we use in Section~\ref{sxn:theory-rand-gaussian}, then 
the running time of dense data in RAM is not substantially faster than 
computing $\matU_1$, while the running time can be much faster for certain 
sparse 
matrices or for computation in parallel or distributed environments.
Alternately, if the sketching matrix $\matS$ is a Fourier-based projection, 
as we use in Section~\ref{sxn:theory-rand-fourier}, then the running 
time for data stored in RAM is typically $O(n^2 \ln k)$, as opposed to the
$O(n^2 k)$ time that would be needed to compute $\matU_1$.
These running times depend sensitively on the size of the data and the 
model of data access; see~\cite{Mah-mat-rev_BOOK,HMT09_SIREV} for detailed 
discussions of these~issues.

In particular, for random sampling algorithms that use a leverage-based 
importance sampling distribution, as we use in 
Section~\ref{sxn:theory-rand-levscore}, it is often said that the running 
time is no faster than that of computing $\matU_1$.
(This $O(n^2 k)$ running time claim is simply the running time of the na\"{i}ve 
algorithm that computes $\matU_1$ ``exactly,'' \emph{e.g.}, with a variant 
of the QR decomposition, and then reads off the Euclidean norms of the rows.)
However, the randomized algorithm of~\cite{DMMW12_JMLR} that computes 
relative-error approximations to \emph{all} of the statistical leverage in 
a time that is qualitatively faster---in worst-case theory and, by using 
existing high-quality randomized numerical 
code~\cite{AMT10,MSM11_TR,HMT09_SIREV}, in practice---gets around this 
bottleneck, as was shown in Section~\ref{sxn:emp}.
The computational bottleneck for the algorithms of~\cite{DMMW12_JMLR} is 
that of applying a random projection, and thus the running time for 
leverage-based Nystr\"{o}m extension is that of applying a (``fast''
Fourier-based or ``slow'' Gaussian-based, as appropriate) random projection 
to $\matA$~\cite{DMMW12_JMLR}.  
See Section~\ref{sxn:emp} or~\cite{AMT10,MSM11_TR,HMT09_SIREV} for 
additional details.

\subsubsection{Sampling with Leverage-based Importance Sampling Probabilities} 
\label{sxn:theory-rand-levscore}

Here, the columns of $\matA$ are sampled with replacement according to a 
nonuniform probability distribution determined by the (exact or approximate) 
statistical leverage scores of $\matA$ relative to the best rank-$k$ 
approximation to $\matA$, which in turn depend on nonuniformity properties 
of the top $k$-dimensional eigenspace of $\matA$. 
To add flexibility (\emph{e.g.}, in case the scores are computed only 
approximately with the fast algorithm of~\cite{DMMW12_JMLR}), we 
formulate the following lemma in terms of any probability distribution that 
is $\beta$-close to the leverage score distribution.
In particular, consider any probability distribution satisfying
\[
  p_j \geq \frac{\beta}{k}\TNormS{(\mat{U}_1)_j} \quad \text{and} \quad \sum\nolimits_{j=1}^n p_j = 1,
\]
where $\beta \in (0,1]$. 
Given these ($\beta$-approximate) leverage-based probabilities, the sketching
matrix is $\matS = \matR \matD$ where $\matR \in \R^{n \times \ell}$ 
is a column selection matrix that samples columns of $\matA$ from the given 
distribution---\emph{i.e.}, $\matR_{ij} = 1$ iff the $i$th column of $\matA$ 
is the $j$th column selected---and $\matD$ is a diagonal rescaling matrix 
satisfying $\matD_{jj} = \frac{1}{\sqrt{\ell p_i}}$ iff $\matR_{ij} = 1$. 
For this case, we can prove the following.

\begin{lemma}
Let $\matA$ be an $n \times n$ SPSD matrix, $q$ be a positive integer, and 
$\matS$ be a sampling matrix of size $n \times \ell$ corresponding to a 
leverage-based probability distribution derived from the top $k$-dimensional 
eigenspace of $\matA$, satisfying
\[
  p_j \geq \frac{\beta}{k}\TNormS{(\mat{U}_1)_j} \quad 
  \text{and} \quad \sum\nolimits_{j=1}^n p_j = 1
\]
for some $\beta \in (0,1].$ 
Fix a failure probability $\delta \in (0,1]$ and approximation factor 
$\epsilon \in (0,1],$ and let 
\[
 \gamma = \frac{\lambda_{k+1}(\matA)}{\lambda_k(\matA)}.
\]

If $\ell \geq 3200 (\beta\epsilon^2)^{-1} k \ln(4 k/(\beta \delta)),$ 
then, when $\matC = \matA^q \matS$ and 
$\matW = \matS^\transp \matA^{2q-1} \matS,$ the corresponding low-rank SPSD 
approximation satisfies
 \begin{align}
\label{eqn:additive1}
  \TNorm{\matA - \matC \matW^\pinv \matC^\transp} & \leq 
    \TNorm{\matA - \matA_k} + 
    \left( \epsilon^2 \tracenorm{(\matA - \matA_k)^{2q-1}} \right)^{1/(2q-1)}  , \\
\label{eqn:additive2}
  \FNorm{\matA - \matC \matW^\pinv \matC^\transp} & 
  \leq \FNorm{\matA - \matA_k} + 
  \left( \sqrt{2} \epsilon \gamma^{q-1}  + 
  \epsilon^2 \gamma^{2(q-1)} \right) \tracenorm{\matA - \matA_k} ,\text{ and} \\
\label{eqn:relative1}
  \tracenorm{\matA - \matC \matW^\pinv \matC^\transp} & \leq 
  (1 + \gamma^{2(q-1)}\epsilon^2) \tracenorm{\matA - \matA_k},
\end{align}
simultaneously with probability at least $1 - 6\delta-0.6.$
\label{lem:sample-lev}
\end{lemma}

\begin{proof}
In \cite[proof of Proposition 22]{MTJ11} it is shown that if $\ell$ satisfies 
the given bound and the samples are drawn from an approximate subspace 
probability distribution, then for any SPSD diagonal matrix~$\matD,$
\[
 \FNorm{\matD \matOmega_2 \matOmega_1^\dagger} \leq \epsilon \FNorm{\matD}
\]
with probability at least $1-2\delta-0.2.$ Thus, the estimates
\[
 \FNorm{\matSig_2^{1/2} \matOmega_2 \matOmega_1^\pinv} \leq 
 \epsilon \FNorm{\matSig_2^{1/2}} = \epsilon \sqrt{\tr{\matSig_2}}
 = \epsilon \sqrt{\tracenorm{\matA - \matA_k}},
\]
and 
\begin{align*}
 \left( \TNorm{\matSig_2^{q-1/2} \matOmega_2 \matOmega_1^\dagger} 
  \right)^{2/(2q-1)} & \leq 
  \left( \FNorm{\matSig_2^{p-1/2} \matOmega_2 \matOmega_1^\dagger}  
  \right)^{2/(2q-1)} \\
  & \leq \left( \epsilon^2 \FNormS{\matSig_2^{q-1/2}} \right)^{1/(2q-1)} \\
  & = \left(\epsilon^2 \tr{\matSig_2^{2q-1}}\right)^{1/(2q-1)} \\
  & = \left(\epsilon^2 \tracenorm{(\matA - \matA_k)^{2q-1}}\right)^{1/(2q-1)} \\
\end{align*}
each hold, individually, with probability at least $1-2\delta-0.2.$ In particular, taking 
$q=1$, we see that
\[
  \TNorm{\matSig_2^{1/2} \matOmega_2 \matOmega_1^\dagger} \leq 
   \epsilon \sqrt{\tracenorm{\matA - \matA_k}}
\]
with the same probability.

These three estimates used in 
Theorems~\ref{thm:spectral-deterministic-error},~\ref{thm:frobenius-deterministic-error},
and~\ref{thm:trace-deterministic-error} yield the bounds given in the 
statement of the~theorem.
\end{proof}


\noindent
\todo[inline]{update this comment: is it still true?}
\textbf{Remark.}
The additive scale factors for the spectral and Frobenius norm bounds are 
much improved relative to the prior results of~\cite{dm_kernel_JRNL}.
At root, this is since the leverage score importance sampling probabilities 
highlight structural properties of the data 
(\emph{e.g.}, how to satisfy the condition in 
Theorems~\ref{thm:spectral-deterministic-error},
\ref{thm:frobenius-deterministic-error}, 
and~\ref{thm:trace-deterministic-error} 
that $\matOmega_1$ has full row rank)
in a more refined way than the importance 
sampling probabilities of~\cite{dm_kernel_JRNL}.

\noindent
\textbf{Remark.}
These improvements come at additional computational expense, but we remind 
the reader that leverage-based sampling probabilities of the form used by 
Lemma~\ref{lem:sample-lev} can be computed faster than the time needed to 
compute the basis $\matU_1$~\cite{DMMW12_JMLR}.
The computational bottleneck of the algorithm of~\cite{DMMW12_JMLR} is the 
time required to perform a random projection on the input matrix.

\noindent
\textbf{Remark.}
Not surprisingly, constant factors such as $3200$ (as well as other 
similarly large factors below) and a failure probability bounded away from zero are artifacts of the analysis; the 
empirical behavior of this sampling method is much better.
This has been observed previously~\cite{DMM08_CURtheory_JRNL,CUR_PNAS}.

\subsubsection{Random Projections with Subsampled Randomized Fourier Transforms}
\label{sxn:theory-rand-fourier}

Here, the columns of $\matA$ are randomly mixed using a unitary matrix 
before the columns are sampled. 
In particular, $\matS = \sqrt{\frac{n}{\ell}} \matD\matT\matR$, where 
$\matD$ is a diagonal matrix of Rademacher random variables, $\matT$ is a 
highly incoherent unitary matrix, and $\matR$ restricts to $\ell$ columns. 
For concreteness, and because it has an associated fast transform, we 
consider the case where $\matT$ is the normalized Fourier transform of 
size $n \times n$. 
For this case, we can prove the following.

\begin{lemma}
Let $\matA$ be an $n \times n$ SPSD matrix, $q$ be a positive integer, and 
$\matS = \sqrt{\frac{n}{\ell}} \matD \matF \matR$ be a sampling matrix of 
size $n \times \ell$, where $\matD$ is a diagonal matrix of Rademacher 
random variables, $\matF$ is a normalized Fourier matrix of size 
$n \times n$, and $\matR$ restricts to $\ell$ columns. 
Fix a failure probability $\delta \in (0,1),$ approximation factor $\epsilon \in (0,1),$ and assume that
$k \geq 4.$ Define
\[
 \gamma = \frac{\lambda_{k+1}(\matA)}{\lambda_k(\matA)}.
\]

If $\ell \geq 24 \epsilon^{-1} [\sqrt{k} + \sqrt{8 \ln( 8 n/\delta)}]^2 \ln(8 k/\delta),$ 
then, when $\matC = \matA^q \matS$ and 
$\matW = \matS^\transp \matA^{2q-1} \matS,$ the corresponding low-rank SPSD 
approximation satisfies
\begin{equation}
 \begin{aligned}
  \TNorm{\matA - \matC \matW^\pinv \matC^\transp} & \leq  
  \left[ 1 + \left(\frac{1}{1 - \sqrt{\epsilon}} \cdot 
  \left(5 + \frac{16 \ln(n/\delta)^2}{\ell}\right) \right)^{1/(2q-1)} \right] 
  \cdot
    \TNorm{\matA - \matA_k} \\
   & \quad +
     \left(\frac{2 \ln(n/\delta)}{(1 - \sqrt{\epsilon})\ell} 
      \right)^{1/(2q-1)} 
     \tracenorm{(\matA - \matA_k)^{2q-1}}^{1/(2q-1)}, \\
 \FNorm{\matA - \matC \matW^\pinv \matC^\transp} & \leq 
  \FNorm{\matA - \matA_k}
  + \left( 7 \gamma^{q-1} \sqrt{\epsilon} 
           + 22 \gamma^{2q-2} \epsilon \right) \tracenorm{\matA - \matA_k},
   \text{ and} \\
\label{eqn:relative2}
 \tracenorm{\matA - \matC \matW^\pinv \matC^\transp} & \leq 
 (1 + 22 \epsilon \gamma^{2(q-1)}) \tracenorm{\matA - \matA_k}
\end{aligned}
\end{equation}
simultaneously with probability at least $1 - 2 \delta.$
\label{lem:proj-fourier}
\end{lemma}


\begin{proof}
In \cite[proof of Theorem 4]{BG12_TR}, it is shown that for this choice of $\mat{S}$ and number of samples $\ell,$ 
\begin{align*}
 \TNormS{\matSig_2^{q-1/2} \matOmega_2 \matOmega_1^\pinv} & 
         \leq \frac{1}{1 - \sqrt{\epsilon}} \cdot 
         \left( 5 \TNormS{\matSig_2^{q-1/2}} + 
        \frac{\ln(n/\delta)}{\ell} \left(\FNorm{\matSig_2^{q-1/2}} + 
        \sqrt{8 \ln(n/\delta)}\TNorm{\matSig_2^{q-1/2}} \right)^2 \right) \\
  & = \frac{1}{1 - \sqrt{\epsilon}} \cdot 
       \left( 5 \TNorm{\matSig}^{2q-1} + 
       \frac{\ln(n/\delta)}{\ell} \left(\tracenorm{\matSig_2^{2q-1}}^{1/2} + 
          \sqrt{8 \ln(n/\delta)}\TNorm{\matSig_2}^{q-1/2} \right)^2 \right) \\
  & \leq \frac{1}{1 - \sqrt{\epsilon}} \cdot 
    \left( \left(5 + \frac{16 \ln(n/\delta)^2}{\ell}\right) 
    \TNorm{\matSig_2}^{2q-1} + 
    \frac{2 \ln(n/\delta)}{\ell} \tracenorm{\matSig_2^{2q-1}} \right)
  \end{align*}
and
\[
 \FNorm{\matSig_2^{1/2} \matOmega_2 \matOmega_1^\pinv} \leq 
 \sqrt{22 \epsilon} \FNorm{\matSig_2^{1/2}} 
 = \sqrt{22 \epsilon \tracenorm{\matSig_2}}
\]
each hold, individually, with probability at least $1 - \delta.$ 
These estimates used in 
Theorems~\ref{thm:spectral-deterministic-error},
and~\ref{thm:trace-deterministic-error}
yield the stated bounds for the spectral and trace norm errors. 

The Frobenius norm bound follows from the same estimates and a
simplification of the bound stated in
 Theorem~\ref{thm:frobenius-deterministic-error}:
\begin{align*}
\FNorm{\matA - \matC \matW^\pinv \matC^\transp} & \leq \FNorm{\matSig_2} 
  +  \gamma^{q-1} \TNorm{\matSig_2^{1/2} \matOmega_2 \matOmega_1^\pinv}
                  \left( 
                       \sqrt{ 2 \tr{\matSig_2} } 
                       + \gamma^{q-1} 
                       \FNorm{\matSig_2^{1/2} \matOmega_2 \matOmega_1^\pinv} 
                   \right) \\
  & \leq   \FNorm{\matSig_2} 
  +  \gamma^{q-1} \FNorm{\matSig_2^{1/2} \matOmega_2 \matOmega_1^\pinv}
     \sqrt{ 2 \tr{\matSig_2} } 
  + \gamma^{2(q-1)} \FNormS{\matSig_2^{1/2} \matOmega_2 \matOmega_1^\pinv} \\
  & \leq \FNorm{\matSig_2}
  + \left( \gamma^{q-1} \sqrt{44 \epsilon} 
           + 22 \gamma^{2q-2} \epsilon \right) \tracenorm{\matSig_2}. 
\end{align*}
We note that a direct application of Theorem~\ref{thm:frobenius-deterministic-error}
gives a potentially tighter, but more unwieldy, bound.
\end{proof}

\noindent
\textbf{Remark.}
Suppressing the dependence on $\delta$ and $\epsilon,$ the spectral norm 
bound ensures that when $p=1,$ $k = \Omega(\ln n)$ and $\ell = \Omega(k \ln k),$ then
\[
\TNorm{\matA - \matC \matW^\pinv \matC^\transp} = 
\mbox{O}\left( \frac{\ln n }{\ln k} \TNorm{\matA - \matA_k} + 
\frac{1}{\ln k} \tracenorm{\matA - \matA_k}\right). 
\]
This should be compared to the guarantee established in Lemma~\ref{lem:proj-gaussian} 
below for Gaussian-based SPSD sketches constructed using the same number of measurements:
\[
\TNorm{\matA - \matC \matW^\pinv \matC^\transp} = 
\mbox{O}\left( \TNorm{\matA - \matA_k} + 
\frac{1}{k\ln k} \tracenorm{\matA - \matA_k}\right). 
\]
Lemma~\ref{lem:proj-fourier} guarantees that errors on this order can be achieved if one
increases the number of samples by a logarithm factor in the dimension: specifically, such
a bound is achieved when $k = \Omega(\ln n)$ and $\ell = \Omega(k \ln k \ln n).$
The difference between the number of samples necessary for Fourier-based sketches and
Gaussian-based sketches is reflective of the differing natures 
of the random projections: the geometry of any $k$-dimensional 
subspace is preserved under projection onto the span of $\ell = \mbox{O}(k)$ Gaussian
random vectors~\cite{HMT09_SIREV}, but the sharpest analysis available suggests that to preserve
the geometry of such a subspace under projection onto the span of $\ell$
SRFT vectors, $\ell$ must satisfy $\ell = \Omega(\max\{k, \ln n\} \ln k)$~\cite{tropp2011improved}. 
We note, however, that in
practice the Fourier-based and Gaussian-based SPSD sketches have similar 
reconstruction errors.

\noindent
\textbf{Remark.}
The structure of the Frobenius and trace norm bounds for the Fourier-based projection are 
identical to the structure of the corresponding bounds from 
Lemma~\ref{lem:sample-lev} for leverage-based sampling (and the bounds 
could be made identical with appropriate choice of parameters). 
This is not surprising since (informally) Fourier-based (and other)
random projections rotate to a random basis where the leverage scores are
approximately uniform and thus where uniform sampling is 
appropriate~\cite{Mah-mat-rev_BOOK}. The disparity of the spectral norm bounds
suggests that leverage-based SPSD sketches should be expected to be more
accurate in the spectral norm than Fourier-based sketches; the empirical results of Section~\ref{sxn:emp-reconstruction}
support this interpretation.
The running times of the Fourier-based and the leverage-based 
algorithms are the same, to leading order, if the algorithm 
of~\cite{DMMW12_JMLR} (which uses the same transform 
$\matS = \sqrt{\frac{n}{\ell}} \matD \matH \matR$) is used to approximate
the leverage scores.

\subsubsection{Random Projections with i.i.d. Gaussian Random Matrices}
\label{sxn:theory-rand-gaussian}

Here, the columns of $\matA$ are randomly mixed using Gaussian random 
variables before sampling. 
Thus, the entries of the sampling matrix $\matS \in \R^{n \times \ell}$ are 
i.i.d.~standard Gaussian random variables. 

\begin{lemma}
Let $\matA$ be an $n \times n$ SPSD matrix, $q$ be a positive integer, 
$\matS \in \R^{n \times \ell}$ be a matrix of i.i.d standard Gaussians, and 
define
\[
 \gamma = \frac{\lambda_{k+1}(\matA)}{\lambda_k(\matA)}.
\]

If $\ell \geq 2 \epsilon^{-2} k \ln k$ where $\epsilon \in (0,1)$ and $k > 4,$ 
then, when $\matC = \matA^q \matS$ and 
$\matW = \matS^\transp \matA^{2q-1} \matS,$ the corresponding low-rank SPSD 
approximation satisfies
\begin{align*}
  \TNorm{\matA - \matC \matW^\pinv \matC^\transp} & \leq 
  \left( 1 + 
     \left(
       89 \frac{\epsilon^2}{\ln k} + 874 \frac{\epsilon^2}{k}
     \right)^{1/(2q-1)}
  \right) \TNorm{\matA - \matA_k} \\
   & \quad + \left( 219 \frac{\epsilon^2}{k \ln k} \right)^{1/(2q-1)} 
              \cdot \tracenorm{\matA - \matA_k}, \\
 \FNorm{\matA - \matC \matW^\pinv \matC^\transp} & \leq
 \FNorm{\matA - \matA_k} + \left[ 
   \gamma^{q-1} \epsilon 
   \left( \frac{42}{\sqrt{k}} + \frac{14}{\sqrt{\ln k}} \right) \right. \\
  & \quad \quad \quad \quad \quad \quad + \left. \gamma^{2q-2} \epsilon^2
   \left(\frac{45}{\ln k} + \frac{140}{\sqrt{k \ln k}}
          + \frac{219}{k \sqrt{\ln k}} 
     \right) \right] 
     \sqrt{\TNorm{\matA - \matA_k} \tracenorm{\matA - \matA_k}} \\
  & \quad + \left( 21 \gamma^{q-1} \frac{\epsilon}{\sqrt{k \ln k}} 
                   + 70 \gamma^{2q-2} \frac{\epsilon^2}{\sqrt{k}\ln k} \right)
            \tracenorm{\matA - \matA_k} \\
  & \quad + \gamma^{2q-2}\epsilon^2 
            \left(\frac{140}{\sqrt{k\ln k}} + \frac{437}{k} \right)
            \TNorm{\matA - \matA_k}, \text{ and}\\
 \tracenorm{\matA - \matC \matW^\pinv \matC^\transp} & \leq 
  \left(1 + 45 \frac{\gamma^{2q-2}\epsilon^2}{\ln k}\right) 
    \tracenorm{\matA - \matA_k} 
   + 437 \frac{\gamma^{2q-2} \epsilon^2 }{k} \TNorm{\matA - \matA_k}
\end{align*}
simultaneously with probability at least $1 - 2 k^{-1} - 4 k^{-k/\epsilon^2}.$
\label{lem:proj-gaussian}
\end{lemma}

\begin{proof}
\noindent
As before, this result is established by bounding the quantities involved in 
Theorems~\ref{thm:spectral-deterministic-error},
\ref{thm:frobenius-deterministic-error},
 and~\ref{thm:trace-deterministic-error}. The following deviation bounds, established in \cite[Section 10]{HMT09_SIREV}, are useful in that regard: if $\matD$ is a diagonal matrix, $\ell = k + p$ with $p > 4$ and $u,t \geq 1,$ then
 \begin{align}
  \Prob{\TNorm{\matD \matOmega_2 \matOmega_1^\pinv} > \TNorm{\matD} \left( \sqrt{\frac{3k}{p+1}} \cdot t 
        + \frac{\e \sqrt{\ell}}{p+1} \cdot tu \right) + \FNorm{\matD} \frac{\e \sqrt{\ell} }{p+1} \cdot t } & 
        \leq 2t^{-p} + \e^{-u^2/2}, \text{ and } \notag \\
 \Prob{\FNorm{\matD \matOmega_2 \matOmega_1^\pinv} > \FNorm{\matD} \sqrt{\frac{3k}{p+1}} \cdot t 
       + \TNorm{\matD} \frac{\e \sqrt{\ell}}{p+1}\cdot tu} & 
       \leq 2t^{-p} + \e^{-u^2/2}.
\label{eqn:gaussian-sampling-deviation-bounds}
\end{align}

Write $\ell = k + p.$ Since $\ell \geq  2 \epsilon^{-2} k \ln k,$ 
we have that $ p \geq \epsilon^{-2} k \ln k.$ Accordingly, the 
following estimates~hold:
\begin{align*}
 \sqrt{\frac{3k}{p+1}} & \leq \sqrt{\frac{3 k}{p}} \leq \sqrt{\frac{3}{\ln k}} \epsilon \\
 \frac{\sqrt{\ell}}{p+1} & \leq \frac{ \sqrt{ k + p}}{p} \leq \sqrt{\frac{\epsilon^4}{k \ln^2 k} + \frac{\epsilon^2}{k \ln k}} < \sqrt{\frac{2}{k \ln k}} \epsilon.
\end{align*}
Use these estimates and take $t = \e$ and $u = \sqrt{2 \ln k}$ in \eqref{eqn:gaussian-sampling-deviation-bounds} to obtain that
\begin{align*}
 \TNormS{\matSig_2^{q-1/2} \matOmega_2 \matOmega_1^\dagger} & \leq  \left[\epsilon \left( \e \sqrt{\frac{3}{\ln k}} +
 2 \e^2 \sqrt{\frac{1}{k}} \right) \cdot \TNorm{\matSig_2^{q-1/2}} + 
  \epsilon \e^2 \sqrt{\frac{2}{k \ln k}} \cdot \FNorm{\matSig_2^{q-1/2}}\right]^2 \\
  & \leq 2 \epsilon^2 \left( \e \sqrt{\frac{3}{\ln k}} +
 2 \e^2 \sqrt{\frac{1}{k}} \right)^2 \cdot \TNorm{\matSig_2}^{2q-1} + \frac{4 \epsilon^2 \e^4}{k \ln k} \cdot \FNorm{\matSig_2^{q-1/2}}^2 \\
  & \leq \left(\frac{12 \e^2 }{\ln k} + \frac{16 \e^4}{k} \right) \epsilon^2 \cdot \TNorm{\matSig_2}^{2q-1} 
  + \frac{4 \epsilon^2 \e^4}{k \ln k} \cdot \tracenorm{\matSig_2^{2q-1}}
\end{align*}
with probability at least $1 - k^{-1}  - 2 k^{-k/\epsilon^2}$ and
\begin{align*}
  \FNorm{\matSig_2^{1/2} \matOmega_2 \matOmega_1^\dagger} & \leq \sqrt{\frac{3}{\ln k}} \epsilon \e \cdot \FNorm{\matSig_2^{1/2}} 
  +  \frac{2 \e^2}{\sqrt{k}} \epsilon \cdot \TNorm{\matSig_2^{1/2}} \\
  & = \epsilon \e \sqrt{\frac{3}{\ln k} \tracenorm{\matSig_2}} + 
   \frac{2\e^2}{\sqrt{k}} \epsilon \cdot \TNorm{\matSig_2}^{1/2}
\end{align*}
with the same probability. Likewise,
\begin{align*}
 \FNormS{\matSig_2^{1/2} \matOmega_2 \matOmega_1^\dagger} &  
 \leq \left(\epsilon \e \sqrt{\frac{3}{\ln k} \tracenorm{\matSig_2}} + 
   \frac{2\e^2}{\sqrt{k}} \epsilon \cdot \TNorm{\matSig_2}^{1/2}\right)^2 \\
  & \leq \frac{6}{\ln k} \epsilon^2\e^2 \cdot \tracenorm{\matSig_2} + \frac{8 \e^4}{k}\epsilon^2 \cdot \TNorm{\matSig_2}
\end{align*}
with the same probability. 

These estimates used in 
Theorems~\ref{thm:spectral-deterministic-error}
and~\ref{thm:trace-deterministic-error}
yield the stated spectral and trace norm bounds. To obtain the corresponding 
Frobenius norm bound, define the quantities
\begin{align*}
 G_1 & = \left(\frac{12\e^2}{\ln k} + \frac{16 \e^4}{k} \right)\epsilon^2 &
 G_3 & = 3 \e^2 \frac{\epsilon^2}{\ln k} \\
 G_2 & = 4 \e^4 \frac{\epsilon^2}{k\ln k} &
 G_4 & = 4 \e^4 \frac{\epsilon^2}{k}
\end{align*}
By Theorem~\ref{thm:frobenius-deterministic-error} and our estimates for
$\TNorm{\matSig_2^{1/2} \matOmega_2 \matOmega_1^\pinv} $ and 
$\FNorm{\matSig_2^{1/2} \matOmega_2 \matOmega_1^\pinv}$, 
\begin{equation}
\label{eqn:gaussian-frob-bound-intermediate2}
\begin{aligned}
 \FNorm{\matA - \matC \matW^\pinv \matC^\transp} & \leq 
 \FNorm{\matSig_2} 
   + \gamma^{q-1} \TNorm{\matSig_2^{1/2} \matOmega_2 \matOmega_1^\pinv} 
     \cdot \left( \sqrt{2 \tr{\matSig_2}} 
       + \gamma^{q-1} \FNorm{\matSig_2^{1/2} \matOmega_2 \matOmega_1^\pinv} 
     \right) \\
  & \leq \FNorm{\matSig_2} 
   + \gamma^{q-1} (G_1 \TNorm{\matSig_2} + G_2 \tracenorm{\matSig_2})^{1/2} 
     \times \\
  & \quad \quad \left( \sqrt{2 \tr{\matSig_2}} 
       + \gamma^{q-1} \sqrt{ G_3 \tracenorm{\matSig_2}} +  \gamma^{q-1}\sqrt{G_4 \TNorm{\matSig_2}} 
     \right) \\
  & \leq \FNorm{\matSig_2} + \left( \gamma^{q-1} \sqrt{2 G_1} 
         + \gamma^{2q-2} (\sqrt{G_1G_3} + \sqrt{G_2 G_4}) \right)
    \cdot \sqrt{\TNorm{\matSig_2} \tracenorm{\matSig_2}} \\
& \quad + \left( \gamma^{q-1} \sqrt{2 G_2} 
          + \gamma^{2q-2} \sqrt{G_2 G_3} \right)
    \cdot \tracenorm{\matSig_2} \\
& \quad + \gamma^{2q-2} \sqrt{G_1 G_4} \TNorm{\matSig_2}.
\end{aligned}
\end{equation}
The following estimates hold for the coefficients in this inequality:
\begin{align*}
\sqrt{2 G_1} & \leq 
  \left( \frac{42}{\sqrt{k}} + \frac{14}{\sqrt{\ln k}} \right) \epsilon &
\sqrt{G_1 G_3} & \leq 
 \left(\frac{45}{\ln k} + \frac{140}{\sqrt{k \ln k}} \right) \epsilon^2 \\
\sqrt{G_2 G_4} & \leq \frac{219}{k \sqrt{\ln k}} \epsilon^2 &
\sqrt{2 G_2} & \leq 21 \frac{\epsilon}{\sqrt{k \ln k}} \\
\sqrt{G_2 G_3} & \leq 70 \frac{\epsilon^2}{\sqrt{k}\ln k} &
\sqrt{G_1 G_4} & \leq 
 \left(\frac{140}{\sqrt{k\ln k}} + \frac{437}{k} \right) \epsilon^2.
\end{align*}
The Frobenius norm bound follows from using these estimates in 
Equation~\eqref{eqn:gaussian-frob-bound-intermediate2} and grouping terms
appropriately:
\begin{align*}
 \FNorm{\matA - \matC \matW^\pinv \matC^\transp} & \leq
 \FNorm{\matSig_2} + \left[ 
   \gamma^{q-1} \epsilon 
   \left( \frac{42}{\sqrt{k}} + \frac{14}{\sqrt{\ln k}} \right) \right. \\
  & \quad \quad \quad \quad \quad \quad + \left. \gamma^{2q-2} \epsilon^2
   \left(\frac{45}{\ln k} + \frac{140}{\sqrt{k \ln k}}
          + \frac{219}{k \sqrt{\ln k}} 
     \right) \right] 
     \sqrt{\TNorm{\matSig_2} \tracenorm{\matSig_2}} \\
  & \quad + \left( 21 \gamma^{q-1} \frac{\epsilon}{\sqrt{k \ln k}} 
                   + 70 \gamma^{2q-2} \frac{\epsilon^2}{\sqrt{k}\ln k} \right)
            \cdot \tracenorm{\matSig_2} \\
  & \quad + \gamma^{2q-2}\epsilon^2 
            \left(\frac{140}{\sqrt{k\ln k}} + \frac{437}{k} \right)
            \TNorm{\matSig_2}.
\end{align*}
\end{proof}

\noindent
\textbf{Remark.}
The way we have parameterized these bounds for Gaussian-based projections 
makes explicit the dependence on various parameters, but hides the 
structural simplicity of these bounds.
In particular,
note that the Frobenius norm bound is upper bounded by a term that depends 
on the Frobenius norm of the error and a term that depends on the trace 
norm of the error; and that, similarly, the trace norm bound is upper 
bounded by a multiplicative factor that can be set to $1+\epsilon$ with an 
appropriate choice of parameters.

\subsubsection{Sampling Columns Uniformly at Random}
\label{sxn:theory-rand-unif}

Here, the columns of $\matA$ are sampled uniformly at random (with or 
without replacement). 
Such uniformly-at-random column sampling only makes sense when the leverage 
scores of the top $k$-dimensional invariant subspace of the matrix are 
sufficiently uniform that no column is significantly more informative than 
the others. 
For this case, we can prove the following.

\begin{lemma}
Let $\matA$ be an $n \times n$ SPSD matrix, $q$ be a positive integer,
and $\matS$ be a sampling matrix  of size $n \times \ell$ corresponding to 
sampling the columns of $\matA$ uniformly at random 
(with or without replacement).
Let $\mu$ denote the coherence of the top $k$-dimensional eigenspace of 
$\matA$ and fix a failure probability $\delta \in (0,1)$ and accuracy factor 
$\epsilon \in (0,1).$  Define
\[
 \gamma = \frac{\lambda_{k+1}(\matA)}{\lambda_k(\matA)}.
\]
If  $\ell \geq 2 \mu \epsilon^{-2} k \ln(k/\delta)$,
then, when $\matC = \matA^q \matS$ and 
$\matW = \matS^\transp \matA^{2q-1} \matS,$ the corresponding low-rank SPSD 
approximation satisfies
 \begin{align*}
  \TNorm{\matA - \matC \matW^\pinv \matC^\transp} & \leq  
    \left( 1 
     + \left(\frac{n}{(1-\epsilon)\ell} \right)^{1/(2q-1)}
    \right) \TNorm{\matA - \matA_k} , \\
 \FNorm{\matA - \matC \matW^\pinv \matC^\transp} & \leq 
  \FNorm{\matA - \matA_k} 
  + \left( \gamma^{q-1} \frac{\sqrt{2} }{\delta\sqrt{1-\epsilon}} 
         + \frac{\gamma^{2q-2}}{(1-\epsilon) \delta^2 }
    \right) \tracenorm{\matA - \matA_k}
  ,\text{ and} \\
 \tracenorm{\matA - \matC \matW^\pinv \matC^\transp} & \leq 
  \left(1 + \frac{\gamma^{2q-2}}{\delta^2 (1-\epsilon)} \right) 
  \tracenorm{\matA - \matA_k},
\end{align*}
simultaneously with probability at least $1 - 3\delta.$
\label{lem:sample-unif}
\end{lemma}




\begin{proof}
 In \cite{Gittens12_TR}, it is shown that 
\[
  \TNormS{\matOmega_1^\pinv} \leq \frac{n}{(1-\epsilon)\ell} 
 \]
with probability at least $1- \delta$ when $\ell$ satisfies the stated bound. 
Observe that $\TNorm{\matOmega_2} \leq 
\TNorm{\matU_2}\TNorm{\mat{S}} \leq 1,$ so that
\begin{equation*}
 \TNormS{\matSig_2^{q-1/2} \matOmega_2 \matOmega_1^\dagger} \leq 
 \TNormS{\matSig_2^{q-1/2}} \TNormS{\matOmega_1^\dagger} \leq 
 \TNorm{\matSig_2}^{2q-1} \frac{n}{(1-\epsilon)\ell}
\end{equation*}
with probability at least $1 - \delta.$ Also,
\begin{equation}
 \FNorm{\matSig_2^{1/2} \matOmega_2 \matOmega_1^\dagger} \leq 
  \sqrt{\frac{n}{(1-\epsilon) \ell}} \FNorm{\matSig_2^{1/2} \matOmega_2}
 \label{eqn:uniform-term-bound}
\end{equation}
with at least the same probability. Observe that since $\mat{S}$ 
selects $\ell$ columns uniformly at random,
\[
 \E \FNormS{\matSig_2^{1/2} \matOmega_2} = 
 \E \FNormS{\matSig_2^{1/2} \matU_2^\transp \mat{S}} = 
 \sum_{i=1}^\ell \E \|\vec{x}_i\|^2,
\]
where the summands $\vec{x}_i$ are distributed uniformly at random over the 
columns of $\matSig_2^{1/2} \matU_2^\transp.$ Regardless of whether $\mat{S}$
selects the columns with replacement or without replacement, the summands all 
have the same expectation:
\[
 \E \|\vec{x}_i\|^2 = 
 \frac{1}{n} \sum_{j=1}^n \| (\matSig_2^{1/2} \matU_2^\transp)^j) \|^2 = 
 \frac{1}{n} \FNormS{\matSig_2^{1/2} \matU_2^\transp} = 
 \frac{1}{n} \FNormS{\matSig_2^{1/2}} =
 \frac{1}{n} \tracenorm{\matSig_2}.
\]
Consequently,
\[
 \E \FNormS{\matSig_2^{1/2} \matOmega_2} = 
 \frac{\ell}{n} \tracenorm{\matSig_2},
\]
so by Jensen's inequality
\[
 \E \FNorm{\matSig_2^{1/2} \matOmega_2} \leq 
 \left(\E \FNormS{\matSig_2^{1/2} \matOmega_2}\right)^{1/2} = 
 \sqrt{\frac{\ell}{n} \tracenorm{\matSig_2}}.
\]
Now applying Markov's inequality to \eqref{eqn:uniform-term-bound}, we see 
that
\[
 \FNorm{\matSig_2^{1/2} \matOmega_2 \matOmega_1^\dagger} \leq 
 \frac{1}{\delta} \sqrt{\frac{1}{(1-\epsilon)} \tracenorm{\matSig_2}}
\]
with probability at least $1 - 2\delta.$ Thus, we also know that 
\[
 \FNormS{\matSig_2^{1/2} \matOmega_2 \matOmega_1^\dagger} \leq 
 \frac{1}{(1-\epsilon) \delta^2} \tracenorm{\matSig_2}
\]
also with probability at least $1 - 2\delta.$
These estimates used in 
Theorems~\ref{thm:spectral-deterministic-error}
and~\ref{thm:trace-deterministic-error} yield the stated spectral and
trace norm bounds. 

To obtain the Frobenius norm bound, observe that 
Theorem~\ref{thm:frobenius-deterministic-error} implies
\begin{align*}
 \FNorm{\matA - \matC \matW^\pinv \matC^\transp} & \leq 
  \FNorm{\matSig_2} + 
  \gamma^{p-1} \TNorm{\matSig_2^{1/2} \matOmega_2 \matOmega_1^\pinv}
   \left( \sqrt{2 \tr{\matSig_2} } 
     + \gamma^{p-1} \FNorm{\matSig_2^{1/2} \matOmega_2 \matOmega_1^\pinv} \right) \\
  & \leq
  \FNorm{\matSig_2} + 
   + \gamma^{p-1} \FNorm{\matSig_2^{1/2} \matOmega_2 \matOmega_1^\pinv}
   \left( \sqrt{2 \tr{\matSig_2} } 
     + \gamma^{p-1} \FNorm{\matSig_2^{1/2} \matOmega_2 \matOmega_1^\pinv} \right) \\
  & \leq
  \FNorm{\matSig_2} 
   + \gamma^{2p-2} \FNormS{\matSig_2^{1/2} \matOmega_2 \matOmega_1^\pinv}
   + \gamma^{p-1} \FNorm{\matSig_2^{1/2} \matOmega_2 \matOmega_1^\pinv}
     \sqrt{2 \tr{\matSig_2} }.
\end{align*}
Now substitute our estimate for 
$\FNormS{\matSig_2^{1/2} \matOmega_2 \matOmega_1^\pinv}$ to obtain
the stated Frobenius norm bound.
\end{proof}

%

\noindent
\textbf{Remark.}
As with previous bounds for uniform sampling, \emph{e.g.}, 
\cite{KMT12,Gittens12_TR}, these results for uniform sampling are much weaker 
than our bounds from the previous subsections, since the sampling complexity 
depends on the coherence of the input matrix.
When the matrix has small coherence, however, these bounds are similar to 
the bounds derived from the leverage-based sampling probabilities.
Recall that, by the algorithm of~\cite{DMMW12_JMLR}, the coherence of an 
arbitrary input matrix can be computed in roughly the time it takes to 
perform a random projection on the input matrix.

\section{Discussion and Conclusion}
\label{sxn:conc}

We have presented a unified approach to a large class of low-rank 
approximations of Laplacian and kernel matrices that arise in machine 
learning and data analysis applications; and in doing so we have provided 
qualitatively-improved worst-case theory and clarified the performance of 
these algorithms in practical settings.
Our theoretical and empirical results suggest several obvious directions for 
future~work.

In general, our empirical evaluation demonstrates that, to obtain 
moderately high-quality low-rank approximations, as measured by minimizing 
the reconstruction error, depends in complicated ways 
on the spectral decay, the leverage score structure, the eigenvalue gaps in 
relevant parts of the spectrum, etc. 
(Ironically, our empirical evaluation also demonstrates that \emph{all} the 
sketches considered are reasonably-effective at approximating both sparse 
and dense, and both low-rank and high-rank matrices which arise in practice. 
That is, with only roughly $\const{O}(k)$ measurements, the spectral, 
Frobenius, and trace approximation errors stay within a small multiplicative 
factor of around $3$ of the optimal rank-$k$ approximation errors. 
The reason for this is that matrices for which uniform sampling is least 
appropriate tend to be those which are least well-approximated by low-rank 
matrices, meaning that the residual error is much larger.)
Thus, \emph{e.g.}, depending on whether one is interested in $\ell$ being
slightly larger or much larger than $k$, leverage-based sampling or a 
random projection might be most appropriate; and, more generally, an 
ensemble-based method that draws complementary strengths from each of 
these methods might be best.

In addition, we should note that, in situations where one is concerned with 
the quality of approximation of the actual eigenspaces, one desires both a 
small spectral norm error (because by the Davis--Kahan sin$\Theta$ theorem 
and similar perturbation results, this would imply that the range space of 
the sketch effectively captures the top $k$-dimensional 
eigenspace of $\mat{A}$) as well as to use as few samples as possible 
(because one prefers to approximate the top $k$-dimension eigenspace of 
$\mat{A}$ with as close to a $k$-dimensional subspace as possible). 
Our results suggest that the leverage score probabilities supply the best 
sampling scheme for balancing these two competing objectives.

More generally, although our empirical evaluation consists of \emph{random} 
sampling and \emph{random} projection algorithms, our theoretical analysis 
clearly decouples the randomness in the algorithm from the structural 
heterogenities in the Euclidean vector space that are responsible for the 
poor performance of uniform sampling algorithms.
Thus, if those structural conditions can be satisfied with a deterministic
algorithm, an iterative algorithm, or any other method, then one can certify (after running the algorithm) that good 
approximation guarantees hold for particular input matrices in less time 
than is required for general matrices.
Moreover, this structural decomposition suggests greedy 
heuristics---\emph{e.g.}, greedily keep some number of columns according to 
approximate statistical leverage scores and ``residualize.''
In our experience, a procedure of this form often performs quite well in 
practice, although theoretical guarantees tend to be much weaker; and thus
we expect that, when coupled with our results, such procedures will perform 
quite well in practice in many medium-scale and large-scale machine learning applications.

\vspace{5mm}
\textbf{Acknowledgments.}
AG would like to acknowledge the support, under the auspice of Joel Tropp, of ONR awards N00014-08-1-0883 and N00014-11-1-0025, AFOSR award FA9550-09-1-0643, and a Sloan Fellowship; and 
MM would like to acknowledge a grant from the Defense Advanced Research 
Projects Agency.
\vspace{5mm}

\bibliographystyle{plain}
\bibliography{communities,mwmbib_jrnl,mwmbib_proc,mwmbib_book,mwmbib_misc,nystromefficacy}

\newpage
\appendix

\end{document}